\documentclass[a4paper, 11pt, oneside]{article}

\usepackage{graphicx}
\graphicspath{Figures/}  

\usepackage[utf8x]{inputenc}

\usepackage{amsmath}
\usepackage{amssymb}
\usepackage{xcolor}
\usepackage[colorlinks, citecolor=blue,urlcolor=blue,linkcolor=black]{hyperref}

\usepackage{bbm}
\usepackage{dirtytalk}
\usepackage{mathtools}

\usepackage[parfill]{parskip}
\usepackage{footmisc}

\usepackage{geometry}
\usepackage[small,sc]{caption}
\usepackage{subcaption}

\usepackage{abstract}

\title{
    \hrule
    \vspace{1em}
    \huge \textbf{Seeing Things in Random-Dot Videos}\\
    \vspace{1.5em}
    \hrule
}
\author{
    \vspace{50pt}
    \LARGE Thomas Dagès, Michael Lindenbaum, Alfred M. Bruckstein \\
    \vspace{10pt}
    \Large Centre for Intelligent Systems \\
    \vspace{10pt}
    \Large Computer Science Department \\
    \vspace{10pt}
    \Large Technion, Israel Institute of Technology \\
    \vspace{10pt}
    \Large 3200003 Haifa, ISRAEL
}
\date{\LARGE July 30, 2019}

\usepackage{dirtytalk}

\usepackage{verbatim}

\usepackage{mathrsfs}

\usepackage{amsthm}

\theoremstyle{plain}
\newtheorem{theorem}{Theorem}

\usepackage[ruled]{algorithm2e}

\SetCommentSty{mycommfont}

\begin{document}

\renewcommand{\abstractname}{}
\renewcommand{\absnamepos}{empty}

\maketitle



\thispagestyle{empty}

\newpage

\begin{abstract}
\noindent
\paragraph{Abstract}
Humans possess an intricate and powerful visual system in order to perceive and understand the environing world. Human perception can effortlessly detect and correctly group features in visual data and can even interpret random-dot videos induced by imaging natural dynamic scenes with highly noisy sensors such as ultrasound imaging. Remarkably, this happens even if perception completely fails when the same information is presented frame by frame rather than in a video sequence. We study this property of surprising dynamic perception with the first goal of proposing a new detection and spatio-temporal grouping algorithm for such signals when, per frame, the information on objects is both random and sparse and embedded in random noise. The algorithm is based on the succession of temporal integration and spatial statistical tests of unlikeliness, the a contrario framework. The algorithm not only manages to handle such signals but the striking similarity in its performance to the perception by human observers, as witnessed by a series of psychophysical experiments on image and video data, leads us to see in it a simple computational Gestalt model of human perception with only two parameters: the time integration and the visual angle for candidate shapes to be detected.

\end{abstract}

\addvspace{30pt}


\section{Introduction}

\subsection{Motivation}

Thanks to the wonders of technological advances for camera design and signal processing software we are now able to take impressively high quality pictures and videos of the world, such as capturing a breathtaking sunset on a well deserved holiday. While the performance of common cameras is ever so amazing, there are other fields using imaging devices that have not yet reached such maturity. Fundamentally, cameras take a snap of the incoming light that passes through the lens and hits the sensors. However, light is not only made of the traditional visible spectrum. Other wavelengths exist to which the human eye receptors are insensitive. Yet we can design cameras with sensors tuned to these wavelengths. Doing this allows us to take a snap of the \say{invisible} visual field, which can provide crucial hidden information. In fact, we need not limit ourselves to light waves and can apply a similar reasoning to design sensors for other waves that the natural human body cannot perceive. An important example is the use of ultrasound imaging for medical applications, which, for example, allows us to see a fetus even though we cannot directly see it with our traditional sensors, or eyes, since it is hidden behind human tissue. Unfortunately, these sensors, while being a true technological prowess, do not always provide very clean high quality images. Sometimes, only an expert can understand the images shown by these sensors: a doctor is often needed to describe an ultrasound to new parents. In this report, we shall consider images created by some special sensors that produce very noisy images. We will work with images that will that consist of binary black and white pixels that look like random unstructured noisy dots. However, at the position where the object is, or often where the boundaries are in the image, the density of points is slightly higher than in the random background: a few extra random points appear at these locations. For an illustration, see Figure \ref{fig: np pts see frame}. If the extra number of points were to be very high, then we would be able to easily see what the object is even if the image is very noisy. However if the extra number of points is low, we will not be able to see the object, or in fact that there is a change of density along some structured shape. However in this context a magical phenomenon appears. If instead of taking a single image, if we take several of them and present them in a video, then we, as humans, without any effort, will \say{see}: we will very well perceive the objects in the video that we were not at all capable of seeing if they were presented as images one at a time! Even though each frame has independent random noise with just a few extra random dots on the object we are imaging, we are suddenly able to perceive the object in the \say{random-dot video} generated by the successive imaging results. While one might quickly say that this is explained by integration and averaging, the phenomenon becomes even more astonishing when the object moves between frames. The phenomenon persists for complex shapes undergoing complex deformation and displacement, and despite the fact that the points do not persist in time between frames. Thus, the problem is not explained by simple integration and averaging. As such, the human visual system is naturally capable of handling highly degraded video data without any conscious effort. On the other hand, one might ask how to replicate this amazing performance on a machine. Solving this technological problem is highly difficult yet easily done by our brain. Therefore analysing the properties of the human visual system can give us insight to conceive novel automatic algorithms for highly noisy video data processing where traditional denoising approaches fail. To the best of our knowledge, this phenomenon has never been thoroughly studied and may provide a valuable tool for analysing images provided by highly noisy sensors.

\subsection{Overview of the Report}

In this report we will generate our noisy random-dot video data and design an algorithm to handle it using the a contrario framework, which is based on statistical tests of unlikeliness. We will study its performance mathematically and empirically. We then define a simple computational model of the human perceptual system and compare it with the algorithm by performing a series of psychophysical experiments. It is based on time integration followed by a spatial search using a contrario algorithm. Our model is based on two parameters: the time integration and the visual angle for the candidate regions that we will search through in the a contrario algorithm. We perform four psychophysical experiments. We will limit our study to straight lines for mathematical simplicity, and they will either be static or move with a fixed speed orthogonally to themselves in the context of the experiments. The first two experiments are on static images, that can be considered either as a single degraded frame with high degradation parameters or as an integrated image of several frames of a video generated from lower degradation values. This will allow us to retrieve the visual angle parameter for our model. The next two experiments are on videos directly and will allow us to retrieve the time integration parameter of the model. In general, similar performances between the algorithm with the chosen parameters and humans give credit to such a model. Our contributions include the modelling of the input data, the adaptation of the a contrario framework to spatio-temporal data with our novel algorithm, the mathematical analysis and prediction a priori of the performance of the algorithm on our data, the psychophysical experiments and the computational model for human perception based on integration and a contrario.

We model the signals and present some simulations in Section \ref{sec: signal modelling}. In Section \ref{sec: literature survey}, we perform a short overview of the psychophysical and computational literature. We then present the a contrario framework in Section \ref{sec: a contrario}. We present the algorithm for handling our videos based on the a contrario framework in Section \ref{sec: initial experiments}. We present our model for human perception and perform a short overview of the psychophysical experiments in Section \ref{sec: humans vs AC overview}. In Section \ref{sec: analysis static edge} we study, mathematically and empirically, from the algorithm's and human's perspective, static images considered as integrations of frames of a video of a static straight line. In Section \ref{sec: analysis dynamic edge}, we perform the same but on frames considered as integrations of frames of a video of a dynamic straight line. We then present the context of the experiments directly on video data and the performance of the algorithm on it in Section \ref{sec: video overview and AC}. We then present the human performance on videos of a static line in Section \ref{sec: exp3.1} and on the dynamic line in Section \ref{sec: exp3.2}. Discussions on future work are done in Section \ref{sec: Future work} and we conclude this report in Section \ref{sec: Conclusion}.

\begin{figure}
    \centering
    \begin{subfigure}[t]{0.45\textwidth}
      \centering
         \includegraphics[width=\textwidth]{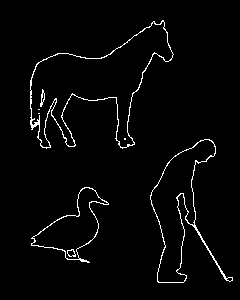}
         \caption{Groundtruth}
    \end{subfigure}%
    \hfill
    \begin{subfigure}[t]{0.45\textwidth}
      \centering
         \includegraphics[width=\textwidth]{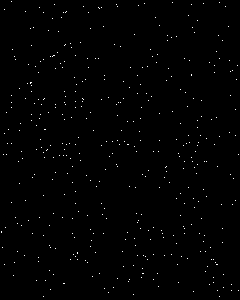}
         \caption{$p_b = 0.005, p_f = 0.05$}
    \end{subfigure}%
    \hfill
    \begin{subfigure}[t]{0.45\textwidth}
      \centering
         \includegraphics[width=\textwidth]{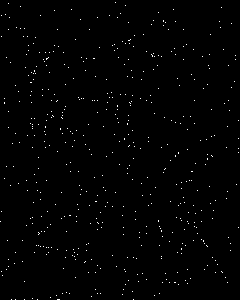}
         \caption{$p_b = 0.005, p_f = 0.1$}
    \end{subfigure}%
    \hfill
    \begin{subfigure}[t]{0.45\textwidth}
      \centering
         \includegraphics[width=\textwidth]{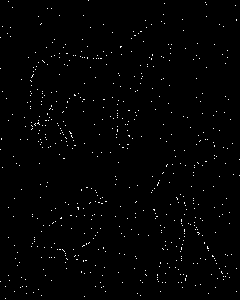}
         \caption{$p_b = 0.005, p_f = 0.2$}
    \end{subfigure}%
    \hfill
\caption{Examples of possible images captured by our sensors. The exact same scene is being imaged in each picture: the same outlines of a horse, a duck and a golfer. The background density is identical. The density of the foreground is lowest on the first image and highest on the last image. Since the number of extra dots on the interesting objects is very high in the last image, we can easily ``see'' just on this image. On the top right image, the number of extra points is too low in order to make us ``see'' the objects. On the bottom left image, one might eventually be able to perceive some parts of the outlines. However, if the sensor provides several such images and shows them in a video then we are easily capable of perceiving the outlines in all three cases. Find video demos at \cite{dagesHorseDemo}.}
\label{fig: np pts see frame}
\end{figure}


\clearpage
\section{Signal Modelling}
\label{sec: signal modelling}

Our signals consist of a noisy degraded version of a black and white video of a moving foreground object on a stationary background. We will call foreground the edges of the object and the background the rest of the image, inside or outside the objects. See Figure \ref{fig: model data acquisition}. Assume for example that the sensor computes the gradient magnitudes of the image and thresholds local maxima of the gradient. The result of this thresholding operation is the output of the sensor. If the sensor were ideal, then the edges of the object, which is defined as the foreground, would be perfectly produced and the stationary uniform background, i.e. the regions both inside and outside the objects, would be simply black. Unfortunately, the sensors are very noisy. This translates to very noisy gradients: for instance, corrupt the input with noise (e.g. iid Gaussian noise). The gradient and thresholding operations will be very noisy, producing many random white points outside the foreground and not producing all white points on the foreground. Yet on the foreground there are more points than in the background. Therefore the sensor produces an image $I$ with two densities of points. A first density corresponding to the background only and a second density on the edges of the objects. See Figure \ref{fig: data acquisition sensor thresh grad} for a one dimensional example. Mathematically, if $i$ is a pixel and $B$ and $F$ the notations for the pixel locations corresponding to the background and the foreground, we have that: $\mathbbm{P}(I_i=1\mid i\in B) = p_0$ and $\mathbbm{P}(I_i=1\mid i\in F) = p_1$, with $p_1>p_0$. Ideally we want $p_1=1$ and $p_0=0$ but unfortunately $p_0>0$ and $p_1$ is only slightly larger than $p_0$. However in some cases the gap between $p_0$ and $p_1$ is not sufficiently large, so that in a single image we can \say{see} the foreground object, but large enough so that we can easily \say{see} the object when presented with a video with frames that display the successive random outputs of these noisy sensors.

In our study we did not work with real sensors. All the data we worked with was simulated. The simulation process we used does not correspond physically to the \say{real degradation process} but is mathematically equivalent to the previous description and this new formulation will help us throughout this report. The process is the following (see Figure \ref{fig: data generation sensor merge} for an illustration). Assume you have a series of clean black and white frames of edges (for instance consider simulated images of edges or outputs of an edge detector on frames of a video). The degradation process $\phi_{vid}$ consists in degradation processes $\phi_{im}$ on each frame independently between frames. For each frame in the original video $I$, we generate two independent degraded images: one consisting only in background white noise in the whole image and one consisting in white noise on the foreground only. We then merge these images to get the degraded frame $I^D$ which will substitute $I$ in the degraded video. Formally, the degraded frame is the union of two random independent signals: $I^{D} = I^B \lor I^F$, where $I^B$ and $I^F$ are the random background and foreground images respectively and are independent from one another, and $\lor$ is the symbol for the logical \say{or} operation. The background image is a 2D array of iid Bernoulli pixels: $(I^B)_{i}\sim B(p_b)$, where $i$ is the pixel index. On the other hand, the random foreground image follows a modified Bernoulli, conditionally to the position of the foreground: if $i$ is a foreground pixel, $(I^F)_{i} \sim B(p_f)$, otherwise $(I^F)_{i} = 0$. The Bernoulli parameters $p_b$ and $p_f$ are the degradation parameters of the process: ${I^D = \phi_{im}(I,p_b,p_f)}$. Note that although $p_b$ does correspond to $p_0$ in the previous formulation, $p_f$ does not correspond to $p_1$. This is because here $p_f$ can be understood as a marginal probability: an increase in probability of appearance on the foreground. We will sometimes call $p_f$ the \say{extra foreground} probability in order to avoid confusion with the probability on the foreground $p_1$. Beware that $p_f\neq p_1-p_0$. This is due to the fact that for a foreground pixel to appear in the degraded image, it can happen either due to the foreground image, or the background image, and both events can simultaneously occur and we must therefore remove the intersection in the computation of the probability. Hence we have:
$$p_1=\mathbbm{P}(I_i=1\mid i\in F) = \mathbbm{P}(I^B_i=1\cup I^F_i=1 \mid i\in F) = p_b+p_f-p_bp_f$$

\begin{figure}
    \centering
    \includegraphics[width=0.7\textwidth]{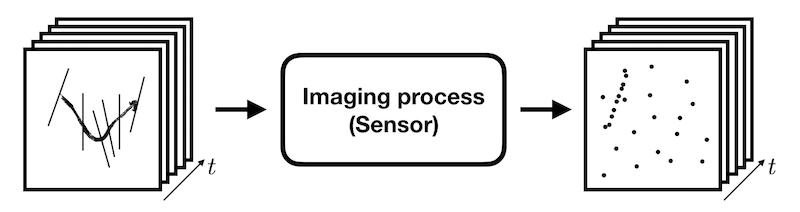}
    \caption[Short version]{Data acquisition model. An input image signal of a foreground, the edges of an object, lying in some background enters the sensor and the sensor outputs a degraded version of the signal, with many random white dots in the background with density $p_0$ and also a few white points on the foreground with slightly bigger density $p_1$. The sensor can process several successive input image signals to produce a degraded sequence of images. For example, the sensor could work on an input video of a dynamic scene. An input video consists in successive input frames, and the sensor would then output a degraded video consisting in successive degraded frames.}
    \hfill
    \label{fig: model data acquisition}
\end{figure}

\begin{figure}
    \centering
    \begin{subfigure}[t]{0.5\textwidth}
      \centering
         \includegraphics[width=\textwidth]{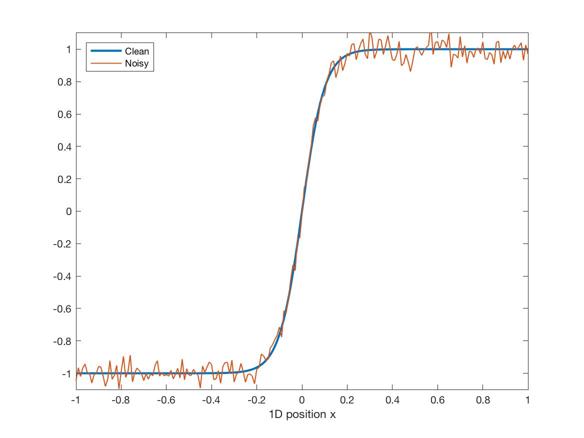}
         \caption{Input signals to sensor}
    \end{subfigure}%
    \hfill
    \begin{subfigure}[t]{0.5\textwidth}
      \centering
         \includegraphics[width=\textwidth]{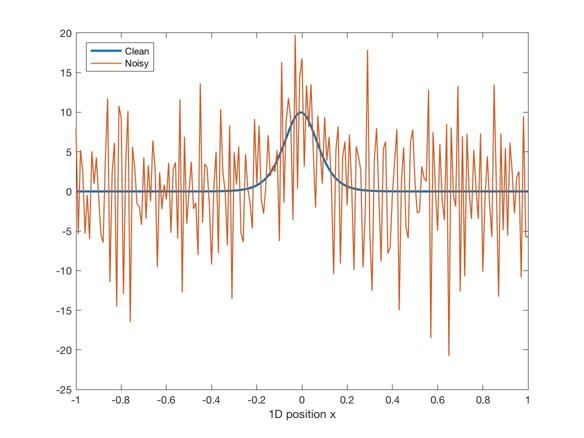}
         \caption{Gradients}
    \end{subfigure}%
    \hfill
    \begin{subfigure}[t]{0.5\textwidth}
      \centering
         \includegraphics[width=\textwidth]{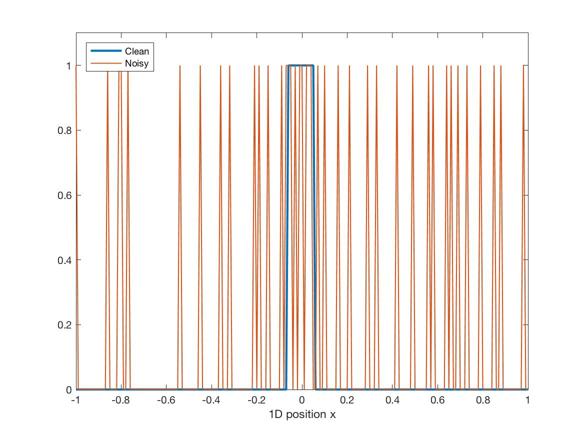}
         \caption{Sensor outputs: thresholding the gradients}
    \end{subfigure}%
    \hfill
    \begin{subfigure}[t]{0.5\textwidth}
      \centering
         \includegraphics[width=\textwidth]{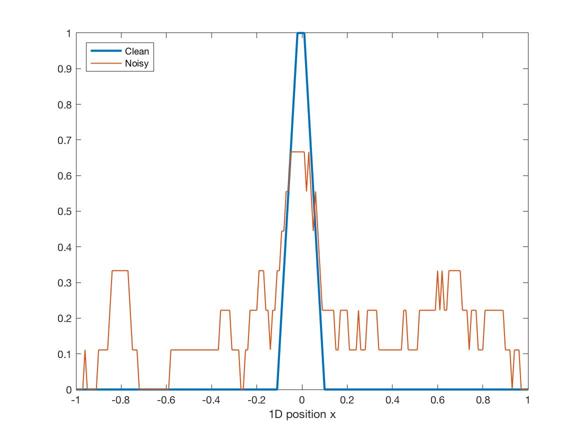}
         \caption{Local output density}
    \end{subfigure}%
\caption{Example of possible sensor acquisition producing our kind of data in a one dimensional setting. Consider one dimensional edge data. Assume the sensor computes the gradient of the input signal and outputs the result of the thresholding operation on the gradient signal. Unfortunately, the input to the sensor is noisy (and eventually the computation of the gradients could be noisy too) which implies high frequency in the gradients. Thresholding the noisy gradients will provide many new random dots (and a few dots less on the edge). The output points lie on two separate densities: outside the edge, i.e. in the background, the density of points is approximately $p_0$ whereas on the line the density is slightly larger $p_1$. For a very noisy sensor such as the one provided, recovering information on just one output of the sensor is difficult if not impossible for larger noise. However using the resampling of randomness, by looking at several independent outputs of the sensor, provide us with more information and would help to recover the clean signal. In this example the signal is a vector of length 201 uniformly sampling $[-1,1]$. The clean signal is $y=tanh(10x)$. We added Gaussian noise with standard deviation $0.05$ to the signal. Gradient computation was done without a smoothing operator to remove the high frequencies. The gradient threshold was chosen to be $7$. The window width for computing local densities was chosen to be $9$.}
\label{fig: data acquisition sensor thresh grad}
\end{figure}

\begin{figure}
    \centering
    \begin{subfigure}[t]{0.45\textwidth}
      \centering
         \includegraphics[width=\textwidth]{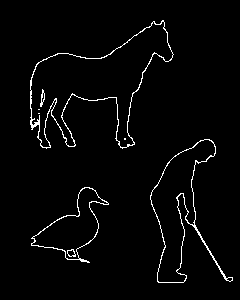}
         \caption{Groundtruth image $I$}
    \end{subfigure}%
    \hfill
    \begin{subfigure}[t]{0.45\textwidth}
      \centering
         \includegraphics[width=\textwidth]{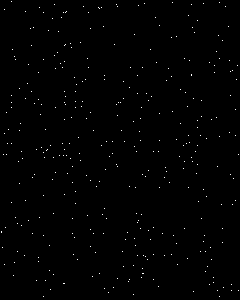}
         \caption{Background image $I^B$}
    \end{subfigure}%
    \hfill
    \begin{subfigure}[t]{0.45\textwidth}
      \centering
         \includegraphics[width=\textwidth]{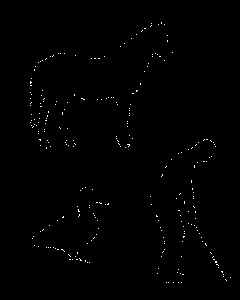}
         \caption{Foreground image $I^F$}
    \end{subfigure}%
    \hfill
    \begin{subfigure}[t]{0.45\textwidth}
      \centering
         \includegraphics[width=\textwidth]{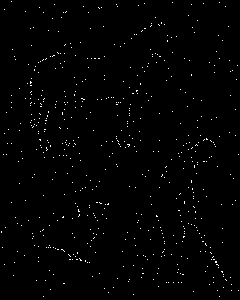}
         \caption{Final degraded image as a merge of both degraded images $I^D = I^B \lor I^F$}
    \end{subfigure}%
    \hfill
\caption{Model of the degradation process done in practice on synthetic data. The degradation procedure in our lab data is done by taking the input groundtruth edge image and generating from it two degraded images. The first is called the background image and is simply white noisy in the entire image with pixel probability of being white of $p_b$. The second is called the foreground image and is also a white noise image but only on the locations where the groundtruth was white (on the edges), in the rest of the image it is simply black. The probability of appearance in the foreground image for a pixel on an edge is $p_f$ which can be arbitrarily small. The final degraded image and output of our degradation procedure is the logical merge of each image: a pixel in the final image is white if and only if it is white in the background image or the foreground image (or both). This procedure is equivalent to the previous description of the degradation procedure with: $p_0 = p_b$ and $p_1 = p_b+p_f-p_b p_f$. In this example, we have $p_b = 0.005$ and $p_f = 0.2$ on an image of resolution $240\times 300$.}
\label{fig: data generation sensor merge}
\end{figure}

\subsection{Initial Simulations}
 
We first simulated the phenomenon in order to experience it first hand. In order to do so, we filmed a scene of a person moving in an environment and ran a Canny edge detector with thresholds hand-tuned. We then applied the degradation procedure described above on the edge video with some hand-tuned parameters that show the effect. 
We used three versions of the parameters for demo purposes (see Figure \ref{fig: me moving demo frames}). The first two choices depend on what is considered the important information to retrieve in the video. The third choice shows the phenomenon in a different range. The first choice is to consider the human moving as the information to be retrieved. For this we chose parameters so that the moving person is not easily seen on each degraded frame but it is if we look at it in a video. In this case, static edges are persistent and can sometimes be seen on frames. However they too are to be considered as noise. The second choice for the degradation parameters is to consider both static edges as well as objects in motion as information hidden in each frame. We had to slightly lower the difference between parameters $p_b$ and $p_f$ in order to mask the straight edges in each frame but have them seen when considering the video. In this case, seeing the moving human can be much harder. In this report, we work in the \say{sparse} domain, which consists in a low background density, $p_b<<\frac{1}{2}$. Nevertheless, for our third demo, we chose a pair of parameters consisting in a much denser image to show that the phenomenon also occurs outside the sparse assumption for the signal. In all cases, for the purpose of the demo, we chose to not lie within the critical thresholds for perception.

An interesting note can already be made. For the first choice of parameters, centred on the dynamic human rather than the edges, we pointed out that the straight edges are slightly more easily seen on each frame than the outline of the human. Although the density of the points along the curves, respectively the long straight edges and the human contour, are identical, perception seems to be easier for the straight edges. This corroborates the hypothesis that an increase in the complexity of the shape leads to a decrease in the perception performance.

\begin{figure}
    \centering
    \begin{subfigure}[t]{0.45\textwidth}
      \centering
         \includegraphics[width=\textwidth]{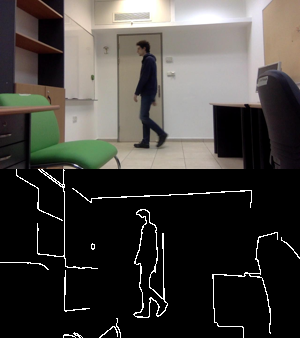}
         \caption{Original}
    \end{subfigure}%
    \hfill
    \begin{subfigure}[t]{0.45\textwidth}
      \centering
         \includegraphics[width=\textwidth]{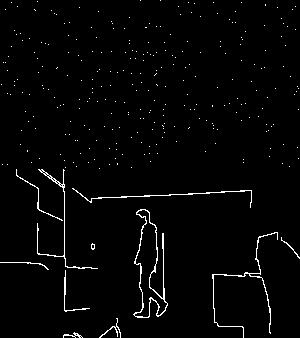}
         \caption{$(p_b = 0.005,p_f = 0.04)$}
    \end{subfigure}%
    \hfill
    \begin{subfigure}[t]{0.45\textwidth}
      \centering
         \includegraphics[width=\textwidth]{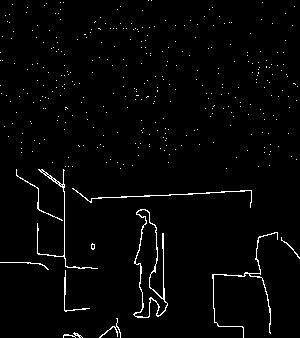}
         \caption{$(p_b = 0.005,p_f = 0.06)$}
    \end{subfigure}%
    \hfill
    \begin{subfigure}[t]{0.45\textwidth}
      \centering
         \includegraphics[width=\textwidth]{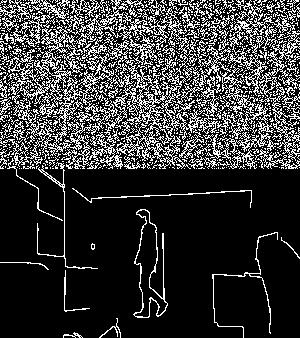}
         \caption{$(p_b = 0.4,p_f = 0.45)$}
    \end{subfigure}%
    \hfill
\caption{Top left: Original frame of a video of a natural scene and the output of a Canny edge detector with parameters $(0.2,0.5)$ and blur standard deviation $3$ for the gradient computation. The output of the Canny edge detector is considered as the clean image in this work and our goal is to recover it. The other pairs of images consist in the same frame and its degradation according to the chosen degradation parameters. Top right: In a video, the straight lines are seen but it can be tricky to clearly see the human moving. Bottom left: In a video, both the straight edges and human can be seen very clearly. One might be able to see part of the straight alignments in the frame, but this is not very clear and is biased by having underneath the groundtruth. Bottom right: A denser version of the phenomenon. In a video, the straight edges and the moving person can be seen clearly. For each of these degradations, viewing simply the degraded frame does not seem to give us much information and it looks like a configuration of random noisy dots. Video demos can be found at \cite{dagesMeDemo}.}
\label{fig: me moving demo frames}
\end{figure}

The phenomenon occurs whether we look at simple shapes such as lines or circles \cite{dagesEdgeCircleConchDemo} or complex shapes such as a human dancer \cite{dagesBalletDemo}. We can still \say{see} these objects even if they move a highly non rigid way. In the second part of demo \cite{dagesBalletDemo}, we range through various values of $p_f$, from when $p_f$ is so high that the objects are easily seen in each frame, to values of $p_f$ so low that we cannot see any structure in the video.

\section{Literature Survey}
\label{sec: literature survey}

\paragraph{Early perception theory}
While psychologists have long studied human perception, it was in the late 19th century that Weber and his student Fechner introduced quantitative experimentation on humans \cite{fechner1860elemente}. This led to the foundation of Psychophysics. Many different theories of vision have appeared since then \cite{palmer1999vision}. The main first two contrasting approaches are Structuralism and Gestaltism. The former claims that basic sensory primitives are indivisible and associated, via experience, to form a more complex concept. The latter claims that the whole is in itself the basic property and it cannot be broken down into sub-parts: \say{the whole is something else than the sum of its parts} \cite{Koffka1935-KOFPOG-3}. Interestingly enough, structuralist principles closely resemble those of cutting-edge signal processing algorithms such as convolutional neural networks, which learn to extract local key features and link them together based on previous experience to obtain higher level features, and concatenate them into a feature vector, but do not necessarily manage to process structures (e.g. shapes) as a whole \cite{baker2018abstract, baker2018deep}. Gibson gave more importance to the surrounding world rather than the organismic processes responsible for perception. This lead to the third main theory of vision called Ecologism. It claims that we can retrieve directly all the information of the world directly from its projection on the retina, setting the ground for modern computer vision. However, it is impossible to fully recover unambiguously the 3D world from 2D data even when including time. Helmholtz claimed that an extra process, innate or learned, happens during perception in order to fully solve the inherent ambiguity of the inverse problem. This lead to the Constructivism theory. This can formulated in a likelihood theory where we have some extra heuristic assumptions and then perceive the most likely explanation to the current stimulus which naturally echoes the Bayesian mathematical approach. It can be also linked to the Prägnanz principle of Gestaltism. Constructivism claims that perception of global phenomena is constructed from local information but also gives importance to perceptions as a whole such as in the Gestalt approach. As such, it is a compromise between all three previous theories. Constructivism is now considered as the main classical approach to perception theory. Nevertheless, the role of Gestalts in visual processing tasks should not be underestimated.

\paragraph{Gestaltism}
The founder of Gestaltism, Wertheimer, designed a series of laws in order to explain human perception of Gestalts \cite{wertheimer1938laws}. These principles predict how humans will group together objects depending on various properties. Human perception tends to group similar objects or objects that tend to make a simple whole. Wertheimer's laws are proximity, similarity, closure, symmetry, common fate, continuity, \say{good Gestalt} also called Prägnanz, and past experience. Many other laws have been introduced such as common region \cite{palmer1992common} or connectedness \cite{palmer1994rethinking}. These laws should be first understood in the following fashion: \say{all else being equal} \cite{palmer1999vision}, the law applies. However, they can reinforce one another or create conflicts and impede the perception, depending on multiple cues that are present. Furthermore, they are recursive in the sense that a first grouping can be done according to a law (e.g. good continuation for detecting a curve) then a sub grouping can be done according to an other (e.g. similarity, if the curve consists in two alternating shapes).

\paragraph{Perception of dot alignment embedded in noise}
In our videos, the clean frames will mostly consist of one dimensional objects or lines. Therefore the Gestalt we are most interested in is the perception of alignment among random dots. Alignment can be understood as rigid alignment where the alignment should be done along a straight line, eventually with some jitter. Or it can also be understood in a more flexible way as an alignment along a smooth curve. It is a well known fact that humans are not able to perceive structure or alignment in a random-dot image \cite{attneave1954some}: the image is primarily understood as a texture. Therefore, in order to perceive an alignment of dots lying in a random-dot image, it is necessary that the statistics along the alignment be different to those in the rest of the image. These different statistics can be of many kinds. Uttal \cite{uttal1975autocorrelation} proposed an autocorrelation theory of form detection, mainly for straight alignments. In this model, due to the definition of autocorrelation, regular spacing is key. However, Kiryati et al. \cite{kiryati1992perception} showed that autocorrelation is insufficient for modelling our perception of alignment. While the regular spacing may help perception, irregular alignments can still be perceived while being strongly masked by regular patterns, something unexplained by the autocorrelation model. As such, the authors propose that human perception should be a combination of autocorrelation and another process independent of spacing such as for instance the Hough transform \cite{duda1971use}. Furthermore, Tripathy et al. \cite{tripathy1999detecting} have shown that the perception of a regular alignment of dots masked by a random-dot background can be improved by breaking the regularity of the alignment by having random dots from the background lie closely to the alignment. Other approaches have based their model on proximity. For instance van Oefellen et al. \cite{van1983algorithm} and the later refinements \cite{smits1985perception,smits1986model} start their model algorithm by convolving the image with Gaussians. This yields a landscape image with peaks and valleys, from which we can then find clusters. This is adaptable to non straight smooth curves. Other approaches for modelling the perception of good continuation of curves exist, such as considering a local associative field which defines a set of rules along which neighbouring points can be considered on the same curve depending mainly on their consistent orientation \cite{field1993contour} or a pyramidal top-down and bottom-up approach that takes into account local and global information \cite{pizlo1997curve}. Other models have been studied such as saliency \cite{ullman1988structural} or curvature based approaches, but they are considered as inappropriate models for human perception. Indeed, the saliency approach from \cite{ullman1988structural} requires a linear time processing with respect to the length of the curve, when it has been shown that our perception is independent of it \cite{pizlo1995exponential}. For the curvature, this measure is scale dependent which would imply that perception performance would strongly depend on the size of the curve. However humans perceive curves with some scale invariance. For instance, the perception of a circle in everyday life does not seem to be affected by our distance to it and this invalidates a curvature based model for perception \cite{pizlo1997curve}. More recently, Preiss \cite{preiss2006theoretical} studies a wide range of methods for the perception of structures, using tools from nearest neighbours to Delaunay triangulations and reflective symmetry and analysing their statistics and implementing algorithms using them.

\paragraph{Just noticeable difference}
Gestalt is not the only way at looking at our problem. Indeed, our signals essentially consist in a dynamic probabilistic density field. In particular our signals' probability field only take two constant values but their location varies continuously through time. We struggle or fail to perceive the underlying structure when the difference in probabilities is very low but when it is very large we easily see. This is typical of the Just Noticeable Difference (JND) concept. This theory is applicable to general signal detection. It consists in finding what is the perception threshold for a type of signal: what is the minimal difference needed in order to perceive the signal. Indeed, our sensitivity is limited and we are not necessarily able to perceive very small differences or variations within some stimulus. The thresholds depend on the task at hand, the nature of the signals and usually varies depending on values such as the intensity of the signal. The earliest and most famous dependency law of the threshold is the Weber law \cite{fechner1860elemente} which states that the minimal difference for perception is proportional to the intensity of the stimulus. While the Weber law is an almost universally applicable classic in visual perception, there are fields where it does not apply \cite{smeets2008grasping} and it has been shown that it is tends to be a bad approximation when the intensity of the stimulus goes into the very high or very low range \cite{stout1915manual}. Many extensions or more refined models exist \cite{wu2019survey}. The thorough study of the JND also leads to computational methods that adapt better to human vision \cite{wu2019survey}. An interesting example is the study of the just noticeable difference in dot densities. Barlow \cite{barlow1978efficiency} measured the human ability to detect regions based on different dot densities for static and dynamic displays by analysing the statistical efficiency\footnote{This value is defined as the squared ratio of the signal-to-noise ratios of a subject and an ideal observer. An ideal observer, from a signal processing point view, is a best performing algorithm under the oracle assumptions: it is the best performance we can expect. From a psychophysical perspective, the signal to noise ration is an empirical estimation of the $d'$ of the subject or observer. This quantity is called the sensitivity index and is commonly used in psychophysics. It evaluates the difference between the means of the signal and noise distributions compared against the standard deviations of these distributions.}. However, the density is studied solely as density per unit of area and there is no study of density per unit of length for one dimensional structures of points (with some density point density per unit length) embedded in a two dimensional background with another density per area. Furthermore, in the dynamic approach, all points persist in time. This diminishes the relevance of this work to our problem. To the best of our knowledge, there is no JND study adapted to our kind of signal: a dynamic probability or density field, where dots do not persist in time, masked by unstructured noise almost identical in behaviour to the statistics of the structures of interest. However JND is not the only alternative to approach the analysis of such problems.

\paragraph{Generalities on motion perception}
Motion perception has been extensively studied from biologists to psychophysicists \cite{johansson1980event,palmer1999vision}. Motion is a complex phenomenon. We naturally perceive motion in every day life, and some phenomena suggest that motion perception is in fact one of the primitives of human perception, and is not a result of some post-processing of the perceived visual field. Two interesting examples are the autokinetic effect, where one tends to perceive a motion while looking at a fixed point in total darkness after a long amount of time, and the waterfall illusion, where one looks at noise moving uniformly in some direction for a long amount of time and then suddenly looks at a static scene but perceives motion in the opposite direction. In both these cases, we perceive motion on static scenes which suggests that motion perception cannot only be a result of post-processing of the visual input. Generally, the concept of motion perception is coined in the terms of real and apparent motion for which there needs not to be real motion in the visual field. Apparent motion is used very often when motion perception is induced by a succession of static frames with dots or features that jump between the frames: nothing moves continuously yet we can perceive a continuous motion. This is traditionally called \say{beta} motion. Depending on the frame-rate change, the perceived motion can be significantly different. Indeed, too high or too low a frame-rate will cause us to perceive just two positions flickering. Moderately slow flickering will create the illusion of continuous motion known as the \say{beta} phenomenon mentioned above. On the other hand, faster frame-rate will create a different kind of motion called the \say{phi} phenomenon where we perceive a discontinuous motion. In other flickering cases, during the \say{phi} phenomenon, we might perceive for instance the background moving. This motion has traditionally been considered pure by psychophysicists. Several models exist to explain motion perception such as Braddick's dual theory consisting in short-range and long-range processes \cite{braddick1974short}: where the maximum displacement to perceive movement in dense random-dot stereograms is much smaller than the one for classically isolated apparent motion. Nowadays, the three orders of perception theory is commonly accepted\footnote{Interestingly enough, the first two orders are linked to Braddick's short-range process whereas the third order would be his long-range process.} \cite{lu1995functional,nishida2018motion}, although other models exist such as zero-crossing or gradient detectors \cite{lu2001three}. The first order consists of the \say{phi} phenomenon. The difference between the first two orders of perception is that the first consists in some flickering or alternation or some other quality that results in a change of the luminance energy in the Fourier domain. As an example, a flickering at the appropriate rate of two separate dots would result in a first order motion whereas a flickering of random textures that does not change the luminance energy would consist in a second order motion. The third order is considered as feature tracking and is considered a higher level process as it requires attention and not only bottom-up processing. Computational models have been proposed especially the Reichardt detector \cite{reichardt1956systemtheoretische,vansanten1984temporal}. The Reichardt detector is fundamentally a delay-and-compare network that looks at sums or products of some signal (or its Fourier transform) and its delayed version and that gets the displacements by finding peaks in its output. It is a linear detector capable of recovering first order motions and requires an extra non linear rectification in order to recover second order motions. Recently, less attention is given to second order motion and its mechanisms remain quite obscure \cite{nishida2018motion}.

\paragraph{Perception and grouping of dynamic dots embedded in noise}
We are particularly interested in the motion perception or grouping literature on random dots with a masking effect. A lot of psychological research has been done on such video data. If one is allowed to make the primary stimulus flicker and allow some other secondary stimulus to appear during the interval between the displays of the primary stimulus, Braddick showed that one could break the perception of apparent motion in random-dot video by adding a bright field \cite{braddick1973masking}. However our data does not flicker in this sense, it is simply resampled per frame. One of the simplest way to group dots together is through the common fate principle \cite{palmer1999vision}. Dots that move together in a euclidean or rigid fashion tend to be easily grouped even if they are immersed in a dense noisy mask. This mask can consist of randomly moving dots or dots moving in a consistent fashion (e.g. static no movement). In fact, Watamaniuck et al. \cite{watamaniuk1995detecting} showed that we can actually perceive the trajectory of a single dot if its movement is somewhat consistent over time (the angular displacement is not too high between frames), even if immersed among randomly moving dots. In this case, the information for two consecutive frames is insufficient in order to perceive consistency. We can even perceive motion in static non flickering random-dots display. Indeed, in Erlikhman et al. \cite{erlikhman2016modeling}, the authors study and model the perception of emerging spatio-temporal boundaries. The videos consist in static dots that all have a same colour (white) except those inside a virtual region that is moving, those in that region have a different colour (black). Most successive frames are identical and there is only a change in a single point when the region passes over it (either during entrance or exit), yet this is enough to clearly perceive the apparent motion of the virtual region: the colour change of three successive points is enough to compute the boundary of the virtual region. Another important aspect that can help perception of moving dots is averaging. There is evidence suggesting humans perform some averaging in order to help them estimate motion and group points together. According to \cite{watamaniuk1992temporal}, we are capable to perceive a global direction of motion of random dots moving according to a random walk where the directions sampled at each frame for each point follow a Gaussian distribution and the direction we perceive is the mean of this distribution. In fact, averaging, whether it be temporal or spatial, is a basic and necessary aspect of the visual memory and it has the greatest influence when subjects are uncertain about the stimulus conditions \cite{dube2015obligatory}, as it happens in our confusing stimuli of almost white noise. However, all studies that we have encountered make a fundamental assumption on the points: they persist in time. To the best of our knowledge, no psychological or psychophysical study so far assumes that each dot disappears after being displayed in a frame and in the next frame an entirely new set of dots is generated to create the next input image to display. In some studies some points disappear, some points appear, or for instance noisy points are randomly resampled each frame, but in no study have we found a totally renewed generation for each frame in time. In our data, not only do we have that the points do not persist in time, even their number may change as well, as each pixel in the foreground and background image is randomly regenerated.

\paragraph{Biological motion perception}
Another interesting topic in motion perception is our ability to perceive amazingly well biological motions and human motions in particular \cite{johansson1980event}. Indeed, grouping is facilitated if the points to be grouped behave consistently with a biological pattern as opposed to arbitrary artificial lab-created patterns. This raises the question: Does biological motion help in tuning the perception in our video data? Johansson demonstrated amazingly the power of biological motion \cite{johansson1973visual}. In his work, Johansson put luminous dots on the joints of human and filmed the scene in the dark. When the person was not moving (i.e. the dots were static), grouping and understanding of the dots could be tricky or even impossible and closely resembling a random configuration (such as when a person was sitting down). But as soon as the person moves we very easily and almost instantly see and recognise a moving human being\footnote{A video demonstrating this amazing yet difficult-to-explain phenomenon is available at \url{https://www.youtube.com/watch?v=1F5ICP9SYLU}.}: grouping and understanding are immediate as soon as there is motion, and that motion is very familiar to us. Johansson gives a mathematical framework for the explanation of this phenomenon by using a vector analysis and decomposing dot movement vectors into a global and relative vectors, and this process can be iterated. For instance, a movement of an arm described as three dots, one on the shoulder, one on the elbow and one on the wrist, is described by a dual pendulum like movement where the elbow relatively rotates around the shoulder dot and the wrist relatively rotates around the elbow dot. Johansson measured the time necessary for us to be able to group and understand the scene between $100$ and $200$ milliseconds \cite{johansson1976spatio}. Following his pioneering work, the perception of biological motion, and in particular human motion, has been extensively studied \cite{blake2007perception}. It emerges that our perception of such motions is robust even in the presence of masks, whether they are simply random dots or even structured random dots moving in some biological fashion \cite{cutting1988masking,bertenthal1994global}. As stated previously, the perception is very fast, with no improvement above $0.8$ seconds of display \cite{cutting1988masking}, however the perception might be holistic rather than in a bottom-up fashion \cite{bertenthal1994global}. We note that although the phenomenon persists for some non human biological motions, we are less tuned to it \cite{pinto2009visual}.


\paragraph{Filtering}
Psychologists and psychophysicists are not the only researchers to have studied perception and recovering information in noisy signals. Due to numerous Computer Vision applications, computational approaches to motion analysis are abundant in the literature. Recall that our data essentially consists of special video sequences contaminated with noise. A first possibility for denoising videos is to use filtering methods. The classical filtering approach for video is the Kalman filter \cite{kalman1960new}. The Kalman filter requires many assumptions, including that all latent and observed variables have a Gaussian distribution and that the state and transition models are linear. The filter can be extended to non linear models essentially by linearising the non linearities. A generalisation of the Kalman filter to general distributions is through the use of particle filters \cite{doucet2000sequential}. Its assumes a \say{hidden markov} model and tries to recover the posterior distribution using Monte Carlo methods since most often explicit computation is infeasible. Unfortunately, filtering methods aim mainly to undo the degradation processes and their models are not sensitive to Gestalts.

\paragraph{Quantitative Gestalt}
While the ideas of Gestaltism are explained by hand-waving and showing simple examples, many attempts of quantitative analysis have been made \cite{sarkar1993perceptual,jakel2016overview}. The empirical and theoretical methods used to study Gestalts vary and the field lacks of a unified framework. While some use an \say{ideal observer} model to analyse the Gestalts, others use statistical approaches, from statistical tests to Bayesian inference. Furthermore the quantity that is being studied or optimised may be very different, such as false alarm rates or posterior probabilities, in attempts to explain and recover the same Gestalt effect. Witkin and Tenenbaum \cite{witkin1983role} along with Lowe \cite{lowe1985perceptual} are pioneers for using image structure to recover causal structures of a scene. Lowe studied the perception of colinear and parallel lines. For this, he used a \say{non-accidental} alignment idea. Under a null hypothesis, such as a random background hypothesis, is the perceived organisation or pattern in the image likely or not? If it very unlikely, then it goes against this hypothesis and should be due to some structured cause, and this leads to a detection. Lowe essentially used a mathematical formulation from the Bayesian theory. However, his formulation was not done in a systematic manner \cite{desolneux2003grouping}, an important requirement for a mathematical framework for image analysis.

\paragraph{The Bayesian approach}
Bayesian theory is one of the most common approach for quantitatively modelling Gestalts \cite{geman1986markov,jakel2016overview}. Its core idea is the use of the Bayes formula to compute the posterior probability for patterns perceived. If $y$ is an observation, such as some grouping, and $x$ a structured cause for such an observation, then we want to estimate what is the probability that the cause is true given the observation of the event. The mathematical formulation then becomes:

\begin{equation*}
    \mathbbm{P}(x \mid y) = \frac{\mathbbm{P}(y \mid x)\mathbbm{P}(x)}{\mathbbm{P}(y)}
\end{equation*}

We intend to find the cause $x$ that maximises the posterior probability and this estimator is called the Maximum A Posteriori (MAP). Since the observation is a given, the denominator is considered a constant and it suffices to maximise the numerator in the previous formula. The probability distribution $\mathbbm{P}(x)$ is called the prior and the conditional distribution $\mathbbm{P}(y \mid x)$ is the likelihood. The likelihood is the probability of the observation given the model and as such it can often be easily formulated or approximated. On the other hand, the prior requires a strong assumption on what objects we might have in the world and their relative frequency of occurrence. The prior is in practice similar to a regularisation term in a variational formulation. The prior is usually hard to define which one of the main weaknesses of the Bayesian approach \cite{desolneux2000meaningful,desolneux2001edge}. Another issue with this approach is that the maximisation of the posterior probability will always yield an explanation $x$, even in the case of noise only data. As such, it will tend to find some structure in data without any prior structure.

\paragraph{Variational methods}
Another common approach is the use of a variational model \cite{mumford1985boundary}. Similarly to the Bayesian approach, we can model the problem and introduce a functional that we wish to minimise. We choose as best the structure that fits to the image data and minimises the designed functional. Traditionally, the functional will be expressed as a sum of a fidelity term of the observation with the structure considered and a weighted regularisation term on the structure that measures how likely it is. Variational approaches have several drawbacks \cite{desolneux2000meaningful}. A first important issue is the choice of regularisation parameters: the results are usually very sensitive to them. While it can be possible to estimate them, for instance by using Lagrange multipliers \cite{rudin1992nonlinear}, this estimation is usually very rough if not inaccurate. Secondly, similarly to the Bayesian approach, minimising a functional will always yield a result even when fed with purely noisy data. Finally, these methods usually tend to be local methods as the functional often depends on local structure. This is a major drawback for a Gestalt oriented approach using a variational formulation.

\paragraph{Global grouping}
In order to recover Gestalts, global computational methods have also been explored. One of the most significant global approach is the Hough transform \cite{hough1962method,duda1971use,maitre1985hough}. Given a parametric model of Gestalt or structure to recover, such a straight line, we can map the image into the dual parameter space. We can recover Gestalts as peaks in the Hough space. Naturally, it is necessary to choose some threshold in order to recover only the most significant peaks. The choice of threshold can be hand tuned, statistically derived or learned through examples. While being an extremely simple idea, the Hough transform is very powerful. In fact, Kiryati et al. \cite{kiryati1991probabilistic} showed that using only a random subset of all points and computing the Hough transform one could recover the line without harming too much the performance, measured in terms of false alarms. Another very famous and relevant approach for detecting line alignment, and more generally parametric Gestalts, is the Random Sample Consensus commonly called RANSAC \cite{fischler1981ransac}. It consists in randomly sampling two points as candidates for a line and count the number of inliers and define that as the score of the sample. We keep the sample that yields the highest score (exactly one output) or those above some threshold (none to several outputs) and recompute more precisely the position of the line based on this sample and it's set of inliers using least squares. Statistically, RANSAC will recover with high probability the line with only a low number of random samples, which can drastically reduce the computation time. Another notorious global method can use saliency, which consists in trying to define and compute a score measure that gives an indication of how salient a position in the image is and to then return some continuous curve through highly salient positions. The curve can be intrinsically linked to the score as in Ullman et al. \cite{ullman1988structural} or explicitly reconstructed as in the extension field approach from Guy et al. \cite{guy1992perceptual,guy1993inferring}. Many other approaches exist from Voronoi diagrams \cite{ahuja1989extraction} to fuzzy logic clustering \cite{dave1990use}. The field for global grouping algorithms is rich, although both two previous methods do not consider masking random dots. Traditional global methods usually fall short of giving a satisfactory justification of the presence or not of the Gestalt structure that they are looking for in the image that they are presented with.

\paragraph{A contrario}
In order to remedy the shortcomings of previously mentioned methods , Desolneux et al. \cite{desolneux2000meaningful} developed the \say{a contrario} framework. A contrario is a global detection method for Gestalt recovery. It is a statistical test approach where the test is designed to minimise the expected number of false alarms. It is based on the \say{Helmhotz principle}: a geometric structure in an image is perceptually significant (and therefore pops out and becomes visible) if its number of occurrences by chance is very low. The overview of the method is the following. We define the space of the Gestalt structures we are looking for and then define a finite sampling of this space. We must estimate the size of this sampling. Given a sample, we look at a score, computed as a count of dots appearing near the sample in our visual input data, and look at how unlikely the score is under a noisy background only assumption $H_0$. This means we look at the tail of the distribution under $H_0$. The trick of the a contrario method is to look at the expected number of false alarms rather than simply the probability of a false alarm. As such, the test criterion is the estimated size of the sampling of space times the tail of the probability distribution under $H_0$ to get a score as extreme as the observed one. This entity is called the number of false alarms and it is then tested against a user defined threshold $\varepsilon$. However, the threshold is not that important due to a log dependency and in practice choices of $\varepsilon\le 1$ give good results. A choice of $\varepsilon = 1$ is a common default \say{good enough} choice for most applications of a contrario. As such, a contrario becomes a parameterless theory for global Gestalt detection and gives a criterion based on the expected number of false alarms for detections. A contrario is inspired from the MINPRAN approach \cite{stewart1995minpran} and can be viewed as a systemisation of the approach initially designed for recovering planar alignments of 3D points which also alleviates crucial unresolved parameters: avoiding the choice of parameters for sampling the candidates by considering the sampling as an implicit parameter deriving by our choice of Gestalt and how we sample, and by overcoming the need for an independence assumption between samples, which is clearly false, by looking at expectations rather than crude probabilities. A contrario has also the power of being a simple framework for most global detection tasks \cite{desolneux2003grouping} from point alignment detection \cite{desolneux2000meaningful}, edge detection \cite{desolneux2001edge}, curve detection \cite{lezama2015grouping,lezama2016good,blusseau2015salience}, vanishing point detection \cite{almansa2003vanishing}, histogram modes \cite{desolneux2003grouping}, clustering \cite{desolneux2003grouping}, face detection \cite{lisani2017contrario} or even forgery detection by performing a contrario on the heatmaps provided by deep neural networks \cite{flenner2018resampling}. In particular, a contrario straight line alignment detection has been extensively studied \cite{desolneux2000meaningful,von2014contrario,lezama2DptAlignDet,lezama2015grouping}. Furthermore, it has also been done on oriented points, where the points can be replaced for instance by small segments \cite{desolneux2003maximal}, Gabor filters \cite{blusseau2015salience} or a 2D vector field \cite{von2012LSD}. The advantage of using oriented points is that for natural images we directly have access to an oriented 2D vector field by computing the gradient of the image. This lead to the very nice Line Segment Detector (LSD) algorithm \cite{von2012LSD} by seeing that using gradient alignment allows us to get edges without having to consider the difference in grey level intensities. When computing straight alignment, having oriented points also helps in the detection performance and fewer points are required to pass the detection threshold. Unfortunately, our data is unoriented. Non straight alignment has also been recently studied through a contrario approaches for curve detection by Lezama \cite{lezama2015grouping,lezama2016good} and Blusseau \cite{blusseau2015salience}. Lezama's approach is based on a local symmetry assumption: along a smooth curve, a triplet of consecutive points should be close to symmetric along the middle point. This assumption implies that the curve should be sampled close to uniformly which is the not the case in our data. Blusseau's approach uses oriented points as Gabor filters and as such cannot be used for our data. Finally, a contrario has been shown to empiricall perform similarly to humans on some tasks \cite{desolneux2003computational,blusseau2015salience} which makes it a good candidate for a simple computational model of the human perception system. Interestingly, Desolneux et al. \cite{desolneux2003computational} use a contrario to estimate a theoretical performance that fits well human performance based on the length and density of the line of the alignment and is somewhat similar to what we will develop in this report.

\paragraph{Computational motion perception}
Motion perception has also been extensively studied from a computational approach \cite{kaushal2018soft} and not just from the psychological one. The main historical approaches to perceive or track motion in natural series of images was to use optical flow methods that try to recover the motion field (projection of the true 3D displacement field onto the 2D image plane) \cite{aggarwal1988computation}. These methods can be structured into dense approaches, that try to recover the motion at every point of the image, or sparse methods, that rely only the motion of a few number of feature points. Optical flow methods are naturally ill-posed due to the fact that the problem is intrinsically underconstrained, even if one can avoid the issue in some cases \cite{uras1988computational}. The art of optical flow methods is in choosing an appropriate additional constraint to overcome the intrinsic edge flow ambiguity. Stochastic approaches using markov random fields models have also been used in order to compensate for the lack of traditional constraints' capability to handle issues such as occlusions \cite{konrad1988estimation,konrad1992bayesian,heitz1993multimodal}. While historically dominant, optical flow methods are not the only methods for computational motion perception. Machine learning approaches including neural networks, fuzzy logic or evolutionary algorithms or hybrid versions have been developed \cite{kaushal2018soft}, even greedy algorithms exist \cite{shafique2005noniterative}. An interesting approach is from Zhou et al. \cite{zhou2012coherent} who use a \say{Coherent Neighbour Invariance} assumption in order to track points moving locally consistently together, even if some points appear or disappear and if masked by random dots. The method was not only tested on natural data but also interestingly on simulated random-dot kinetograms. The \say{Coherent Neighbour Invariance} consists in two important assumptions. First, the neighbourhoods of coherent points tend to remain invariant over time and secondly that the correlation of the velocity vectors of neighbouring coherent points remains high when averaged over time. Unfortunately, since in our data points to do not persist in time, i.e. all points appear and disappear at each frame, these assumptions do not hold for us. To the best of our knowledge, we do not know of any method that does not assume persistence over time of points.

\paragraph{Computational biological motion perception}
Biological motion perception has also received interest from the computational literature. Johansson's \say{vector analysis} \cite{johansson1973visual} was not only a pioneering work for psychologists but also for computational methods. Indeed, this has led to approaches based on the perception and analysis of the motion of the human skeleton \cite{kale2016study}. Gershman et al. \cite{gershman2016discovering} even showed that a Bayesian formulation of a hierarchical tree structure could computationally fit well Johansson's \say{vector analysis} theory. Unfortunately, to the best of our knowledge, these methods do not apply to our kind of data with only random dots that do not survive a frame.

\clearpage
\newpage
\section{The ``A Contrario'' Framework}
\label{sec: a contrario}

A recent and effective way to quantify Gestalt rules into a mathematical framework was proposed by Desolneux, Moisan and Morel \cite{desolneux2000meaningful,desolneux2003grouping}. The framework adapts to any detection task. The idea is to assume that most of the observed data was generated by some unstructured random distribution $H_0$ and under this distribution to look at unlikely patterns to be detected as objects of interest. However, simply looking at the probability of a data configuration will most often not encapsulate the meaningfulness of the outlier configuration. To explain the a contrario framework, we shall look at a simple example.

\paragraph{Example}
\label{par: simple example presentation nfa}
In this Section, we will consider a simple example, following one from von Gioi \cite{von2014contrario}, that gives a good idea of how the a contrario framework works. Consider strings of binary data where $H_0$ defines that each bit independently takes the value 1 with probability $\frac{1}{2}$. In this setting, for any detection task and any candidate sequence, the probability of observing this sequence under $H_0$ only depends on the length $L$ of the sequence and is given by $\frac{1}{2^L}$. This does not provide any information on the structure of the observed sequence. For instance, the sequences $1111111111$ and $1001101001$ have the same probability under $H_0$ yet we perceive that the first pattern is clearly highly structured and we feel that it is quite unlikely under $H_0$. Therefore rather than looking at the probability of specific sequences, the idea is to associate some other measure to an observation and then look at how likely this measure is. This measure should clearly depend on the detection task at hand and on the nature of the data. Most often, the measure we consider will be the result of a counting procedure. For instance, if the task is to detect a subsequence with many 1s in very long binary sequences, then a good value to associate to an observation is the number of ones in the observed sequence. Let us fix the length of subsequence strings of interest to some $L$. Let $c_i$ be a candidate, i.e. subsequence of this length, within the observed long sample $x$. Denote the counted value by $k(c_i,x)$. The problem now boils down to defining some threshold $k_i$ for each candidate subsequence giving us a decision level between being acceptable under $H_0$ (i.e. no detection) and rejecting $H_0$ (i.e. detecting something of interest). Note that $k_i$ could explicitly depend or not on the candidate $c_i$. Rather than choosing $k_i$ to be a decision level with an $\alpha$-confidence value that we will not make a false detection with probability greater than $\alpha$, a powerful key idea of the a contrario framework is to look at the expected number of false alarms. The idea is that under the assumption on unstructured data $H_0$, we want to control the expected total number of false alarms on a single piece of data, which in our case is a long sequence $x$ of bits. Under $H_0$, each sequence $x$ is a realisation of a random variable $X$, following $H_0$. A \say{false alarm} will be considered as a detection by the a contrario algorithm. We shall denote the number of false alarms as $d$, which is a random variable depending on $X$. The number of false alarms is given by:

\begin{equation*}
    d(x) = \sum\limits_{i=1}^{N_T}\mathbbm{1}_{k(c_i,x)\ge k_i}
\end{equation*}

where $N_T$ is the number of candidate positions, also called the number of tests. The expected number of false alarms is on the other hand:

\begin{align*}
    \mathbb{E}(d(X)) & = \mathbb{E}\Bigg( \sum\limits_{i=1}^{N_T} \mathbbm{1}_{k(c_i,X) \ge k_i} \Bigg) \\
    & = \sum\limits_{i=1}^{N_T}\mathbb{P}(k(c_i,X) \ge k_i)
\end{align*}

We choose the values $k_i$ so that $\mathbb{E}(d(X)) \le \varepsilon$, where $\varepsilon$ is the confidence level. While there are many possibilities to choose from, most often a uniform splitting of $\varepsilon$ is done: ${k_i = \min\limits_{\widetilde{k}}\mathbb{P}(k(c_i,X) \ge \widetilde{k}) \le \frac{\varepsilon}{N_T}}$. Equality can be reached in this definition for sufficient continuity assumptions \cite{von2014contrario}. Note that, in our simple example, all the $\mathbbm{P}(k(c_i,X))$ are identical, due to the fixed length $L$ and iid generation for the bits in the string $x$ under $H_0$. This implies that all thresholds could be equal, say to $k_{th}$, and the expected number of false alarms is just a sum of identical terms: $\mathbbm{E}(d(X)) = N_T\mathbbm{P}(k(c_i,X)\ge k_{th})$. Here, as previously discussed, we can choose $k_{th}$ to correspond to the level $\frac{\varepsilon}{N_T}$, i.e. ${k_{th} = \min\limits_{\widetilde{k}}\mathbb{P}(k(c_i,X) \ge \widetilde{k}) \le \frac{\varepsilon}{N_T}}$, and this will enforce the expected number of false alarms to be at most $\varepsilon$.

The a contrario test for a candidate region $c_i$ (i.e. a possible detection at location $c_i$) is:

\begin{equation*}
    k(c_i,x) \ge k_i
\end{equation*}

which can be rewritten thanks to the definition of $k_i$ as:

\begin{equation*}
    N_T \mathbbm{P}\Big(k(c_i,X)\ge k(c_i,x)\Big) \le \varepsilon
\end{equation*}

However, in the second formulation, we do not have to explicitly compute $k_i$ but only the probability of having at least as many points as the number of points observed. Explicitly computing $k_i$ can be difficult whereas computing the probability can be simple as we will see later. See Figure \ref{fig: statistical test} for an illustration on statistical tests and the link between thresholds and tail of distributions. The quantity $\mathcal{N}_{FA}(c_i,x) = N_T \mathbbm{P}\Big(k(c_i,X)\ge k(c_i,x)\Big)$ is traditionally called the \say{number of false alarms} in the a contrario literature. Note that it is in fact the contribution of candidate $c_i$ to the expected total number of false alarms.

\begin{figure}
    \centering
    \includegraphics[width=0.7\textwidth]{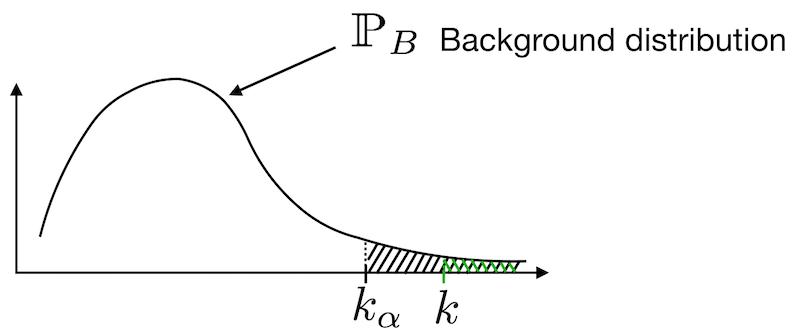}
    \caption[Short version]{Classical right tailed statistical test. Under $H_0$ the distribution assumption for our random variable $K$ (for us $H_0=B$ the random background only distribution assumption), we want to have a statistical test that guarantees that false alarms (also called type 1 error) occur with probability lower than the user-defined confidence level $\alpha$. For this we choose a critical region of weight $\alpha$, typically the tails of the distribution, or in the case of a right sided tail distribution such as the binomial distribution we choose the rightmost confidence region of complementary cumulative distribution (also called right tail distribution) closest to $\alpha$ from below. This defines a threshold value $k_\alpha$, such that if our statistical test is $k\ge k_\alpha$ where $k$ a random realisation of $K\sim H_0$, then we wrongfully reject $H_0$ when it was true (i.e. a false alarm) with probability smaller than $\alpha$. Rather than computing thresholds which requires computing the inverse of the cumulative distribution function, we can simply look at the value taken by the right tail distribution. The test $k\ge k_\alpha$ is equivalent to testing if the right tail distribution from $k$ is smaller than $\alpha$ by definition of $k_\alpha$.}
    \hfill
    \label{fig: statistical test}
\end{figure}

Note that in this formulation, $H_0$ implicitly appears as the distribution of $X$. We can also write this slightly differently. If $k(c_i,X)$ corresponds to a counting variable under the assumption that $X$ follows $H_0$, we can then view $k(c_i,\cdot)$ as a random variable $K_{c_i}$ that counts the number of occurrences (i.e. ones in our example) in the candidate area $c_i$ and that can be conditioned, for instance on $X$ distributed according to $H_0$. Its random realisation on the sequence $x$ can be denoted $k_{c_i}$. Then, the a contrario test can be rewritten by conditioning the distribution of $K_{c_i}$ on the $H_0$ assumption:

\begin{equation*}
    N_T \mathbbm{P}_{H_0}(K_{c_i} \ge k_{c_i}) \le \varepsilon
\end{equation*}

We will use both notations in the rest of this report.

\begin{figure}
    \centering
    \begin{subfigure}[t]{0.5\textwidth}
      \centering
         \includegraphics[width=\textwidth]{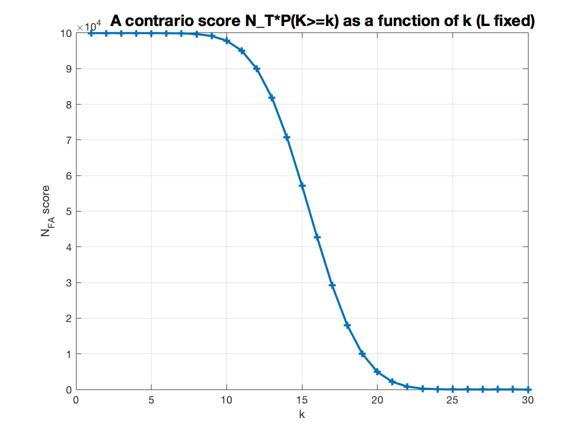}
    \end{subfigure}%
    \hfill
    \begin{subfigure}[t]{0.5\textwidth}
      \centering
         \includegraphics[width=\textwidth]{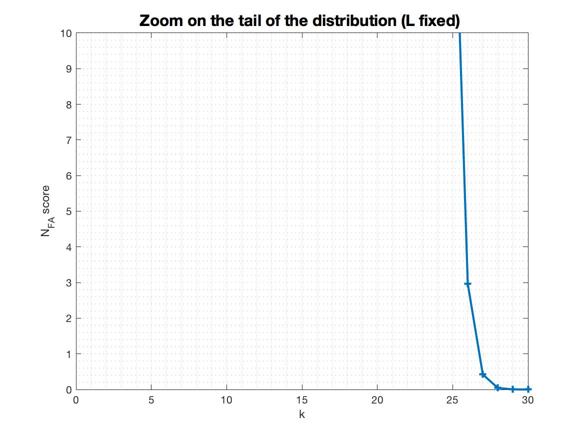}
    \end{subfigure}%
    \hfill
\caption{Behaviour of $k_{th}$ with respect to $\varepsilon$ in our simple sequences of bits example. We fix the length of candidate sequences to be $L$. We assume that the single piece of data $x$, a very long string of bits, is of length $n_x=100000$. Sequences of length $L=30$, with an abnormally large proportion of $1$s, are chosen to be the interesting structures we are looking for. This means that there are $N_T = n_x-(L-1) = 99971$ candidate positions, i.e. candidate positional tests to perform. The measure we will study is the count of $1$s in the candidate sequences of length $L$. Under the background only assumption that each bit independently takes the value $1$ with probability $p=\frac{1}{2}$ (and so with probability $1-p=\frac{1}{2}$ take the value 0), the count measure, a random variable denoted $K$, follows a binomial distribution $K\sim B(L,p)$. Left: the a contrario number of false alarm score $\mathcal{N}_{FA} = N_T\mathbbm{P}_{B}(K\ge k)$ as a function of $k$. Right: the same score but zoomed in around the traditional values of interest. Note that we can here see that for $\varepsilon=1$ the associated threshold level is $k_{th}(\varepsilon=1) = 27$. This means that in order to enforce that the a contrario algorithm makes at most $\varepsilon$ false alarms on average when ranging through the $N_T\approx 100000$ candidate positions, we must only decide to detect windows of length $L=30$ that have at least $27$ bits taking the value 1!}
\label{fig: threshold simple example nfa plot}
\end{figure}

\begin{figure}
    \centering
    \begin{subfigure}[t]{0.5\textwidth}
      \centering
         \includegraphics[width=\textwidth]{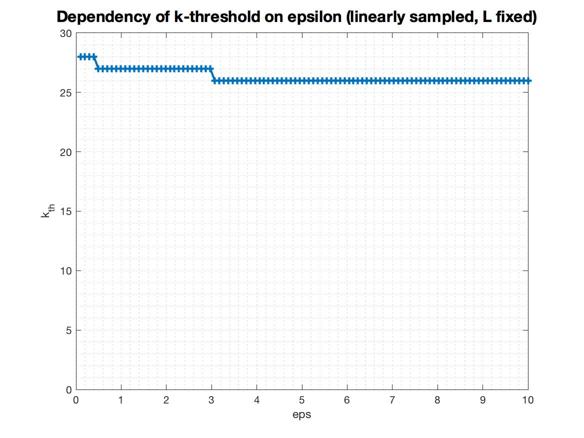}
    \end{subfigure}%
    \hfill
    \begin{subfigure}[t]{0.5\textwidth}
      \centering
         \includegraphics[width=\textwidth]{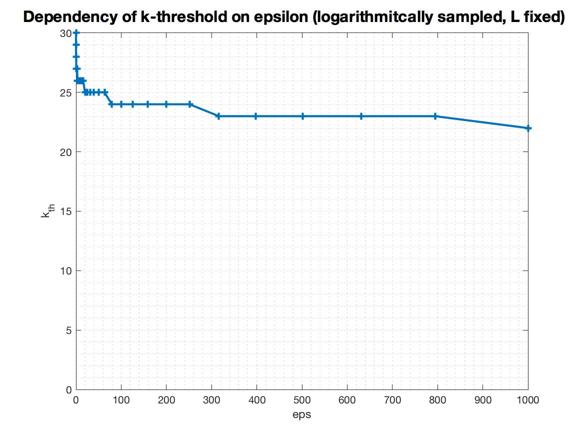}
    \end{subfigure}%
    \hfill
    \begin{subfigure}[t]{0.5\textwidth}
      \centering
         \includegraphics[width=\textwidth]{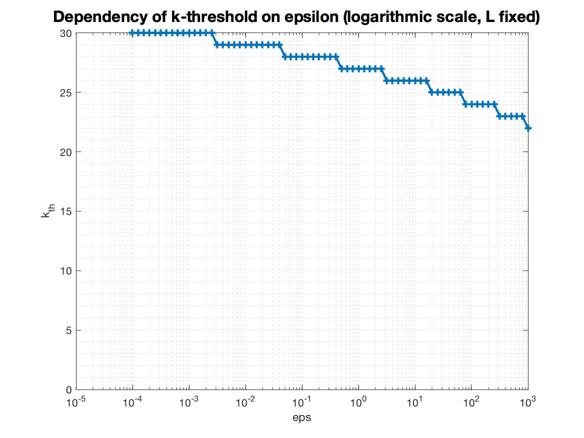}
    \end{subfigure}%
    \hfill
\caption{We are working on the same simple problem as in the previous Figure \ref{fig: threshold simple example nfa plot}. Top left: plot of the threshold $k_{th}$ for a linearly sampled $\varepsilon$ in $[0.1,10]$. Note that you only require two more points for $\varepsilon=0.1$ than for $\varepsilon=10$ which is not much and thus does not significantly increase the number of false alarms and the performance of the algorithm. Arbitrary choices of $\varepsilon$ within this range will generate similar performances. It is also empirical evidence of the ``logarithmic dependency'' of the threshold with respect to $\varepsilon$. Top right: plot of the threshold $k_{th}$ for a logarithmically  sampled $\varepsilon$ in $[10^{-4},10^3]$ but displayed in linear scale. Bottom: plot of the threshold $k_{th}$ for a logarithmically  sampled $\varepsilon$ in $[10^{-4},10^3]$ but displayed in logarithmic scale. Note that for $\varepsilon < 10^{-4}$ that the false alarm score for $k=30$ is not below $\varepsilon$, therefore any detection for these tiny $\varepsilon$ cannot guarantee that the expected number of false alarms will be below the tolerated number $\varepsilon$. Therefore for such small values the only decision possible is to never reject the background hypothesis, leading to $0$ false alarms. Note how the last three plots reveal the logarithmic dependency of $k_{th}$ on $\varepsilon$.}
\label{fig: threshold simple example plot thresh}
\end{figure}

The power of the a contrario formulation is that it guarantees on average an expected number of false alarms lower than $\varepsilon$, regardless of all other parameters, such as the nature of the model we are looking for, the size or resolution of the data. It gives a theory with only a single parameter that adapts naturally to several types of data and detection problems. Desolneux et al.  showed in \cite{desolneux2000meaningful} that the choice of $\varepsilon$ does not have a significant impact on the detection results. In fact, due to a logarithmic dependency of the decision thresholds on $\varepsilon$, a choice of expected number of allowable false alarms, between for example 0.1 and 10, will lead to similar decisions and hence to similar detection performance. Furthermore, depending on the task at hand and the model we consider, estimating the number of candidate positions $N_T$ can become particularly tricky and so in practice rough estimates of orders of magnitude often suffice. Usually, we select $\varepsilon\le 1$ in order to limit the number of false alarms of the detector to less than one, on average. If we are required to make much fewer false alarms then naturally we can take $\varepsilon \ll 1$. This being said, empirically, a choice of $\varepsilon = 1$ gives good and coherent results for traditional detection problems \cite{desolneux2000meaningful,desolneux2007gestalt,blusseau2015salience} and so in practice this is the choice we make for $\varepsilon$. This means that we guarantee that the expected number of false alarms for each piece of data will be at most 1 ($\varepsilon=1$) which is a powerful result. Using the expectation also results in dealing with a mathematical quantity that is easier to manipulate than the crude probabilities, since candidate positions for detection  are not independent, and in fact can be significantly correlated (due for instance to large overlap) \cite{desolneux2000meaningful}.

An other advantage of dealing with expectations is to introduce $N_T$ in the decision. This can be easily seen with a bit of heuristic argumentation. The more complex is the structure we are trying to detect, the larger is the discretised space that this structure can live in i.e. the larger is the number of candidate positions $N_T$. If $N_T$ increases, then $k_i$ decreases and in order to get a detection in a contrario we need a larger observed number of points. And so for more complex models we need more points which seems logical, although some contradictory evidence also exists (e.g. Pizlo et al. \cite{pizlo1997curve}). See Figure \ref{fig: nb points complexity shape} for an illustration.

\begin{figure}
    \centering
    \begin{subfigure}[t]{0.3\textwidth}
      \centering
         \includegraphics[width=\textwidth]{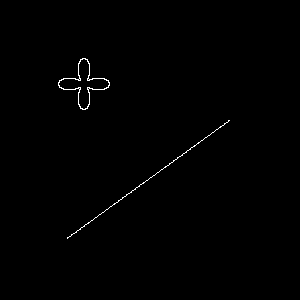}
         \caption{Groundtruth}
    \end{subfigure}%
    \hfill
    \begin{subfigure}[t]{0.3\textwidth}
      \centering
         \includegraphics[width=\textwidth]{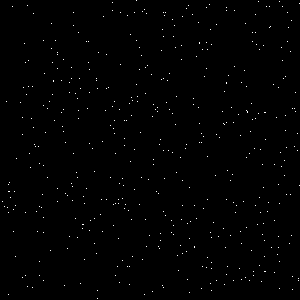}
         \caption{10 points}
    \end{subfigure}%
    \hfill
    \begin{subfigure}[t]{0.3\textwidth}
      \centering
         \includegraphics[width=\textwidth]{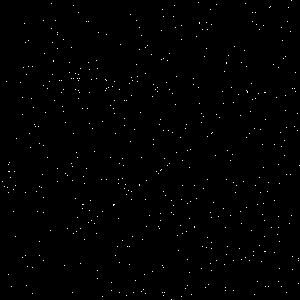}
         \caption{20 points}
    \end{subfigure}%
    \hfill
    \begin{subfigure}[t]{0.3\textwidth}
      \centering
         \includegraphics[width=\textwidth]{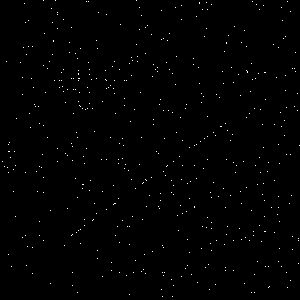}
         \caption{30 points}
    \end{subfigure}%
    \hfill
    \begin{subfigure}[t]{0.3\textwidth}
      \centering
         \includegraphics[width=\textwidth]{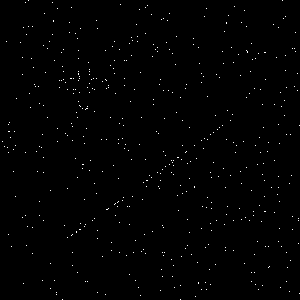}
         \caption{40 points}
    \end{subfigure}%
    \hfill
    \begin{subfigure}[t]{0.3\textwidth}
      \centering
         \includegraphics[width=\textwidth]{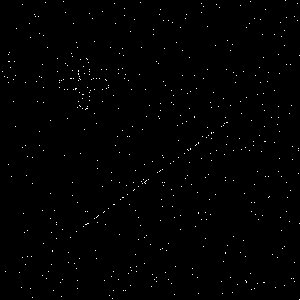}
         \caption{50 points}
    \end{subfigure}%
    \hfill
\caption{Degraded versions of the same image (left) where the number of points on each structure, the line and the conchoide, contain 10 to 50 points. The background degradation parameter is $p_b = 0.005$. For 10 points, both structures are invisible. For 20 points, the line is visible but the conchoide is invisible. For 30 points, a structure can be seen where the conchoide is but there can be ambiguity on its shape. For 40 points the ambiguity is lessened. For 50 points the ambiguity is removed. Both structures have the same arc-length. This example shows that the complexity interferes with the needed number of points in order to perceive a structure. Increased geometric complexity of the structure seems to require more points. This phenomenon occurs in a contrario as well as higher complexity of shapes leads to a higher number of candidates $N_T$ which implies a higher threshold in order to have a number of false alarm score to remain below level $\varepsilon$}
\label{fig: nb points complexity shape}
\end{figure}

The a contrario test can be linked to classical statistics. See Figure \ref{fig: statistical test}. In classical statistical test theory, we perform some test under a hypothesis $H_0$ and test whether to reject it or not. In a contrario, $H_0$ is the background only hypothesis. We decide on a significance level $\alpha$, such as $\alpha = 5\%$. This value is the probability we are ready to accept for making false alarms, i.e. rejecting $H_0$ when it is true. We then partition the realisation space of our random variable into two parts: one with probability measure $1-\alpha$ and one with measure $\alpha$. The latter one is called the critical region and values for which the random measurement falls within this region will imply that we decide to reject the background hypothesis. While the decision on how to partition the space can be done in various or sometimes even in an infinite number of ways, we traditionally prefer to centre the non critical region around the mean of the random variable and consider the critical regions as the tails of the distribution. Depending on the problem and the distribution, one either splits the critical region to cover both tails or one can just consider one of the tails. In classical a contrario, we consider $\mathbbm{P}(k(c_i,X)\ge k(c_i,x))$ which corresponds to looking at the one sided right tail for the critical region. Denote now $Y$ the random variable we are looking at and $y$ its random realisation. Assume $Y$ is one sided such as in our contrario. The critical region is defined as $[y_\alpha,+\infty[$ were $y_\alpha$ is such that $\mathbbm{P}_{H_0}(Y< y_\alpha) = 1-\alpha$ i.e. $\mathbbm{P}_{H_0}(Y\ge y_\alpha) = \alpha$. However, explicitly computing the value $y_\alpha$ can be difficult to determine, and then we can resort to an equivalent approach by considering the $p$-value. The $p$-value, is the probability to have a random realisation even more extreme that the random realisation we are considering. In our example, the $p$-value is simply $\mathbbm{P}_{H_0}(Y\ge y)$. We then observe that $y$ is in the critical region if and only if its $p$-value is smaller than $\alpha$. The trick here is that we need only compute the tail probability once which can be easier than computing the entire distribution necessary for estimating $y_\alpha$ in the first approach. This being said, a contrario is simply a $p$-value test $\mathbbm{P}(k(c_i,X)\ge k(c_i,x)) \le \frac{\varepsilon}{N_T} ?$ The $\alpha$ value is $\alpha = \frac{\varepsilon}{N_T}$. Nevertheless, there are differences in the philosophy of the a contrario approach with respect to the classical statistical test. First, we do not decide on a value of $\alpha$ for the confidence level but on a value of $\varepsilon = \alpha N_T$. This is because in a contrario we control the expected number of false alarms rather than the probability. Secondly, we cannot control easily the exact number of false alarms directly without passing through its expected number. This is due to the fact that the candidate regions are not independent, and so a meaningful alignment, i.e. detection, can be done while being included in an other region and detected there also \cite{desolneux2000meaningful}. Finally, fixing $\varepsilon$ rather than $\alpha$ allows us to not depend on the image size. Choosing $\varepsilon$ fixes the expected number of false alarms regardless of the image resolution whereas choosing $\alpha$ does not: with an increase in resolution we increase the number of false alarms.

Another crucial aspect of a contrario that we have not mentioned yet is how to define precisely $H_0$. Traditionally $H_0$ consists in an assumption of lack of orderly structure. However, it is not precise enough to say this. Consider for instance our images: we have random unstructured background dots and some structured dots on some alignment we are trying to find. Just saying that the background is unstructured does not give enough information. We must define the statistics of the unstructuredness in order to be able to compute the probabilities in the a contrario test. We assume that an unstructured background means white noise: each white background pixel is independent and appears with a probability $p_b$. However, depending on how the image was acquired, $p_b$ is not necessarily known. Several options are possible for estimating this $p_b$. The first option is to have an oracle: someone to tell us what $p_b$ is without having to work for it. in practice, this could happen if for instance the manufacturer of the imaging sensor providing our data would give it to us. When an oracle is not available, we must estimate the background density. Following Lezama et al. \cite{lezama2DptAlignDet}, we have several possibilities. The first option is to estimate $p_b$ globally by counting the total number of points in the image. This estimate works well if the background probability is in fact uniform in the image and does not change locally. We could imagine having white background dots densely appearing in one region and less densely in an other. The less densely appearing white dots will contribute to lowering the estimate of the global $p_b$ and therefore we will artificially have a high number of dots in the high density random background which would lead to many detections of random noise in the high density noise. Therefore the authors of \cite{lezama2DptAlignDet} propose a local estimation of the background density. Their first approach is to consider wider rectangles around the candidate alignments and count the points there. When computing the number of false alarm score, we would compute the probabilities conditionally to the background density being observed: given the fact that a specific number of points appears in the large local region that we have, and under the assumption that the points uniformly and independently appear in the large region (and so do not have a particular preference for the smaller central region of interest), what is the probability that we get the observed number of points in the small candidate alignment region. The probability that one of the appearing white point to be in the alignment region is simply given by the quotient between the areas of alignment and local estimation of the density. This leads to a tail of a new binomial, similar to the traditional oracle a contrario with similar properties. The authors admit that this version is sensitive to background texture boundaries and therefore ultimately refine it by estimating the densities left and right of the candidate alignment and keeping as estimate the maximum one. In our types of data, the background noise is uniform in the image and we will assume that an oracle gives us the background density value. Replacing the oracle by a global estimation would lead us almost to the same results up to minor differences.

There can be a theoretical further complexity depending on the type of data you are working with for performing a contrario. We theoretically require that the points appear independently to one another in each candidate. Naturally, due to the possible overlapping of candidates, the number of white points appearing in each candidate region can be highly correlated and this issue is overcome by looking at the expected number of false alarms rather than the probability of false alarms. However inside each candidate, the points should behave independently. Depending on the data, this is not always rigorously the case. Consider for example an edge detector in natural images such as in \cite{von2012LSD}. We compute the gradient of each image and look at alignment of gradients. However the computed gradients are not independent between neighbouring points. For instance if the computation of the gradient is done using a $3\times 3$ filter, then the gradient of a pixel is correlated with the gradient of any of its 8 neighbouring pixels, but is independent with the gradients of all other pixels. Therefore, the values of the gradients are almost independent in the candidate. Theoretically, this is unsatisfactory and we should find a subsampling strategy to ensure full independence. However, empirically, von Gioi et al. \cite{von2008straight} found that when working on white noise, omitting a sub sampling and assuming that the gradients were independent even if it was not the case did not significantly increase the number of false alarms compared to working on a synthetic truly independent 2D random field. This means that even if we work with the dependent points, we preserve the \say{Helmholtz principle} of not detecting structure in white noise. Thus in practice we need not worry much about a strict independence between points within a candidate region if the correlations happen within a tiny neighbourhood. We do not encounter this issue in our data since we simulated the data in a way that each pixel behaves independently to its neighbours. The issue could arise in a real application when looking at a real sensor that would threshold noisy gradients depending on how the gradients are computed.

It is also necessary to specify what output should our process return. By default, a contrario rejects the background hypothesis, i.e. decides there is a structure and so a detection, whenever the score is below the level $\varepsilon$. For this reason, many detections can be made, including redundant detections. While it is possible to keep all detections, a post-processing step is often desired. Several strategies exist to do post-processing \cite{von2014contrario}. A simple strategy from Desolneux et al. \cite{desolneux2007gestalt} is to consider the exclusion principle. It states that a point should not belong to two different groups obtained with the same Gestalt law (such as the same a contrario algorithm). This principle is linked to the concept of maximal meaningful events \cite{desolneux2003maximal}. The authors provide a simple yet time expensive algorithm following this principle. The idea is the following: each pixel can only vote for a single detected structure. The power of structures is ranked by sorting in increasing order the number of false alarm score. The algorithm is simple. Take the most unlikely candidate (with the smallest score) and consider all the pixels covered by the candidate structure as unavailable. Remove from all the remaining candidate structures those that pass through at least one pixel of the previously made unavailable structure. Recompute the expected number of false alarm score for the remaining candidate structure without counting the unavailable pixels and resort the structures according to this updated score. Repeat the process until the minimum number of false alarm score is larger than the confidence level $\varepsilon$. The crude exclusion principle is not always satisfying as it can prefer long alignments to two separate aligned segments. This leads to the multisegment detector \cite{von2014contrario}. in practice, the computation of the multisegment detector is expensive since it is an exhaustive search and so if possible heuristic approaches similar to the one used in the LSD \cite{von2012LSD} are preferred for real time scenarios \cite{von2014contrario}.

\clearpage
\newpage
\section{Some Initial Experiments: A Contrario on Random-Dot Videos}
\label{sec: initial experiments}

We first tried an a contario-like approach on our video data. Here, we will assume that we know what we are looking for and have a precise model for it. In practice, we worked on videos, of size $300\times300$ pixels ($10.4\times10.4$ cm), where the clean foreground was an edge of known length $L$ and width $w_e$, both in pixel size, in smooth Euclidean movement (rotation and translation) with a small displacement between frames. When we mention pixel size, we mean in pixel side unit size, and we will use this meaning for the rest of the report. We also fix the frame-rate to $FPS = 30$ frames per second. For further simplicity we assume we also know the parameters of the degradation process $(p_b,p_f)$ although it is possible to estimate them locally \cite{lezama2DptAlignDet}.

A summary of the model of the algorithm we will use is shown in Figure \ref{fig: model a contrario merge}.

\subsection{A Merging Strategy for Temporal Integration}

In this paragraph we detail our approach on how to handle such data. Consider a frame $I_t$ at time step $t$ of the video. Should we try and run the a contrario just on $I_t$ then we would lose the time information and we would struggle due to our choice of degradation parameters: recall that we assume that $p_b$ and $p_f$ are such that it is extremely hard, nearly impossible to detect objects on each frame, independently. Instead, a simple way of using temporal information is to consider several frames for the time step $t$. Consider using $n_{f}$ frames for predicting information at time step $t$. For instance we could take only past information $(I_{t-n_{f}+1},\hdots,I_t)$ or only future frames $(I_t,\hdots,I_{t+n_{f}-1})$ or combining past and future information. A 
simple way of combining frames at different times is to simply merge them. Several options are available for merging. For instance one could consider a pixel-wise average. Averaging would change our data from the boolean (white on black) to a continuous space. Furthermore, the movement of the edge, while still being slow and smooth, can be fast enough such that the object might move by one pixel per frame. Then, averaging pixel values will not yield much information due to the fact that the foreground is not on the same pixels between frames. To bypass this, we use a boolean union over frames. If a pixel is activated, i.e. becomes white on the black background, in any of the considered $n_{f}$ frames, then it will be considered as activated in the merge. We take $t$ to be the middle of the time neighbourhood considered. If the movement is nearly uniform in time, all the pixels activated at $t$ on the object will be centred in a cloud of activated pixels on the object in neighbouring times.

Our algorithm will then run on sliding windows of $n_f$ frames through the video. At each time step, it merges the current frame with its temporal neighbours and then runs the a contrario algorithm on the merged data. In order to run the a contrario algorithm, it is necessary to define a model of the structures we are looking for and sample the space in which the structure lives in. We then run the object detection sampling process: each sample being a candidate position. For each candidate position, we take a spatial window and count the number of points in the candidate window and see if their amount is unlikely large under the background only assumption. In order to incorporate scale, for each candidate position we can test with several window sizes. Note that for different sizes the threshold for the number of points to reject the null hypothesis are not the same, therefore only the number of false alarms can be compared between different frames rather than simply the number of points.

\begin{figure}
    \centering
    \begin{subfigure}[t]{0.7\textwidth}
      \centering
         \includegraphics[width=\textwidth]{Chapters/Pictures/model_data_acquisition.jpg}
         \caption{Data acquisition model}
    \end{subfigure}%
    \hfill
    \begin{subfigure}[t]{\textwidth}
      \centering
         \includegraphics[width=\textwidth]{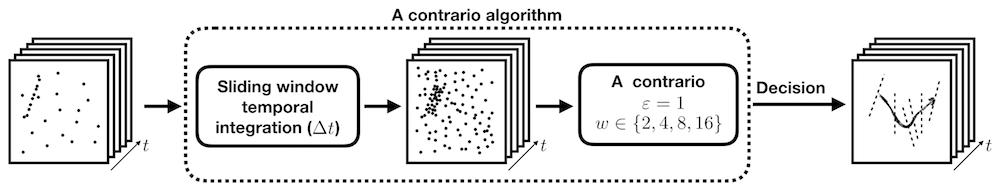}
         \caption{A contrario algorithm model on random-dot videos}
    \end{subfigure}%
    \hfill
\caption{Model of the a contrario algorithm for random-dot videos. The input to the algorithm is the noisy output of the sensors as previously described. The algorithm fundamentally consists in two stages. First a temporal integration which converts 3D time data into a 2D image by merging video frames within a time window, followed by a a contrario spatial search through the merged image for finding structures of interest, in our case straight lines of length $L$ and width $w_e$.}
\label{fig: model a contrario merge}
\end{figure}

\subsection{Sampling the Space of Candidate Objects}

A crucial choice in the design of the a contrario algorithm is to define a spatial structure to be detected and the sampling space for the search process. There are four important steps concerning the spatial search sampling process. First we must choose the structure we are looking for. Then we must define a discretised spatial sample space for this model. Third we must count the total number of spatial positions in the discrete sampling and define it as $N_T$ (also called the number of tests), which will be used for computing the expectation in the a contrario decision. Finally we must define a way to go through the sampling space, in practice, when we run the algorithm.

\paragraph{Model} We already defined the model of the structure we are looking for, as straight edges of fixed length $L$ and width $w_e$. We can assume for simplicity that the sought edge lies entirely within the image and does not leave the image domain, entirely or partially. We can naturally modify this algorithm to more complex shapes. Note that this model of structure applies for each separate frame. However, for several merged frames, due to movement, the area covered by the dynamic edge could increase in length or width or not even ressemble a straight edge anymore. Nevertheless, due to small displacement in the small number of frames considered, this modification will only affect a few number of pixels. The change in length should be minimal and not drastically change the performance of the a contrario, on the other hand the change in width is more important. However, considering several widths in the a contrario algorithm allows to account for this effect. Therefore on the merge of frames, we decide that we are still trying to detect edges of length $L$ and width $w_e$. When defining our candidate windows, we will enforce the length constraint $L$ as it does not significantly change in merges of frames, but we will not enforce a width of $w_e$ to the candidate windows since the width of region traversed by the edge of width $w_e$ in the merge of frames could be significantly higher.

\paragraph{Discretised sampling without length constraint} First, let us forget about the length constraint of the edge. In this case, Desolneux et al. \cite{desolneux2003grouping} provide two possibilities for doing so. The first and most common way is to consider all pairs of pixels as defining an edge. This is called the dense sampling strategy. An other strategy they propose is to only consider pairs of pixels that are activated in the image. This means that candidate lines necessarily have a support already appearing in the image. This strategy is defined as the sparse strategy. The main difference in practice for the a contrario between these strategies is the size of the sampling space. If the image $I$ is of size $N\times N$ then the dense strategy's sample space is of size $\frac{N^2(N^2-1)}{2}$, whereas the sparse strategy's sample space is of size only $\frac{M(M-1)}{2}$ where $M$ is the empirical number of white pixels on the black background, i.e. $M = \# \{i\in\{1,\hdots,N^2\}\,; \, I_i = 1\}$. Naturally, $M<N^2$, which implies $\frac{M(M-1)}{2}<\frac{N^2(N^2-1)}{2}$. However, if the image is very sparse meaning $M<<N^2$, then $\frac{M(M-1)}{2}<<\frac{N^2(N^2-1)}{2}$ which induces a very different threshold $k_i$ for the level $\varepsilon$ in the a contrario test. Desolneux et al. show that in practice, in a sparse image where humans can group together a few dots in very sparse images, the sparse approach allows to correctly recover this edge whereas the dense approach will reject it since its threshold level is very high compared to the one of the dense one due to a much larger number of tests. They also show that on denser images, where humans cannot perceive a line, both methods fail. The authors argue that human vision is sensitive to densities and therefore it is legal to choose the sampling strategy best adapted to our input data. We chose to work with the sparse strategy. An additional reason for choosing this strategy is the following. If we imagine a human as a living a contrario machine, then if we display two images of the same physical size (unit in meters) but with different resolutions, should the human perform a different number of tests on each image solely because of that? If the image is understood as a sampling of the $[0,1]^2$ unit square, and the model of the interesting structure is edges that lie within the unit square, then discretising the sampling space solely on the resolution seems absurd and arbitrary. However, the sparse sampling strategy overcomes this issue since edges in the sample space must have as a support two appearing white pixels. If we look at it from a human perspective, this means that humans will try to imagine lines passing through two points and see if there is any meaningful structure. For this reason, we decide to use the sparse sampling strategy, when the model is simple enough to allow us to do so as our line model clearly is! For a more complex model, it might be impossible or very hard to find a similar sparse strategy and therefore we would have to use another approach with less tight bounds, and eventually very approximate estimates on the number of tests.

\paragraph{Discretised sampling with length constraint} Although we have chosen a way of sampling our image model in the previous paragraph, the sampling strategy does not exactly fit the model we are considering here, since we are looking for edges of a given length $L$. Therefore, we adapt the sampling strategy to incorporate this information. There are mainly two possibilities for doing so. The first possibility is to explicitly consider the subset of the samples that respect the length condition. For the sparse strategy, the appearing white pixels that will support the edge may be constrained to be at a distance of at most $L$: $||\mathrm{pix}_i-\mathrm{pix}_j||_2 \le L$. However, proceeding like this has two issues. The first is that estimating the size of the sampling becomes much harder since $M$ the number of white pixels in an image is a realisation of a random variable with statistics depending on $(p_b,p_f)$, and estimating that number under a length constraint is even harder. The other, more important reason, is that although we know what the underlying structure should be, we can't force humans from just looking for this structure. Humans might effortlessly see an alignment of points that is longer than the desired length $L$, or might see other structures that are not covered by our simple model (see Figure 2.2 of \cite{blusseau2015salience}). Therefore it is not desirable to restrict the sampling to some dimensions constraint. This is crucial as it has a major impact on the value of $N_T$. Sampling under constraints might considerably lower $N_T$, to a level where the thresholds $k_i$ for the confidence levels $\varepsilon$ become too low: a contrario will then detect edges where humans do not see any structure. Therefore, we decided to sample without the length constraint. However, in the algorithm, when we run through samples, we will skip any candidate position that does not respect the a priori constraint.

\paragraph{Estimating the size of the discretised sampling} With the previous sampling descriptions, the size of sampling space is taken to be $\frac{M(M-1)}{2}$ where $M$ the number of white pixels, which for each image is a realisation of a random variable depending on the model, its motion in time if we consider merging of frames, and on the degradation parameters $(p_b,p_f)$. If we are working with several window widths $N_w$ for the candidate windows, then the number of tests is linearly increased by $N_w$: $N_T = N_w\frac{M(M-1)}{2}$.

\paragraph{Run-time through the samples} During run time in the a contrario, we traverse the spatial sample space and test each candidate position. While ranging through the whole space is possible, the number of candidate positions can be huge and the operations in each can be quite slow. If working with static image data (non video), then this would not be a problem. However, since we are working with videos, a single piece of data consists in a video, which translates to many frames depending on the length of the video. Therefore in practice, our algorithm should not take excessively long to process each time step, i.e. our a contrario on a merge of frames should be quite fast. This means that in practice it might not be possible to run through the entire sample space in reasonable time. Therefore we have to devise a way of choosing some subset of samples and run through them. We chose the sparse sampling rather than the dense sampling and we can use this to our advantage to solve this issue. Consider our problem of seeing the edge in a merge of frames. For this, we should probably have $p_f$ large compared to $p_b$ which increases the density of points on the edge compared to the density of points on the background in order to see the edge or a cloud corresponding to the edge in the merge of frames. Since our sampling strategy enforces that candidate edges pass through two appearing white points, due to the difference in density, a significant number of pair of points should exist on the true edge compared to the number of pair of points that exist in the background in a comparable area. Furthermore, the sparse sampling strategy implies that a candidate pair of supporting points of distance at most $L$ are not necessarily at a distance of exactly $L$. Therefore for pairs of points respecting the length constraint, we consider extensions of the edge beyond the segment defined by two supporting points in both directions such that the sum of this segment and of the two extensions is exactly $L$. The a contrario will give a score to each moving window and we can for instance only keep the window that yields the best score: the most unlikely window under the background assumption. A choice of any two points on the edge's mid range will necessarily generate a candidate window exactly fitting to the edge. For this reason, a RANSAC-like approach is compatible with the sparse approach. Inspired by the RANSAC algorithm \cite{fischler1981ransac} and the probabilistic Hough transform \cite{kiryati1991probabilistic}, we choose to randomly sample a subset of the total candidate positions. The higher the number of random samples, the higher the probability that at least one candidate sample corresponds to a pair of points on the edge's mid range and will therefore generate a candidate window for the true edge. Therefore for appropriate sampling definitions of the structure space, the RANSAC-like approach allows to speed-up the a contrario algorithm without harming the performance too much.

\subsection{Output}

Since we are looking for only a single edge, we decide to return only the candidate position that has the smallest \say{expected-number-of-false-alarms} score and that has this score smaller than the confidence level $\varepsilon$. Note that it is possible that all candidate positions have a score larger than $\varepsilon$ and therefore the a contrario never rejects the background hypothesis hence it does not detect any line.

We provide a pseudo-code of the a contrario algorithm on random-dot video data that we have just presented in Algorithm \ref{alg: AC algorithm on video}. Find a demo at \cite{dagesACResDemo} of an output of an a contrario algorithm, as described in this section, for recovering a moving edge.

\newpage
\begin{algorithm}[H]
\label{alg: AC algorithm on video}
    \SetAlgoLined
    \KwIn{Random-dot video $I$ with $N$ frames, expected tolerated number of false alarms $\varepsilon$, length of the target line $L_e$, set of widths for candidate rectangles $W$, number of frames for integration $n_f$, number of iterations in RANSAC $N_{iter}$, background degraded parameter provided by an oracle $p_b$}
    \KwResult{Output video of detected edge locations $O$}
    $N_w = \#W$  \tcp*{Number of candidate widths} 
    \For{time step $t \in [\frac{n_f}{2},N-\frac{n_f}{2}]$}{
        $C = \emptyset$ \tcp*{Set of candidates with a contrario score lower than $\varepsilon$}
        Get the neighbouring frames in the sliding window around time step $t$: $I_1,\hdots,I_{n_f}$ \;
        $I^M = \bigvee\limits_{i\in\{1,\hdots,n_f\}}I_i$ \tcp*{Merge frames}
        Count the number of white pixels $M$ in $I^M$ \;
        $N_T = \frac{M(M-1)}{2}$ \tcp*{Number of tests in a contrario}
        \For{$step \in [1,N_{iter}]$}{
            Take a random pair of white pixels in $I^M$ that are distant by less than $L_e$\;
            \For(\tcp*[h]{Width of the a contrario window}){$w\in W$}{
                $c_{min} = \emptyset$ \tcp*{The position of the sliding window giving the smallest $\mathcal{N}_{FA}$ score}
                $(\mathcal{N}_{FA})_{min} = \infty$ \tcp*{Smallest $\mathcal{N}_{FA}$ score of sliding windows with dimensions $L_e\times w$ supported by both pixels}
                \For{all sliding windows $c$ of length $L_e$ and width $w$ supported by the chosen pixels}{
                    Count the number of white pixels $k$ in the considered rectangle of $I^M$ \;
                    $\mathcal{N}_{FA} = N_T \mathbbm{P}_{B}(K\ge k)$ \tcp*{Under $B$, $K$ follows a known binomial distribution with parameter $p_b$}
                    \If{$\mathcal{N}_{FA} < (\mathcal{N}_{FA})_{min}$}{
                        $c_{min} = c$ \tcp*{Update}
                        $(\mathcal{N}_{FA})_{min} = \mathcal{N}_{FA}$ \tcp*{Update}
                    }
                }
                \If(\tcp*[h]{Detection}){$(\mathcal{N}_{FA})_{min} \le \varepsilon$}{ 
                    $C = C \cup (c_{min},(\mathcal{N}_{FA})_{min})$ \tcp*{Add candidate location to the detections}
                }
            }
        }
        Generate an entirely black image $O_t$ of same size as those of $I$\;
        \If(\tcp*[h]{At least one candidate is $\varepsilon$-meaningful}){$C\neq \emptyset$}{
            Get $(c,\mathcal{N}_{FA})\in C$ with minimum $\mathcal{N}_{FA}$ \tcp*{Detection in the whole image with lowest $\mathcal{N}_{FA}$ score}
            Colour in white a straight edge of length $L_e$ and width $w_e$ at the position corresponding to $c$ in $O_t$\;
        }
        Add image $O_t$ as a frame at the end of $O$\;
    }
    \caption{A contrario algorithm for detecting a line in a random-dot video}
\end{algorithm}

\clearpage
\newpage
\section{Human Vision Versus A Contrario: Overview}
\label{sec: humans vs AC overview}

Our aim is to understand the process of extracting moving boundaries in random-dot videos and compare the performance of automated detection processes with the performance of human observers. Taking cues from statistical detection theory, we shall propose to model the human visual grouping and understanding processes in sparse random video data as generated in the assumed imaging process. We wish to compare human visual grouping with the a contrario model in order to estimate how our mental grouping processes are linked to the decision making involved in detecting and grouping unlikely configurations under some naturally occurring random assumptions. We summarise our model for the human perceptual system in Figure \ref{fig: perception model}.

In the previous Section, a crucial idea was the accumulation of spatio-temporal data by merging $n_{f}$ consecutive frames: temporal evolution of the data is transformed into 2D frames over sliding windows by collecting votes on active pixels over a set of consecutive frames. In an analogy with the human visual system, taking information over a time span is modelling short memory visual persistence. A natural question would then be if this \say{sliding-window merging} model of $n_{f}$ frames is consistent with our human visual system. Note that, naturally, the parameter $n_{f}$ depends on the frame rate of the video, while human short memory will primarily be linked to some absolute time duration rather than the technological frame-rate of the video.

Consider the $(p_b,p_f)$ plane, where $p_b$ is the probability of the background pixel activation while $p_f$ can be understood as the additional activation probability at boundaries, as defined in subsection \ref{sec: signal modelling}. Humans should perform well for high values of $p_f$ combined with low values of $p_b$. Inversely, they should perform poorly if $p_f$ is very low and $p_b$ is significantly high. In the degradation parameters plane, there should be some transition zone for the human performance. Similarly, there is a transition zone for an a contrario algorithm. If these transition zones are similar, then there is reason to assert that the algorithm models well the performance of human vision on this kind of data. Our videos are generated so that it is impossible to see the objects in a single frame but are nevertheless perceived in the video. This means that in the $(p_b,p_f)$ space, for $p_b$ and $p_f$ corresponding to the degradation parameters of each single frame, we are below the transition zone for humans. However, the temporal integration performed by the visual system on consecutive frames increases the density of the background and of the foreground in a particular way. Recall the discussion we had about $p_b$ and $p_f$, the degradation parameters, versus $p_0$ and $p_1$ the densities in the background and on the object edges. Merging frames increases $p_0$ and $p_1$, which then implies an increase in $p_b$ and $p_f$. Thereby, by merging frames, we move from an initial position below the transition area to a position above the transition area, where it is possible to \say{see} the object due to the sliding-window merge of frames. We can then compare if this idea is coherent with some choice of $n_{f}$ frames for the a contrario algorithm.

An other important choice for the model is the choice of widths for the candidate window used in the a contrario framework. Considering larger widths allows us to analyse larger spatial scales. Should we want to model human perception by a contrario, a natural issue that must be addressed is whether we base our decision on a single width or on multiple widths.

\begin{figure}
    \centering
    \begin{subfigure}[t]{0.7\textwidth}
      \centering
         \includegraphics[width=\textwidth]{Chapters/Pictures/model_data_acquisition.jpg}
         \caption{Data acquisition model}
    \end{subfigure}%
    \hfill
    \begin{subfigure}[t]{\textwidth}
      \centering
         \includegraphics[width=\textwidth]{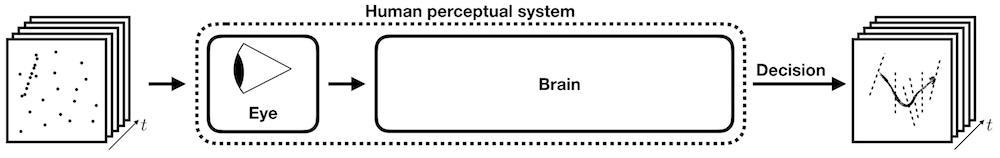}
         \caption{Black box model of the human perceptual system}
    \end{subfigure}%
    \hfill
    \begin{subfigure}[t]{\textwidth}
      \centering
         \includegraphics[width=\textwidth]{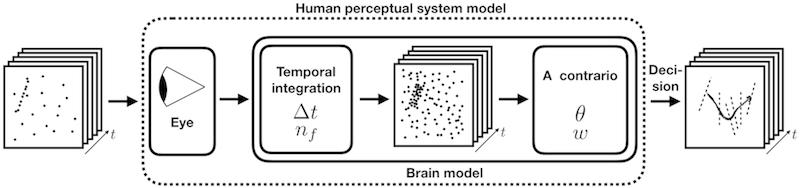}
         \caption{Algorithmic model for the human perceptual system}
    \end{subfigure}%
    \hfill
\caption{Our model of human perception. Visual information is captured by our sensors. The brain then processes the raw data in order to perceive some structure. First, we integrate time information to generate a single image by combining the images in the duration $\Delta t$. Then, we process this single combined image by using a statistical test of unlikeliness such as a contrario, with single candidate width corresponding to the visual angle $\theta$. For a numerical signal such as a video, time integration is equivalent to merging a certain number of frames $n_{f}$. Furthermore, if we assume that we position ourselves always in the same way with respect to the screen, then the visual angle becomes equivalent to a pixel width $w$.}
\label{fig: perception model}
\end{figure}

It is a well-known fact that humans perceive the world at different scales \cite{gibson1950perception}. This would lead to the idea of using a multiple-width a contrario algorithm. However, this is unwarranted for several reasons. First, the a contrario paradigm consists in sampling some candidate shapes where we count the number of white points and test if this is unlikely under a background only assumption. This test heavily relies on the way we sample since we use expectations rather than probabilities. Unfortunately, it is impossible to know how a human observer performs such samplings. Even if we encourage samplings of, say, straight edges of a fixed, known, length with some fixed width, a human brain will most probably most probably consider many other candidate shapes. For instance, in Figure \ref{fig: various shapes and widths}, no matter how much we shall beg humans to try and see a straight edge, they will readily see the more complex shapes. Therefore we do not have full control on what shapes humans try. Even if we had control on, say, the length of a candidate linear region, we do not have full control on the range of widths that humans will take into consideration, see Figure \ref{fig: various shapes and widths}! Secondly, in the a contrario framework, given a choice of position, the probabilities, and hence the decisions, for candidates at the same location and direction but with different widths are not independent, nevertheless during the decision process, each decision is done independently. As such, a multiple-width a contrario will not perform a merging process between different widths in order to detect an edge at various candidate positions. The combination of information is minimal, only manifested in the linear increase in the number of samples considered. However, for a small number of different widths (in practice, between 2 and 4 different widths), the increase in $N_T$ by a small factor does not have a big impact on the threshold of the number of points to appear necessary in order to have a detection at level $\varepsilon$. For these reasons, we prefer to consider a model based on single-width a contrario algorithms.

\begin{figure}
    \centering
    \begin{subfigure}[t]{0.45\textwidth}
      \centering
         \includegraphics[width=\textwidth]{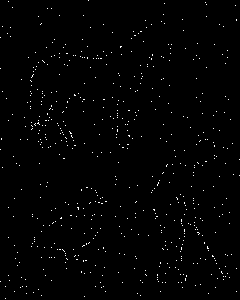}
    \end{subfigure}%
    \hfill
    \begin{subfigure}[t]{0.45\textwidth}
      \centering
         \includegraphics[width=\textwidth]{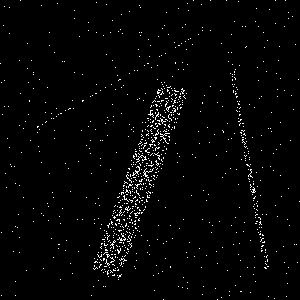}
    \end{subfigure}%
    \hfill
\caption{Left: Natural shapes (horse, golfer, duck). Right: Edges of same length with various widths. What is the good shape for a candidate region? What is a good width? We cannot enforce humans to not test complex shapes and various widths.}
\label{fig: various shapes and widths}
\end{figure}

It is crucial to point out that, should we try to match the human vision performance by an a contrario process on a single width, we do not know in advance what this width should be. Hence, we should test several single-width a contrario algorithms each with a different width and see which width gives an a contrario algorithm that best corresponds to human performance. This will be the task of the first set of experiments, as described below. This width should be to some extent related to the just noticeable jitter for line alignment that has been studied in psychophysics. In practice we shall consider widths between 2 and 16 pixels, which yield reasonable displacements given the configuration for our image display system (pixel size, viewing distance...). These considerations and the results of the human experiments are presented and discussed in Sections \ref{exp1}, \ref{exp2}, \ref{sec: exp3.1} and \ref{sec: exp3.2} below.

Once we estimated a width for a possible a contrario model of the human visual system, we shall estimate the number of frames we use for integration in a video. Note that this corresponds to an integration time, for a given frame-rate. This was done in the last two experiments, to be described later.

The first two experiments were done on static images (considered as sliding window mergers of frames of a video of a stationary and of a moving straight edge). They yield the width parameter $w$ for the a contrario black box modelling human perception, or equivalently the visual angle $\theta$ if the subjects view the data from the same position. In the subsequent experiments we will evaluate human performance on video data of stationary and moving straight edges. This will yield the time integration parameter $\Delta t$ for the human perception model of random-dot videos, or equivalently the number of frames to integrate on $n_f$ given the video's frame-rate. See the Figure \ref{fig: flowchart experiments} for a flowchart of the experiments.

\begin{figure}
    \centering
    \includegraphics[width=\textwidth]{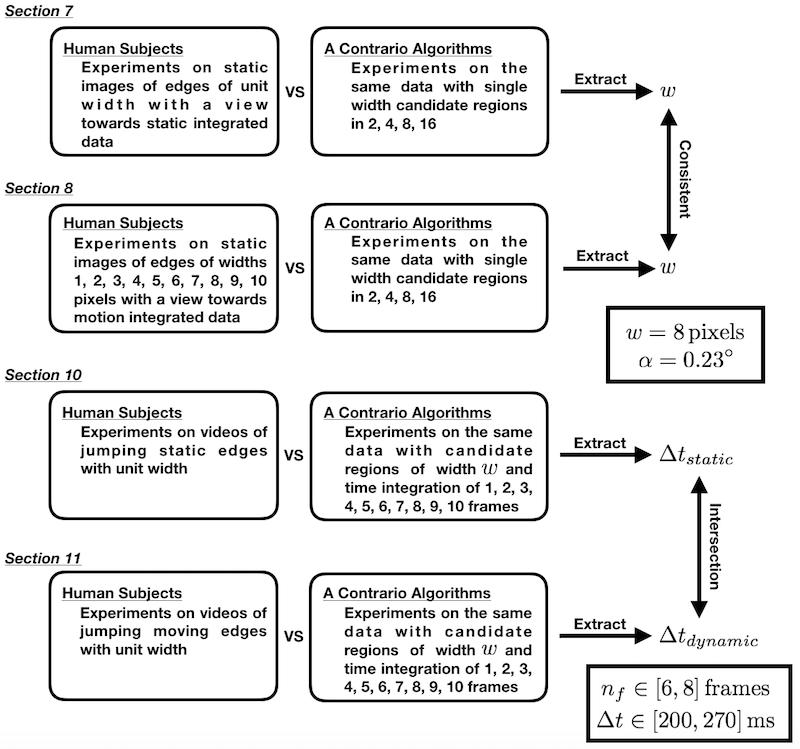}
    \caption[Short version]{Flowchart for the human experiments versus the a contrario tests. Humans will be tested on random-dot data and several a contrario algorithms will be compared with the subjects in order to test our model of the human perceptual system and recover its parameters. In the first two experiments, we work on static data, i.e. images of edges. In parallel several a contrario algorithms will run on the same data, but each working on a different width for its candidate windows. These experiments allow us to find the parameter corresponding to the width of the candidate windows $w$ (or equivalently the visual angle $\theta$). This data will be later understood as merges of consecutive frames of a video of static and of a dynamic edges. In the next experiments we work on videos of static and dynamic edges. In order to avoid undesired effects, we make the edges jump regularly in these videos. In parallel, we run several a contrario algorithms on the same data, each with the same previously found $w$ for the candidate window width, but each using a different number of frames $n_f$ as time integration. This allows to recover the number of frames $n_f$ for integration in our model of the human perceptual system (or equivalently the time integration $\Delta t$).}
    \hfill
    \label{fig: flowchart experiments}
\end{figure}

\clearpage
\newpage
\section{Analysis of the Static Edge Case}
\label{sec: analysis static edge}

In this Section we study the performances of humans and of a contrario algorithms on static images, that may also be considered as the result of merging degraded frames taken from a video of a static scene. This will allow us to challenge our model and recover a visual angle $\theta$ or equivalently a width $w$ for candidate regions in our model of human perception.

We will actually mean hit performance when mentioning performance, since we already have a low bound on the false alarms thanks to the design of a contrario which ensure on average $\varepsilon$ false alarm, which in practice will be $\varepsilon=1$, and since humans will be allowed to make only one detection, thus at most one false alarm. Unless explicitly mentioned otherwise, performance will mean hit performance in the rest of this report.

Consider the simplest occurrence when the foreground consists of a static edge of known length $L_e$ and width $w_{e}$. Since the edge is static, all non-degraded frames are identical and denoted $I$. Consider superpositions of $n_{f}$ frames from the degraded video. In this case, each merged frame is $I^{D,M} \sim \bigvee\limits_{1\le j \le n_{f} }I_{j}^{D}$ where $I_{1}^D,\hdots,I_{n_{f}}^D$ are generated identically and independently as $\phi_{im}(I,p_b,p_f)$. The notation $\bigvee$ is the pixelwise logical \say{OR} operator: if a pixel is white in any of the degraded frames then it will be white in the merged frame. In this case we have the following theorem:

\bigskip
\begin{theorem}
    For $t\ge 1$ and $I$ a static deterministic boolean image, if $I_1^D,\hdots,I_t^D$ are iid and $I_j^D\sim\phi_{im}(I,p_b,p_f)$ for all $j\in\{1,\hdots,t\}$, denote $I^{D,M} = \bigvee\limits_{1\le j \le t}I_{j}^{D}$, then $I^{D,M} \sim \phi_{im}(I,(p_b^M)_{t},(p_f^M)_t)$ where $(p_b^M)_t = 1-(1-p_b)^t$ and $(p_f^M)_t = 1-(1-p_f)^t$.
\end{theorem}

\begin{proof}
    Denote $F = \{i; (I)_i = 1\}$ and $B = \{i; (I)_i = 0\}$ the foreground and background of the image. Since the image is static, $F$ and $B$ are the foreground and background sets for each frame and for the merge. Let $i$ be a pixel. We have that, recalling independence and identical distribution between frames:
    \begin{align*}
        (p_0^M)_t \triangleq \mathbb{P}((I^{D,M})_i = 1 \mid i\in B,I,p_b,p_f) &= 1-\mathbb{P}(\bigwedge\limits_{j=1}^{t}(I_{j}^D)_i = 0\mid i\in B,I,p_b,p_f) \\
        &= 1-\prod\limits_{j=1}^{t}\mathbb{P}((I_{j}^D)_i = 0\mid i\in B,I,p_b,p_f) \\
        &= 1-(1-p_b)^{t}
    \end{align*}
    and similarly, for the same independence reasons and since the foreground and background signals are independent:
    \begin{align*}
        (p_1^M)_t \triangleq \mathbb{P}((I^{D,M})_i = 1 \mid i\in F,I,p_b,p_f) &= 1-\mathbb{P}(\bigwedge\limits_{j=1}^{t}(I_{j}^D)_i = 0\mid i\in F,I,p_b,p_f) \\
        &= 1-\prod\limits_{j=1}^{t}\mathbb{P}((I_{j}^D)_i = 0\mid i\in F,I,p_b,p_f) \\
        &= 1-\Big( (1-p_b)(1-p_f) \Big) ^{t}
    \end{align*}
    Hence we have $(p_b^M)_t = (p_0^M)_t = 1-(1-p_b)^{t}$ and $(p_1^M)_t$, the probabilities of respectively a background pixel to appear white in the merged image and of a foreground pixel to appear white in the merged image. We also want to find $(p_f^M)_t$ implicitly defined by:
    \begin{align*}
        (p_1^M)_t \triangleq 1-(1-(p_b^M)_t)(1-(p_f^M)_t) &\iff (p_f^M)_t = \frac{(p_1^M)_t - (p_b^M)_t}{1-(p_b^M)_t} \\
        &\iff (p_f^M)_t = \frac{(1-p_b)^t(1-(1-p_f)^t)}{(1-p_b)^t} = 1-(1-p_f)^t
    \end{align*}
    Since $(p_b^M)_t$ and $(p_f^M)_t$ are independent of the location in the image, $i$, they are valid for the entire image and so we can conclude that $I^{D,M}\sim \phi_{im}(I,(p_b^M)_t,(p_f^M)_t)$.

\end{proof}

Note that the result is almost trivial for the change for $p_b$ in the merge of frames, it is not as straightforward for the change for $p_f$, although they seem to change in a similar fashion.

Therefore increasing $n_{f}$ and looking at $I^{D,M}$ is equivalent to degrading $I$ with some specifically increasing $p_b$ and $p_f$, given by the formula above with $t=n_f$.

\subsection{A Theory for Performance Levels of A Contrario}

Our first question is how do humans manage to temporally integrate and spatially group sparse random data and whether the the detection of is linked to grouping unlikely configurations under a uniform random assumption, similarly to a contrario.

In the $(p_b,p_f)$ space, schematically, we can expect a contrario to detect correctly an edge when $p_f$ is significantly high compared to $p_b$ and to reject it when $p_f$ is not large enough compared to $p_b$. We can in fact give a mathematical estimate of the prediction power of the a contrario algorithm on images generated from $\phi_{im}(I,p_b,p_f)$ in the $(p_b,p_f)$ space. We can randomly select the position of the true edge, since a contrario does not explicitly depend on the position of the true edge due to it's strategy for sampling the space of potential object locations. Likewise, we can estimate human performance on the same input data and compare the resulting detection performances. If humans perform in a similar fashion to a contrario, this would suggest that human visual grouping is indeed a \say{Gestalt} of grouping unlikely configurations under a uniform randomness assumption.

The data we imagine to feed a priori to the a contrario algorithms consist in degraded images of a groundtruth image that has exactly one edge of length $L$ and width $w_e$. Performance of the algorithm is then measured as follows: if the algorithm correctly detects an edge at the true location of the edge, then it is a success, otherwise it is a failure. This corresponds to looking only at the hit score of the algorithm. The reason why we do not consider the false alarm score in the performance is that we already have a bound on the false alarm rate since on average for each image it should be $\epsilon$ which is very small (around 1 in practice) and consistent between all algorithms. Furthermore, having a score solely based on the region located at the true position allows to generalise the estimated performance to input images that could have many edges of length $L$ and width $w_e$.

\subsubsection{Predicting A Contrario Performance on a Simple One Dimensional Example}

Before diving into the mathematical complexities of estimating a priori the performance of the a contrario algorithm on our image data, let us return to the study of our very simple one dimensional example from Section \ref{par: simple example presentation nfa} in order to get a good understanding of how we will proceed. Later we will go back to our more complex problem on 2D images and reuse a similar reasoning. Recall that in our example, the data $x$ is a one dimensional long string of $N$ bits, where $N$ very big. We are looking for contiguous subsequences of bits of fixed length $L$ that contain an abnormally high number of $1$s. Assume now that we impose that each piece of data $x$ is the degraded version of a clean long string $x_{GT}$, which consists in $0$s everywhere except on one contiguous sequence of length $L$ where all bits take the value 1. The pixels that take the 0 value lie in the background $B$ whereas the sequence of length $L$ of values only $1$ is the foreground $F$. Due to noise in the acquisition of the signal, $x$ is a degraded version of $x_{GT}$ where each of its bits are drawn independently, conditionally on the location of the subsequence of interest. We will denote $GT$ the conditioning with respect to the position of the subsequence of interest, which can also be seen as a conditioning with respect to the groundtruth $x_{GT}$. If $i$ is the $i$-th bit of $x$, with value $x_i$, then we have: $\mathbbm{P}(x_i=1\mid i\in B \,\mathrm{and}\,GT)= p_b = p_0$, and $\mathbbm{P}(x_i=1\mid i\in F \,\mathrm{and}\,GT)= p_1 = p_b+p_f-p_bp_f$. The reason for having $p_b$ and $p_f$ is that $x$ can be seen as being degraded everywhere with some background degradation parameter $p_b$ on one hand, and on the other hand to be degraded in parallel just on the position of the structure of $1$s with parameter $p_f$, and then merge both degradations into the final degradation.

We now want to predict a priori what the performance of the a contrario algorithm on such data $x$ would be. Eventually, at the end, we should run empirical tests and compare it to our theoretical estimation. We must first define the concept of performance of an algorithm. Here, performance of an algorithm trying to recover a subsequence of length $L$ of $1$s in a background of $0$s will be defined as follows. If the algorithm recovers the exact location of the subsequence of length $L$, i.e. detects that the true underlying sequence of $1$s is a sequence of $1$s, then it is a success. Otherwise, it is a failure. Note that in our definition of performance, we do not care about the decisions made at other candidate positions. The reason for this is because we already have a guarantee that the number of detections at other positions (and so of false alarms), will be on average smaller than $\varepsilon = 1$, and so we can focus only on the hit score of the algorithm. For our analysis, we neglect possible detection overlapping subsequences in the vicinity of the true location, which are theoretically less likely to lead to a detection.

We can now predict a priori the performance of the algorithm. Imagine that we feed such data $x$ into an a contrario machine. Then, at run time, the algorithm will test all candidate contiguous sequences of length $L$ and test whether they are very unlikely under the background only assumption. The candidate locations are simply sliding windows of length $L$. As such, there are $N_T = N-(L-1)$ such sliding windows and therefore there are $N_T$ tests in the a contrario process. This is the size of the candidate locations space. Now assume that we have reached the candidate location $c_i^*$ corresponding to the true underlying sequence of $1$s. The decision on this candidate is the only decision that matters to us as it defines the performance of the algorithm. If a contrario rejects the background hypothesis, i.e. detects a structure of $1$s in $c_i^*$, then we have a success. Otherwise it does not reject the background assumption and does not detect a structure of $1$s and the performance is a failure. The decision, and hence the performance, consists in the test: $N_T\mathbbm{P}_{B}(K\ge K_{GT}^*) \le \varepsilon$, where we will define and explain the difference between $K$ and $K_{GT}^*$ in the following.

We define $K$ as the random variable counting the number of $1$s in the candidate window $c_i^*$ under the assumption of background only. Its distribution is known and corresponds to $\mathbbm{P}_{B}(K=l) = {L \choose l} p_b^l(1-p_b)^{L-l}$. On the other hand, $K_{GT}^*$ is defined as the random variable counting the number of $1$s in the candidate window $c_i^*$ under the groundtruth assumption and assumption that the candidate window corresponds to the window located exactly at the position of the sequences of $1$s we are looking for. Recall that traditionally, when working on a given image, then we have the realisation of $K_{GT}^*$ as $k$, and thus in practice during runtime we plug into the score of the a contrario algorithm the realisation $k$. But here we are working a piori: we wish to predict how the algorithm will perform should we feed it with the data we are interested in, and therefore we do not have a random realisation $k$ yet but only the random variable $K_{GT}^*$. We denote $BF$ the conditioning of the groundtruth and that the candidate sequence is at the true location of sequences of $1$s $c_i^*$. The reason for this choice of notation is that in this assumption we have that the bits become $1$ due to background and foreground signals, thus both $p_b$ and $p_f$ play a role. The distribution of $K_{GT}^*$ is also a given and is, since the entire candidate window $c_i^*$ covers the foreground: $\mathbbm{P}_{BF}(K_{GT}^*=l) = {L \choose l} p_1^l(1-p_1)^{L-l} = {L \choose l} (p_b+p_f-p_bp_f)^l(1-p_b-p_f+p_bp_f)^{L-l}$.

Now that we understand better the definitions of $K$ and $K_{GT}^*$, let us look back to the score that the algorithm will be faced with, at run time, at the ideal candidate $c_i^*$: $N_T\mathbbm{P}_{B}(K\ge k)$. We here have that $N_T$ is a global given constant, and that $\mathbbm{P}_{B}(K\ge l) = \sum\limits_{y=l}^L {L \choose y} p_b^y(1-p_b)^{L-y}$ is a deterministic function. Denote $\mathcal{F}_B(l) = \mathbbm{P}_{B}(K\ge l)$ this function. This is a strictly decreasing function and as such it is \say{pseudo invertible}. It is not rigorously invertible since it is defined at discrete values, nevertheless we can define its \say{pseudo} inverse. If $y\in[0,1]$ is the query for the pseudo inverse, then we can find $l$ the smallest integer such that $\mathcal{F}_{B}(l)\le y$, and then define $\mathcal{F}_B^{-1}(y)=l$. Back to the a contrario score, we now see that this score is, a priori, a random value $N_T\mathcal{F}_{B}(K_{GT}^*)\le \varepsilon$, depending on the random value $K_{GT}^*$, which is conditioned to $BF$. Success of the algorithm is then whether or not $K_{GT}^*$ falls greater than $\mathcal{F}_B^{-1}(\frac{\varepsilon}{N_T})$. If it is greater, then it is a success, otherwise it is a failure. The first way of estimating a priori the performance of the algorithm is to give the probability of success for such data. In this setting, we are looking at
\begin{align*}
    \mathbbm{P}_{BF}(N_T\mathbbm{P}_B(K\ge K_{GT}^*) \le \varepsilon) &= \mathbbm{P}_{BF}(\mathcal{F}_B(K_{GT}^*) \le \frac{\varepsilon}{N_T}) \\
    &= \mathbbm{P}_{BF}(K_{GT}^*\ge \mathcal{F}_B^{-1}(\frac{\varepsilon}{N_T}))
\end{align*}
since $\mathcal{F}_B$ is strictly decreasing. This probability is explicitly known and is: 
\begin{align*}
    \mathbbm{P}_{BF}(N_T\mathbbm{P}_B(K\ge K_{GT}^*) \le \varepsilon) &= \sum\limits_{l=\mathcal{F}_B^{-1}(\frac{\varepsilon}{N_T})}^L {L\choose l} p_1^l (1-p_1)^{L-l} \\
    &= \sum\limits_{l=\mathcal{F}_B^{-1}(\frac{\varepsilon}{N_T})}^L {L\choose l} (p_b+p_f-p_bp_f)^l (1-p_b-p_f+p_bp_f)^{L-l}
\end{align*}
While this gives a good understanding a priori of the algorithm it requires us to be capable of inverting the right cumulative of the distribution of $K$, which can be difficult, especially for more complicated problems. This leads us to the second way of estimating the performance of the algorithm. Instead of giving the probability of success, we can directly estimate what the random value $K_{GT}^*$ will be and test whether or not it falls in the right range: we plug the estimate into the a contrario score and if the score is lower than $\varepsilon$ then we predict a success, otherwise we predict a failure. The issue is to choose a deterministic estimator $\widehat{K_{GT}^*}$ that represents well enough in some sense the random variable $K_{GT}^*$. A possible choice, which we will now do, is to choose the expectation of $K_{GT}^*$: $\widehat{K_{GT}^*} = \mathbbm{E}_{BF}(K_{GT}^*) = Lp_1 = L(p_b+p_f-p_bp_f)$. The estimated score is then: 
$${\widehat{\mathcal{N}_{FA}^{*}}(p_b,p_f,L) = N_T\mathbbm{P_B}(K\ge \widehat{K_{GT}^*}) = N_T\sum\limits_{l=\widehat{K_{GT}^*}}^L {L\choose l} p_b^l (1-p_b)^{L-l} = N_T\sum\limits_{l=Lp_1}^L {L\choose l} p_b^l (1-p_b)^{L-l}}$$
We now look at whether this score is lower or not than $\varepsilon$ to claim a priori if the algorithm will succeed or fail. What is important to point out is that the deterministic estimated value for the a priori score of the algorithm $\widehat{\mathcal{N}_{FA}^{*}}$ is a function that only depends on $p_b$ and $p_f$, if we fix the length of the candidate region $L$. The reason why taking $\widehat{K_{GT}^*}=\mathbbm{E}_{BF}(K_{GT}^*)$ is a good choice for estimating $K_{GT}^*$ in the prediction of a contrario performance is because the random variable $K_{GT}^*$ is very concentrated around its mean. Indeed, the standard deviation of $K_{GT}^*$, denoted $\sigma_{K_{GT}^*}$, is: $\sigma_{K_{GT}^*} = \sqrt{Lp_1(1-p_1)} = \sqrt{L(p_b+p_f-p_bp_f)(1-p_b-p_f+p_bp_f)}$ since $K_{GT}^*\sim_{BF}B(L,p_1)$ (a binomial distribution). For example, if $L=30$, $p_b=0.1$ and $p_f=0.6$, then $\sigma_{K_{GT}^*}\approx 2.7$, thus $\frac{\sigma_{K_{GT}^*}}{L} \approx 9\%$. Thus choosing any value around the mean within a distance of $\sigma_{K_{GT}^*}$ is a good estimator of $K_{GT}^*$ and thus translates to a good estimate of the prediction of the performance of a contrario when plugging the estimated value into the a contrario score. Also, we could do the same reasoning for other a contrario algorithms but this time looking for other lengths $L_c$ of sequences different from $L$, and this will also give us a function of $p_b$ and $p_f$, although we must be careful in devising the distribution of $K_{GT}^*$ under $BF$ as if the candidate sequence window is longer than $L$ then there are exactly $L$ bits that each take the value 1 independently with probability $p_1$ and the others in the window independently take the value 1 with probability $p_0$, whereas before in our example all bits took the value $1$ with probability $p_1$ independently under $BF$. Nevertheless, each of the $L_c$ algorithms will provide a deterministic function $f_{L_c}(p_b,p_f)=\widehat{\mathcal{N}_{FA}^{*}}(p_b,p_f,L_c)$ which depends only on the degradation parameters. The decision curve can be understood as the contour level $\varepsilon$ of this two dimensional function. In the $(p_b,p_f)$ plane, this curve separates the plane into two regions: below the curve is the region of undetectability, i.e. predicted failure of the algorithm, and above the curve is the region of detectability, i.e. predicted success of the algorithm. See the empirical results in Figures \ref{fig: simple example thresh pb static single window} to \ref{fig: simple example comparison estimated prediction vs prob prediction static single window}.

\begin{figure}
    \centering
    \includegraphics[width=0.7\textwidth]{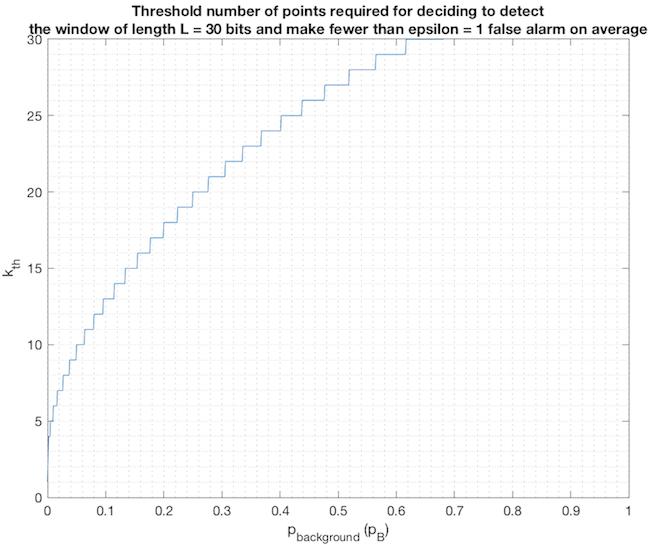}
    \caption[Short version]{Dependency of the threshold number of points, in our simple example, required to reject the background hypothesis for the $BF$ window candidate that is positioned exactly on the sequence of $1$s in the groundtruth, with respect to the background density $p_b$. Unsurprisingly, this function is a non decreasing staircase function, as when there are more 1 values in the background, we need more 1 values in the observed sequence for it to be more unlikely. We are still using $N=100000$ and $L_c=L=30$.}
    \hfill
    \label{fig: simple example thresh pb static single window}
\end{figure}

\begin{figure}
    \centering
    \begin{subfigure}[t]{0.7\textwidth}
      \centering
         \includegraphics[width=\textwidth]{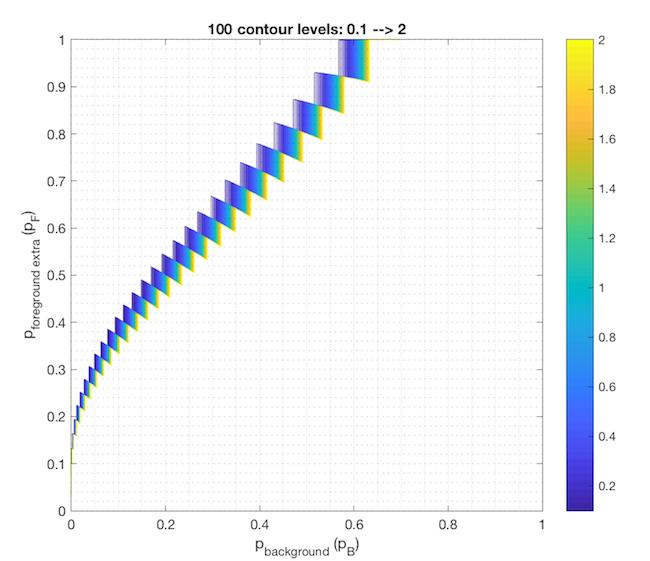}
    \end{subfigure}%
    \hfill
    \begin{subfigure}[t]{0.7\textwidth}
      \centering
         \includegraphics[width=\textwidth]{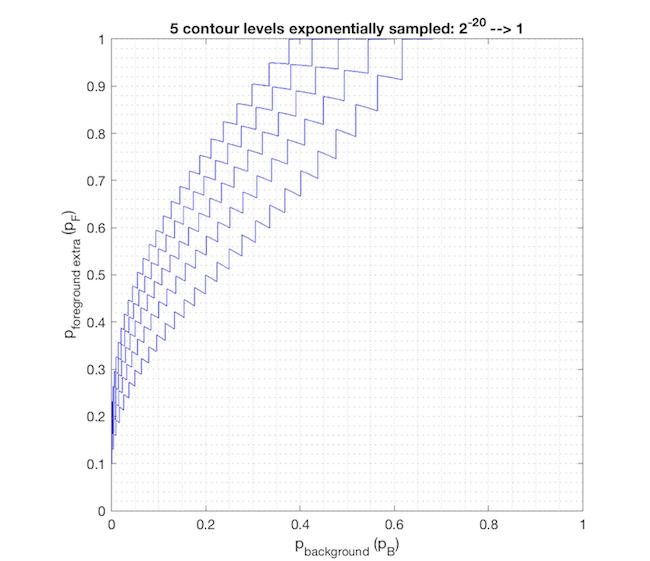}
    \end{subfigure}%
    \hfill
\caption{A priori predictions of the performance of a contrario on our simple example. We use $N=100000$ the length of the long string of bits $x$, and $L=30$ the length of the contiguous sequence of $1$s in the groundtruth signal. The candidate windows of a contrario have length $L_c=L$. Top: 100 contour levels linearly sampled in $[0.1,2]$ of the estimated a priori success of a contrario, i.e. of the function of $(p_b,p_f)$: $N_T\mathbbm{P}_B(K\ge \widehat{K_{GT}^*}) = N_T\mathbbm{P}_B(K\ge \mathbbm{E}_{BF}(K_{GT}^*))$. Each contour level corresponds to a different value of $\varepsilon$. Note how a small change of $\varepsilon$ does not significantly impact the position of the decision level and thus does not significantly impact the performance a priori of the algorithm. Bottom: 5 contour levels of the same estimated a priori performance of a contrario but sampled exponentially between $0.1$ and $2^{-20}$. Note how in order to get a significant change in the decision level one needs to drastically change $\varepsilon$. For instance a value of $\varepsilon = 2^{-20}$ seems extreme and unreasonable: it would mean that for about 100000 tests we accept to make on average up to $2^{-20}$ false alarms! Choosing $\varepsilon\approx 1$ seems like a reasonable order of magnitude, and since the dependency is very slow on $\varepsilon$ the choice of $\varepsilon=1$ seems like a good possible choice.}
\label{fig: simple example estimated prediction static single window}
\end{figure}

\begin{figure}
    \centering
    \begin{subfigure}[t]{0.7\textwidth}
      \centering
         \includegraphics[width=\textwidth]{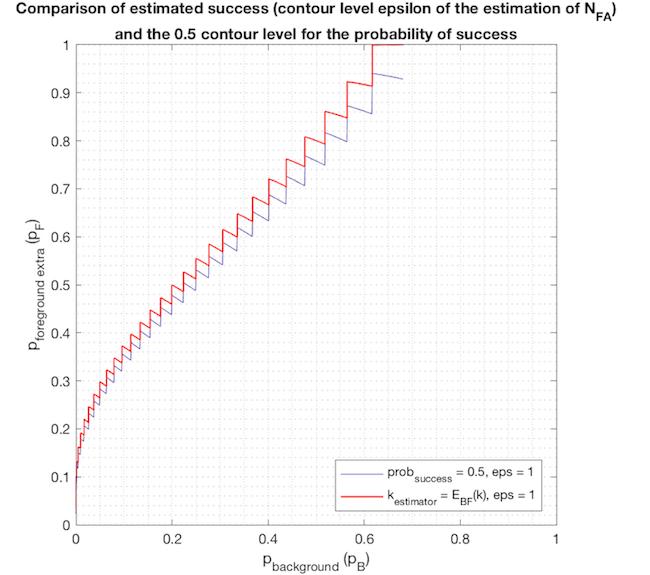}
    \end{subfigure}%
    \hfill
    \begin{subfigure}[t]{0.7\textwidth}
      \centering
         \includegraphics[width=\textwidth]{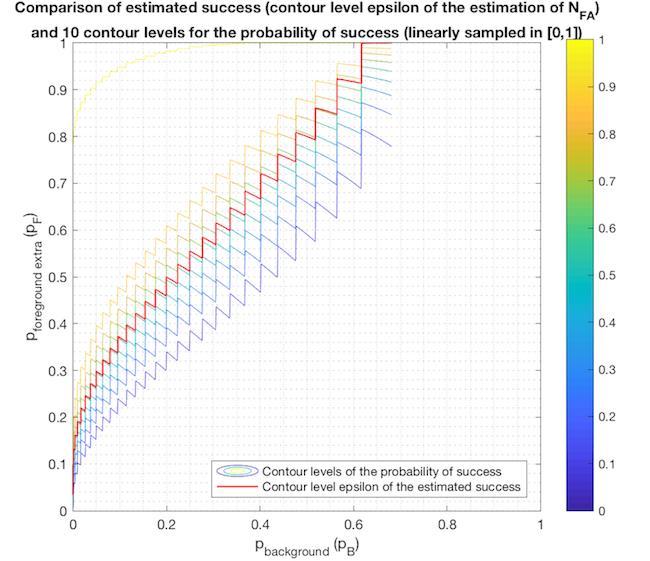}
    \end{subfigure}%
    \hfill
\caption{Top: Comparison between the contour levels, i.e. the decision levels, of the estimated performance of a contrario when plugging in $\widehat{K_{GT}^*} = \mathbbm{E}_{BF}(K_{GT}^*)$ into $\mathcal{N}_{FA}$, versus the 1/2 contour level of the exact probability of success of a contrario, i.e. of the function of $(p_b,p_f)$: $\mathbbm{P}_{BF}(N_T\mathbbm{P}_{B}(K\ge K_{GT}^*)\le \varepsilon)$. Note how close both decision levels are in the $(p_b,p_f)$ space. This means that our choice to plug into the a contrario score the mean of $k$ as its estimate is a good choice and predicts well the behaviour of a contrario. Bottom: Comparison between the contour $\varepsilon=1$ of the estimated a contrario prediction versus $10$ linearly sampled contour levels in $[0,1]$ of the probability of success of a contrario.}
\label{fig: simple example comparison estimated prediction vs prob prediction static single window}
\end{figure}

\clearpage
\subsubsection{Predicting A Contrario Performance on Our Image Data}

Let us come back to our slightly more complex problem of random-dot 2D images. We can mathematically estimate the performance a priori of the a contrario algorithm in the $(p_b,p_f)$ space. Recall that the a contrario algorithm will sample the space of objects or curves, in our case of regions in the image plane defined and edge of length $L_e$ and a certain width $w$, and evaluate the size of the space of samples. Then for each sample region, it will count the number of active (white) pixels $k$ appearing in it, and test whether this observation is $\varepsilon$-meaningful or not with respect to a random variable $K$ with a probability distribution that can be precomputed from the data model, i.e. whether $N_T \mathbb{P}_B(K\ge k) \le \varepsilon$ holds. In an a priori context, we do not have the random realisation $k$, only the random variable counting the number of white pixels in the considered candidate region conditioned to the groundtruth $K_{GT}$. Furthermore, as in the simple example, we will only focus on the candidate region $c_i^*$ located on the true edge, for which the associated count conditioned to the groundtruth is $K_{GT}^*$. In order to predict a priori the a contrario performance, we therefore have to estimate: the number of tests in the a contrario criterion $N_T$ and the probability distribution for $K_{GT}^*$ the number of white pixels to appear in a candidate window of length $L$ and width $w$ conditionally to being on the true edge. Recall that $K$ is the random variable counting the number of white points appearing in the same candidate window under the background only assumption, unlike $K_{GT}^*$.

For the a priori estimation, when looking at the $(p_b,p_f)$ space, we shall consider that the width of the candidate window to be fixed, but remember that this lies in a context where we could want to test for several widths the a contrario hypothesis and so the number of tests should be linear in the number of widths we will in total consider. We denote $N_w$ the number of widths, and by $w$ the considered window sample width.

Unless specified otherwise, we work conditionally to the position of the edge. In the end, we will return to unconditioned results by noticing that the probabilities conditioned to being on the true edge, do not really depend on its position.

As mentioned previously in our discussion on sampling the space of candidate objects, we shall consider all pairs of appearing white points as defining the candidate edges, and all all widths around the line defined by them. Hence we shall have $N_T = \frac{M(M-1)}{2}N_w$ where $M$ is the number of white pixels to appear in the image and $N_w$ is the number of widths to use in the a contrario algorithm. Since $M$ is a random variable so is $N_T$ and we will therefore estimate it by its expectation ${\widehat{N}_t = \mathbb{E}(N_T)}$. Denote $e$ and $\Bar{e}$ the pixels on the edge and respectively off the edge. We have, if we omit the stair-casing effect of straight lines, that ${\#e \approx L_e w_e}$ and ${\#\Bar{e} \approx N^2 - L_e w_e}$, where the image is of size $N\times N$. Let $M_e$ and $M_{\Bar{e}}$ be respectively the number of appearing white pixels on the edge and the number of white pixels appearing off the edge: $M = M_e + M_{\Bar{e}}$. Since all pixels off the edge follow an iid bernoulli $B(p_b)$ (recall that we are here working conditionally on the position of the edge), we have that $M_{\Bar{e}}$ follows a binomial distribution ${M_{\Bar{e}}} \sim Bin(\#\bar{e},p_b)$. Likewise, we get ${M_e\sim Bin(\#e,p_b+p_f-p_bp_f)}$. Recall the well known moments of binomials: if $X\sim Bin(n,p)$, then $\mathbb{E}(X) = np$ and $\mathbb{E}(X^2) = np+n(n-1)p^2$. This trivially leads to:

\begin{align*}
    \mathbb{E}(M_e) &= L_e w_e(p_b+p_f-p_b p_f) \\
    \mathbb{E}(M_{\bar{e}}) &= (N^2-L_e w_e)p_b \\
    \mathbb{E}(M_e^2) &= L_e w_e(p_b+p_f-p_b p_f)+L_e w_e(L_e w_e -1)(p_b+p_f-p_b p_f)^2 \\
    \mathbb{E}(M_{\bar{e}}^2) &= (N^2-L_e w_e)p_b+(N^2-L_e w_e)(N^2-L_e w_e -1)p_b^2
\end{align*}

 From there, we get:

\begin{align*}
    \mathbb{E}(M) &= \mathbb{E}(M_e+M_{\bar{e}}) \\
    &= \mathbb{E}(M_e) + \mathbb{E}(M_{\bar{e}}) \\
    &= L_e w_e(p_b+p_f-p_b p_f) + (N^2-L_e w_e)p_b
\end{align*}

and by independence of $M_e$ and $M_{\bar{e}}$:

\begin{align*}
    \mathbb{E}(M^2) &= \mathbb{E}((M_e+M_{\Bar{e}})^2) \\
    &= \mathbb{E}(M_e^2) + \mathbb{E}(M_{\bar{e}}^2)+2\mathbb{E}(M_e)\mathbb{E}(M_{\bar{e}}) \\
    &= L_e w_e(p_b+p_f-p_b p_f)+L_e w_e(L_e w_e -1)(p_b+p_f-p_b p_f)^2 + (N^2-L_e w_e)p_b\\
    &\quad+(N^2-L_e w_e)(N^2-L_e w_e -1)p_b^2 + 2L_e w_e(p_b+p_f-p_b p_f)(N^2-L_e w_e)p_b
\end{align*}

We can now estimate $N_T$:

\begin{align*}
    \widehat{N}_T &= \mathbb{E}(N_T) = \mathbb{E}(\frac{M(M-1)}{2}N_w) \\
    &= \frac{1}{2}\big(\mathbb{E}(M^2)-\mathbb{E}(M)\big)N_w \\
    &= \frac{1}{2}\Bigg(L_e w_e(L_e w_e -1)(p_b+p_f-p_b p_f)^2 +(N^2-L_e w_e)(N^2-L_e w_e -1)p_b^2 \\
    &\qquad+ 2L_e w_e(p_b+p_f-p_b p_f)(N^2-L_e w_e)p_b \Bigg)N_w
\end{align*}

Note that the conditional expectations do not explicitly depend on the position of the edge in the image and therefore the expectations without the conditional knowledge of the position of the edge take the same value. This is due to the fact that if $X,Y$ two random variables, we have $\mathbb{E}(\mathbb{E}(X\mid Y)) = \mathbb{E}(X)$. In particular, if $\mathbb{E}(X\mid Y)$ is a constant independent of values taken by $Y$, i.e. it is deterministic, then with the previous formula we see that $\mathbb{E}(X)$ takes that same value. 

To predict the performance of the a contrario process of detection, we need to next consider the distribution of $K_{GT}^*$, and ask what is the probability of detection when we test a candidate location that corresponds to an existing edge there. Recall that for our definition of performance, we solely consider the performance of the algorithm at the candidate region that is located at the true position of the edge $c_i^*$. Therefore we must consider the distribution of $K_{GT}^*$ under the assumption that the candidate region is located at the true location of the edge. Evaluating the distribution of $K_{GT}^*$ is thus less simple. We could proceed in two different ways. We could look at the exact\footnote{If we replace $N_T$ by its deterministic estimate in $\mathbbm{P}(N_T\mathbbm{P}_{B}(K\ge K_{GT}^*))$.} probability of a contrario to succeed. Or we can estimate $K_{GT}^*$ by a deterministic value and simply plug it into the a contrario score and see if the score passes the $\varepsilon$ threshold or not. We choose to work with the second option. The reason for this is that $K_{GT}^*$ will have a small variance and so is essentially concentrated around its expectation. Thus any deterministic estimate of $K_{GT}^*$ within this centred region around the mean is a very good estimate of $K_{GT}^*$. Therefore plugging this estimate into the score function also yields a very good estimate of the random score function of a contrario on the candidate region corresponding to the true location.

The estimate for $K_{GT}^*$ will be its expected value (conditionally to the edge being in the middle of the window, which we will omit mentioning unless explicitly otherwise): $\widehat{K_{GT}^*} = \mathbb{E}(K_{GT}^*)$. The candidate window is of dimensions $L_e w$. The number of pixels in the window is $n_w = L_ew$ (for non integer dimensions we round this number $n_w = [L_e w]$). Similarly to the way we computed $M$ previously by counting the number of white pixels appearing conditionally to being on the edge and the number of white pixels to appear conditionally to not being on the edge, we here get, since we are working conditionally to the edge is in the centre of the window that: 

\begin{align*}
    \widehat{K_{GT}^*} &= \mathbb{E}(K_{GT}^*) \\
    &= L_e \widetilde{w}_e(p_b+p_f-p_b p_f)+L_e(w-\widetilde{w}_e)p_b \\
    &=\frac{\widetilde{w}_e}{w}n_w(p_f+p_b-p_f p_b)+(1-\frac{\widetilde{w}_e}{w})n_w p_b
\end{align*}
where $\widetilde{w}_e = \mathrm{min}(w_e,w)$.

The variance of $K_{GT}^*$, denoted $\sigma_{K_{GT}^*}^2$, is given as the sum of the variance of the number of points on the true line and the variance of the number of points outside the true line by independence of the number of white points appearing in areas that do not intersect (conditionally to the areas). The variance of a binomial $X\sim B(n,p)$ is given by $\mathbbm{V}(x) = np(1-p)$. Thus:

\begin{equation*}
    \sigma_{K_{GT}^*} = \sqrt{(\frac{\widetilde{w}_e}{w}n_w(p_f+p_b-p_fp_b)(1-p_f-p_b+p_bp_f))^2+(1-\frac{\widetilde{w}_e}{w}n_wp_b(1-p_b))^2}
\end{equation*}

For example, if $L=200$, $p_b = 0.03 \approx 1-(1-0.005)^6$, $p_f = 0.2 = 1-(1-0.365)^6$, and $w=8$, which are typical values in our later experiments, then $n_w = 1600$ and $\sigma_{K_{GT}^*} \approx 50.3$. Which means that the spread of $K_{GT}^*$ is essentially within a region of size $\frac{2\sigma_{K_{GT}^*}}{n_w} \approx 7\%$ of its possible range. This shows that $K_{GT}^*$ is here too concentrated around its mean and thus its expected value is a very good deterministic estimator.

Under the background only hypothesis, the probability to have more than $\widehat{K_{GT}^*}$ white points appearing (conditionally to the choice of window) is given by the tail of the binomial distribution since pixel value distributions are independent and Bernoulli $B(p_b)$: 
\begin{equation*}
    \mathbb{P}_{B}(K \ge \widehat{K_{GT}^*}) = \sum\limits_{i=\widehat{K_{GT}^*}}^{n_w}{n_w\choose i}p_b^i(1-p_b)^{n_w-i} = 1-\sum\limits_{i=0}^{i=\widehat{K_{GT}^*}-1}{n_w\choose i}p_b^i(1-p_b)^{n_w-i}
\end{equation*}

Therefore the a contrario false alarm score to test $\widehat{\mathcal{N}_{FA}^{*}}(p_b,p_f,w)$, at location $c_i^*$, is: 
\begin{align*}
    \widehat{\mathcal{N}_{FA}^{*}}(p_b,p_f,w) &= \widehat{N}_t \mathbbm{P}_B(K \ge \widehat{K_{GT}^*}) \\
    &= \frac{1}{2}\Bigg(L_e w_e(L_e w_e -1)(p_b+p_f-p_b p_f)^2 +(N^2-L_e w_e)(N^2-L_e w_e -1)p_b^2 \\
    &\qquad+ 2L_e w_e(p_b+p_f-p_b p_f)(N^2-L_e w_e)p_b \Bigg)n_w \\
    &\qquad\times \sum\limits_{i=\frac{\widetilde{w}_e}{w}n_w(p_f+p_b-p_f p_b)+(1-\frac{\widetilde{w}_e}{w})n_w p_b}^{n_w}{n_w\choose i}p_b^i(1-p_b)^{n_w-i}
\end{align*}

Next, due to our sampling strategy of the space of edges, we further slightly modify the previous formula. Indeed, since we enforce that each candidate region is determined by two \say{appearing} (white) pixels that are assumed to lie on the line, we must remove these two pixels from the count. On the other hand, if we had sampled edges in a more naive way such as considering all pairs of locations without the \say{two white pixels appearing} constraint, then the above formula would apply. Our choice influences the definition of $K_{GT}^*$ and therefore slightly affects its value. It will also have a slight impact on the value of the probability of being white under the background assumption and slightly modify the value of $\widehat{\mathcal{N}_{FA}^{*}}(p_b,p_f,w)$. If the window is considered without the two given pixels on the underlying assumed line, this implicitly changes the definition of $k$ and the modification in its value is:

\begin{align*}
    \widehat{K_{GT}^*} &= \mathbb{E}(K_{GT}^*) \\
    &= (L_e \widetilde{w}_e-2)(p_b+p_f-p_b p_f)+L_e(w-\widetilde{w}_e)p_b \\
    &= (\frac{\widetilde{w}_e}{w}n_w-2)(p_f+p_b-p_f p_b)+(1-\frac{\widetilde{w}_e}{w})n_w p_b
\end{align*}

We must update similarly the definition of the random variable $K$ since the target area has lost two pixels:

\begin{equation*}
    \mathbb{P}_{B}(K\ge \widehat{K_{GT}^*}) = \sum\limits_{i=\widehat{K_{GT}^*}}^{n_w-2}{n_w-2\choose i}p_b^i(1-p_b)^{n_w-2-i} = 1-\sum\limits_{i=0}^{i=\widehat{K_{GT}^*}-1}{n_w-2\choose i}p_b^i(1-p_b)^{n_w-2-i}
\end{equation*}

The estimated number of false alarms in the a contrario therefore updates to:

\begin{align*}
    \widehat{\mathcal{N}_{FA}^{*}}(p_b,p_f,w) &= \widehat{N}_t \mathbbm{P}_B(K \ge \widehat{K_{GT}^*}) \\
    &= \frac{1}{2}\Bigg(L_e w_e(L_e w_e -1)(p_b+p_f-p_b p_f)^2 +(N^2-L_e w_e)(N^2-L_e w_e -1)p_b^2 \\
    &\qquad+ 2L_e w_e(p_b+p_f-p_b p_f)(N^2-L_e w_e)p_b \Bigg)n_w \\
    &\qquad\times \sum\limits_{i=(\frac{\widetilde{w}_e}{w}n_w-2)(p_f+p_b-p_f p_b)+(1-\frac{\widetilde{w}_e}{w})n_w p_b}^{n_w-2}{n_w-2\choose i}p_b^i(1-p_b)^{n_w-2-i}
\end{align*}

For a fixed width $w$, the estimated criterion at the position of the true edge is just a function of $p_b$ and $p_f$: $\widehat{\mathcal{N}_{FA}^{*}}_w$. For each considered width we can plot $\widehat{\mathcal{N}_{FA}^{*}}(p_b,p_f,w) = \widehat{\mathcal{N}_{FA}^{*}}_w(p_b,p_f)$ and compute its contours levels.  For the a contrario algorithm, the interesting contour level is $\widehat{\mathcal{N}_{FA}^{*}}_w(p_b,p_f) = \varepsilon$. We would like to know whether the decision boundary for humans lies along one of the contour levels of this function. If this is the case then it would suggest that humans indeed use a Gestalt grouping perception on sparse data that is linked to unlikeliness against a random uniform assumption. Note that for a fixed width, then we implicitly assume a single width $N_w=1$. If we are comparing humans against several widths, then we should take $N_w$ to be larger which very slightly changes the contour levels.

Concluding, we can empirically sample the $(p_b,p_f)$ space and evaluate for a given $w$ the  $\widehat{\mathcal{N}_{FA}^{*}}_w(p_b,p_f)$. We have a parametric surface and compute its contour levels. Each contour level of value $\varepsilon$ corresponds to the critical partition line for the confidence level $\varepsilon$: above the line there should be a correct detection and below no detection. We plot the coutour levels of each of these functions in Figure \ref{fig: various widths contour level eps 1 static} left. We explain how to read these plots in the context of merging frames in Figure \ref{fig: various widths contour level eps 1 static} right.

\begin{figure}
    \centering
    \begin{subfigure}[t]{0.5\textwidth}
      \centering
      \includegraphics[width=\textwidth]{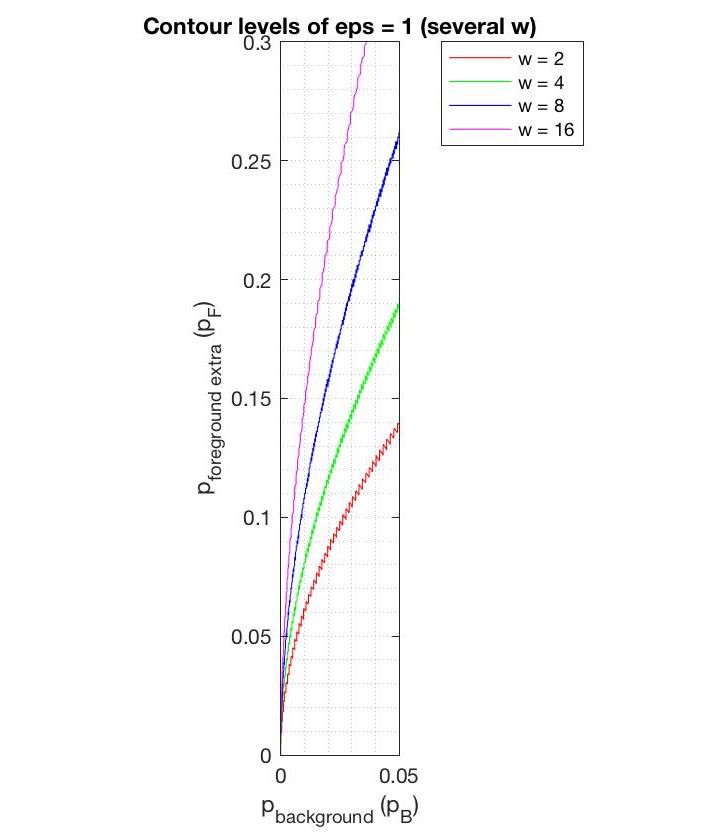}
    \end{subfigure}%
    \hfill
    \begin{subfigure}[t]{0.5\textwidth}
      \qquad\quad
      \includegraphics[width=\textwidth]{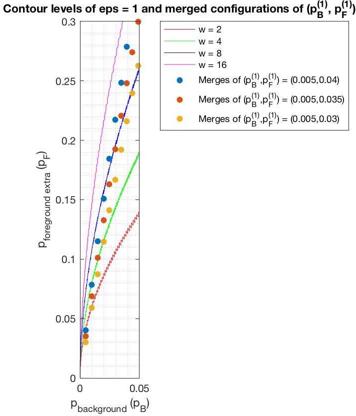}
    \end{subfigure}%
    \hfill
\caption{Left: Predicted a priori contour levels at level $\varepsilon=1$ of $\widehat{\mathcal{N}_{FA}^{*}}_w(p_b,p_f)$ for $w\in\{2,4,8,16\}$ in the static edge case. Right: same contour levels with configurations corresponding to merged frames. Fix a $p_b^{(1)}=0.005$ the background degradation parameter on one frame, and consider several foreground degradation parameters on one frame $p_f^{(1)} \in \{0.03,0.035,0.04\}$. Thus we get three pairs of degradation parameters, corresponding to the three coloured points vertically aligned which we call a column. All these configurations, corresponding to the first column of points on the figure on the right, are below each decision curve and so clearly no matter which $w$ we use, we should not be able to see anything on a per frame basis. A merge of $t$ frames corresponds to looking at the $t$-th column. Should we integrate on $t=7$ frames, and with a width of $w=8$, then we are looking at the 7-th column of dots and the blue threshold curve. Being above the curves means detection and below means miss. In this case, the point corresponding to the video with $p_f^{(1)} = 0.04$ is clearly above the threshold and should be easily detected, the one corresponding to $p_f^{(1)} = 0.03$ is clearly below and should not be detected and the one corresponding to $p_f^{(1)} = 0.035$ is not far from the threshold and should be in the transition or difficult zone to perceive, with some failures and successes.} 
\label{fig: various widths contour level eps 1 static}
\end{figure}

\subsection{Empirical Performances: Data Generation}

Now that we have introduced a theory, we will put it to the test in an empirical way.

Our empirical tests of the a contrario algorithms (as well as the experiments with human subjects later) are done on images generated on a computer. The size of the images is  $300\times300$ pixels, that corresponds in our display to a $10.4\times10.4$ cm viewing window. Each image will correspond to a degradation of a clean image that consists in an image of a single straight edge in a random position of length $L$ and width $w_e$, both measured in pixel size units, over a uniform black background. Then we fix the model parameters and generated a random dataset of images with various degradation parameters. We first chose the edge to have width $w_e = 1$ and length $L=200$. We sampled a grid of configurations in the $(p_b,p_f)$ space. In the context of very sparse videos, we chose to work in the range $(p_b,p_f)\in[0,0.05]\times[0,0.25]$. The grid sampling for this space is the following: $p_b\in\{0.005,0.01,0.02,0.03,0.04,0.05\}$ and $p_f$ uniformly sampled between $0.02$ and $0.25$ with a spacing of $0.01$ between samples. For each degradation configuration we generate $5$ random images. Each random image is generated independently from all other images, including those with the same degradation parameters. For each image, the edge's position is chosen at random. This means that its position and orientation are chosen randomly and uniformly. We enforce that the edge must lie entirely in the image. Therefore we sampled the midpoint of the edge according to a random uniform variable in $[100,200]\times[100,200]$ and the angle according to a uniform random variable in $[0,2\pi]$. Given a random position for an edge described in a ground truth clean image $I$ and degradation parameters $(p_b,p_f)$, the degraded image added to the dataset is $\phi_{im}(I,p_b,p_f)$. We did not generate samples for degradation parameter configurations where $p_f<2p_b$ since in practice in those cases the increase in probability of appearance on the edge due to $p_f$ is so marginal since $p_b$ is very small that for a thin one dimensional object not enough extra points appear in order to see an alignment. The total size of the dataset is $620$ images.

Once the data has been generated and saved, it is not regenerated: all experiments will be done on the same data (although shown in a different random order).

\subsection{Fixing $\varepsilon$}

Our model for the human perceptual system is a prior time integration module modelled as a merger of frames for video input followed by a spatial detector modelled by an a contrario algorithm. The only parameter for the first module is the time or number of frames of integration, and we need not discuss this here since each image can be considered as a single frame or as a merger of frames with lower initial degradation parameters. What we are interested in testing is the a contrario black box. The a contrario algorithm depends on several parameters. Some influence the counting function. Here, these are the parameters describing the region in which we cound the white pixels: rectangles of length $L_e$, the known length of the edge, and of width $w$. Finally, the a contrario needs a decision level $\varepsilon$ as the expected number of false alarms we are ready to accept. Therefore, for these experiments, there are essentially two parameters for the a contrario: $w$ and $\varepsilon$. 

We want to find the appropriate $w$ to model the \say{default} width humans will look at for finding edges when working of this type of noisy images. For this reason we will work with several $\widehat{\mathcal{N}_{FA}^{*}}_w$ with various $w$. Therefore we will sample $w$ in some way and this will yield a set of prior a contrario functions to compare with. On the other hand, we do not know in advance what kind of $\varepsilon$ would be most appropriate for humans. As explained previously, $\varepsilon$ designates the number of false alarms we are ready to accept (on average), and is at the core of the a contrario algorithm. Here, we consider the a contrario as a black box already designed for us. Therefore, $\varepsilon$ is just some confidence level parameter and its original meaning is not as important. We may expect humans to have $\varepsilon$ around $1$. The hand-waiving justification for this is that it seems unlikely that humans are ready to accept many false alarms (which are decisions that they see an edge) while at the same time they can sometimes imagine or convince themselves that they see some structure in some random sets. Also, if told to try to detect one edge, they will never return two. Out of hundreds of thousands of candidate positions, to accept on average one false alarms seems reasonable. Consider for instances how humans see shapes in a cloud, which are essentially random agglomerations of water. We shall assume that $\varepsilon$ stays within the same order of magnitude between different individuals.

The exact $\varepsilon$ value does not matter much compared with the $w$ value, as is common in traditional a contrario studies. We plot the predicted performance contour levels for various $\varepsilon$ in Figures \ref{fig: log dependency static w=2} to \ref{fig: log dependency static w=16}. As seen in these figures empirically, we need a substantial change in $\varepsilon$ by orders  of magnitude in order to get the decision curve to significantly change. As such, the dependence of performance in $\varepsilon$ is in effect very slight (and some analysis provided that it is logarithmic in $\varepsilon$), and therefore to choose the value $\varepsilon=1$ seems reasonable.

\begin{figure}
    \centering
    \begin{subfigure}[t]{0.3525\textwidth}
      \centering
      \includegraphics[width=\textwidth]{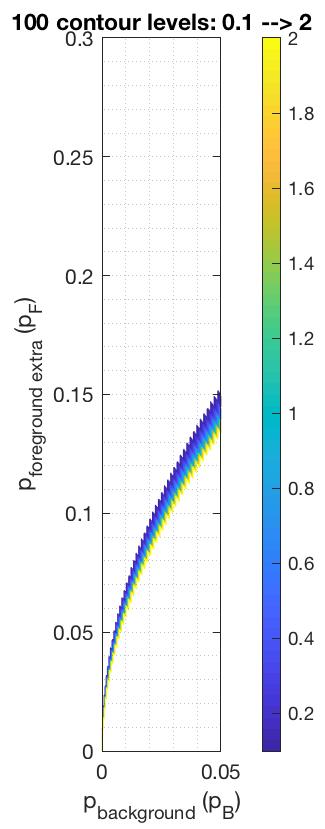}
    \end{subfigure}%
    \hfill
    \begin{subfigure}[t]{0.6\textwidth}
      \qquad\quad
      \includegraphics[width=\textwidth]{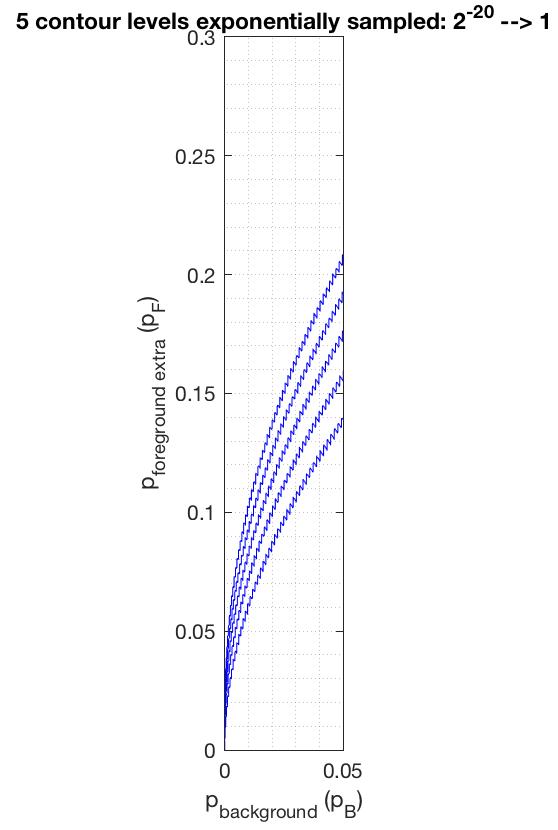}
    \end{subfigure}%
    \hfill
\caption{Left: 100 linearly sampled contour levels of $\widehat{\mathcal{N}_{FA}^{*}}_w(p_b,p_f)$ with $w=2$ in the static edge case, with $\varepsilon\in[0.1,2]$. There is not much change in the position of the decision levels for such a range of $\varepsilon$. Right: 5 exponentially sampled contour levels of $\widehat{\mathcal{N}_{FA}^{*}}_w(p_b,p_f)$ with $w=2$ in the static edge case, with $\varepsilon\in\{2^{-20},2^{-15},2^{-10},2^{-5},2^{0}=1\}$. In order to get a significant change in the decision level we have to drastically change $\varepsilon$. Empirically, to get a linear displacement of the curves we need to linearly change the logarithm of $\varepsilon$ thus an empirical log-dependency to $\varepsilon$. Furthermore, taking $\varepsilon=2^{-5}$ or less means that we expect on average, over a few hundred thousand tests, to make $2^{-5}$ or fewer mistakes, which does not seem reasonable when considering humans looking at signals similar to white noise. The choice of value $\varepsilon=1$ is reasonable and crude enough since small displacements of $\varepsilon$ do not influence much the decision levels.} 
\label{fig: log dependency static w=2}
\end{figure}

\begin{figure}
    \centering
    \begin{subfigure}[t]{0.3625\textwidth}
      \centering
      \includegraphics[width=\textwidth]{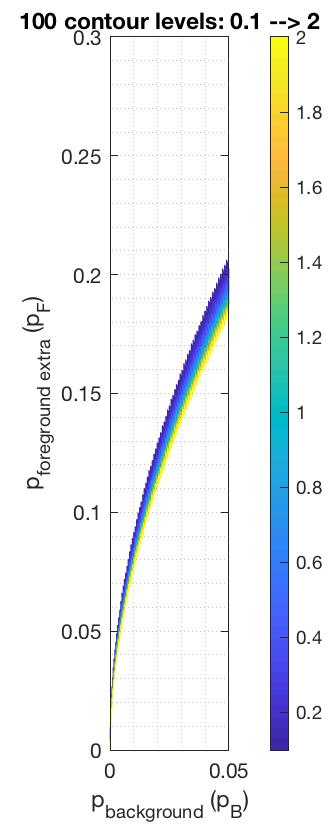}
    \end{subfigure}%
    \hfill
    \begin{subfigure}[t]{0.6\textwidth}
      \qquad\quad
      \includegraphics[width=\textwidth]{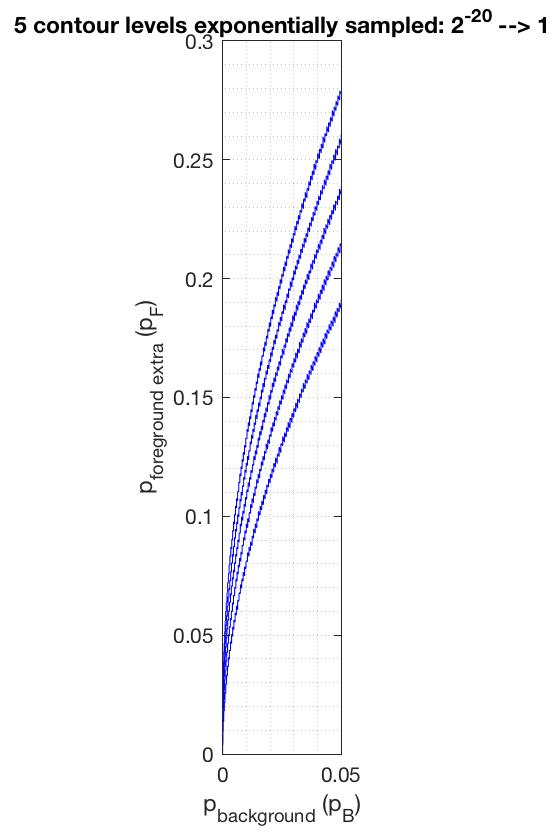}
    \end{subfigure}%
    \hfill
\caption{Left: 100 linearly sampled contour levels of $\widehat{\mathcal{N}_{FA}^{*}}_w(p_b,p_f)$ with $w=4$ in the static edge case, with $\varepsilon\in[0.1,2]$. There is not much change in the position of the decision levels for such a range of $\varepsilon$. Right: 5 exponentially sampled contour levels of $\widehat{\mathcal{N}_{FA}^{*}}_w(p_b,p_f)$ with $w=2$ in the static edge case, with $\varepsilon\in\{2^{-20},2^{-15},2^{-10},2^{-5},2^{0}=1\}$. In order to get a significant change in the decision level we have to drastically change $\varepsilon$. Empirically, to get a linear displacement of the curves we need to linearly change the logarithm of $\varepsilon$ thus an empirical log-dependency to $\varepsilon$. Furthermore, taking $\varepsilon=2^{-5}$ or less means that we expect on average, over a few hundred thousand tests, to make $2^{-5}$ or fewer mistakes, which does not seem reasonable when considering humans looking at signals similar to white noise. The choice of value $\varepsilon=1$ is reasonable and crude enough since small displacements of $\varepsilon$ do not influence much the decision levels.} 
\label{fig: log dependency static w=4}
\end{figure}

\begin{figure}
    \centering
    \begin{subfigure}[t]{0.37\textwidth}
      \centering
      \includegraphics[width=\textwidth]{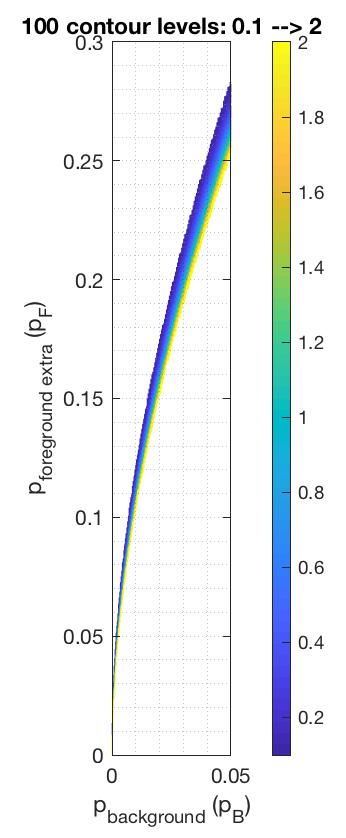}
    \end{subfigure}%
    \hfill
    \begin{subfigure}[t]{0.6\textwidth}
      \qquad\quad
      \includegraphics[width=\textwidth]{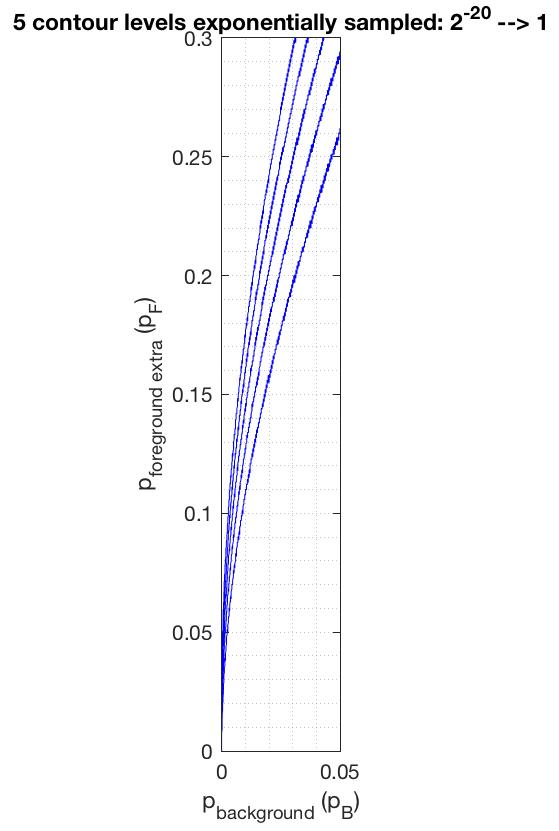}
    \end{subfigure}%
    \hfill
\caption{Left: 100 linearly sampled contour levels of $\widehat{\mathcal{N}_{FA}^{*}}_w(p_b,p_f)$ with $w=8$ in the static edge case, with $\varepsilon\in[0.1,2]$. There is not much change in the position of the decision levels for such a range of $\varepsilon$. Right: 5 exponentially sampled contour levels of $\widehat{\mathcal{N}_{FA}^{*}}_w(p_b,p_f)$ with $w=2$ in the static edge case, with $\varepsilon\in\{2^{-20},2^{-15},2^{-10},2^{-5},2^{0}=1\}$. In order to get a significant change in the decision level we have to drastically change $\varepsilon$. Empirically, to get a linear displacement of the curves we need to linearly change the logarithm of $\varepsilon$ thus an empirical log-dependency to $\varepsilon$. Furthermore, taking $\varepsilon=2^{-5}$ or less means that we expect on average, over a few hundred thousand tests, to make $2^{-5}$ or fewer mistakes, which does not seem reasonable when considering humans looking at signals similar to white noise. The choice of value $\varepsilon=1$ is reasonable and crude enough since small displacements of $\varepsilon$ do not influence much the decision levels.} 
\label{fig: log dependency static w=8}
\end{figure}

\begin{figure}
    \centering
    \begin{subfigure}[t]{0.3575\textwidth}
      \centering
      \includegraphics[width=\textwidth]{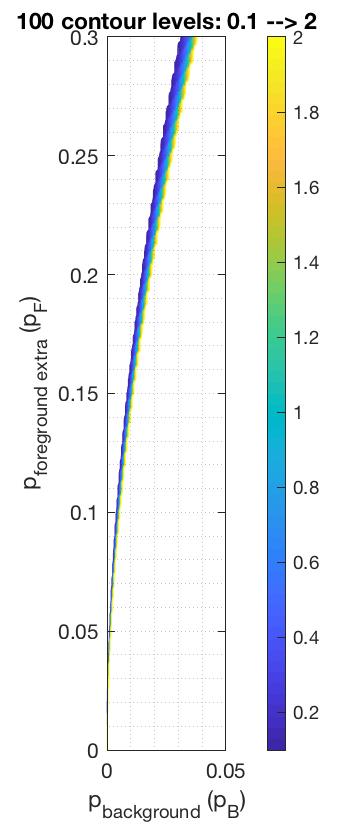}
    \end{subfigure}%
    \hfill
    \begin{subfigure}[t]{0.6\textwidth}
      \qquad\quad
      \includegraphics[width=\textwidth]{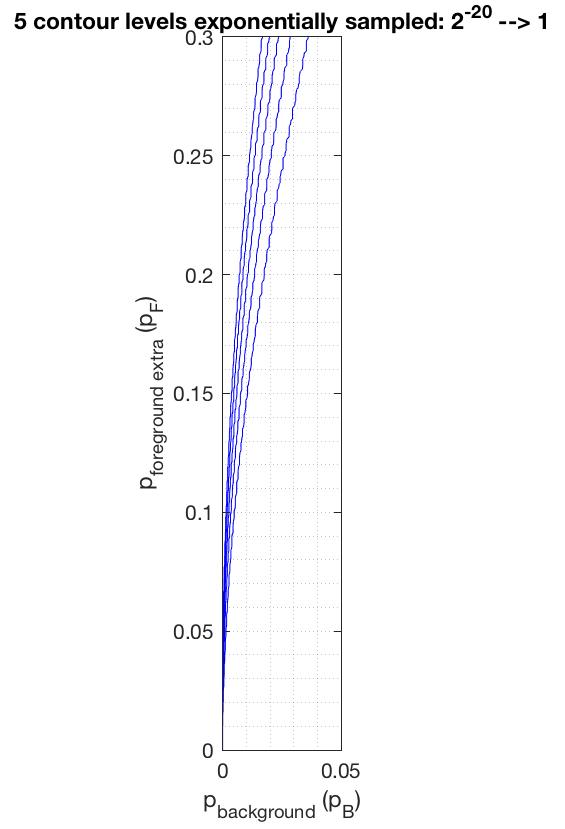}
    \end{subfigure}%
    \hfill
\caption{Left: 100 linearly sampled contour levels of $\widehat{\mathcal{N}_{FA}^{*}}_w(p_b,p_f)$ with $w=16$ in the static edge case, with $\varepsilon\in[0.1,2]$. There is not much change in the position of the decision levels for such a range of $\varepsilon$. Right: 5 exponentially sampled contour levels of $\widehat{\mathcal{N}_{FA}^{*}}_w(p_b,p_f)$ with $w=2$ in the static edge case, with $\varepsilon\in\{2^{-20},2^{-15},2^{-10},2^{-5},2^{0}=1\}$. In order to get a significant change in the decision level we have to drastically change $\varepsilon$. Empirically, to get a linear displacement of the curves we need to linearly change the logarithm of $\varepsilon$ thus an empirical log-dependency to $\varepsilon$. Furthermore, taking $\varepsilon=2^{-5}$ or less means that we expect on average, over a few hundred thousand tests, to make $2^{-5}$ or fewer mistakes, which does not seem reasonable when considering humans looking at signals similar to white noise. The choice of value $\varepsilon=1$ is reasonable and crude enough since small displacements of $\varepsilon$ do not influence much the decision levels.} 
\label{fig: log dependency static w=16}
\end{figure}

\subsection{Empirical Performance Versus the Predicted Performance of A Contrario}

The estimate $\widehat{\mathcal{N}_{FA}^{*}}(p_b,p_f,w)$ was an a priori estimate of the score the a contrario algorithm will get when it is considering a candidate position that is in the right position on the true edge. We can show the relevance of this estimator by comparing its associated performance, for the level $\varepsilon$, with the empirical performance of a real a contrario algorithm working on a single width on real data.

For candidate region widths ranging in ${w\in\{2,4,8, 16\}}$, we run the a contrario algorithm working on a single window width $w$ and confidence level $\varepsilon=1$. The empirical performance (fraction of runs in which the minimal a contrario score indeed fitted the true edge) of the single width a contrario processes are displayed in Figures \ref{fig: single widhts a contrario vs prediction static w 2 4} and \ref{fig: single widhts a contrario vs prediction static w 8 16} as colour coded points measuring the edge detection rates. We also plot the estimated contour levels of the detection performance, measured as in the static case, according to our a priori estimates. The contour levels for $\varepsilon=1$ and window width $w$ fits well the empirical transition area of the a contrario with single width $w$ and confidence level $\varepsilon=1$. This provides us empirical confirmation that the derived a priori formula does indeed apply for the a contrario.

Recall that the performance is the fraction of experiments returning the lowest candidate location with score lower than $\varepsilon$, if any, that are at the true position of the edge, a.k.a hit performance. We found that cases where the output, i.e. the candidate that has lowest a contrario score and below the level $\varepsilon$, would seldom disagree with the true edge.

In the static case, for high $p_f$ and low $p_b$ (many points on the line and few in the background), the performance will saturate close to 1. On the contrary, for low $p_f$ and high $p_b$, the performance saturates close to 0. In between there is a transition area. This corresponds to the change between failure to detect and success of detection. Between these extremes, we have the transition decision area. For every $p_b$, there should be some value $v$ for $p_f$ such that, ideally, for every $p_f>v$, the true line segment would always be detected, and that for every $p_f<v$ the line segment would not. In practice, the algorithms deviate from the idealised behaviour, and sometimes they do not detect the true line segment even for $p_f>v$ or inversely sometimes they detect the true line segment for $p_f<v$. Denote $\mathbbm{P}_h(p_b,p_f)$ the empirical hit probability: the probability to correctly detect the line segment at the configuration $(p_b,p_f)$ (the average performance at this configuration). We specify the empirical decision at $(p_b,v)$ such that there are as many hits of the true edge for $(p_b,p_f)$ with $p_f$ smaller than $v$ than misses of the true edge for $(p_b,p_f)$ with $p_f$ greater than $v$: $\sum\limits_{p_f\ge v}1-\mathbbm{P}_h(p_b,p_f) \approx \sum\limits_{p_f <v}\mathbbm{P}_h(p_b,p_f)$. Note that since we did not fully sample the $[0,1]$ space for $p_f$, there are values for which, in the cases of $w\in\{8,16\}$, we have not yet reached the performance saturation of 1. Therefore in this definition, we will assume that the performance function equals to $1$ for all values of $p_f$ above the sampled space. This was done since we could not work on all $p_f$ values as it would be too time consuming. This means that in some extreme areas when we have not yet reach saturation (the yellow dots in the figures), the empirical decision threshold is taken to be lower than the true empirical decision threshold.

\begin{figure}
    \centering
    \begin{subfigure}[t]{0.49\textwidth}
      \centering
      \includegraphics[width=\textwidth]{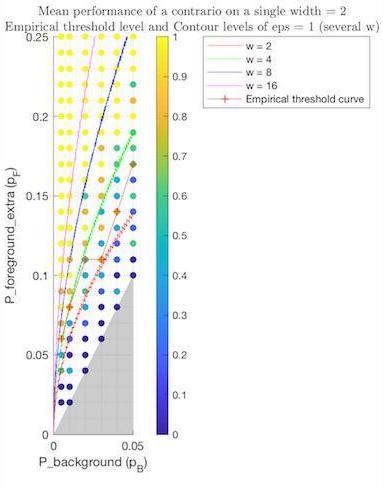}
      \caption{A contrario $w=2$}
    \end{subfigure}%
    \hfill
    \begin{subfigure}[t]{0.49\textwidth}
      \centering
      \includegraphics[width=\textwidth]{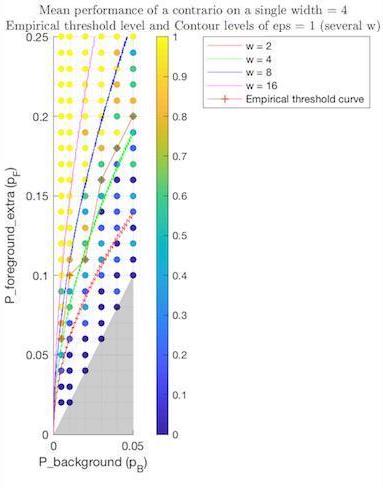}
      \caption{A contrario $w=4$}
    \end{subfigure}%
    \caption{Empirical performance of single width a contrario algorithms with $w=2$ and $w=4$ and their comparison with the predicted a priori mathematical model. The empirical performance fits very well for the width $w=4$ but less so for width $w=2$. This is mainly due to the issue of digitisation. In this case, depending on the orientation of the line, a width of $w=2$ pixels of the line means we look on each side at pixels within distance $1$. For horizontal or vertical edges, this means we are looking at band of width $3$ pixels whereas for a perfectly diagonal edge, only the pixels that are exactly on the line will be considered. This digitisation impact is significant for $w=2$ since in this case it drastically changes the size of the candidate area and its number of pixels. The effect is less important for larger widths and can be forgotten. In our mathematical model, we did not worry about digitisation artefacts and always assumed that each candidate sample had the same area: $n_w = [Lw]-2$.} 
\label{fig: single widhts a contrario vs prediction static w 2 4}
\end{figure}

\begin{figure}
    \centering    
    \begin{subfigure}[t]{0.49\textwidth}
      \centering
      \includegraphics[width=\textwidth]{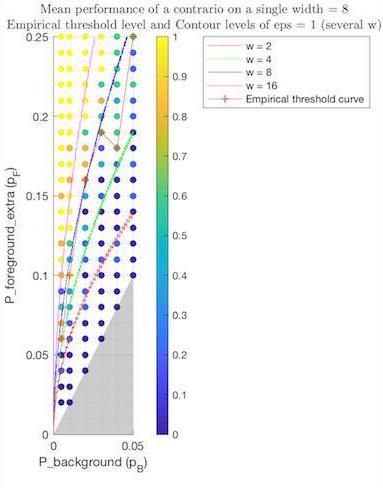}
      \caption{A contrario $w=8$}
    \end{subfigure}%
    \hfill
    \begin{subfigure}[t]{0.49\textwidth}
      \centering
      \includegraphics[width=\textwidth]{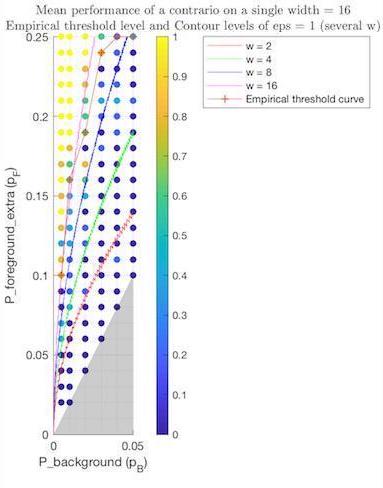}
      \caption{A contrario $w=16$}
    \end{subfigure}%
    \hfill
\caption{Empirical performance of single width a contrario algorithms with $w=8$ and $w=16$ and their comparison with the predicted a priori mathematical model. The empirical performance fit well to the mathematical prediction. } 
\label{fig: single widhts a contrario vs prediction static w 8 16}
\end{figure}

\subsection{Humans and A Contrario}
\label{exp1}

In this Section we present empirical results comparing human performance and the a contrario algorithm, to challenge the a contrario model of the human visual system. We devised a set of experiments as described below. The experiments detailed here were approved by the Ethics committee of the Technion as conforming to the standards for performing psychophysics experiments on volunteering human subjects.

\paragraph{Stimuli} The stimuli is the random-dot images dataset we have randomly generated and previously described. All subjects were presented with the entire dataset (620 images). The order in which they are shown is chosen at random and therefore differs between subjects.

\paragraph{Apparatus} Each subject was tested on exactly the same display. The display is the screen of a MacBook Pro retina 13-inch display, from mid 2014. Each subject was seated directly in front of the screen. The images were displayed in a well-lit room. The average distance between the subjects' eyes and the screen was about $70\,\mathrm{cm}$. The $300\times300$ image is displayed having a size of $10.4 \,\mathrm{cm}\times10.4\,\mathrm{cm}$. This translates to a pixel size of approximately $0.35\,\mathrm{mm}\times0.35\,\mathrm{mm}$. On average the visual angle for observing the image is approximately $8.5°\times8.5°$. Next to the image, we displayed a red box for user decisions (see the procedure paragraph) of size $3.7\,\mathrm{cm}\times3.7\,\mathrm{cm}$. The average visual angle for the red box is approximately $3.0°\times3.0°$. The distance between the border of the image and the border of the red box is $5.1\,\mathrm{cm}$. The average visual angle for the separation between the image and the box is of approximately $4.1°$. See Figure \ref{fig: screenshot exp 1} for a screenshot of the display.

\paragraph{Subjects} The experiments were performed on the first author and on thirteen other subjects. All the other subjects were unfamiliar with psychophysical experiments. All subjects were international students of the university, coming from various countries around the globe including China, Vietnam, India, Germany, Greece, the United States of America and France. Gender parity was almost respected, with eight males and six females. The age range of the subjects was between 20 and 31 years of age. All subjects had normal or corrected to normal vision. To the best of our knowledge, no subject had any mental condition. All other subjects were unaware of the aim of our study.

\paragraph{Procedure} Subjects were told that each image they would be presented with would consist in random points but that there is an alignment of points that should form a straight edge at some random position of length equal to two thirds of the horizontal dimension of the image. They were asked to try and detect it. In the entire experiment, a Matlab window covers the entire screen. In the left we display the random image the subject must work on. On the right we have the fixed red square. If the subjects could not see the edge, they were instructed to click once on the red box and the next image would appear. If they did see and edge, they had to click twice on the edge to roughly define its extremities. Subjects were encouraged to click on locations as far as possible while still on the edge of length two-thirds of the image width. They were told that if they could only see a sub-part of the total edge that was relevant then it was all right to click on what they see as the edge. They were strongly encouraged to not click on very small alignments or to click at the same positions. They were also encouraged to be as precise as possible. They were strongly encouraged to click on the red box in case of doubt and were told that it is all right to click on the red box if they saw no edge. The cursor consisted in an easily visible cross thin enough so as not to harmfully obstruct visibility. For the mouse used for clicking, subjects had the choice between using the built-in touchpad of the MacBook Pro or to use a Asus N6-Mini mouse. If the subject clicked on any area not in the image or the red box, that click was dismissed. If the subject clicked once on the image and then on the red square then the next image was shown. If the subject clicked twice on the image the next image is shown. Subjects had a 10 second limit to answer for each image (to click on the red box or click twice on the image), after which the next image would be automatically shown. They were encouraged to click on the image if the detection was obvious for them. However we did not explain or give a definition of \say{obvious} to subjects. Subjects were not shown a progress bar. Subjects could not take a break once having started the experiment, but were allowed to abandon the experiment at any time, should they desire it. The time for each click, along with the pixel in the image corresponding to a click on the random image were saved.

\paragraph{Discussion} Performance of individuals was measured as previously as a $0-1$ score where $0$ would be a decision to not see, to run out of time or to make an incorrect detection. On the other hand, $1$ is given for a correct decision. We had to allow a high tolerance for the clicking as subjects tended to not be particularly accurate in the position of the click compared to what they saw. A correct decision had to satisfy the three following tolerances. The angle of the line between both clicks had to be within a range of $0.1\,\mathrm{rad}$ of the angle of the line (near parallelism test). The maximum distance between a click and its orthogonal projection onto the true line had to be smaller or equal to $15$ pixels (distance to line test). Each click had to be no further away than the mid point of the true edge than by half the length of the edge plus a tolerance of $20$ pixels. The experimental results are plotted in Figures \ref{fig: static perf 0 and all} and \ref{fig: static perf 1 to 4} to \ref{fig: static perf 13}. First, the predicted (and empirical) threshold curves of single widths a contrario seem particularly relevant to the performance of humans. Second, the performance differs between individuals. Not all decision thresholds are exactly at the same position. Most seem to fit well around the threshold line of $w=8$ for low and moderate $p_b$ but some tend to become closer to the line $w=4$ for high $p_b$. While this can be due to intrinsic differences between each individual, it is also linked to the fact that not all subjects made the same mental effort to detect the lines. In particularly, those who decided faster to detect or reject seemed to fit better to a higher width a contrario than those who took their time. This is unsurprising since those who take their time can focus better and use the multiscale of the brain rather than solely rely on their \say{default vision} when looking at such data. Nevertheless, we can say that the single width a contrario process with $w=8$ seems to be a good candidate for modelling human performance when looking at random-dot static data corresponding to an image of a static edge. From now on we fix this as the choice for $w$. Note that this width corresponds to a visual angle of $\theta=0.23°$.

\begin{figure}
    \centering
    \includegraphics[width=\textwidth]{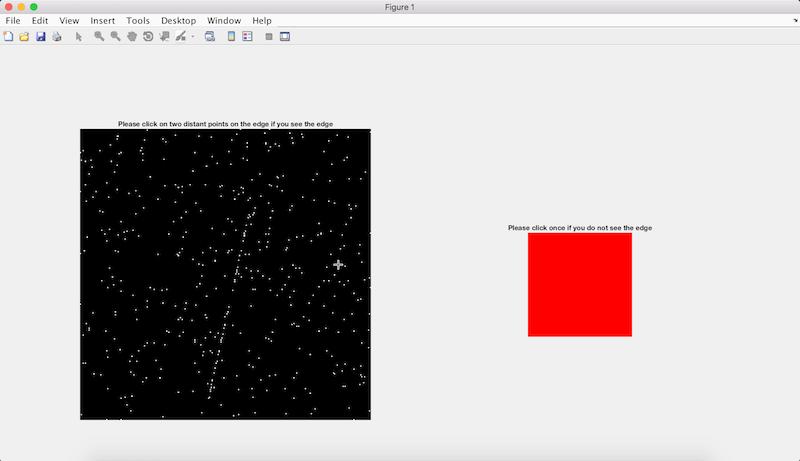}
    \hfill
    \caption[Short version]{Screenshot of the display.}
    \label{fig: screenshot exp 1}
\end{figure}

\begin{figure}
    \centering
    \begin{subfigure}[t]{0.49\textwidth}
      \centering
         \includegraphics[width=\textwidth]{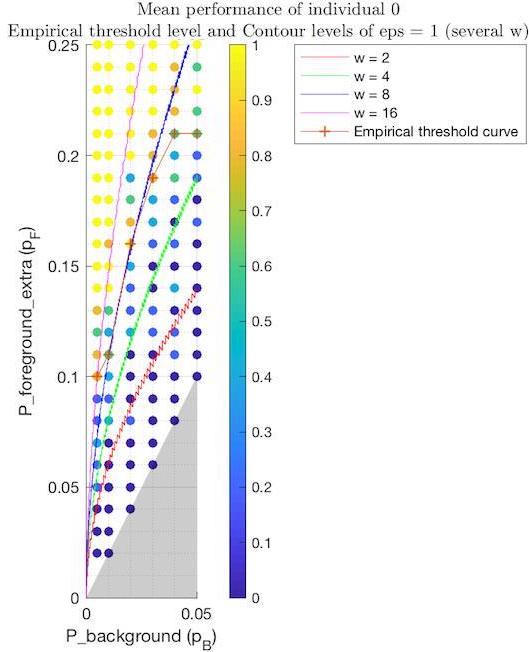}
    \end{subfigure}%
    \hfill
    \begin{subfigure}[t]{0.49\textwidth}
      \centering
         \includegraphics[width=\textwidth]{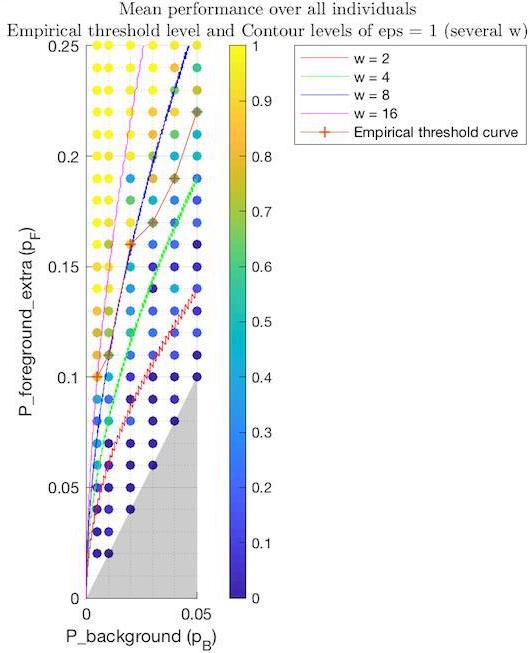}
    \end{subfigure}%
    \hfill
\caption{Empirical performance of humans on static images of a static edge. Left: performance of a single person. Right: average performance over all subjects. Each dot corresponds to a pair of degradation parameters, each corresponding to $5$ samples. The colour corresponds to the averaged score on those samples per degradation parameter. A yellow dot corresponds to perfect success in recovering the line whereas as a dark blue dot corresponds to systematic failure to recover the line. The transition zone seems to fit the transition zones for a contrario. More formally, the empirical decision level fits quite well the predicted performance of a single width a contrario algorithm with candidate widths $w=8$.}
\label{fig: static perf 0 and all}
\end{figure}

\begin{figure}
    \centering
    \begin{subfigure}[t]{0.49\textwidth}
      \centering
         \includegraphics[width=\textwidth]{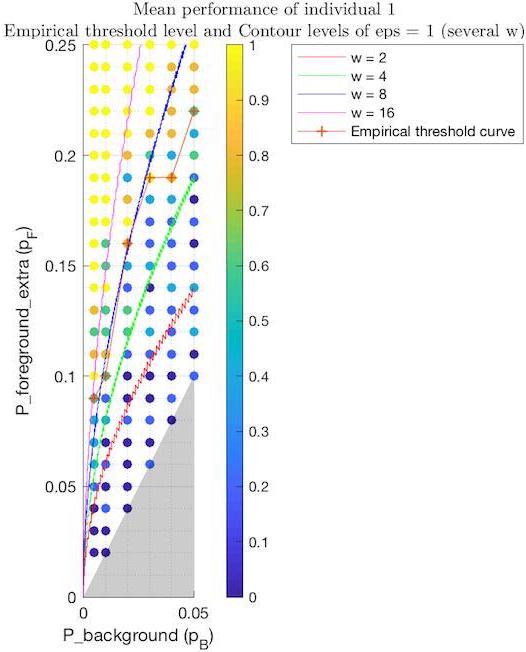}
    \end{subfigure}%
    \hfill
    \begin{subfigure}[t]{0.49\textwidth}
      \centering
         \includegraphics[width=\textwidth]{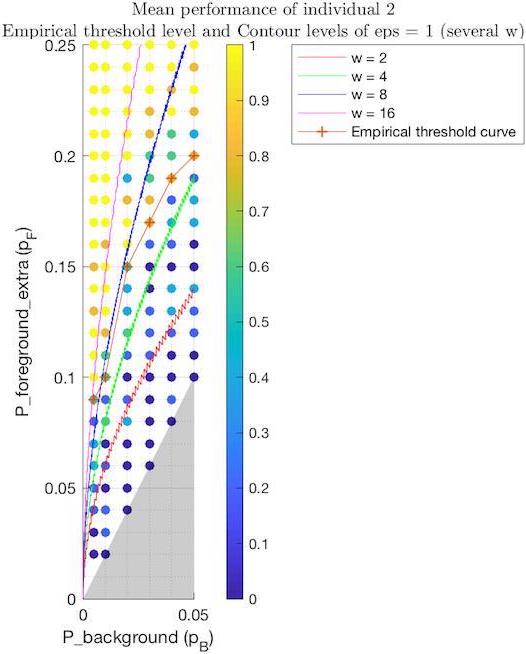}
    \end{subfigure}%
    \hfill
    \centering
    \begin{subfigure}[t]{0.49\textwidth}
      \centering
         \includegraphics[width=\textwidth]{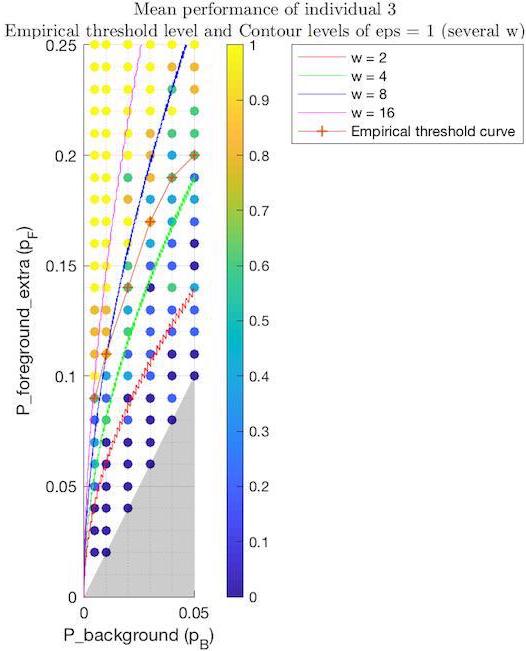}
    \end{subfigure}%
    \hfill
    \begin{subfigure}[t]{0.49\textwidth}
      \centering
         \includegraphics[width=\textwidth]{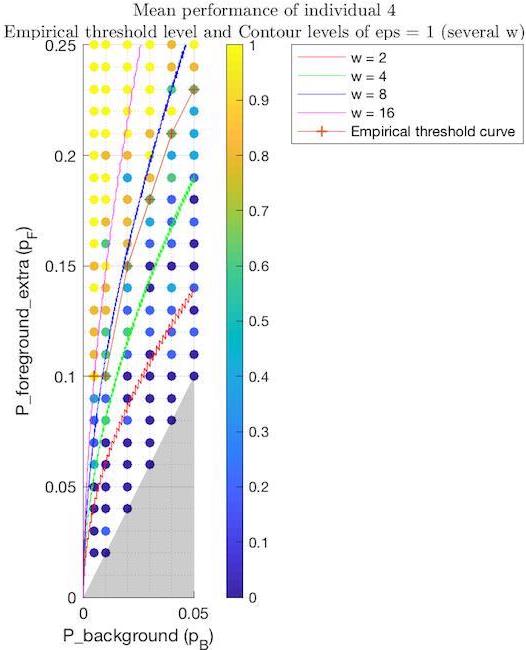}
    \end{subfigure}%
\caption{Empirical performance of humans on static images of a static edge for subjects 1 to 4. The empirical performance of most subjects lies close to the predicted performance of a contrario working on $w=8$. Some humans tend to have a performance that goes down towards the predicted performance of a contrario on $w=4$ for high $p_b$. This could be due to the fact that humans inevitably use to some extent multiscale and also due to the fact that some subjects made more effort to detect the line and by focusing more they could tune their visual angle to a smaller angle.}
\label{fig: static perf 1 to 4}
\end{figure}
\begin{figure}
    \centering    
    \begin{subfigure}[t]{0.49\textwidth}
      \centering
         \includegraphics[width=\textwidth]{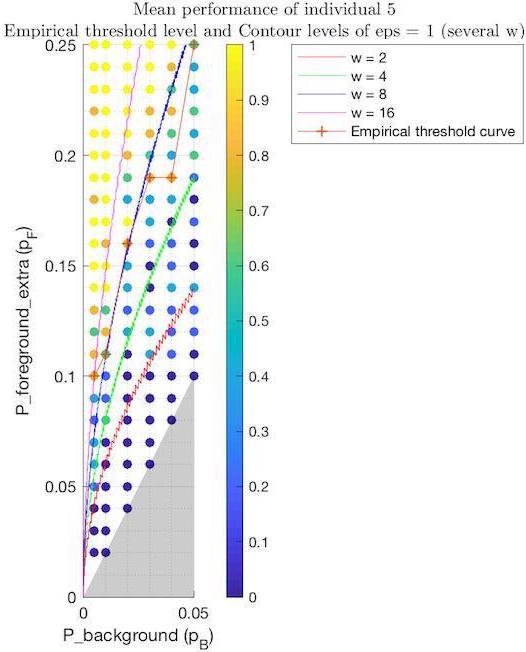}
    \end{subfigure}%
    \hfill
    \begin{subfigure}[t]{0.49\textwidth}
      \centering
         \includegraphics[width=\textwidth]{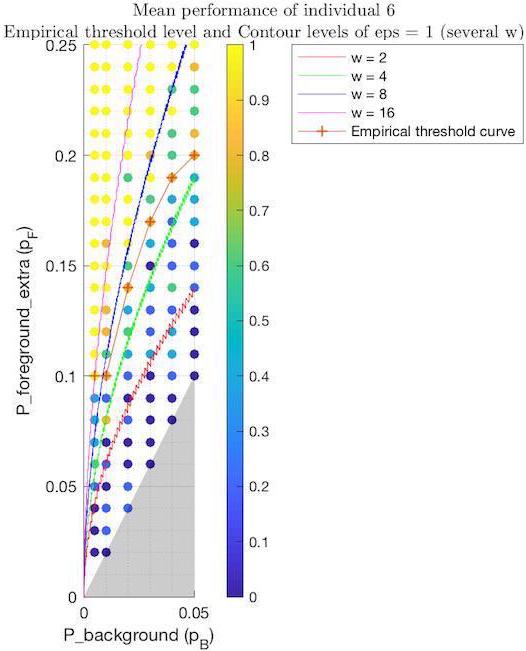}
    \end{subfigure}%
    \hfill
    \centering    
    \begin{subfigure}[t]{0.49\textwidth}
      \centering
         \includegraphics[width=\textwidth]{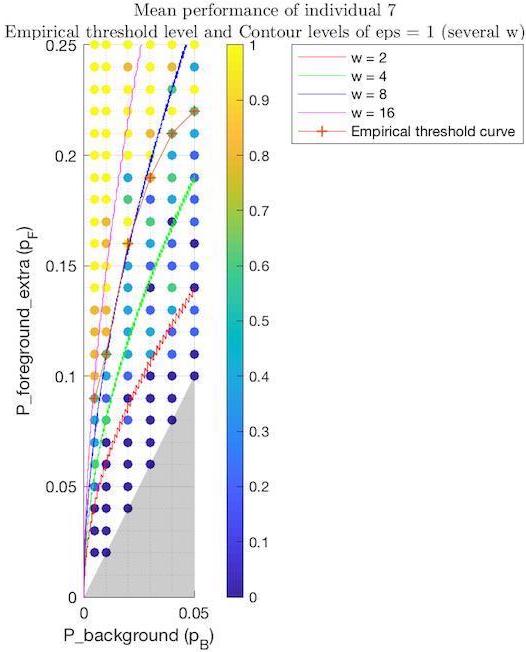}
    \end{subfigure}%
    \hfill
    \begin{subfigure}[t]{0.49\textwidth}
      \centering
         \includegraphics[width=\textwidth]{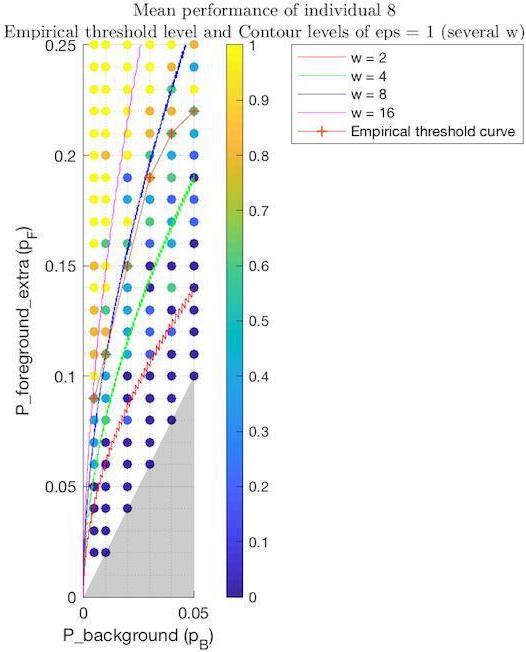}
    \end{subfigure}%
\caption{Empirical performance of humans on static images of a static edge for subjects 5 to 8. The empirical performance of most subjects lies close to the predicted performance of a contrario working on $w=8$. Some humans tend to have a performance that goes down towards the predicted performance of a contrario on $w=4$ for high $p_b$. This could be due to the fact that humans inevitably use to some extent multiscale and also due to the fact that some subjects made more effort to detect the line and by focusing more they could tune their visual angle to a smaller angle.}
\label{fig: static perf 5 to 8}
\end{figure}
\begin{figure}
    \centering    
    \begin{subfigure}[t]{0.49\textwidth}
      \centering
         \includegraphics[width=\textwidth]{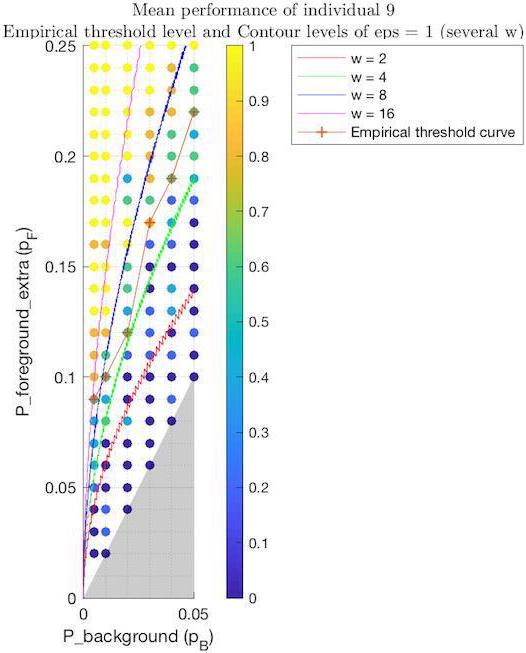}
    \end{subfigure}%
    \hfill
    \begin{subfigure}[t]{0.49\textwidth}
      \centering
         \includegraphics[width=\textwidth]{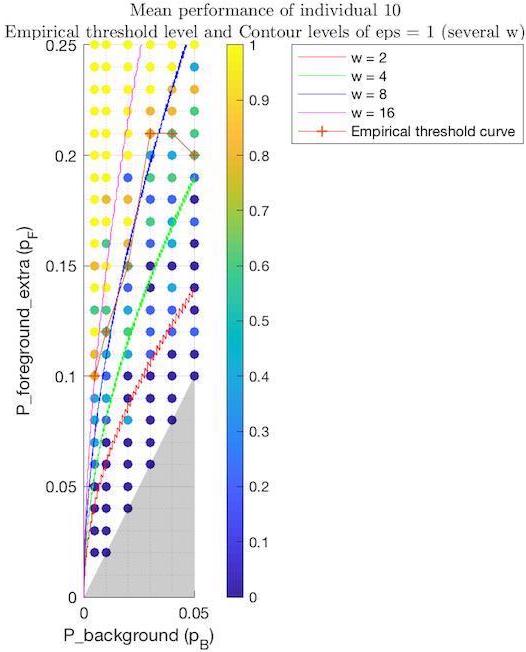}
    \end{subfigure}%
    \hfill
    \centering    
    \begin{subfigure}[t]{0.49\textwidth}
      \centering
         \includegraphics[width=\textwidth]{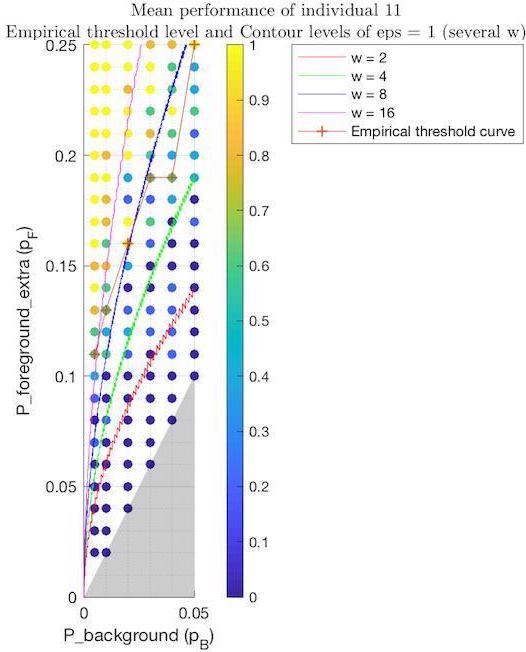}
    \end{subfigure}%
    \hfill
    \begin{subfigure}[t]{0.49\textwidth}
      \centering
         \includegraphics[width=\textwidth]{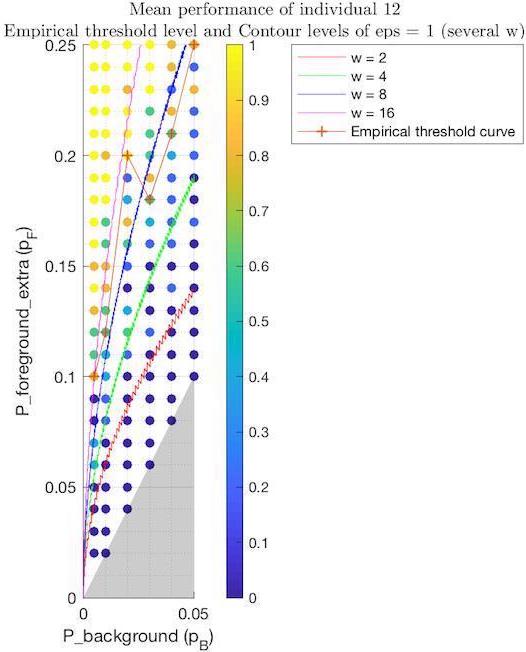}
    \end{subfigure}%
\caption{Empirical performance of humans on static images of a static edge for subjects 9 to 12. The empirical performance of most subjects lies close to the predicted performance of a contrario working on $w=8$. Some humans tend to have a performance that goes down towards the predicted performance of a contrario on $w=4$ for high $p_b$. This could be due to the fact that humans inevitably use to some extent multiscale and also due to the fact that some subjects made more effort to detect the line and by focusing more they could tune their visual angle to a smaller angle.}
\label{fig: static perf 11 12}
\end{figure}
\begin{figure}
    \centering    
    \begin{subfigure}[t]{0.49\textwidth}
      \centering
         \includegraphics[width=\textwidth]{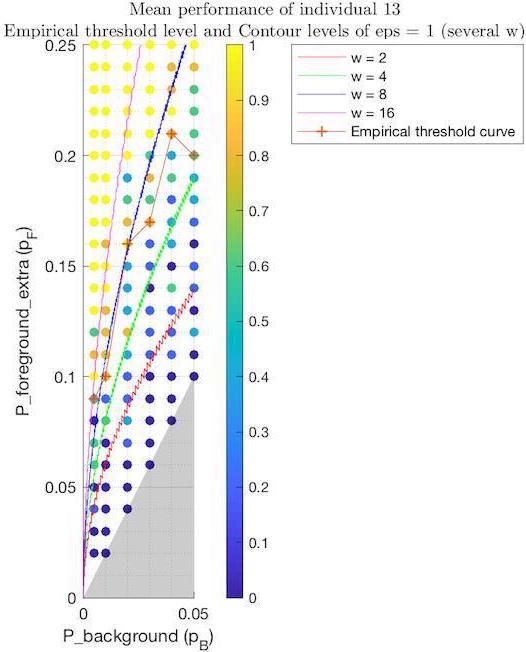}
    \end{subfigure}%
    \hfill
\caption{Empirical performance of humans on static images of a static edge for subject 13. The empirical performance of most subjects lies close to the predicted performance of a contrario working on $w=8$. Some humans tend to have a performance that goes down towards the predicted performance of a contrario on $w=4$ for high $p_b$. This could be due to the fact that humans inevitably use to some extent multiscale and also due to the fact that some subjects made more effort to detect the line and by focusing more they could tune their visual angle to a smaller angle.}
\label{fig: static perf 13}
\end{figure}

\clearpage
\newpage
\section{Analysis of the Dynamic Edge Case}
\label{sec: analysis dynamic edge}

In section \ref{sec: analysis static edge}, we analysed the case of a single image, or a merge of image of a video, of a non moving line. Now consider that the foreground object in the video is a line of known length $L_e$ and width $w_e=1$ that moves in a non static fashion. For mathematical simplicity, we shall consider that the edge undergoes a uniform translatory movement orthogonal to itself with speed ${v_e = 1 \, \mathrm{pixel/frame}}$. Recall that we defined pixel size as the unit side size of a pixel. This means that in a second, under a frame-rate of $30\,\mathrm{frames/second}$, the edge will have moved by 30 pixels, which is fast. While this is simple, it is not unreasonable as shown in demos of natural movements (see \cite{dagesMeDemo} where the torso of the walker moves at about ${v_e \approx 1 \, \mathrm{pixel/frame}}$). Note that it is also the lowest speed for which no matter how many merges of frames we do, a pixel in the image can only be traversed by the edge in at most one frame. We shall analyse this case similarly to the static edge case.

We shall denote $I_{(w)}$ the non degraded image of an edge of length $L_e$ and width $w$. It can be also be seen as the merge of the successive $w$ non degraded frames of a video of a dynamic edge.

Consider a sequence of $n_{f}$ consecutive frames in a video. Denote $I_1,\hdots,I_{n_{f}}$ the corresponding non degraded frames. Denote the degraded frames of the sequence $I_{j}^D = \phi_{im}(I_j,p_b,p_f)$ for $j\in{1,\hdots,n_{f}}$. Denote the merged frame: $I^{D,M} = \bigvee\limits_{1\le j \le n_{f} }I_{j}^{D}$. Recall that $\bigvee$ is the pixelwise logical OR operator. We will also use $\bigwedge$ the pixelwise logical AND operator. Conditionally to the groundtruth (or displacement of the edge), the degraded frames are independent.

For mathematical simplicity, we will here assume that the line is vertical moving in the horizontal direction. In this case we have the following theorem:

\bigskip
\begin{theorem}
    For $t\ge 1$ and $(I_1,\hdots,I_t)$ a sequence of static non random boolean images of a vertical line moving orthogonally to itself (horizontally) of length $L_e$ and at a speed of $v_e = 1\,\mathrm{pixel/frame}$. Denote $I_{(t)} = \bigvee\limits_{1\le j \le t}I_{j}$ the static non random merge of non degraded image. We have that $I_{(t)}$ a clean image of a rectangle of length $L_e$ and width $t$. For some degradation parameters $(p_b,p_f)$, denote $I_j^D = \phi_{im}(I_j,p_b,p_f)$ the degraded frames. They are independent conditionally to the displacement of the line. Denote $I^{D,M} = \bigvee\limits_{1\le j \le t}I_{j}^{D}$. Then we have $I^{D,M} \sim \phi_{im}(I_{(t)},(p_b^M)_{t},(p_f^M)_t)$ where $(p_b^M)_t = 1-(1-p_b)^t$ and $(p_f^M)_t = p_f$. This theorem generalises naturally only conditionally to $I_{(w)}$.
\end{theorem}

\begin{proof}
    The proof is similar to the static case. It consists in looking at the probabilities of white pixels appearing on the merged image conditionally to the positions of the line. First we will work conditionally to $(I_1,\hdots,I_t)$ and then generalise to working conditionally to $I_{(w)}$ and then further generalise to a random initial configuration of the line and its initial displacement direction.

    Denote $F_j = \{i; (I_j)_i = 1\}$ and $B_j = \{i; (I_j)_i = 0\}$ the foreground and background of the images. Denote $F = \bigvee\limits_{1\le j \le t}F_{j}$ and $B = \bigwedge\limits_{1\le j \le t}B_{j}$ the union of foreground pixels and the intersection of background pixels of each frames. We naturally have that $F = \{i; (I_{(t)})_i = 1\}$ and $B = \{i; (I_{(t)})_i = 0\}$ are the foreground and background areas in the merged image. We can define for every pixel $i\in F$ the set $F(i) = \{j\in\{1,\hdots,t\},i\in F_j\}$. This set is deterministic conditionally to the displacement of the line. Since the movement of the line is uniform, we have that $\#F(i)$ is independent on $j\in F$. We denote this constant as $s = \# F(i)$ for any $i\in F$. Note that since here the speed is $1\,\mathrm{pix/frame}$, any pixel traversed by the line in the $t$ frames can only have been traversed at only one of the frames. This means that for any $i\in F$, $F(i)$ is a set reduced to a single element (singleton) which means that $s=1$. Lower speeds would mean greater $s$.
    
    We have that, recalling the independence between frames (conditionally to the displacement of the line):
    \begin{align*}
        (p_0^M)_t &\triangleq \mathbb{P}((I^{D,M})_i = 1 \mid i\in B,(I_1,\hdots,I_t),p_b,p_f) \\
        &= 1-\mathbb{P}(\bigwedge\limits_{j=1}^{t}(I_{j}^D)_i = 0\mid i\in B,(I_1,\hdots,I_t),p_b,p_f) \\
        &= 1-\prod\limits_{j=1}^{t}\mathbb{P}((I_{j}^D)_i = 0\mid i\in B,I_1,\hdots,,I_t,p_b,p_f) \\
        &= 1-(1-p_b)^{t}
    \end{align*}
    and similarly, for the same reasons and since the foreground and background signals are independent:
    \begin{align*}
        (p_1^M)_t &\triangleq \mathbb{P}((I^{D,M})_i = 1 \mid i\in F,(I_1,\hdots,I_t),p_b,p_f) \\
        &= 1-\mathbb{P}(\bigwedge\limits_{j=1}^{t}(I_{j}^D)_i = 0\mid i\in F,(I_1,\hdots,I_t),p_b,p_f) \\
        &= 1-\mathbb{P}(\bigwedge\limits_{\substack{j=1\\j\notin F(i)}}^{t}(I_{j}^D)_i = 0\mid i\in F,(I_1,\hdots,I_t),p_b,p_f) \\
        &\qquad \times\mathbb{P}(\bigwedge\limits_{\substack{j=1\\j\in F(i)}}^{t}(I_{j}^D)_i = 0\mid i\in F,(I_1,\hdots,I_t),p_b,p_f)\\
        &= 1-\prod\limits_{\substack{j=1\\j\notin F(i)}}^{t}\mathbb{P}((I_{j}^D)_i = 0\mid i\in F,I_,p_b,p_f)\\ 
        &\qquad \quad \times\prod\limits_{\substack{j=1\\j\in F(i)}}^{t}\mathbb{P}((I_{j}^D)_i = 0\mid i\in F,I_,p_b,p_f)\\
        &= 1-(1-p_b)^{t-\# F(i)}\Big( (1-p_b)(1-p_f) \Big) ^{\# F(i)} \\
        &= 1-(1-p_b)^{t}(1-p_f)^{\# F(i)} \\
        &= 1-(1-p_b)^{t}(1-p_f)^{s} \\
        &= 1-(1-p_b)^{t}(1-p_f) 
    \end{align*}
    
    We now work only conditionally to $I_{(w)}$. In this case, there are only two, equally probable, scenarios given the assumption on the type of displacement the line is going through. Either the sequence of non degraded images is $(I_1,\hdots,I_t)$, or it is the flipped one $(I_t,\hdots,I_1)$. We shall call arbitrarily for simplicity $(I_1,\hdots,I_t)$ the normal sequence and $(I_t,\hdots,I_1)$ the reverse. We shall use the $N$ and $R$ upper indices as natural notations referring to the normal sequence case and the reverse sequence case. Both sequences have their corresponding sequences of degraded images and the corresponding foreground and background regions. What is important to point out is that regardless of which sequence we are considering, if $i\in F$, then $F^R(i) = \{t-j,j\in F^N(i)\}$ which implies $\#F^N(i) = \#F^R(i) = s$.
    
    We thus have, using the result proven conditionally to one of the sequence:
    \begin{align*}
        &\mathbb{P}((I^{D,M})_i = 1 \mid i\in B,I_{(t)},p_b,p_f) \\
        &\qquad= \mathbb{P}((I^{D,M})_i = 1 \mid i\in B,(I_1,\hdots,I_t),p_b,p_f)\mathbb{P}((I_1,\hdots,I_t)) \\
        &\qquad\qquad + \mathbb{P}((I^{D,M})_i = 1 \mid i\in B,(I_t,\hdots,I_1),p_b,p_f)\mathbb{P}((I_t,\hdots,I_1)) \\
        &\qquad= \frac{1}{2}\mathbb{P}((I^{D,M})_i = 1 \mid i\in B,(I_1,\hdots,I_t),p_b,p_f)\\
        &\qquad\qquad+\frac{1}{2}\mathbb{P}((I^{D,M})_i = 1 \mid i\in B,(I_t,\hdots,I_1),p_b,p_f) \\
        &\qquad= \frac{1}{2}(1-(1-p_b)^{t}) + \frac{1}{2}(1-(1-p_b)^{t}) \\
        &\qquad= 1-(1-p_b)^{t}
    \end{align*}
    and similarly:
    \begin{align*}
        &\mathbb{P}((I^{D,M})_i = 1 \mid i\in F,I_{(t)},p_b,p_f) \\
        &\qquad= \mathbb{P}((I^{D,M})_i = 1 \mid i\in F,(I_1,\hdots,I_t),p_b,p_f)\mathbb{P}((I_1,\hdots,I_t)) \\
        &\qquad\qquad+ \mathbb{P}((I^{D,M})_i = 1 \mid i\in F,(I_t,\hdots,I_1),p_b,p_f)\mathbb{P}((I_t,\hdots,I_1)) \\
        &\qquad= \frac{1}{2} \mathbb{P}((I^{D,M})_i = 1 \mid i\in F,(I_1,\hdots,I_t),p_b,p_f) \\
        &\qquad\qquad+ \frac{1}{2} \mathbb{P}((I^{D,M})_i = 1 \mid i\in F,(I_t,\hdots,I_1),p_b,p_f) \\
        &\qquad= \frac{1}{2}\Big(1-(1-p_b)^{t}(1-p_f)^{\# F^N(i)}\Big) + \frac{1}{2}\Big(1-(1-p_b)^{t}(1-p_f)^{\# F^R(i)}\Big) \\
        &\qquad= 1-(1-p_b)^{t}(1-p_f)^{s} \\
        &\qquad= 1-(1-p_b)^{t}(1-p_f)
    \end{align*}
    
    Hence we have $(p_b^M)_t = (p_0^M)_t = 1-(1-p_b)^{t}$ and $(p_1^M)_t = \mathbb{P}((I^{D,M})_i = 1 \mid i\in F,I_{(t)},p_b,p_f)$ the probabilities of respectively a background pixel to appear in the merge image and of a foreground pixel to appear in the merge image. We want to find $(p_f^M)_t$ such that:
    \begin{align*}
        (p_1^M)_t = 1-(1-(p_b^M)_t)(1-(p_f^M)_t) &\iff (p_f^M)_t \triangleq \frac{(p_1^M)_t - (p_b^M)_t}{1-(p_b^M)_t} \\
        &\iff (p_f^M)_t = \frac{(1-p_b)^t(1-(1-p_f)^{s})}{(1-p_b)^t} \\
        &\iff (p_f^M)_t = 1-(1-p_f)^s = p_f
    \end{align*}
    Since $(p_b^M)_t$ and $(p_f^M)_t$ are independent of $i$, they are valid for the entire image and so we can conclude that $I^{D,M}\sim \phi_{im}(I_{(t)},(p_b^M)_t,(p_f^M)_t)$.

\end{proof}

Note that the result is almost trivial for the change for $p_b$ in the merge of frames, it is not as straightforward for the change for $p_f$. Also note that for other line orientations, due to digitisation, there might be some overlap in two successive frames over some pixels and so the theorem is in those cases not exact. Nevertheless, the impact is minor and we will still use the equivalence in those cases.

Therefore increasing $n_{f}$ and looking at $I^{D,M}$ is equivalent to degrading $I$ with some specifically increasing $p_b$, in the way given by the formula above with $t=n_f$, and with $p_f$ constant.

\subsection{A Theory for Performance Levels of A Contrario}

Similarly to the static case, we can here as well estimate a priori how the a contrario algorithm will perform should we feed it with some image $I^{D,M} \sim \phi_{im}(I_{(t)},p_b,p_f)$.

However a tricky issue must be pointed out. In the static case, when looking at values of $\widehat{\mathcal{N}_{FA}^{*}}_{w}(p_b,p_f)$, it was equivalent to consider that we were looking at a single frame generated as $\phi_{im}(I,p_b,p_f)$ or as a merge of several frames with each frame generated as $I^D \sim \phi_{im}(I,\widetilde{p_b},\widetilde{p_f})$, with in particular that $\widetilde{p_b}<p_b$ such that $1-(1-\widetilde{p_b})^t = p_b$ for some non zero integer $t$. In the dynamic case, when looking at the configuration $(p_b,p_f)$, the performance will not be the same if the image we have was degraded from a frame of the video or from a merge of several frames of the video. The reason for this is that in the merged frames the edge now appears thicker. We thus have in the dynamic case that the performance not only depends on $p_b$ and $p_f$ but also on the apparent width of the edge in the image we are considering, or equivalently to how many frames we have merged. See Figure \ref{fig: edge profile} for an illustration of the difference between the static case and the dynamic case when considering merged frames. We will still want to plot performance curves in the $(p_b,p_f)$ plane. In order to do this, we fix the background degradation parameter for the frames of the videos as $p_b^{(1)}$. We then only look at $p_b$ values that are given from merges of degraded frames with $p_b^{(1)}$. This means that we will only look at values of $p_b$ that are $(p_b^M)_t= p_b^{(t)}$ where $p_b^{(t)} = 1-(1-p_b^{(1)})^t$. Each value of $p_b^{(t)}$ defines a vertical line, or column. Each column is implicitly associated with a number merged frames, and so with a rectangle width for the edge in the merged frame we are degrading with the new parameters $(p_b^{(t)},(p_f^M)_t)$. Since, in the dynamic edge case, $(p_f^M)_t$ is constant regardless of looking at an image as a merge of several frames or generated by a single frame with larger width, we do not have to be careful in sampling $p_f$ with associating it to some width of a line. We therefore discretely fix a sampling of $p_b$ corresponding to merges of frames from an initial fixed $p_b^{(1)}$ and sample continuously $p_f$ in our study of the a priori performance of the a contrario algorithm.

\begin{figure}
    \centering
    \begin{subfigure}[t]{0.49\textwidth}
      \centering
         \includegraphics[width=\textwidth]{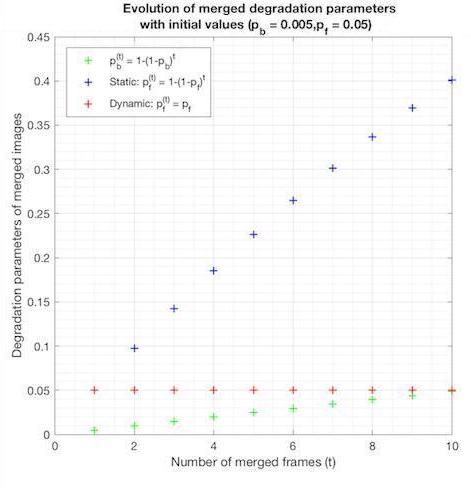}
    \end{subfigure}%
    \hfill
    \begin{subfigure}[t]{0.49\textwidth}
      \centering
         \includegraphics[width=\textwidth]{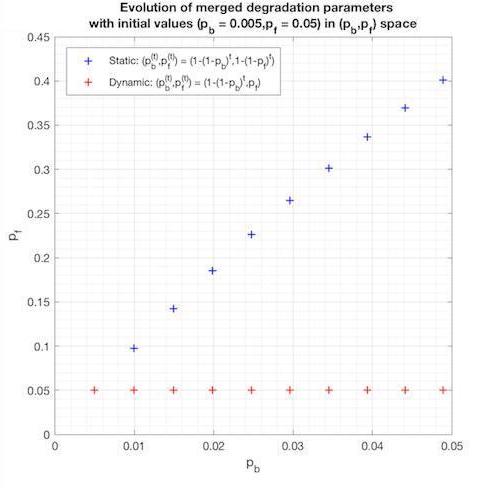}
    \end{subfigure}%
    \hfill
        \begin{subfigure}[t]{0.25\textwidth}
      \centering
         \includegraphics[width=\textwidth]{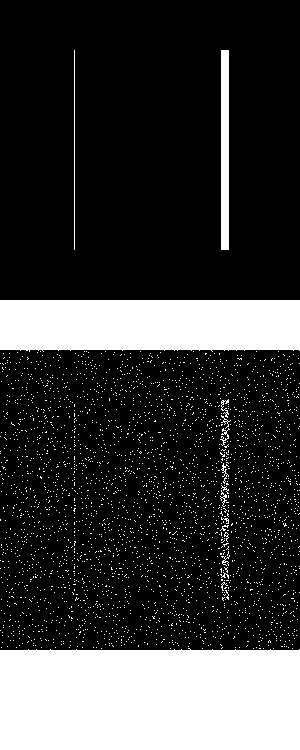}
    \end{subfigure}%
    \hfill
        \begin{subfigure}[t]{0.75\textwidth}
      \centering
         \includegraphics[width=\textwidth]{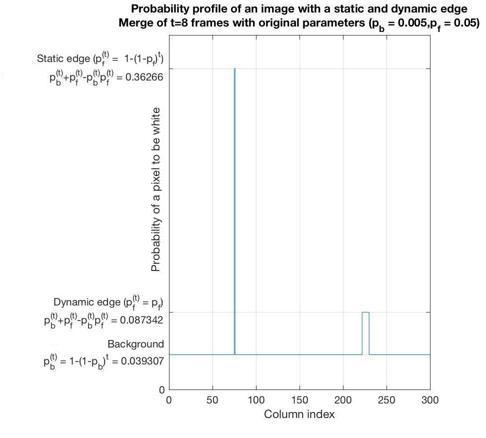}
    \end{subfigure}%
    \hfill
    
\caption{Top left: Evolution of the degradation parameters of the merged image when the number of frames increases. Top right: Same evolution but plotted in the $(p_b,p_f)$ space directly. Left middle: Clean merged image of two edges for $t=8\,\mathrm{frames}$. The one on the left is static and the one on the right is dynamic with a uniform velocity of $1\,\mathrm{pix/frame}$ orthogonal to itself. Left bottom: Degraded version of the merged image. The degradation parameters for each frame are $(p_b=0.005,p_f=0.05)$. Bottom right: Probability profile for pixel appearance conditional to the positions of the edges for a horizontal cross section located at the centre of the vertical axis. Each peak corresponds to an edge ($(p_1^M)_t$), and the rest corresponds to the background ($(p_0^M)_t = (p_b^M)_t$). Note that although the peak is highest for the static edge, it seems easier to see the dynamic edge. This is because although each pixel is less likely to appear, the area for which this is valid is much larger in the dynamic case which allows a lot of points to appear nevertheless.}
\label{fig: edge profile}
\end{figure}

We consider a single width a contrario algorithm working on candidate samples that are rectangles of length $L_e$ and width $w$. For estimating the performance of a contrario a priori, we imagine that we feed the a contrario algorithm with a merge of degraded images of some line $I^{D,M}$ or equivalently a degraded image of a thicker line with higher background density: $I^D \sim \phi_{im}(I_{(t)},p_b^{(t)},p_f)$. At some point, the a contrario algorithm will consider the candidate position that fits the line along its mid-width. This is the best position as it maximise the expected number of points that can appear in a rectangular candidate window of size $L_e\times w$. We want to know whether the algorithm will decide that there is an edge at this location or that there is not. 

As previously, two approaches are possible. The first is to compute the exact\footnote{Not entirely exact as we would use the estimator $\widehat{N_T}$ and plug it into the a contrario score.} probability distribution of the score function $\mathcal{N}_{FA}$. However this can be rather tricky. The other possibility is to take an estimate $\widehat{K_{GT}^*}$ for the number of points to appear in the candidate region placed at the true location of the edge $K_{GT}^*$ and to plug it into the a contrario score and see if it is below or above the $\varepsilon$ level. As in the static case, we will work with the second option: estimating the a contrario score by plugging into it the estimate for $k$ (and an estimate for $N_T$).

Therefore the quantity we wish to estimate is ${\mathcal{N}_{FA}^*}_w(p_b^{(t)},p_f) = N_T\mathbb{P}_B(K\ge k)$, where $K$ is a random variable giving the number of points that appear in this considered candidate position under the background only assumption $p_b^{(t)}$, and $k$ is the realisation of the random number of points $K_{GT}^*$ that actually appear in this candidate position when observing the image $I_{(t)}$ conditionally to the fact that the candidate region is at the true location. By choice of the conditioning to the background assumption, we known the distribution of $K$ and it follows a binomial $K\sim Bin(n_w,p_b)$ where $n_w\approx L_e w$. We also need to estimate $N_T$ and $K_{GT}^*$ in order to estimate the number of false alarms score.  

We sample lines in the same way as in the static case. As such we do not change our estimation of $N_T$.

Similarly, we keep our estimate for $K_{GT}^*$ as $\widehat{K_{GT}^*} = \mathbb{E}(K_{GT}^*)$. As before, we could condition this expectation to the position of the groundtruth rectangle and show that it is independent on its position and orientation up to minor negligible changes due to aliasing as long as the entire rectangle is in the image. The formulae for the statistics of $k$ are identical as in the static case is we replace $w_e$ by $t$ (the width of the line in the merged image, since we took $w_e=1$ here).

Therefore, we can compute the formula for the estimated a contrario score using the formula from the static case and replacing $w_e$ by $t$. We have in the dynamic case, just as in the static case, that $\widehat{\mathcal{N}_{FA}^{*}}(p_b^{(t)},p_f,w) = \widehat{\mathcal{N}_{FA}^{*}}_w(p_b^{(t)},p_f)$ is a function of the degradation parameters $(p_b,p_f)$. We plot the coutour levels of these functions, depending on the width parameter $w$, in Figure \ref{fig: various widths contour level eps 1 dynamic}.

Similarly to the static case, the reason why directly estimating $K_{GT}^*$ as its expectation works so well is because $K_{GT}^*$ is narrowly distributed along its mean. We have that $K_{GT}^*$ is the sum of two independent binomials, the binomials counting the white points on the edge in the merged frame and the one counting the white points in the candidate region that are not on the edge. We can apply the formula we got in the static case with $w_e = t$, and $\widetilde{w}_e = \mathrm{min}(t,w)$:

\begin{equation*}
    \sigma_{K_{GT}^*} = \sqrt{(\frac{\widetilde{w}_e}{w}n_w(p_f+p_b-p_fp_b)(1-p_f-p_b+p_bp_f))^2+(1-\frac{\widetilde{w}_e}{w}n_wp_b(1-p_b))^2}
\end{equation*}

For example, if $L=200$, $p_b^{(1)} = 0.005$, $t = 6$, $p_f = 0.07$ and $w=8$, which are typical values in our later experiments, then $n_w = 1600$, $p_b = p_b^{(6)} = 1-(1-0.005)^6\approx 0.03$, and $\sigma_{K_{GT}^*} \approx 106.3$. Which means that the spread of $k$ is essentially located within a region of size $\frac{2\sigma_{K_{GT}^*}}{n_w} \approx 13\%$ of its possible range. This shows that $k$ is, here too, essentially concentrated around its mean and thus its expectation is a good deterministic estimator.

\begin{figure}
    \centering
    \begin{subfigure}[t]{0.5\textwidth}
      \centering
        \includegraphics[width=\textwidth]{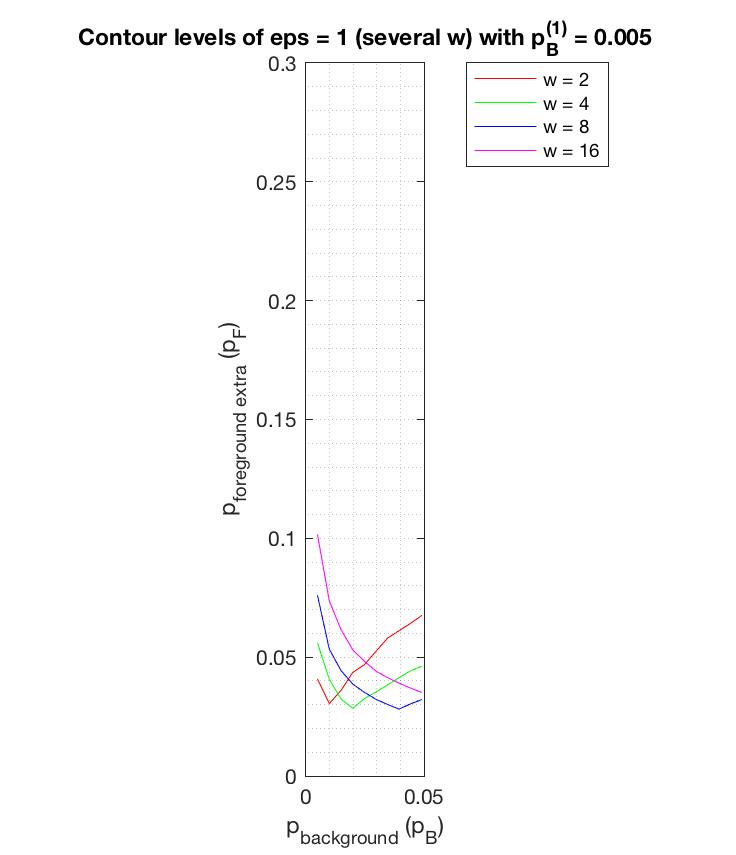}
    \end{subfigure}%
    \hfill
    \begin{subfigure}[t]{0.5\textwidth}
        \centering
        \includegraphics[width=\textwidth]{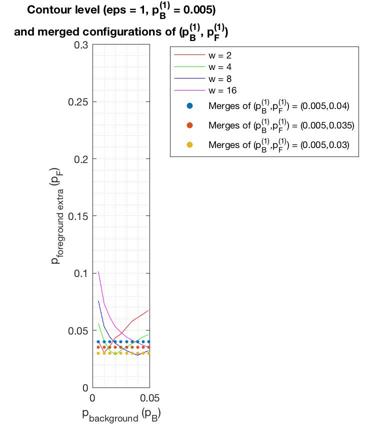}
    \end{subfigure}%
    \hfill
    
    \caption{The background probability parameter on a single frame is chosen to be ${p_b^{(1)} = 0.005}$. Recall that in the dynamic case, higher background density in the plot (or in the merged image), corresponds to an integration over more frames, so the line is thicker in the merged image. Left: Predicted a priori contour levels at level $\varepsilon=1$ of $\widehat{\mathcal{N}_{FA}^{*}}_w(p_b,p_f)$ for $w\in\{2,4,8,16\}$ in the dynamic edge case. Right: same contour levels with configurations corresponding to merged frames of videos. Fix $p_b^{(1)}=0.005$ the background degradation parameter on one frame, and consider several foreground degradation parameters on one frame $p_f^{(1)} \in \{0.03,0.035,0.04\}$. All these configurations, corresponding to the first column of dots on the figure on the right, are below each decision curve and so clearly no matter which $w$ we use, we should not be able to see anything on a per frame basis. A merge of $t$ frames corresponds to looking at the $t$-th column. Should we integrate on $t=6$ frames, and with a width of $w=8$, then we are looking at the 6-th column of dots and the blue threshold curve. Being above the curves means a hit, i.e. detection, and below means a miss. In this case, the point corresponding to the video with $p_f^{(1)} = 0.04$ is clearly above the threshold and should be easily detected, the one corresponding to $p_f^{(1)} = 0.03$ is clearly below and should not be detected and the one corresponding to $p_f^{(1)} = 0.035$ is not far from the threshold and should be in the transition or difficult zone to perceive, with some failures and successes. It is very logical that the decision curves decrease until $t = w$ where the time integration is the width of the candidate a contrario and then increases. The reason is the following. For the first frames, the more we integrate over time, the more area the line covers, and so we fill more and more the candidate rectangle with potential points. Note that unlike in the static case the foreground probability $(p_f^M)_t$ for a pixel to be white due to the foreground signal does not increase then since the pixels are not traversed by several frames, therefore it is much more likely to get many more points in the total area if the area of the foreground covers more area in that rectangle. The line fills in the candidate rectangle for $t=w$. However, for larger time integrations, the background density keeps on increasing but there are no additional points that appear on the foreground in the candidate region: the extra points that appear on the foreground in the merge are outside the candidate region and so invisible during the test of the candidate shape. Thus the decision line increases.} 
    \label{fig: various widths contour level eps 1 dynamic}
\end{figure}

\subsection{Empirical Performances: Data Generation}

We also put this theory to empirical test, by first applying the a contrario detection process by examining the corresponding human detection performance.

We work with an edge of dimensions $w_e=1$ and $L_e=200$ in the single frame video context. We then generate a random dataset of images with various degradation parameters. However, as argued previously, we had to fix an initial degradation parameter $p_b^{(1)}$. We chose to work with $p_b^{(1)} = 0.005$, as this generates sparse images that are nevertheless sufficiently dense in a video context. Our sampling of the background parameter are the merged parameters $p_b^{(t)} = 1-(1-p_b^{(1)})^t$, where the number of frames $t$ is taken to be in the range $t\in\{1,\hdots,10\}$, corresponding to a time integration of up to $0.3$ seconds. Remember that for ${p_b = p_b^{(1)}}$, we also consider that the image corresponds to a merge and so, equivalently, it is a degradation of a rectangle of length $L_e$ and width $t=1$. The foreground parameter was uniformly sampled in $[0.015,0.07]$ with a spacing of $0.005$ between samples. For each configuration of degradation parameters we generate 5 random images. Each random image is generated independently from all other images, including those with the same degradation parameters. For each image, the edge’s position is chosen at random. This means that its position and orientation are chosen randomly and uniformly. We enforce that the edge must lie entirely in the image. The images are of size $300\times300$. Therefore we sampled the midpoint of the edge according to a random uniform variable in $[100,200]\times[100,200]$ and the angle according to a uniform random variable in $[0, 2\pi]$. Given a random position for an edge described in a ground truth clean image $I$, we generate $I_{(t)}$ by translating the edge of $\pm\frac{t}{2}$ orthogonally to itself. Thus given $I$, a number of frames to merge $t$ and a foreground degradation parameter $p_f$, the degraded image added to the dataset is $\varphi_{im}(I_{(t)},p_b^{(t)},p_f)$. The total size of the dataset is 600 images.

Once the data has been generated and saved, it is not regenerated: all experiments will be done on the exact same data (although shown to humans in a different random order).

Unlike in the images of the static case, we do work on samples of $p_f$ that are much smaller, which would be considered in the previous experiment as too small compared with $p_b$ in order to bring enough points to see something. The reason for this is that since the width of the edge increases with the number of merged frames, the number of points that can appear due to the foreground probability $p_f$ linearly increases. As such, the intuition is the following: due to higher areas of coverage, although the number of points on just one line might not be enough for detection, if we multiply this by $t$ then it could now be enough to be detected, and so smaller foreground densities become relevant.

\subsection{Fixing $\varepsilon$}

We have already fixed $\varepsilon=1$ in the static case. The choice for this was the realisation that this value was reasonable and the dependency on $\varepsilon$ is very slow (somewhat logarithmic): a very large change in $\varepsilon$ does not significantly change the decision curve. Recall that $\varepsilon$ is the expected number of false alarms we tolerate. There is no reason to change it between the static case and the dynamic case. Hence, we continue to work with $\varepsilon=1$. Nevertheless, to be confident that $\varepsilon=1$ is indeed here too a good choice, we provide the predicted decision curves for other contour levels, showing that the log-dependency on $\varepsilon$ empirically holds in this case. See Figures \ref{fig: log dependency dynamic w=2} to \ref{fig: log dependency dynamic w=16}.

\begin{figure}
    \centering
    \begin{subfigure}[t]{0.4\textwidth}
      \centering
      \includegraphics[width=\textwidth]{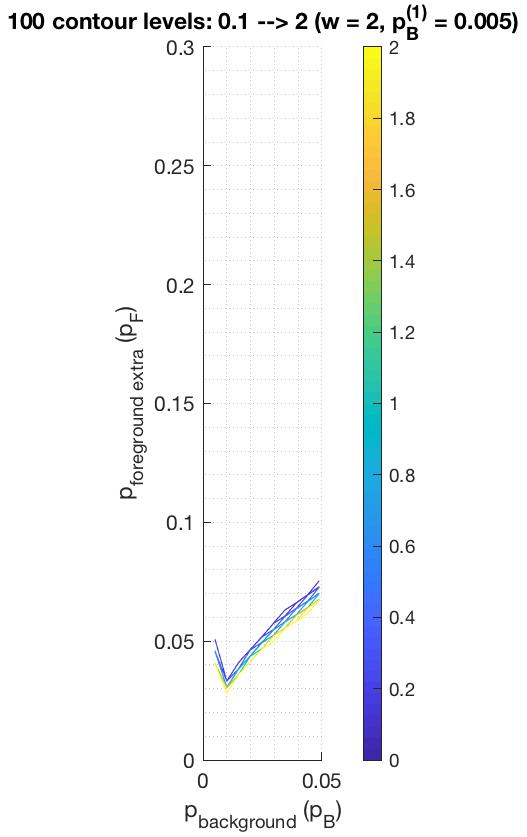}
    \end{subfigure}%
    \hfill
    \begin{subfigure}[t]{0.58\textwidth}
      \qquad\quad
      \includegraphics[width=\textwidth]{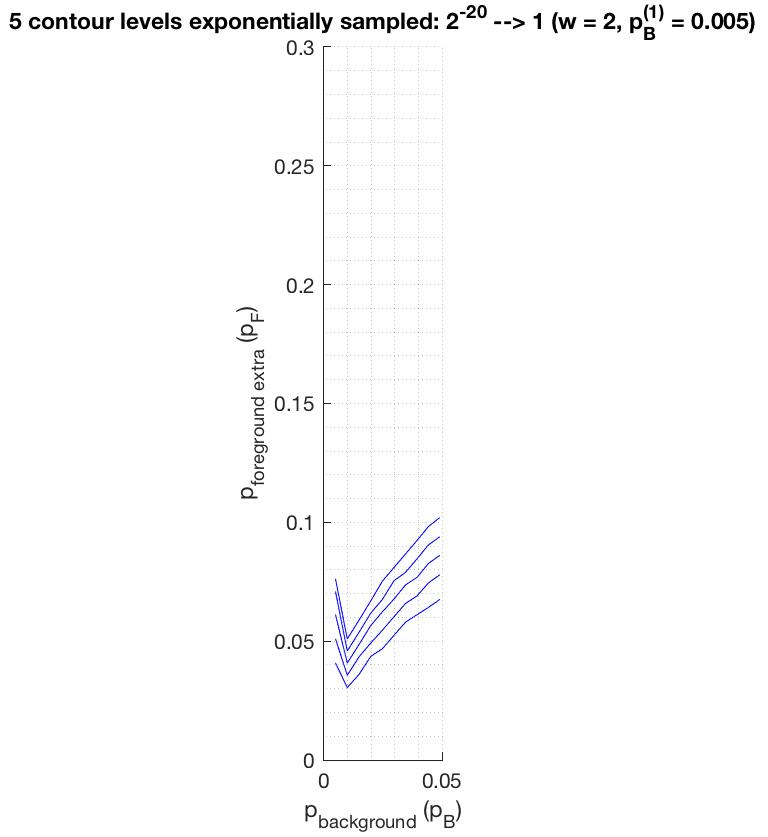}
    \end{subfigure}%
    \hfill
\caption{Left: 100 linearly sampled contour levels of $\widehat{\mathcal{N}_{FA}^{*}}_w(p_b,p_f)$ with $w=2$ in the dynamic edge case, with $\varepsilon\in[0.1,2]$. There is not much change in the position of the decision levels for such a range of $\varepsilon$. Right: 5 exponentially sampled contour levels of $\widehat{\mathcal{N}_{FA}^{*}}_w(p_b,p_f)$ with $w=2$ in the static edge case, with $\varepsilon\in\{2^{-20},2^{-15},2^{-10},2^{-5},2^{0}=1\}$. In order to get a significant change in the decision level we have to drastically change $\varepsilon$. Empirically, to get a linear displacement of the curves we need to linearly change the logarithm of $\varepsilon$ thus an empirical log-dependency to $\varepsilon$. Furthermore, taking $\varepsilon=2^{-5}$ or less means that we expect on average, over a few hundred thousand tests, to make $2^{-5}$ or fewer mistakes, which does not seem reasonable when considering humans looking at signals similar to white noise. The choice of value $\varepsilon=1$ is reasonable and crude enough since small displacements of $\varepsilon$ do not influence much the decision levels.} 
\label{fig: log dependency dynamic w=2}
\end{figure}

\begin{figure}
    \centering
    \begin{subfigure}[t]{0.4\textwidth}
      \centering
      \includegraphics[width=\textwidth]{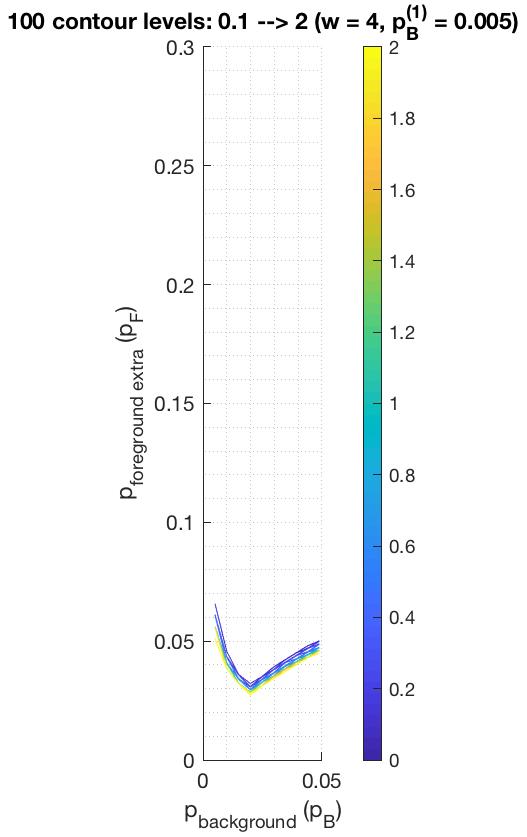}
    \end{subfigure}%
    \hfill
    \begin{subfigure}[t]{0.58\textwidth}
      \qquad\quad
      \includegraphics[width=\textwidth]{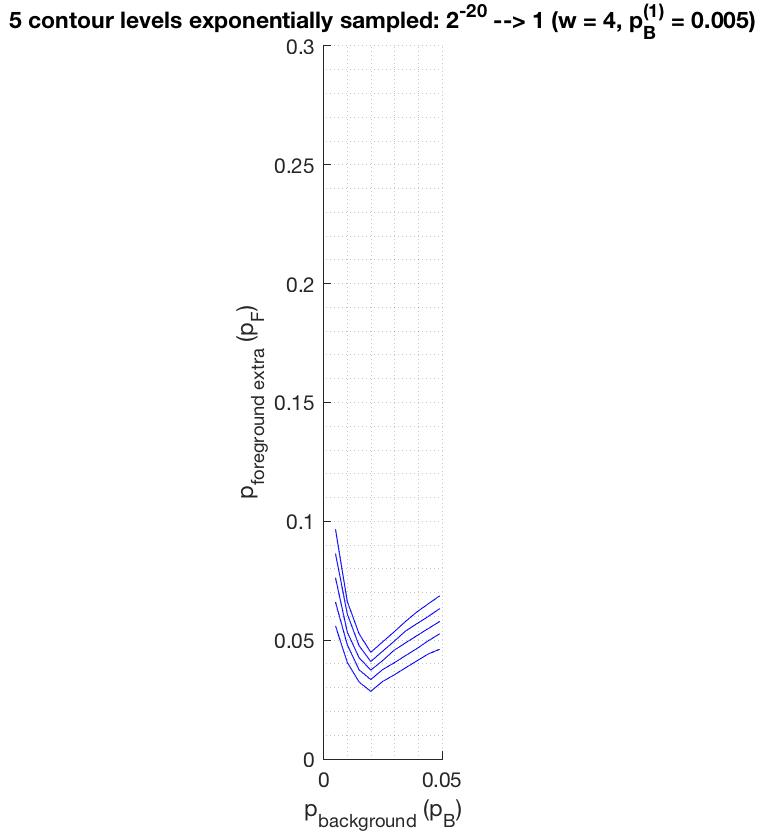}
    \end{subfigure}%
    \hfill
\caption{Left: 100 linearly sampled contour levels of $\widehat{\mathcal{N}_{FA}^{*}}_w(p_b,p_f)$ with $w=4$ in the dynamic edge case, with $\varepsilon\in[0.1,2]$. There is not much change in the position of the decision levels for such a range of $\varepsilon$. Right: 5 exponentially sampled contour levels of $\widehat{\mathcal{N}_{FA}^{*}}_w(p_b,p_f)$ with $w=2$ in the static edge case, with $\varepsilon\in\{2^{-20},2^{-15},2^{-10},2^{-5},2^{0}=1\}$. In order to get a significant change in the decision level we have to drastically change $\varepsilon$. Empirically, to get a linear displacement of the curves we need to linearly change the logarithm of $\varepsilon$ thus an empirical log-dependency to $\varepsilon$. Furthermore, taking $\varepsilon=2^{-5}$ or less means that we expect on average, over a few hundred thousand tests, to make $2^{-5}$ or fewer mistakes, which does not seem reasonable when considering humans looking at signals similar to white noise. The choice of value $\varepsilon=1$ is reasonable and crude enough since small displacements of $\varepsilon$ do not influence much the decision levels.} 
\label{fig: log dependency dynamic w=4}
\end{figure}

\begin{figure}
    \centering
    \begin{subfigure}[t]{0.4\textwidth}
      \centering
      \includegraphics[width=\textwidth]{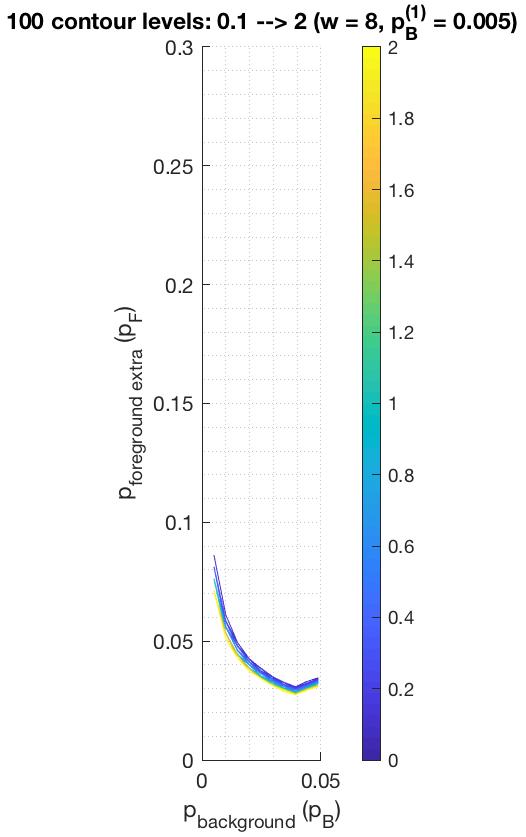}
    \end{subfigure}%
    \hfill
    \begin{subfigure}[t]{0.58\textwidth}
      \qquad\quad
      \includegraphics[width=\textwidth]{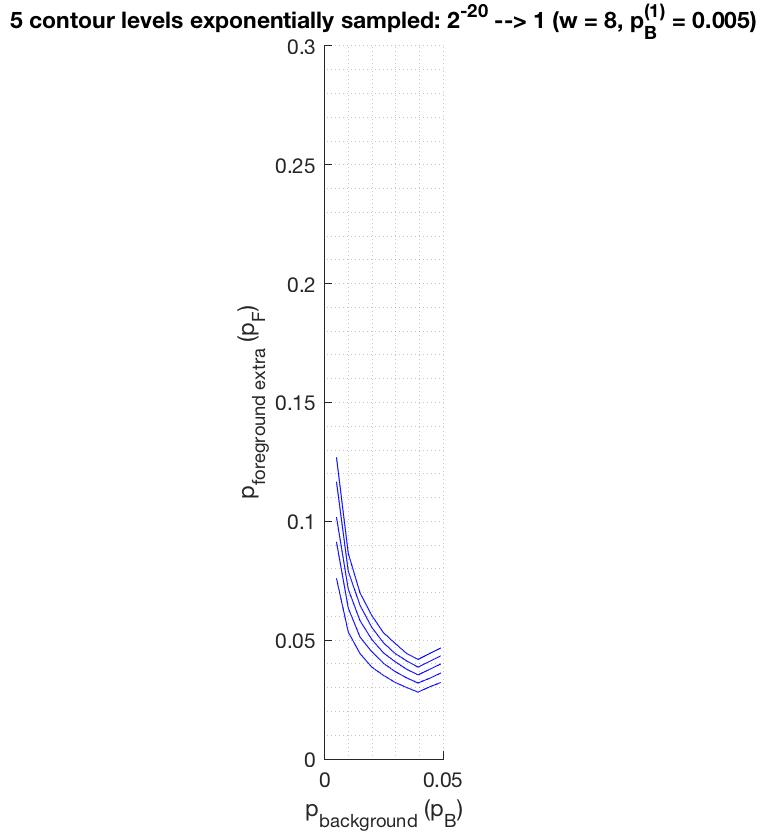}
    \end{subfigure}%
    \hfill
\caption{Left: 100 linearly sampled contour levels of $\widehat{\mathcal{N}_{FA}^{*}}_w(p_b,p_f)$ with $w=8$ in the dynamic edge case, with $\varepsilon\in[0.1,2]$. There is not much change in the position of the decision levels for such a range of $\varepsilon$. Right: 5 exponentially sampled contour levels of $\widehat{\mathcal{N}_{FA}^{*}}_w(p_b,p_f)$ with $w=2$ in the static edge case, with $\varepsilon\in\{2^{-20},2^{-15},2^{-10},2^{-5},2^{0}=1\}$. In order to get a significant change in the decision level we have to drastically change $\varepsilon$. Empirically, to get a linear displacement of the curves we need to linearly change the logarithm of $\varepsilon$ thus an empirical log-dependency to $\varepsilon$. Furthermore, taking $\varepsilon=2^{-5}$ or less means that we expect on average, over a few hundred thousand tests, to make $2^{-5}$ or fewer mistakes, which does not seem reasonable when considering humans looking at signals similar to white noise. The choice of value $\varepsilon=1$ is reasonable and crude enough since small displacements of $\varepsilon$ do not influence much the decision levels.} 
\label{fig: log dependency dynamic w=8}
\end{figure}

\begin{figure}
    \centering
    \begin{subfigure}[t]{0.4\textwidth}
      \centering
      \includegraphics[width=\textwidth]{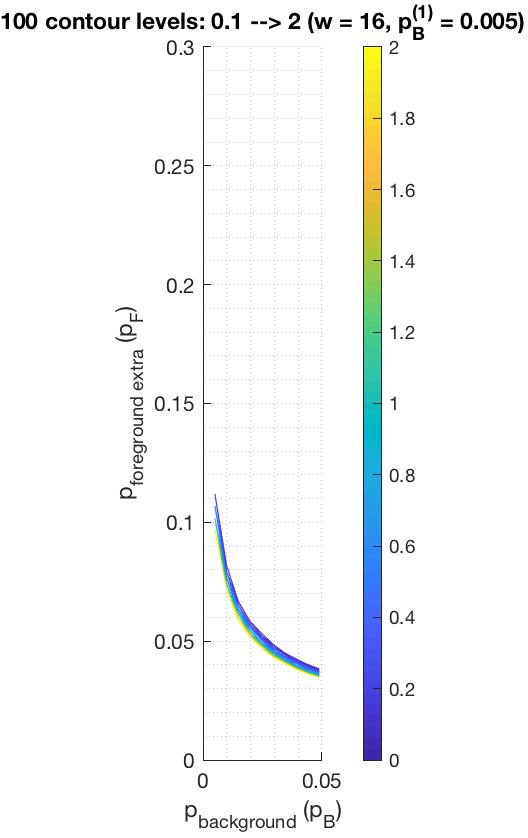}
    \end{subfigure}%
    \hfill
    \begin{subfigure}[t]{0.58\textwidth}
      \qquad\quad
      \includegraphics[width=\textwidth]{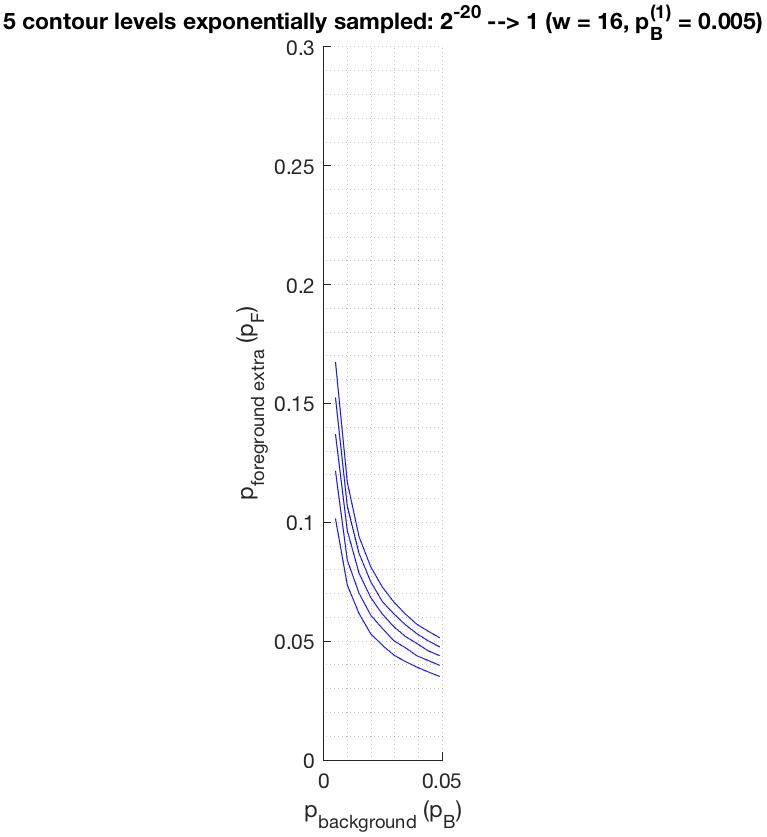}
    \end{subfigure}%
    \hfill
\caption{Left: 100 linearly sampled contour levels of $\widehat{\mathcal{N}_{FA}^{*}}_w(p_b,p_f)$ with $w=16$ in the dynamic edge case, with $\varepsilon\in[0.1,2]$. There is not much change in the position of the decision levels for such a range of $\varepsilon$. Right: 5 exponentially sampled contour levels of $\widehat{\mathcal{N}_{FA}^{*}}_w(p_b,p_f)$ with $w=2$ in the static edge case, with $\varepsilon\in\{2^{-20},2^{-15},2^{-10},2^{-5},2^{0}=1\}$. In order to get a significant change in the decision level we have to drastically change $\varepsilon$. Empirically, to get a linear displacement of the curves we need to linearly change the logarithm of $\varepsilon$ thus an empirical log-dependency to $\varepsilon$. Furthermore, taking $\varepsilon=2^{-5}$ or less means that we expect on average, over a few hundred thousand tests, to make $2^{-5}$ or fewer mistakes, which does not seem reasonable when considering humans looking at signals similar to white noise. The choice of value $\varepsilon=1$ is reasonable and crude enough since small displacements of $\varepsilon$ do not influence much the decision levels.} 
\label{fig: log dependency dynamic w=16}
\end{figure}

\subsection{Empirical Performance Versus the Predicted Performance of the A Contrario Algorithm}

The estimate $\widehat{\mathcal{N}_{FA}^{*}}(p_b^{(t)},p_f,w)$ is an a priori estimate of the score the a contrario algorithm will get when it is considering a candidate position that is in the right position on the true edge. We show the relevance of this estimator by comparing it with the empirical performance of a real a contrario algorithm working on a single width on real data.

For candidate region widths ranging in ${w\in\{2,4,8,16\}}$, we run the a contrario working on a single window width $w$ and confidence level $\varepsilon=1$. The performance of the single widths a contrario are displayed in the Figures \ref{fig: single widhts a contrario vs prediction dynamic w 2 4} and \ref{fig: single widhts a contrario vs prediction dynamic w 8 16}. We also plot the estimated contour levels of the detection performance according to our a priori estimators. The contour levels for $\varepsilon=1$ and window width $w$ fits well the empirical transition area of the a contrario with single width $w$ and confidence level $\varepsilon$. This provides an empirical confirmation that the derived a priori formula does indeed apply for the a contrario decision process.

Performance is measured like in the static case.

\begin{figure}
    \centering
    \begin{subfigure}[t]{0.49\textwidth}
      \centering
      \includegraphics[width=\textwidth]{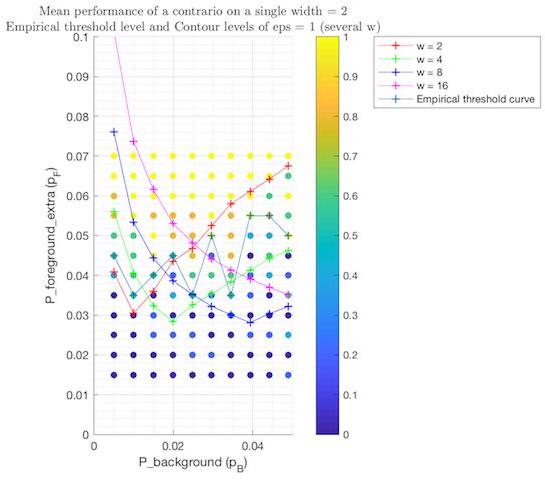}
      \caption{A contrario $w=2$}
    \end{subfigure}%
    \hfill
    \begin{subfigure}[t]{0.49\textwidth}
      \centering
      \includegraphics[width=\textwidth]{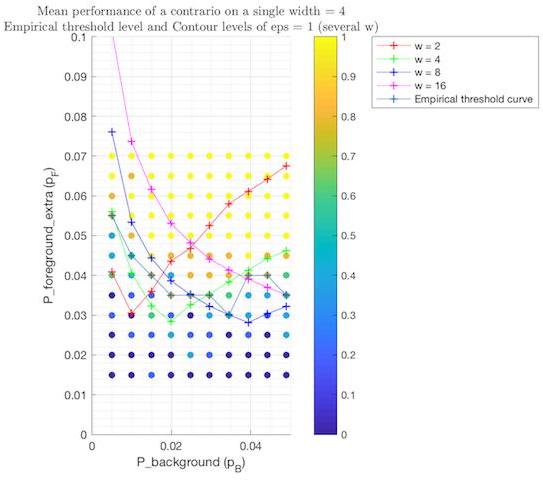}
      \caption{A contrario $w=4$}
    \end{subfigure}%
\caption{Empirical performance of several single width a contrario algorithms and their comparison with the predicted a priori mathematical model. The empirical performance fits well for the width $w=4$ but less so for width $w=2$. This is mainly due to the issue of digitisation. In this case, depending on the orientation of the line, a width of $w=2$ pixels of the line means we look on each side at pixels within distance $1$. For horizontal or vertical edges, this means we are looking at band of width $3$ pixels whereas for a perfectly diagonal edge, only the pixels that are exactly on the line will be considered. This digitisation impact is significant for $w=2$ since in this case it drastically changes the size of the candidate area and its number of pixels. The effect is less important for larger widths and can be forgotten. In our mathematical model, we did not worry about digitisation artefacts and always assumed that each candidate sample had the same area: $n_w = [Lw]-2$.} 
\label{fig: single widhts a contrario vs prediction dynamic w 2 4}
\end{figure}
\begin{figure}
    \centering
    \begin{subfigure}[t]{0.49\textwidth}
      \centering
      \includegraphics[width=\textwidth]{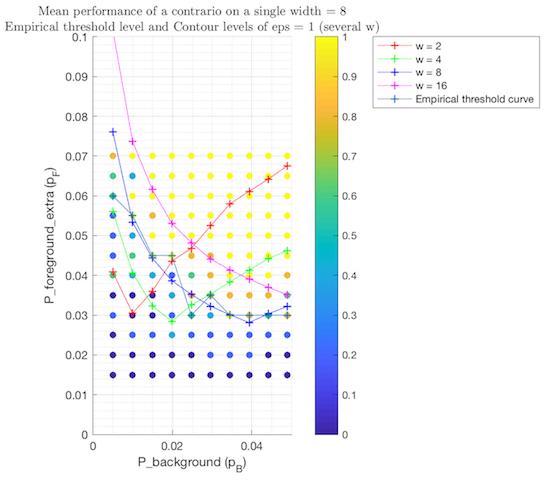}
      \caption{A contrario $w=8$}
    \end{subfigure}%
    \hfill
    \begin{subfigure}[t]{0.49\textwidth}
      \centering
      \includegraphics[width=\textwidth]{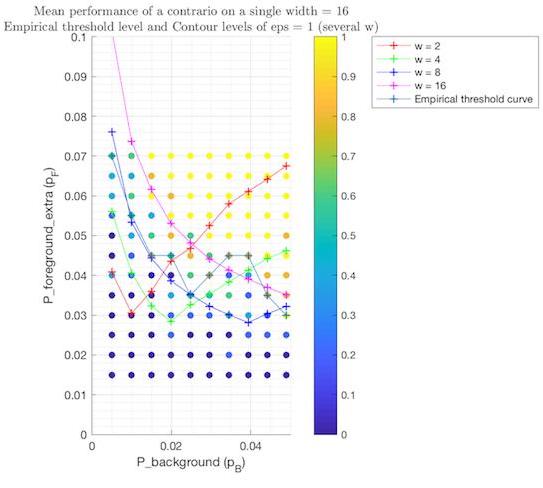}
      \caption{A contrario $w=16$}
    \end{subfigure}%
    \hfill
\caption{Empirical performance of several single width a contrario algorithms and their comparison with the predicted a priori mathematical model. The empirical performances fit well to the mathematical predictions. The performance curve in the first columns for the a contrario working on width $16$ is inaccurate since we have not sampled enough vertically $p_f$ and have not reached the domains of fully correct decisions. Thus the assumption that we have a success in the unsampled region $p_f>0.7$ is wrong and thus falsifies the empirical decision curve in the first columns of $w=16$.} 
\label{fig: single widhts a contrario vs prediction dynamic w 8 16}
\end{figure}

\subsection{Humans and A Contrario}
\label{exp2}

Next we compare human performance and the a contrario algorithm, in order to challenge the a contrario as a reasonable model of the human visual system. We devised a set of experiments as described below. The experiments detailed here were approved by the Ethics committee of the Technion as conforming to the standards for performing psychophysics experiments on volunteering human subjects.

\paragraph{Stimuli} The stimuli is the random-dot image dataset we have randomly generated and previously described. All subjects were presented with the entire dataset (600 images). The order in which they are shown is chosen at random and therefore differs between subjects.

\paragraph{Apparatus} Each subject was tested on exactly the same display. The display is the screen of a MacBook Pro retina 13-inch display, from mid 2014. Each subject was seated directly in front of the screen. The images were displayed in a well-lit room. The average distance between the subjects' eyes and the screen was about $70\,\mathrm{cm}$. The $300\times 300$ pixels image is displayed having a size of $10.4 \,\mathrm{cm}\times10.4\,\mathrm{cm}$. This translates to a pixel size of approximately $0.35\,\mathrm{mm}\times0.35\,\mathrm{mm}$. On average the visual angle for observing the image is approximately $8.5°\times8.5°$. Next to the image, we displayed a red box for user decisions (see the procedure paragraph) of size $3.7\,\mathrm{cm}\times3.7\,\mathrm{cm}$. The average visual angle for the red box is approximately $3.0°\times3.0°$. The distance between the border of the image and the border of the red box is $5.1\,\mathrm{cm}$. The average visual angle for the seperation between the image and the box is of approximately $4.1°$. See Figure \ref{fig: screenshot exp 2} for a screenshot of the display.

\paragraph{Subjects} The experiments were performed on the first author and on thirteen other subjects as previously. All subjects were the same as in the previous experiment. All the other subjects were unfamiliar with psychophysical experiments. All subjects were international students of the university, coming from various countries around the globe including China, Vietnam, India, Germany, Greece, the United States of America and France. Gender parity was almost respected with eight males and six females. The age range of the subjects was between 20 and 31 years of age. All subjects had normal or corrected to normal vision. To the best of our knowledge, no subject had any mental condition. All other subjects were unaware of the aim of our study.

\paragraph{Procedure} Subjects were told that each image they would be presented with will consist in random points but that there is an alignment of points that should form a straight edge at some random position of length equal to two thirds of the horizontal dimension of the image. They were asked to try and detect it. Unlike in the first experiment, they were told that this time the width of the line could change. As such, they were told that they were looking for a rectangular line of maximal width less than half a centimetre wide (since we could not expect them to picture precisely what $0.35$ cm is and that would only create confusion). In the entire experiment, a Matlab window covers the entire screen. In the left we display the random image the subject must work on. On the right we have the fixed red square. If the subjects could not see the line, they were instructed to click once on the red box and the next image would appear. If they did see a line, they had to click twice on the line to roughly define its extremities. Subjects were encouraged to click on locations as far as possible while still on the edge of length two-thirds of the image width. They were told that if they could only see a sub-part of the total edge that was relevant then it was all right to click on what they see as on the edge. They were strongly encouraged to not click on very small alignments or to click at the same positions. Subjects were told to try and click on the line at the centre of the rectangle but that it would still be fine if they clicked on the rectangle line but not on the centre line of it. They were also encouraged to be as precise as possible. They were strongly encouraged to click on the red box in case of doubt and were told that it is all right to click on the red box if they saw no edge. The cursor consisted in an easily visible cross thin enough so as not to obstruct visibility. For the mouse used for clicking, subjects had the choice between using the built-in touchpad of the MacBook Pro or to use a Asus N6-Mini mouse. If the subject clicked on any area not in the image or the red box, that click was dismissed. If the subject clicked once on the image and then on the red square then the next image was shown. If the subject clicked twice on the image the next image is shown. Subjects had a 10 second limit to answer for each image (to click on the red box or click twice on the image), after which the next image would be automatically shown. They were encouraged to click on the image if the detection was obvious for them. However we did not explain or give a definition of \say{obvious} to subjects. Subjects were not shown a progress bar. Subjects could not take a break once having started the experiment, but were allowed to abandon the experiment at any time, should they desire it. The time for each click made, along with the pixel in the image corresponding to a click on the random image were saved.

\paragraph{Discussion} Performance of individuals was measured as previously as a $0-1$ score where $0$ would be a decision to not see, to run out of time or to make an incorrect detection. On the other hand, $1$ is given for a correct decision. We had to allow a high tolerance for the clicking as subjects tended to not be particularly accurate in the position of the click compared to what they saw. A correct decision had to satisfy the three following tolerances. The angle of the line between both clicks had to be within a range of $0.1\,\mathrm{rad}$ of the angle of the line (near parallelism test). The maximum distance between a click and its orthogonal projection onto the mid line of the true line had to be smaller or equal to half the width of the line plus $15$ pixels (distance to line test). Recall that here the width of the line is the number of merges we are considering since we are considering merges of frames of a dynamic edge. Each click had to be no further away than the mid point of the true edge than by half the diagonal of the edge plus a tolerance of $20$ pixels. The experimental results are plotted in Figures \ref{fig: dynamic perf 0 and all} and \ref{fig: dynamic perf 1 to 4} to \ref{fig: dynamic perf 13}. First, the predicted (and empirical) threshold curves of single widths a contrario seem relevant to the performance of humans. However, the results seem more noisy than in the static case. Subjects seem to unpredictably fail too much for high $p_b$ and low $p_f$ around the predicted decision level. This may be due to the fact that this is quite difficult and in such a long experiment subjects tend to slack off and give up a bit too easily on hard yet possible examples. It could also be a bias of our model, being based on some idealised approximations. Furthermore, it could be due to the inherent multiscale  image interpretation of humans. Thus, for $p_b \approx 0.4$ all widths $w\in\{2,4,16\}$ predict a failure when $p_f\approx 0.035$, whereas $w=8$ predicts a success, hence if we too would do some multiscaling then having $3$ decisions predicting that this is just random whereas one saying that this is not should lower the importance of the $w=8$ decision and cause us not to detect. Finally, it could also simply be due to the fact that we have only 5 examples per probability configuration, therefore the standard deviation of the decisions is quite high for each probability configuration. As previously, the performance differs between individuals. Not all decision thresholds are exactly at the same position. Most seem to fit well around the threshold line of $w=8$ for low and moderate $p_b$ but some tend to become closer to the lines $w=4$ and $w=16$ for high $p_b$. While this can be due to intrinsic differences between each individual, it is also linked to the fact that not all subjects made the same mental effort to detect the lines. In particularly, those who decided faster to detect or reject seemed to fit better to a higher width a contrario than those who took their time. This is logical as those who take their time can focus better and use the multiscale of the brain rather than solely rely on their \say{default vision} when looking at such data. Nevertheless, the results seem to nicely confirm that the choice of $w=8$ pixels (or $\theta=0.23°$ in terms of visual angle) as a single width a contrario process on static data is reasonable and consistent between the static edge and dynamic case. As such, it validates our choice of width for the single width a contrario algorithm.

\begin{figure}
    \centering
    \includegraphics[width=\textwidth]{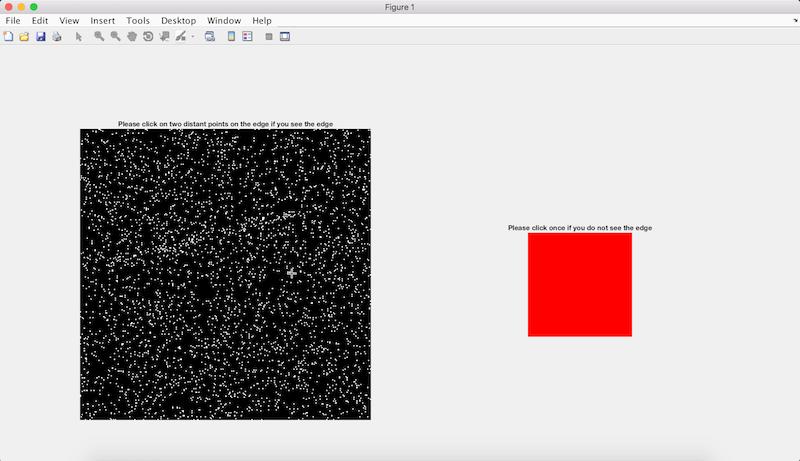}
    \hfill
    \caption[Short version]{Screenshot of the display.}
    \label{fig: screenshot exp 2}
\end{figure}

\begin{figure}
    \centering
    \begin{subfigure}[t]{0.49\textwidth}
      \centering
         \includegraphics[width=\textwidth]{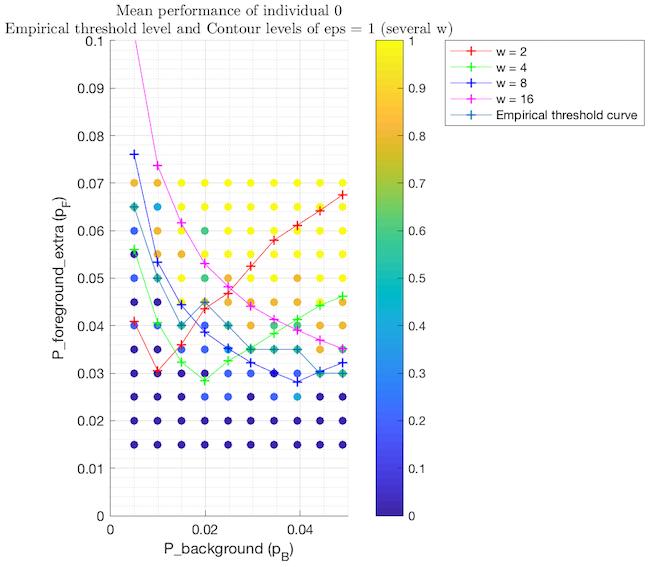}
    \end{subfigure}%
    \hfill
    \begin{subfigure}[t]{0.49\textwidth}
      \centering
         \includegraphics[width=\textwidth]{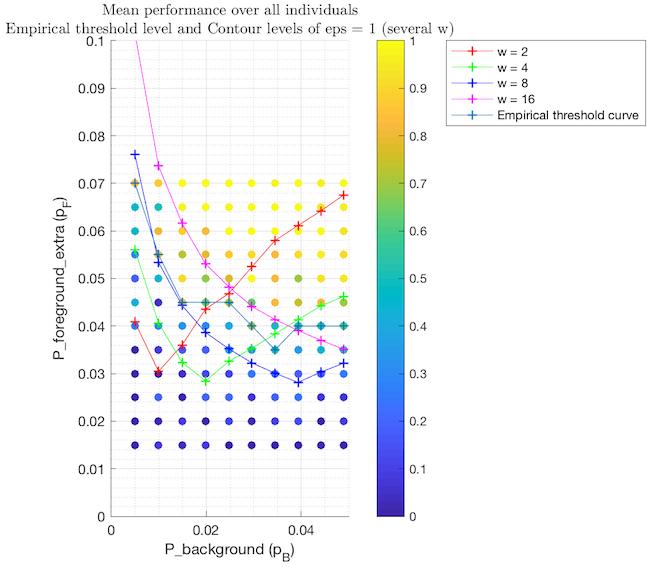}
    \end{subfigure}%
    \hfill
\caption{Empirical performance of humans on static images of a dynamic edge. Left: performance of a single person. Right: average performance over all subjects. Each dot corresponds to a pair of degradation parameters, each corresponding to $5$ samples. The colour corresponds to the averaged score on those samples per degradation parameter. A yellow dot corresponds to perfect success in recovering the line whereas as a dark blue dot corresponds to systematic failure to recover the line. The transition zone seems to fit the transition zones for a contrario. The static experiment gave us a choice of $w=8$ as a good candidate for width. This choice is consistent with the results on the dynamic edge as the empirical decision level seems here to also fit well the predicted decision level of the a contrario algorithm working on width $w=8$.}
\label{fig: dynamic perf 0 and all}
\end{figure}

\begin{figure}
    \centering
    \begin{subfigure}[t]{0.49\textwidth}
      \centering
         \includegraphics[width=\textwidth]{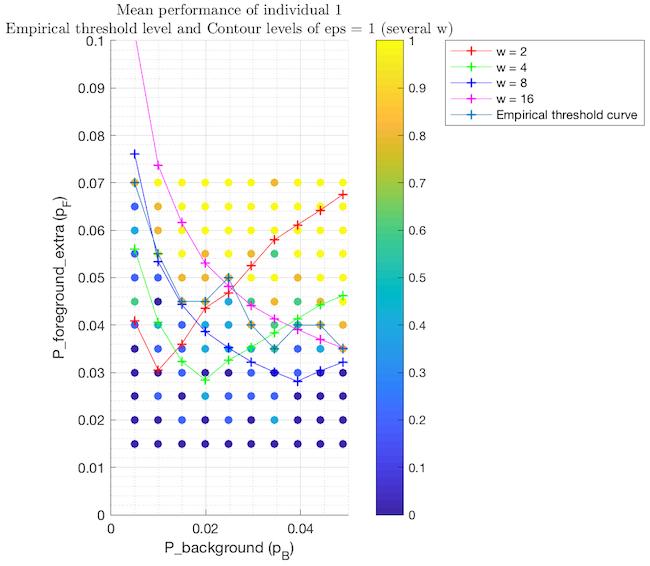}
    \end{subfigure}%
    \hfill
    \begin{subfigure}[t]{0.49\textwidth}
      \centering
         \includegraphics[width=\textwidth]{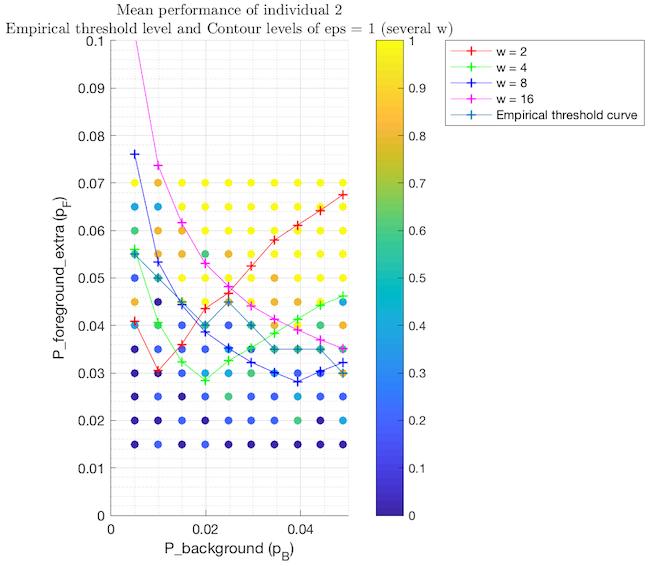}
    \end{subfigure}%
    \hfill
    \begin{subfigure}[t]{0.49\textwidth}
      \centering
         \includegraphics[width=\textwidth]{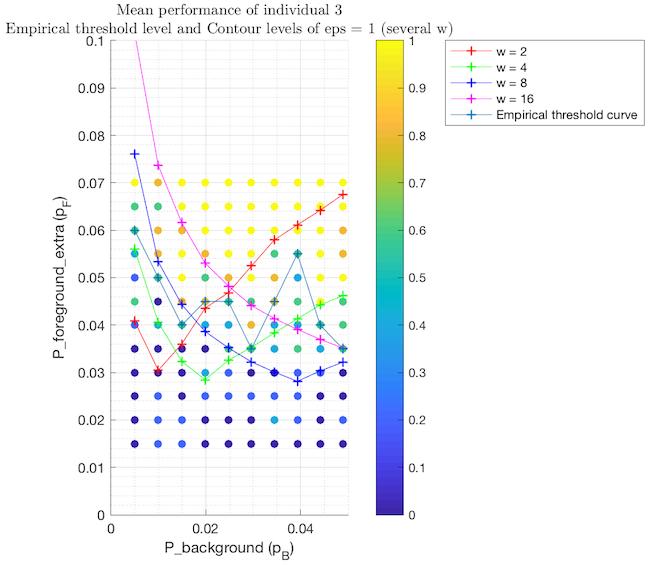}
    \end{subfigure}%
    \hfill
    \begin{subfigure}[t]{0.49\textwidth}
      \centering
         \includegraphics[width=\textwidth]{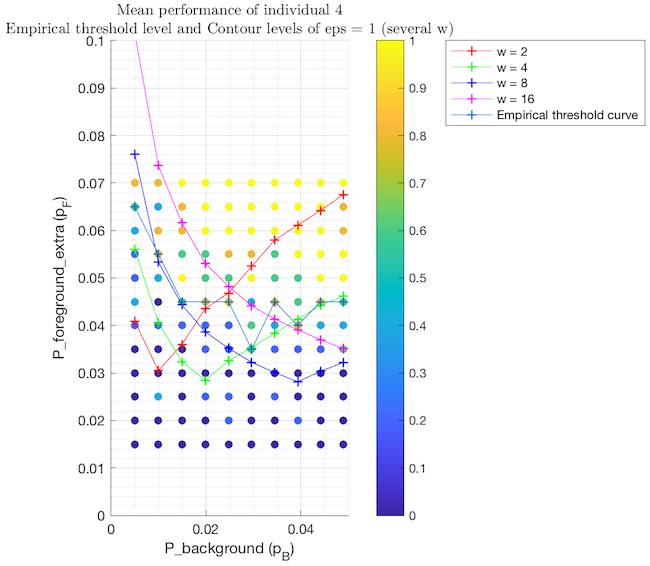}
    \end{subfigure}%
\caption{Empirical performance of humans on static images of a dynamic edge for subjects 1 to 4. The empirical performance of most subjects lies close to the predicted performance of a contrario with $w=8$. However, for higher $p_b$, the decision threshold seems to lie a bit higher than the prediction. This is due to several facts. First humans are not exactly a contrario machines. Second, they are inevitably multiscale, and thus for high values of $p_b$ the performance will inevitably lie somewhat higher than the absolute prediction. Third, the experiments were quite long and tiring and all subjects necessarily made the effort to see in the difficult cases thus leading to a higher decision level. Nevertheless, the choice of $w=8$ seems relevant and confirmed in the dynamic case as well.}
\label{fig: dynamic perf 1 to 4}
\end{figure}
\begin{figure}
    \centering
    \begin{subfigure}[t]{0.49\textwidth}
      \centering
         \includegraphics[width=\textwidth]{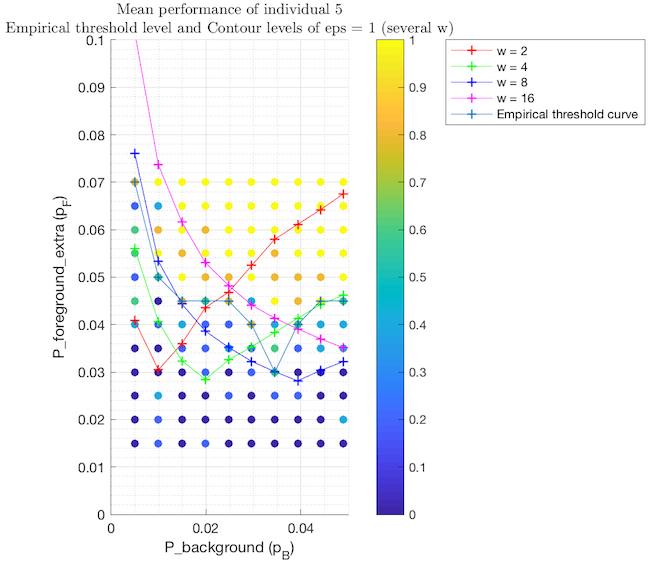}
    \end{subfigure}%
    \hfill
    \begin{subfigure}[t]{0.49\textwidth}
      \centering
         \includegraphics[width=\textwidth]{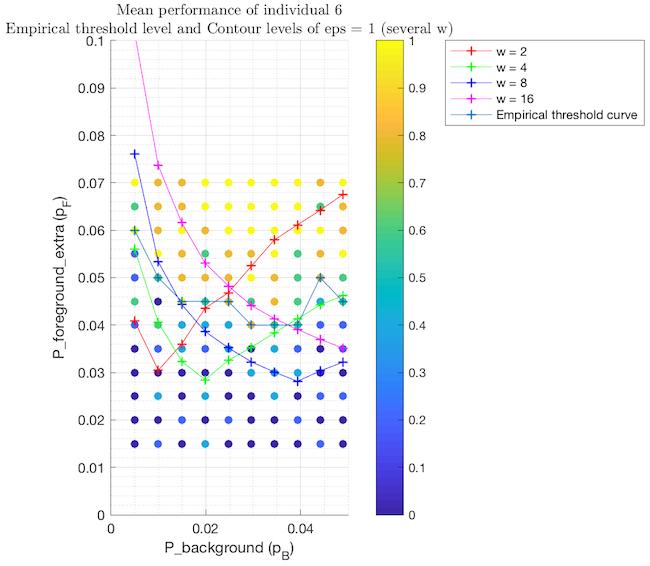}
    \end{subfigure}%
    \hfill
    \begin{subfigure}[t]{0.49\textwidth}
      \centering
         \includegraphics[width=\textwidth]{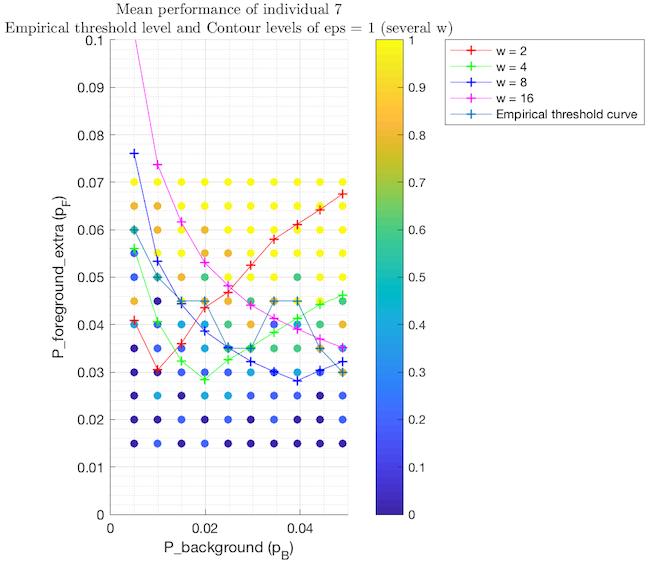}
    \end{subfigure}%
    \hfill
    \begin{subfigure}[t]{0.49\textwidth}
      \centering
         \includegraphics[width=\textwidth]{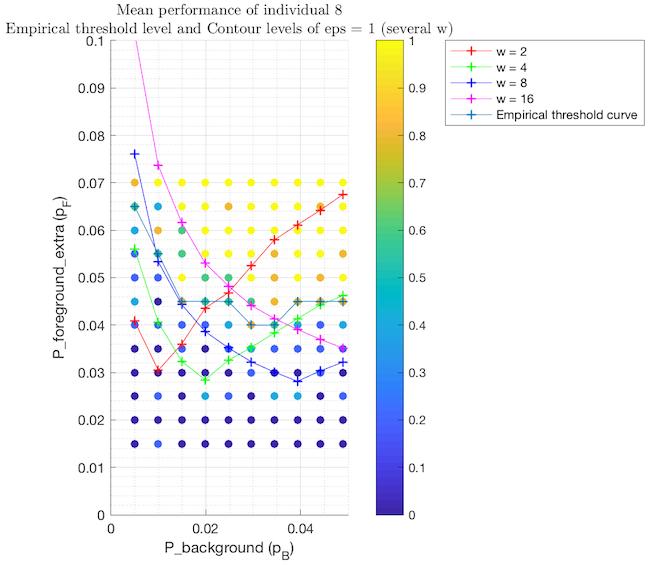}
    \end{subfigure}%
\caption{Empirical performance of humans on static images of a dynamic edge for subjects 5 to 8. The empirical performance of most subjects lies close to the predicted performance of a contrario with $w=8$. However, for higher $p_b$, the decision threshold seems to lie a bit higher than the prediction. This is due to several facts. First humans are not exactly a contrario machines. Second, they are inevitably multiscale, and thus for high values of $p_b$ the performance will inevitably lie somewhat higher than the absolute prediction. Third, the experiments were quite long and tiring and all subjects necessarily made the effort to see in the difficult cases thus leading to a higher decision level. Nevertheless, the choice of $w=8$ seems relevant and confirmed in the dynamic case as well.}
\label{fig: dynamic perf 5 to 8}
\end{figure}
\begin{figure}
    \centering
    \begin{subfigure}[t]{0.49\textwidth}
      \centering
         \includegraphics[width=\textwidth]{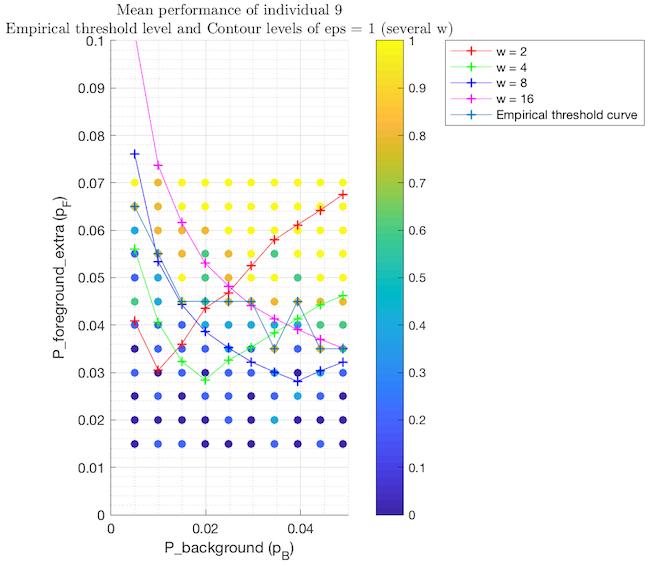}
    \end{subfigure}%
    \hfill
    \begin{subfigure}[t]{0.49\textwidth}
      \centering
         \includegraphics[width=\textwidth]{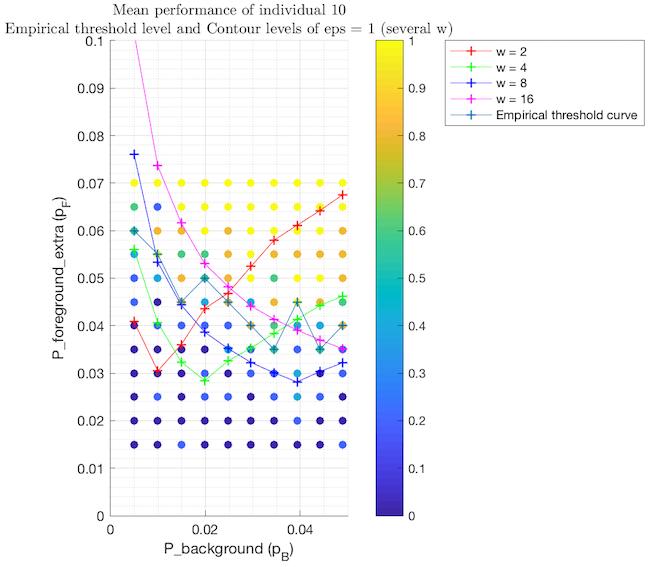}
    \end{subfigure}%
    \hfill
    \begin{subfigure}[t]{0.49\textwidth}
      \centering
         \includegraphics[width=\textwidth]{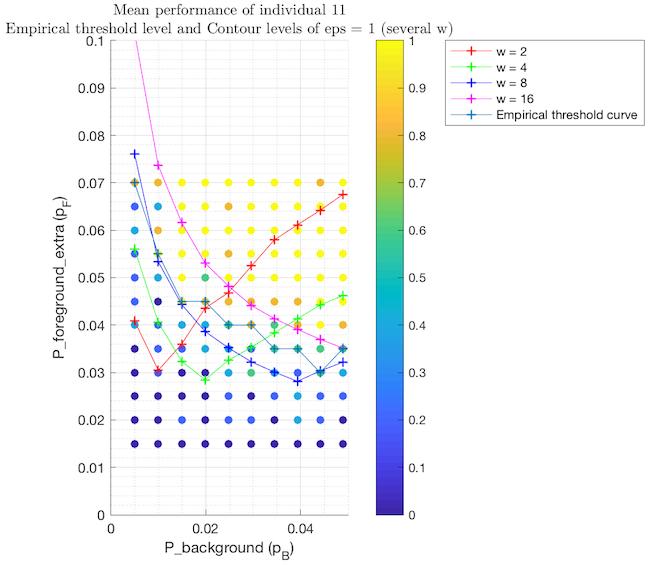}
    \end{subfigure}%
    \hfill
    \begin{subfigure}[t]{0.49\textwidth}
      \centering
         \includegraphics[width=\textwidth]{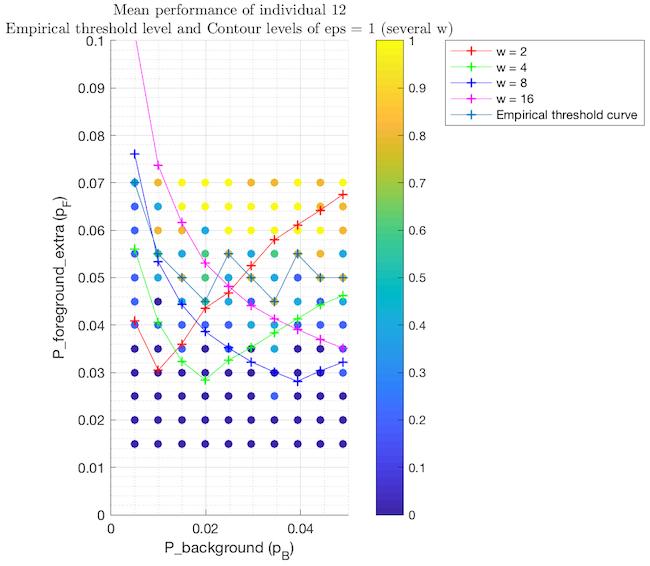}
    \end{subfigure}%
\caption{Empirical performance of humans on static images of a dynamic edge for subjects 9 to 12. The empirical performance of most subjects lies close to the predicted performance of a contrario with $w=8$. However, for higher $p_b$, the decision threshold seems to lie a bit higher than the prediction. This is due to several facts. First humans are not exactly a contrario machines. Second, they are inevitably multiscale, and thus for high values of $p_b$ the performance will inevitably lie somewhat higher than the absolute prediction. Third, the experiments were quite long and tiring and all subjects necessarily made the effort to see in the difficult cases thus leading to a higher decision level. Nevertheless, the choice of $w=8$ seems relevant and confirmed in the dynamic case as well.}
\label{fig: dynamic perf 9 to 12}
\end{figure}
\begin{figure}
    \centering
    \begin{subfigure}[t]{0.49\textwidth}
      \centering
         \includegraphics[width=\textwidth]{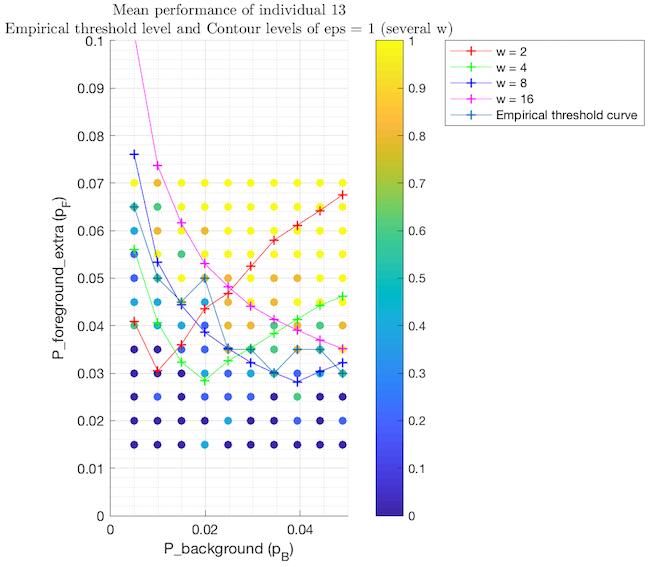}
    \end{subfigure}%
    \hfill
\caption{Empirical performance of humans on static images of a dynamic edge for subject 13. The empirical performance of most subjects lies close to the predicted performance of a contrario with $w=8$. However, for higher $p_b$, the decision threshold seems to lie a bit higher than the prediction. This is due to several facts. First humans are not exactly a contrario machines. Second, they are inevitably multiscale, and thus for high values of $p_b$ the performance will inevitably lie somewhat higher than the absolute prediction. Third, the experiments were quite long and tiring and all subjects necessarily made the effort to see in the difficult cases thus leading to a higher decision level. Nevertheless, the choice of $w=8$ seems relevant and confirmed in the dynamic case as well.}
\label{fig: dynamic perf 13}
\end{figure}

\clearpage
\newpage
\section{Video Experiments: Getting the Time Integration}
\label{sec: video overview and AC}

In the previous Sections we modelled the human perceptual system as a pipeline of a time integrator considered as a merge of frames over some fixed time followed by a spatial a contrario algorithm on images for detecting interesting structures that are unlikely under some random assumptions. We devised two experiments on static data (images) considered as the output of the time integrator and analysed the model for the a contrario black box. We found that humans seem to perform similarly on static data to a single-width a contrario algorithm of width $w=8$ which approximately corresponds to a visual angle of $0.23°$. In this subsection, we wish to compare humans with our model on dynamic data and in particular challenge our model of time integration.

In the previous experiments, we displayed static data (images) of lines. These images could be understood as an integration of a certain number of consecutive frames of a video of a line that was either static (no movement) or dynamic (movement). If an image of the previous experiments can be considered as merges of $t$ frames each sampled with a background probability of $p_b$ and a foreground probability parameter of $p_f$, then we can show the same subjects a video generated with parameters $(p_b,p_f)$. For a fixed value of $p_b$, we can vary $p_f$ and study the performance of humans on these videos. From that we can extrapolate a threshold level of $p_f$ for that background level $p_b$, below which humans cannot see the line in a video and over which they can. We can then compare with our previous experiment on static data to find with which number of merged frames this threshold corresponds to. For instance, if we are looking at configurations of dynamic edges (second experiment), then we look if there is a column of the performance in the $(p_b,p_f)$ space that fits well the performance level on video data.  This allows us to choose a time for integration that fits to human behaviour. Finally, we can do the same with the static edge by comparing the performance on the dynamic data and on the static data considered as merges of frames and look at which column performance in the static data experiment fits well the performance in the dynamic experiment and thus get a time integration. Both time integrations should be the same or within the same range and thus having both independent experiments getting the same result would validate the choice of $\Delta t$ as time integration in our human perception model.

\subsection{Defining Video Data}

We generate two video datasets for the following experiments. We chose to fix the background parameter to $p_b = 0.005$ for each frame of each video. This choice is motivated by our choice in the second experiment to work with a fixed background parameter for the single frame $p_b^{(1)} = 0.005$, so that we can use the results of this previous experiment and compare it with those of the new ones. Similarly, we sampled $p_f$ according to its sampling in the second experiment: uniform sampling in the range $[0.015,0.07]$ with a separation of $0.005$ between two configurations. The dimension of the edge is the same. We fix its width $w_e = 1$ and length $L=200$. Each frame of the videos are generated independently (conditional to the position of the edge). All frames have the same size $300\times300$. Since we have fewer degradation configurations than in the static experiments, we can afford to sample more videos per configuration. For this reason, for each degradation configuration we generated $20$ videos.

\paragraph{Static edge} In the first data set, we consider the case of the static line. We first consider data with a line. The edge's position is chosen at random. This means that its position and orientation is chosen randomly and uniformly. We enforce that the edge must lie entirely in the image. Therefore we sampled the midpoint of the edge according to a random uniform variable in $[100,200]\times[100,200]$ and the angle according to a uniform random variable in $[0,2\pi]$. The duration of the video is $10s$ and with a frame rate of $FPS = 30\,\mathrm{frames/second}$. However, we should not sample just once the position of the line and display such a long video. If we did do so, then correlations between decisions at the exact same position but separated by a longer time will influence the perception of the line: higher level memory of human subjects necessarily comes into play. Since we are interested in low level memory (maybe even at the level of neuron activations) we do not desire to have this additional factor influencing the data. In order to avoid this issue, we make the edges jump. Every $n_{\mathrm{jump}}$ frames, the position of the line is re-sampled randomly and independently. We estimate that the time for integration should be around 0.27 seconds (or $8$ frames) from simple tests on the first author of this report. If this is the case, then it takes 0.27 seconds to gather information in order to start seeing a line. However, if the line jumps every 0.27 seconds then in proportion, we would be able to see and decide of the presence of a line for only the last frame for every 8 frames. This would mean that for the large majority of the duration of the video, we do not see anything and only sporadically see a line for a very short time of one frame, i.e. $\frac{1}{FPS} \approx 0.03s$. This would bias us to decide that we do not see a line. In order to overcome this issue, we decided to sample the position of the line every $0.53s$ which corresponds to 16 frames, the double of our estimate on the first author of this report. This would mean that if the time integration is indeed around $0.27s$, then for half of the video we do not see the line (we are integrating the information and do not see yet) and for the other half we do see: hence we have overcome the bias. Furthermore, this duration is a good compromise between a short display to avoid higher order memory and showing the data long enough for us to decide. This generates for us a dataset of degraded videos of a static edge. We also generate the same number of videos of degraded videos without any edge: we generate just as many white noise videos with parameter $p_b$. The total size of the dataset is 480 videos.

\paragraph{Dynamic edge} In the second data set, we consider the case of the dynamic line with the same kind of movement as we assumed to have had in the second experiment. Namely that the movement is uniform, orthogonal to the length direction and at a speed of $1\,\mathrm{pixel/frame}$. We first consider data videos with a line. The edge's initial position is chosen at random. This means that its initial position and orientation is chosen randomly and uniformly. We also sample independently the direction of translation with probability $\frac{1}{2}$ for both opposite directions. The initial mid point of the edge is sampled according to a random uniform variable in $[100,200]\times[100,200]$ and the angle according to a uniform random variable in $[0,2\pi]$. The duration of the video is $10s$ and with a frame rate of $FPS = 30\,\mathrm{frames/second}$. Similarly to the static edge data set, we actually make the edge jump regularly: every $16$ frames (corresponding to approximately $0.53$ seconds) the initial position and direction of translation are re-sampled in order to avoid long term memory influence on the decision. This generates for us a dataset of degraded videos of a dynamic edge. As above, we also generate the same number of videos of degraded videos without any edge: we generate just as many white noise videos with parameter $p_b$. The total size of the dataset is here too 480 videos.

Once the data has been generated and saved, it is not regenerated: all experiments will be done on the exact same data (although shown in a different random order).

We will work on both data sets independently to get the time integration (or equivalently a number of frames of integration on a frame-rate of $FPS=30\,\mathrm{frames/second}$ video), and see that the two sets of experiments yield the same result and thereby validating the result. The dynamic edge video experiment was done in practice after the static edge experiment since we found it slightly easier on small tests to perceive the static edge than the dynamic edge. It was also done shortly after, on the same day, in order to minimise the number of times the subjects had to come and since both experiments feel the same for the human subjects, unlike the static image experiments from the previous sections.

\subsection{A Contrario}

We ran several single width a contrario algorithms with width $w=8$ but each working with a different time integration $n_{f} \in\{1,\hdots,10\}\,\mathrm{frames}$. This allows us to get a performance measure for a contrario algorithms of $(w=8,n_{f})$ and thus later compare with human performance on dynamic data and decide which time integration corresponds to humans. Note that proceeding in such a way might seem redundant with the empirical a contrario estimations on static data since we said we considered our static data as merged frames, and a contrario only works on merged frames. This is not entirely true since we did not necessarily sample our data in the static experiments exactly according to how we sampled them in the video context. For instance, in the static case, we needed not worry about how many frames we are considering since the edge does not move (which is not the case in the dynamic case), and thus sampling uniformly the space allowed us to better understand human and a contrario behaviour and thus tune the later experiments. For this reason, we reran the a contrario algorithm on these videos.

In order to save computation time, each a contrario will only test one merge of frames for each videos. This might seem very underwhelming since at each frame of the video the a contrario performs the time integration and tests it, thus on a $10s$ video of $30$ frames per second there are about $300$ tests on merged frames. Furthermore, the edge jumps between positions thus looking at only one merge only looks at one of these positions of the edge. Humans might not be able to see at each instant but over the range of the video there might be times where they see and this will lead them to deciding that this is edge data. For simplicity and in order to save computation time, we did not do this for a contrario and a contrario is only tested on the first $n_{f}$ merge of frames, where $n_{f}$ is the time integration assumed for the a contrario algorithm.

In order to estimate performance, we should normally look at both the true positive rate of the algorithms with respect to $p_f$ but also the false alarm rate, and design a score combining both rates. Indeed, if we only look at one of these values and discard the other, we could end up rating some performance as excellent when in fact it is terrible. For instance imagine an algorithm that always finds the edge but when there isn't any always return a false alarm. Then it detects all edge data, thus $100\%$ true positive rate, but also a $100\%$ false alarm rate, which is bad: the true decisions are not reliable. Thus if we omit the false alarm rate and only look at the true positive rate we get an incorrect evaluation of the performance.  

Empirically, we found that the false alarm rate is low, close to 0, as shown in the left column of Figures \ref{fig: video static AC} (static edge) and \ref{fig: video dynamic AC} (dynamic edge), which was expected by construction for the a contrario algorithm since we chose $\varepsilon$ to be low. Also, it is consistent between all a contrario algorithms and will be similar to the false alarm rates of humans (see Sections \ref{sec: exp3.1} and \ref{sec: exp3.2}). Therefore the false alarm rates do not provide any information for comparison between algorithms or with humans, in this case. We thus chose to use the plain true positive rate as our performance score. For each probability configuration $p_f$, we can look at the true positive rate on all edge samples of that configuration. Doing this over all sampled $p_f$ yields a performance score vector or curve. Each algorithm will provide its performance curve and similarly, each human subject will also yield its own performance curve. We can then compare human performance to algorithmic performance by comparing these vectors or curves. This is somewhat equivalent to fitting the true positive vector performance of humans to the columns of the previous performances on the static data (up to dilation in the static case since $p_f$ almost linearly increases with the number of frames of the merge). We first ran the a contrario decision algorithms on both datasets. The score curves (the true positive curves as a function of $p_f$) are displayed in the right column of Figures \ref{fig: video static AC} (static edge) and \ref{fig: video dynamic AC} (dynamic edge). 

\begin{figure}
    \centering
    \begin{subfigure}[t]{0.5\textwidth}
      \centering
         \includegraphics[width=\textwidth]{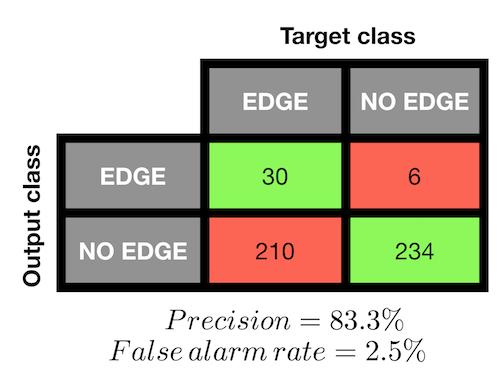}
         \caption{$n_{f} = 1$ frame}
    \end{subfigure}%
    \hfill
    \begin{subfigure}[t]{0.5\textwidth}
      \centering
         \includegraphics[width=\textwidth]{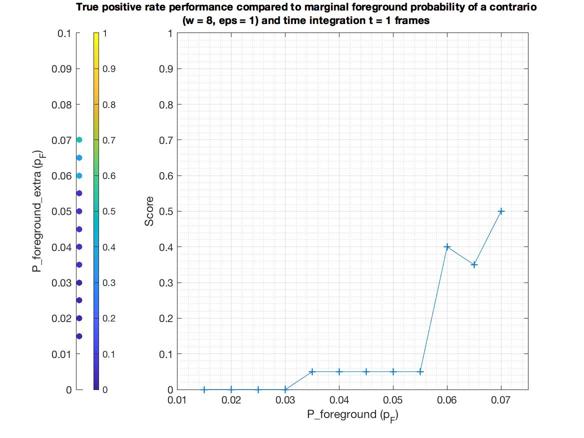}
         \caption{$n_{f} = 1$ frame}
    \end{subfigure}%
    \hfill
    \begin{subfigure}[t]{0.5\textwidth}
      \centering
         \includegraphics[width=\textwidth]{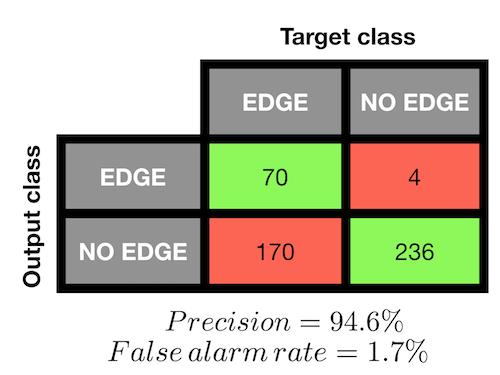}
         \caption{$n_{f} = 2$ frames}
    \end{subfigure}%
    \hfill
    \begin{subfigure}[t]{0.5\textwidth}
      \centering
         \includegraphics[width=\textwidth]{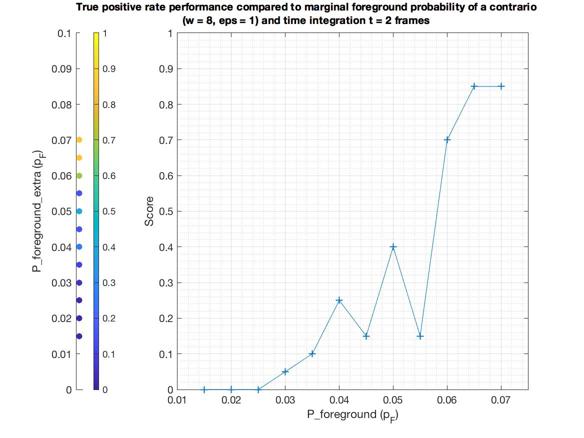}
         \caption{$n_{f} = 2$ frames}
    \end{subfigure}%
    \hfill
    \begin{subfigure}[t]{0.5\textwidth}
      \centering
         \includegraphics[width=\textwidth]{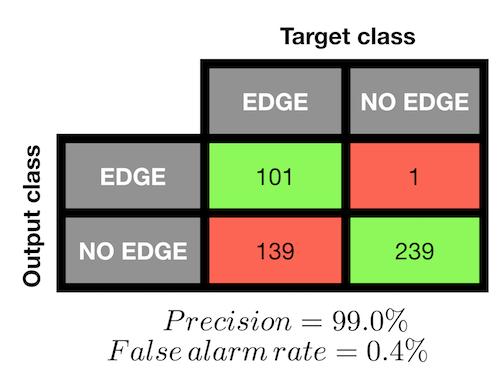}
         \caption{$n_{f} = 3$ frames}
    \end{subfigure}%
    \hfill
    \begin{subfigure}[t]{0.5\textwidth}
      \centering
         \includegraphics[width=\textwidth]{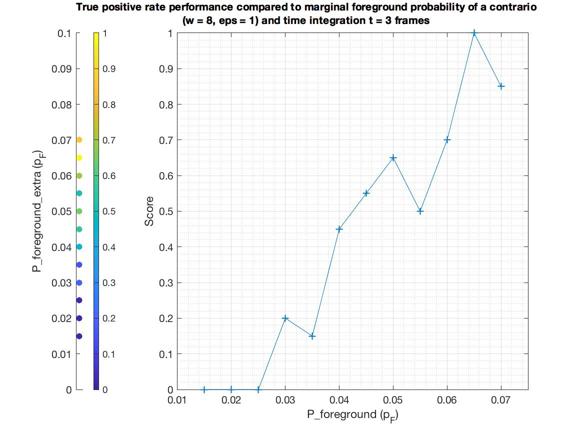}
         \caption{$n_{f} = 3$ frames}
    \end{subfigure}%
    \hfill
\end{figure}
\begin{figure}\ContinuedFloat
    \centering
    \begin{subfigure}[t]{0.5\textwidth}
      \centering
         \includegraphics[width=\textwidth]{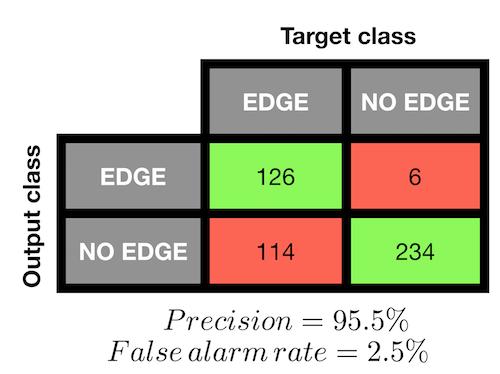}
         \caption{$n_{f} = 4$ frames}
    \end{subfigure}%
    \hfill
    \begin{subfigure}[t]{0.5\textwidth}
      \centering
         \includegraphics[width=\textwidth]{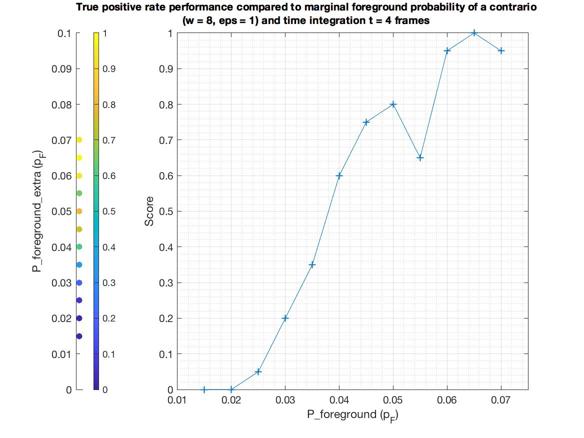}
         \caption{$n_{f} = 4$ frames}
    \end{subfigure}%
    \hfill
    \begin{subfigure}[t]{0.5\textwidth}
      \centering
         \includegraphics[width=\textwidth]{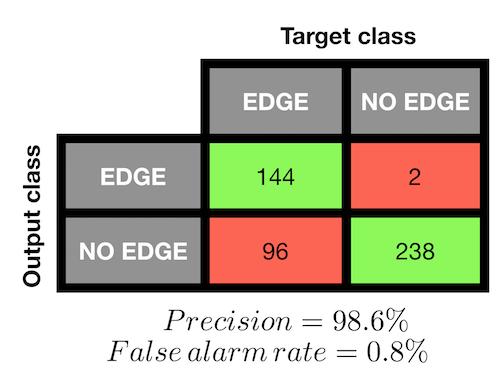}
         \caption{$n_{f} = 5$ frames}
    \end{subfigure}%
    \hfill
    \begin{subfigure}[t]{0.5\textwidth}
      \centering
         \includegraphics[width=\textwidth]{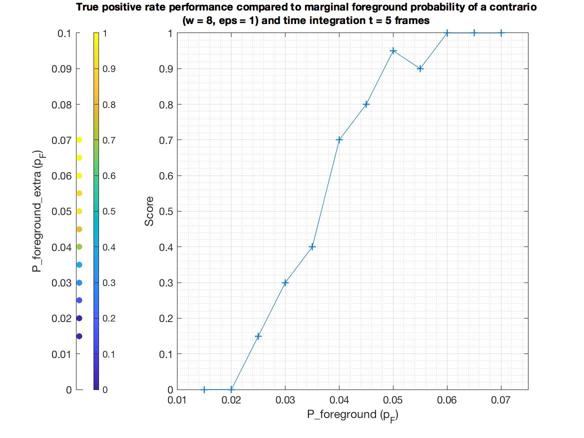}
         \caption{$n_{f} = 5$ frames}
    \end{subfigure}%
    \hfill
    \begin{subfigure}[t]{0.5\textwidth}
      \centering
         \includegraphics[width=\textwidth]{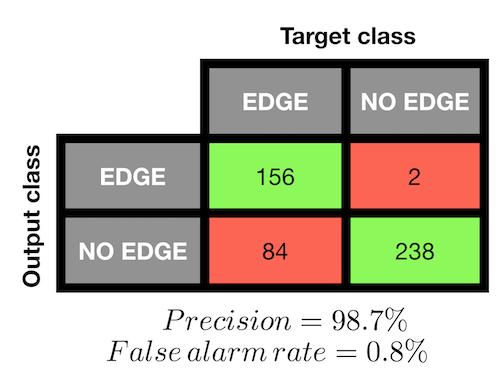}
         \caption{$n_{f} = 6$ frames}
    \end{subfigure}%
    \hfill
    \begin{subfigure}[t]{0.5\textwidth}
      \centering
         \includegraphics[width=\textwidth]{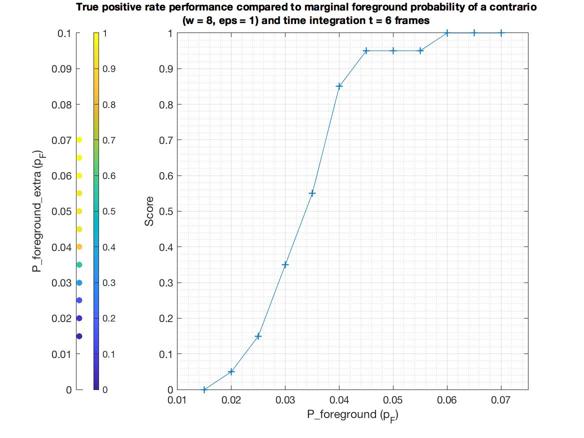}
         \caption{$n_{f} = 6$ frames}
    \end{subfigure}%
    \hfill
\end{figure}
\begin{figure}\ContinuedFloat
    \centering
    \begin{subfigure}[t]{0.5\textwidth}
      \centering
         \includegraphics[width=\textwidth]{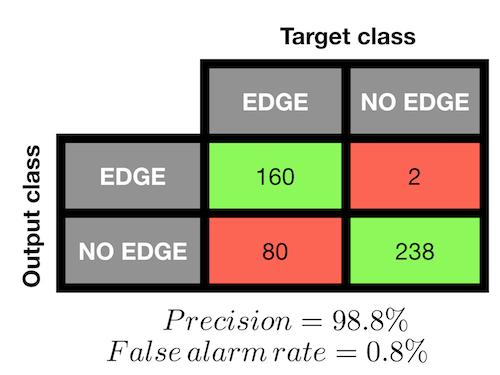}
         \caption{$n_{f} = 7$ frames}
    \end{subfigure}%
    \hfill
    \begin{subfigure}[t]{0.5\textwidth}
      \centering
         \includegraphics[width=\textwidth]{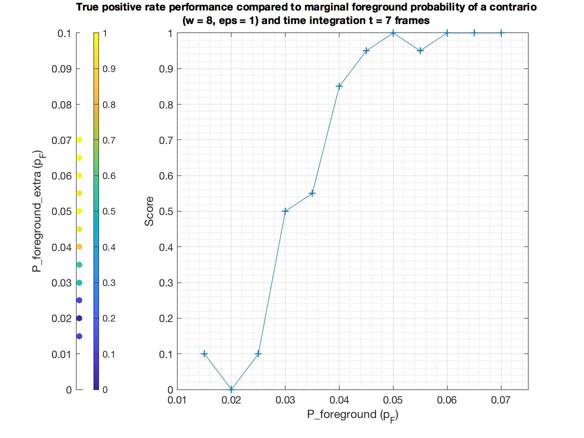}
         \caption{$n_{f} = 7$ frames}
    \end{subfigure}%
    \hfill
    \begin{subfigure}[t]{0.5\textwidth}
      \centering
         \includegraphics[width=\textwidth]{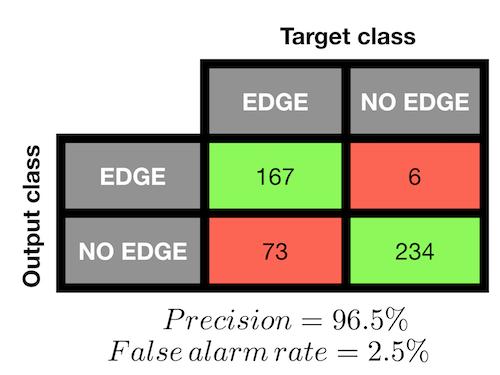}
         \caption{$n_{f} = 8$ frames}
    \end{subfigure}%
    \hfill
    \begin{subfigure}[t]{0.5\textwidth}
      \centering
         \includegraphics[width=\textwidth]{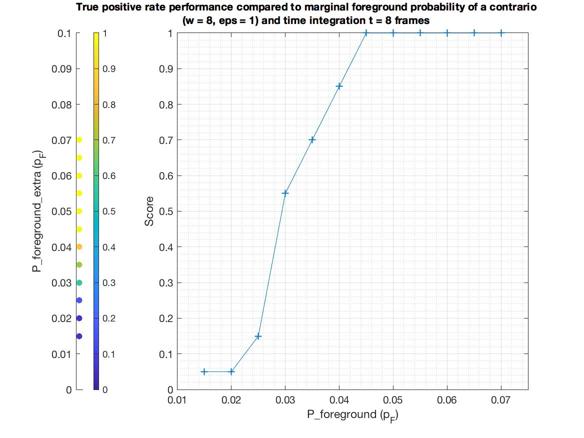}
         \caption{$n_{f} = 8$ frames}
    \end{subfigure}%
    \hfill
    \begin{subfigure}[t]{0.5\textwidth}
      \centering
         \includegraphics[width=\textwidth]{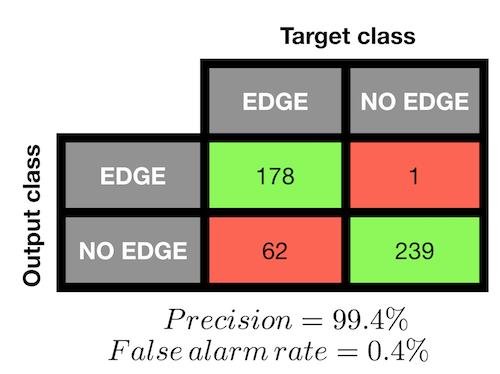}
         \caption{$n_{f} = 9$ frames}
    \end{subfigure}%
    \hfill
    \begin{subfigure}[t]{0.5\textwidth}
      \centering
         \includegraphics[width=\textwidth]{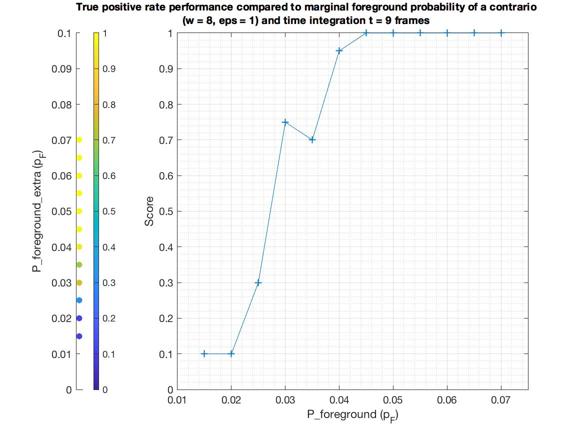}
         \caption{$n_{f} = 9$ frames}
    \end{subfigure}%
    \hfill
\end{figure}
\begin{figure}\ContinuedFloat
    \centering
    \begin{subfigure}[t]{0.5\textwidth}
      \centering
         \includegraphics[width=\textwidth]{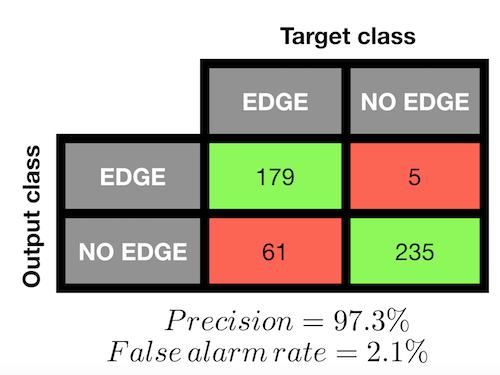}
         \caption{$n_{f} = 10$ frames}
    \end{subfigure}%
    \hfill
    \begin{subfigure}[t]{0.5\textwidth}
      \centering
         \includegraphics[width=\textwidth]{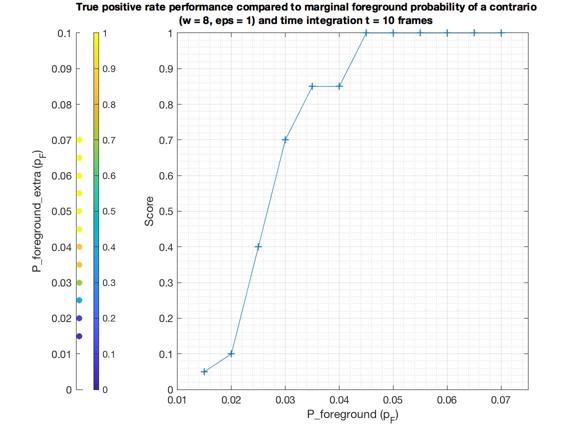}
         \caption{$n_{f} = 10$ frames}
    \end{subfigure}%
    \hfill
\caption{Performances of single width a contrario algorithms on video data of a static edge. Each algorithm corresponds to a different time integration $n_{f}$. Left: global confusion matrix of each algorithm on the entire dataset. Right: true positive rate of each algorithm with respect to $p_f$. In particular, note how low the false alarm rate is for each algorithm. It is not worrying to have many true negatives as the dataset is generated such that many of the true videos have low $p_f$ and as such should make the algorithms fail. Since the false alarm rate is particularly low, the precision of each algorithm is close to $100\%$ and as such we can trust a lot a decision to detect an edge. This leads us to only study the true positive rates (along with the fact that humans will have similar false alarm rates). Unsurprisingly, the true positive rates increase from close to $0$ for low $p_f$ to close to $1$ for high $p_f$ (except for merges of 1 or 2 frames, those need higher $p_f$ to saturate at $1$).}
\label{fig: video static AC}
\end{figure}

\begin{figure}
    \centering
    \begin{subfigure}[t]{0.5\textwidth}
      \centering
         \includegraphics[width=\textwidth]{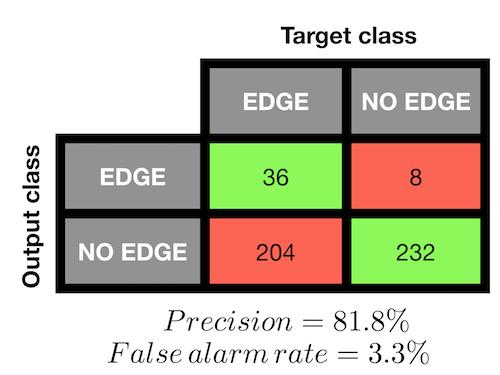}
         \caption{$n_{f} = 1$ frame}
    \end{subfigure}%
    \hfill
    \begin{subfigure}[t]{0.5\textwidth}
      \centering
         \includegraphics[width=\textwidth]{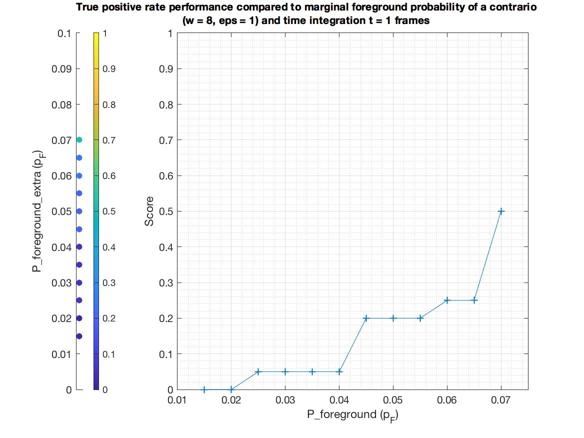}
         \caption{$n_{f} = 1$ frame}
    \end{subfigure}%
    \hfill
    \begin{subfigure}[t]{0.5\textwidth}
      \centering
         \includegraphics[width=\textwidth]{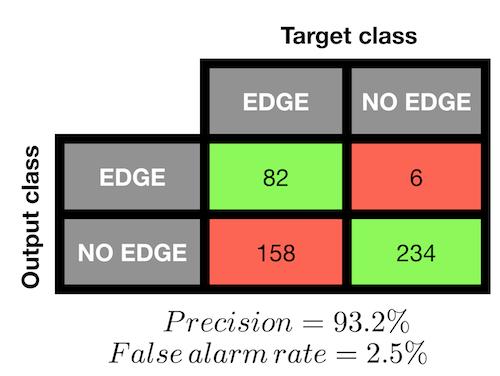}
         \caption{$n_{f} = 2$ frames}
    \end{subfigure}%
    \hfill
    \begin{subfigure}[t]{0.5\textwidth}
      \centering
         \includegraphics[width=\textwidth]{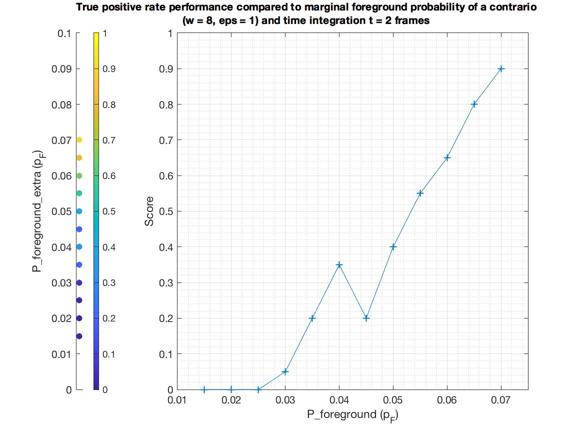}
         \caption{$n_{f} = 2$ frames}
    \end{subfigure}%
    \hfill
    \begin{subfigure}[t]{0.5\textwidth}
      \centering
         \includegraphics[width=\textwidth]{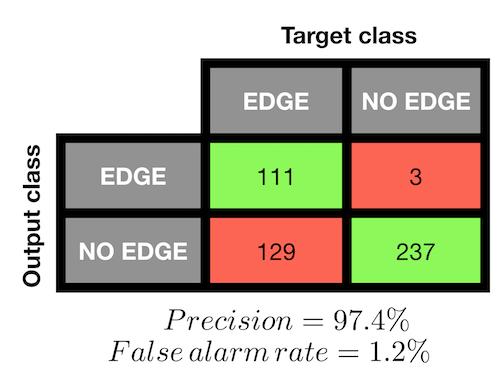}
         \caption{$n_{f} = 3$ frames}
    \end{subfigure}%
    \hfill
    \begin{subfigure}[t]{0.5\textwidth}
      \centering
         \includegraphics[width=\textwidth]{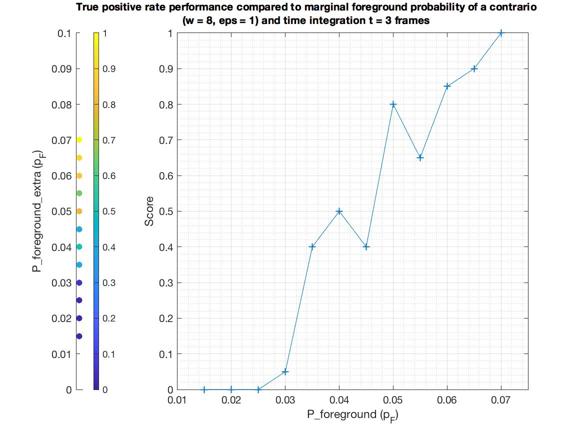}
         \caption{$n_{f} = 3$ frames}
    \end{subfigure}%
    \hfill
\end{figure}
\begin{figure}\ContinuedFloat
    \centering
    \begin{subfigure}[t]{0.5\textwidth}
      \centering
         \includegraphics[width=\textwidth]{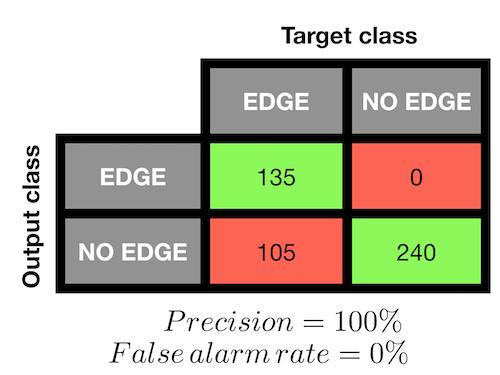}
         \caption{$n_{f} = 4$ frames}
    \end{subfigure}%
    \hfill
    \begin{subfigure}[t]{0.5\textwidth}
      \centering
         \includegraphics[width=\textwidth]{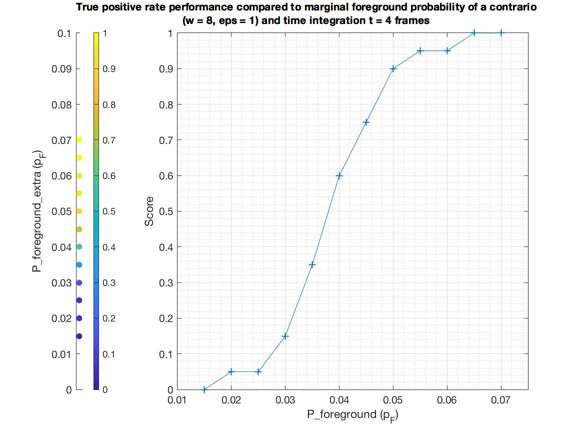}
         \caption{$n_{f} = 4$ frames}
    \end{subfigure}%
    \hfill
    \begin{subfigure}[t]{0.5\textwidth}
      \centering
         \includegraphics[width=\textwidth]{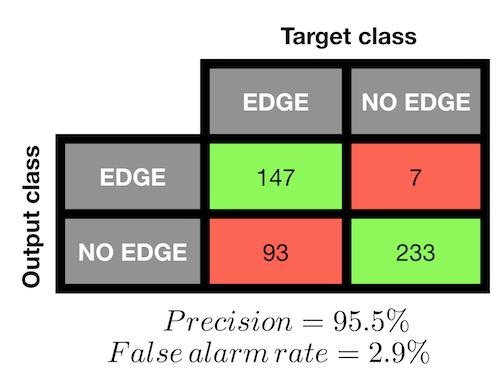}
         \caption{$n_{f} = 5$ frames}
    \end{subfigure}%
    \hfill
    \begin{subfigure}[t]{0.5\textwidth}
      \centering
         \includegraphics[width=\textwidth]{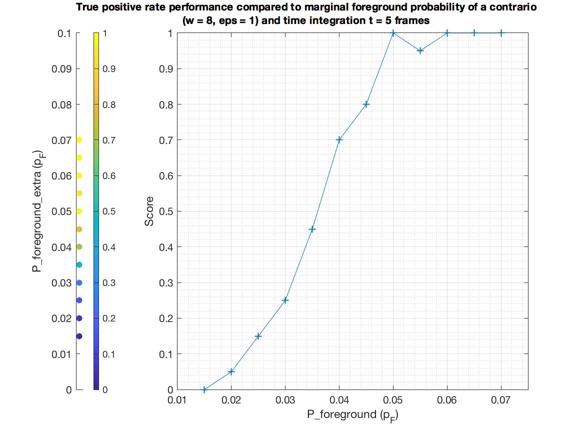}
         \caption{$n_{f} = 5$ frames}
    \end{subfigure}%
    \hfill
    \begin{subfigure}[t]{0.5\textwidth}
      \centering
         \includegraphics[width=\textwidth]{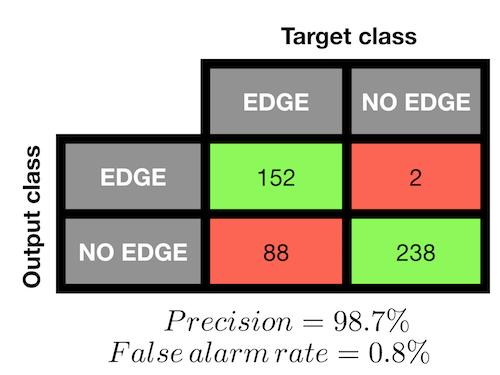}
         \caption{$n_{f} = 6$ frames}
    \end{subfigure}%
    \hfill
    \begin{subfigure}[t]{0.5\textwidth}
      \centering
         \includegraphics[width=\textwidth]{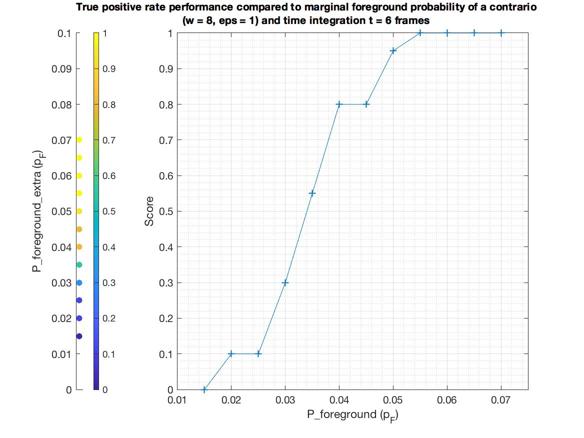}
         \caption{$n_{f} = 6$ frames}
    \end{subfigure}%
    \hfill
\end{figure}
\begin{figure}\ContinuedFloat
    \centering
    \begin{subfigure}[t]{0.5\textwidth}
      \centering
         \includegraphics[width=\textwidth]{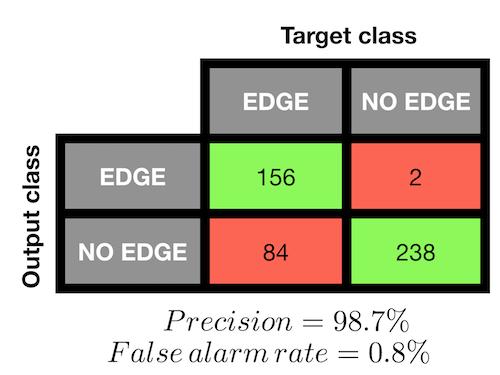}
         \caption{$n_{f} = 7$ frames}
    \end{subfigure}%
    \hfill
    \begin{subfigure}[t]{0.5\textwidth}
      \centering
         \includegraphics[width=\textwidth]{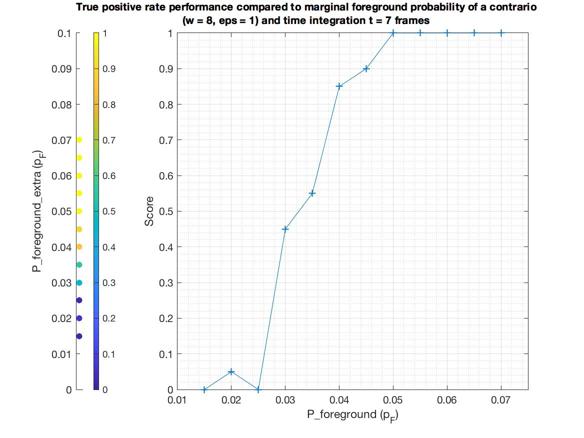}
         \caption{$n_{f} = 7$ frames}
    \end{subfigure}%
    \hfill
    \begin{subfigure}[t]{0.5\textwidth}
      \centering
         \includegraphics[width=\textwidth]{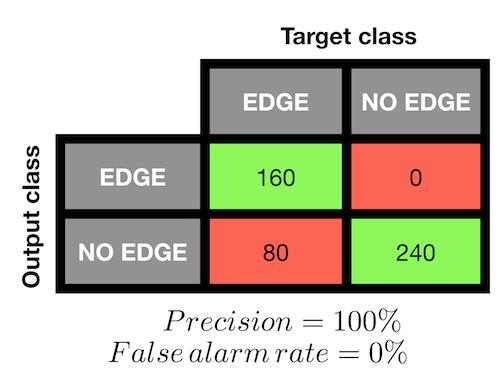}
         \caption{$n_{f} = 8$ frames}
    \end{subfigure}%
    \hfill
    \begin{subfigure}[t]{0.5\textwidth}
      \centering
         \includegraphics[width=\textwidth]{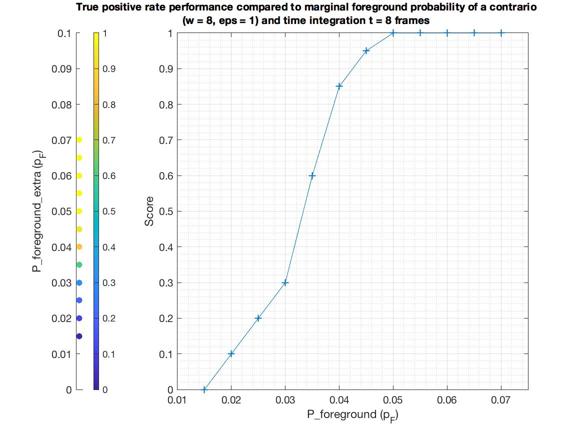}
         \caption{$n_{f} = 8$ frames}
    \end{subfigure}%
    \hfill
    \begin{subfigure}[t]{0.5\textwidth}
      \centering
         \includegraphics[width=\textwidth]{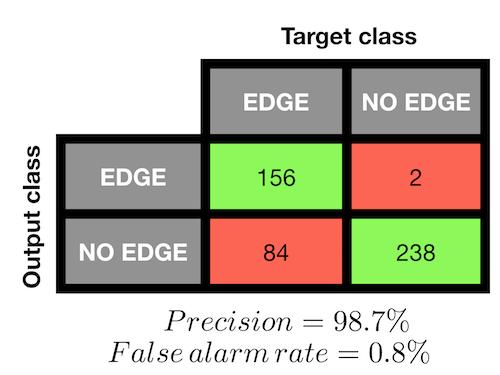}
         \caption{$n_{f} = 9$ frames}
    \end{subfigure}%
    \hfill
    \begin{subfigure}[t]{0.5\textwidth}
      \centering
         \includegraphics[width=\textwidth]{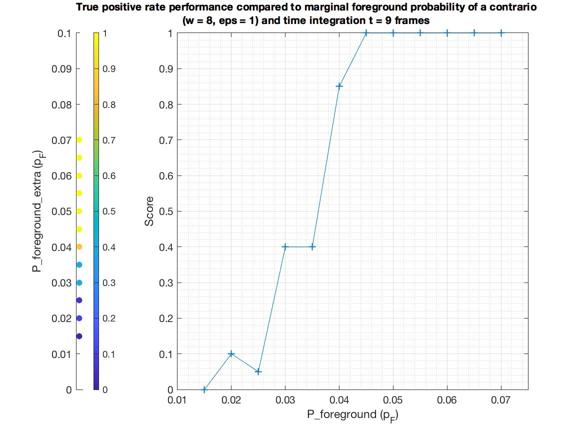}
         \caption{$n_{f} = 9$ frames}
    \end{subfigure}%
    \hfill
\end{figure}
\begin{figure}\ContinuedFloat
    \centering
    \begin{subfigure}[t]{0.5\textwidth}
      \centering
         \includegraphics[width=\textwidth]{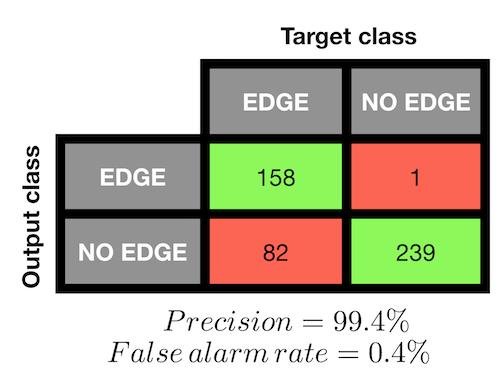}
         \caption{$n_{f} = 10$ frames}
    \end{subfigure}%
    \hfill
    \begin{subfigure}[t]{0.5\textwidth}
      \centering
         \includegraphics[width=\textwidth]{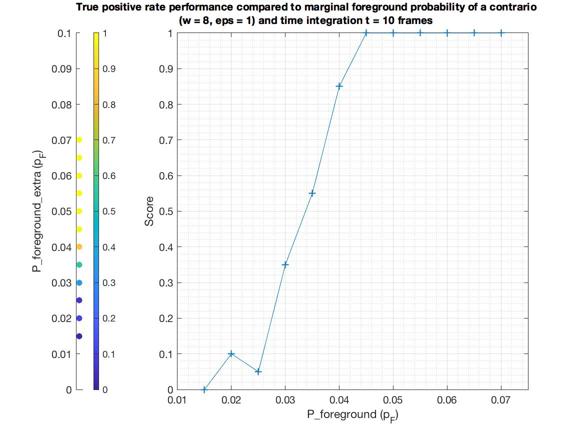}
         \caption{$n_{f} = 10$ frames}
    \end{subfigure}%
    \hfill
\caption{Performances of single width a contrario algorithms on video data of a dynamic edge. Each algorithm corresponds to a different time integration $n_{f}$. Left: global confusion matrix of each algorithm on the entire dataset. Right: true positive rate of each algorithm with respect to $p_f$. In particular, note how low the false alarm rate is for each algorithm. It is not worrying to have many true negatives as the dataset is generated such that many of the true videos have low $p_f$ and as such should make the algorithms fail. Since the false alarm rate is particularly low, the precision of each algorithm is close to $100\%$ and as such we can trust a lot a decision to detect an edge. This leads us to only study the true positive rates (along with the fact that humans will have similar false alarm rates). Unsurprisingly, the true positive rates increase from close to $0$ for low $p_f$ to close to $1$ for high $p_f$ (except for merges of 1 or 2 frames, those need higher $p_f$ to saturate at $1$).}
\label{fig: video dynamic AC}
\end{figure}

\clearpage
\newpage
\section{Static Edge Experiments}
\label{sec: exp3.1}

In this Section we detail the experiment performed by human observers on the videos of \say{jumping} static edges. This will allow us to extract an empirical estimate for the time integration for our model of human perception. The experiments detailed here were approved by the Ethics committee of the Technion as conforming to the standards for performing psychophysics experiments on volunteering human subjects.

\paragraph{Stimuli} The stimuli is the random-dot video static edge dataset we have generated and previously described. Recall that half of these videos are just noise without an underlying line. All subjects were presented with the entire dataset (480 videos). The order in which they are shown is chosen at random and therefore differs between subjects.

\paragraph{Apparatus} Each subject was tested with exactly same display. The display is the screen of a MacBook Pro retina 13-inch display from mid 2014. Each subject was seated directly in front of the screen. The images were displayed in a well-lit room. The average distance between the subjects' eyes and the screen was about $70\,\mathrm{cm}$. The $300\times300$ pixels video is displayed with a size of $10.4 \,\mathrm{cm}\times10.4\,\mathrm{cm}$. This translates to a pixel size of approximately $0.35\,\mathrm{mm}\times0.35\,\mathrm{mm}$. On average the visual angle for observing the image is approximately $8.5°\times8.5°$. Below the video we display two buttons for user decision: a YES button and a NO button (see the procedure paragraph). The YES button is green with YES written in a large font on it. The NO button is red with NO written in a large font on it, the same font as on the YES button. Both buttons are horizontally aligned. The YES button is located on the left whereas the NO button is located on the right. Each button is of size $5.6\,\mathrm{cm}\times 2.8\,\mathrm{cm}$. The average visual angle for the buttons are approximately $4.6°\times2.3°$. The distance between the image and the horizontal line of boxes is $0.5\,\mathrm{cm}$ which translates to a visual angle of seperation of approximately $0.4°$. The buttons are separated by a distance of $5.8\,\mathrm{cm}$, which corresponds to a visual angle of approximately $4.7°$. See Figure \ref{fig: screenshot exp 3.1} for a screenshot of the display.

\paragraph{Subjects} The experiments were performed on the first author and on the same further thirteen subjects as previously. All the other subjects were unfamiliar with psychophysical experiments. All subjects were international students of the university, coming from various countries around the globe including China, Vietnam, India, Germany, Greece, the United States of America and France. Gender parity was almost respected with eight males and six females. The age range of the subjects was between 20 and 31 years of age. All subjects had normal or corrected to normal vision. To the best of our knowledge, no subject had any mental condition. All other subjects were unaware of the aim of our study.

\paragraph{Procedure} The experiment is a yes-no task. Subjects were told that some of the videos they would be presented with will contain a line and some would not. The proportion of videos without a line was not given before-hand. However, subjects were told it was all right if they found that they had to click a lot on YES or a lot on NO and just focus on the current video at hand. Subjects were asked to click on the YES button if they believed that the video contained the line, and click NO if they thought it did not. They were encouraged to make a decision rather than wait for the timer to run out. They were told that they were looking for a line of length two thirds of the width of the video frame and width one pixel. They were told that should the video be a video with a line in it, then the line would be there during the entire video. They were also told that if a line is present, it would jump regularly from a random position to another random position. They were told no video would ever have were two underlying lines that would be present simultaneously. In the entire experiment, a Matlab window covered the entire screen. In the centre we displayed the video. Below it, on the left would be a big green button with YES written on it and symmetrically on the right a big red button with NO written on it. Above the video was written the question \say{Is the line present?}, to which they had to answer by clicking on the appropriate button. For responding, subjects had the choice between using the built-in touchpad of the MacBook Pro or to use a Asus N6-Mini mouse. If the subject clicked on any area that is not within one of the two buttons then that click was dismissed. If the subject clicked on any of the buttons then the videos stops and moves on to the next one. Subjects had a 10 second limit to answer for each video, after which the next video would be automatically shown. If subjects ran out of time the next video was automatically shown. Subjects were not shown a progress bar. Subjects could not take a break once having started the experiment, but were allowed to abandon the experiment at any time, should they desire it. The time for each click to be made, along with the corresponding button clicked was saved.

\paragraph{Discussion} In order to estimate detection performance, we should normally look at both the true positive rate of the subjects with respect to $p_f$ but also the false alarm rate. Indeed, if we only look at one of these values and discard the other, we could end up rating some performance as excellent when in fact they are terrible. For instance imagine a subject that always finds the edge but when there isn't any always return a false alarm: he always clicks on YES. Then he detects all edge data, thus $100\%$ true positive rate, but also a $100\%$ false alarm rate, which is bad: the true decisions are not reliable. Thus if we omit the false alarm rate and only look at the true positive rate we get an incorrect evaluation of the performance. Nevertheless, this is what we are going to do. The justification for such a dangerous procedure is the following. The false alarm rate of the subjects is particularly low, close to $0\%$. Also, it is similar to the false alarm rate of the a contrario algorithms. As such, we are justified in the decision to disregard the false alarm rate and only concentrate on the true positive rate. The true positive rate with respect to $p_f$ provides a performance vector or curve for each subject. We can then compare humans and a contrario by comparing these vectors or functions. This is somewhat equivalent to fitting the true positive vector performance of humans to the columns of the previous performances on the static data (up to dilation since $p_f$ almost linearly increases with the number of frames of the merge). The comparison measure is chosen to be the $L_2$ distance between the performance vectors. For better display, we plot the $L_2$ distance of the subjects' performance and the performance of each algorithm. The algorithms with lower distance thus perform the most similarly to the human and thus give us a candidate time integration. In practice, we see that for most subjects and on average that the algorithms with time integration in $n_f\in[6,8]$ frames perform similarly and yield the best $L_2$ fit and so this time integration range is a good candidate for the time integration of humans. This corresponds to a time integration of $\Delta t\in[0.2,0.27]$ seconds. We display the global confusion matrices and false alarm rates of the subjects in Figure \ref{fig: exp3 static confusion humans}. We display the true positive rates with respect to $p_f$ in the Figure \ref{fig: exp3 static TPR humans}. We display in Figure \ref{fig: exp3 static L2 humans} the $L_2$ distance between the true positive rate curves of the subjects with respect to those of the a contrario algorithms working on $n_{f}$ frames for time integration.

\begin{figure}
    \centering
    \includegraphics[width=\textwidth]{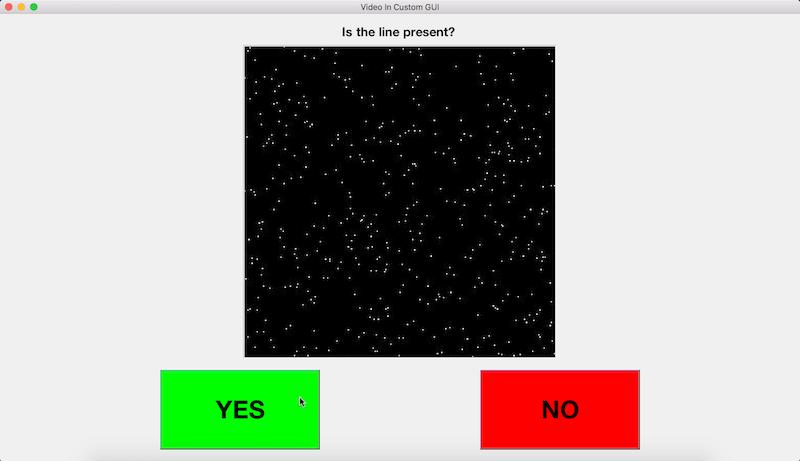}
    \hfill
    \caption[Short version]{Screenshot of the display.}
    \label{fig: screenshot exp 3.1}
\end{figure}

\begin{figure}
    \centering
    \begin{subfigure}[t]{0.3\textwidth}
      \centering
         \includegraphics[width=\textwidth]{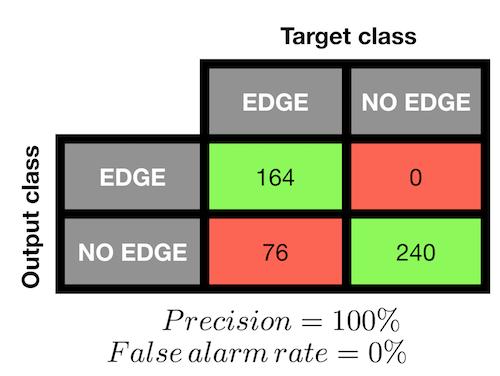}
         \caption{Subject 0}
    \end{subfigure}%
    \hfill
    \begin{subfigure}[t]{0.3\textwidth}
      \centering
         \includegraphics[width=\textwidth]{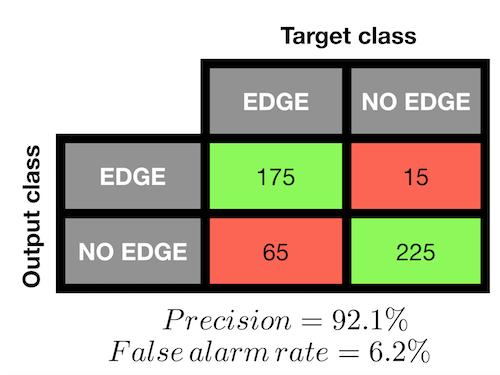}
         \caption{Subject 1}
    \end{subfigure}%
    \hfill
    \begin{subfigure}[t]{0.3\textwidth}
      \centering
         \includegraphics[width=\textwidth]{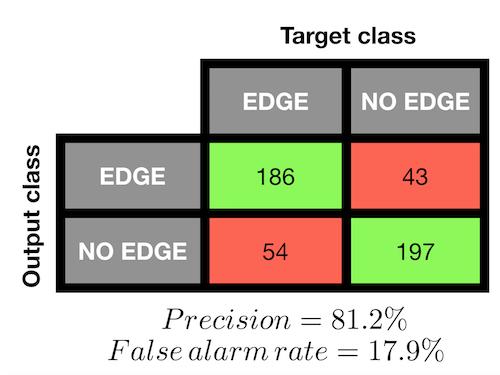}
         \caption{Subject 2}
    \end{subfigure}%
    \hfill
    \begin{subfigure}[t]{0.3\textwidth}
      \centering
         \includegraphics[width=\textwidth]{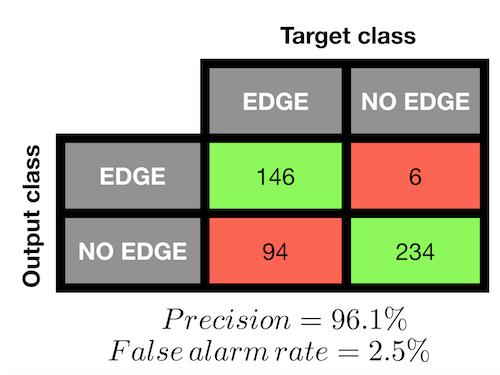}
         \caption{Subject 3}
    \end{subfigure}%
    \hfill
    \begin{subfigure}[t]{0.3\textwidth}
      \centering
         \includegraphics[width=\textwidth]{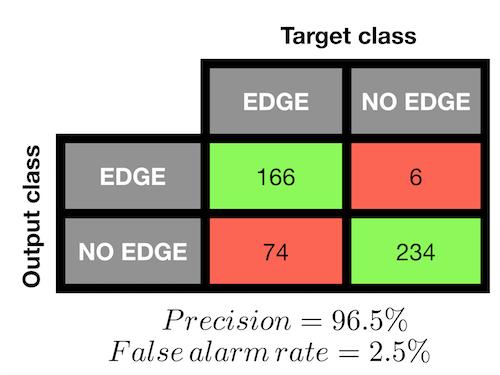}
         \caption{Subject 4}
    \end{subfigure}%
    \hfill
    \begin{subfigure}[t]{0.3\textwidth}
      \centering
         \includegraphics[width=\textwidth]{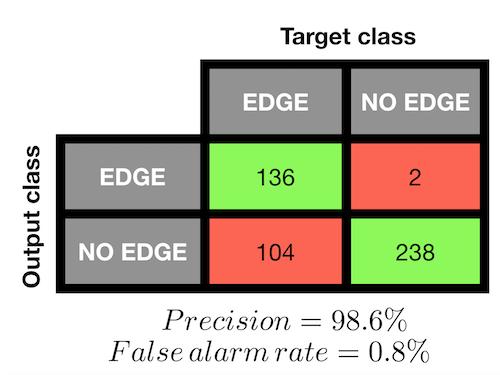}
         \caption{Subject 5}
    \end{subfigure}%
    \hfill
    \begin{subfigure}[t]{0.3\textwidth}
      \centering
         \includegraphics[width=\textwidth]{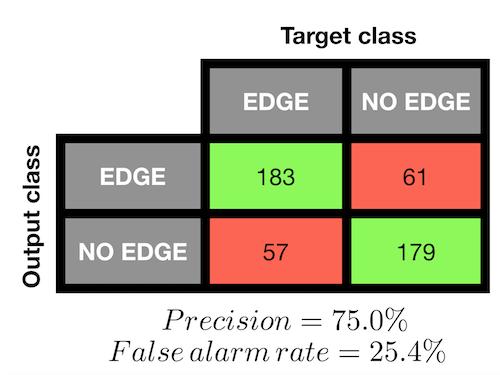}
         \caption{Subject 6}
    \end{subfigure}%
    \hfill
    \begin{subfigure}[t]{0.3\textwidth}
      \centering
         \includegraphics[width=\textwidth]{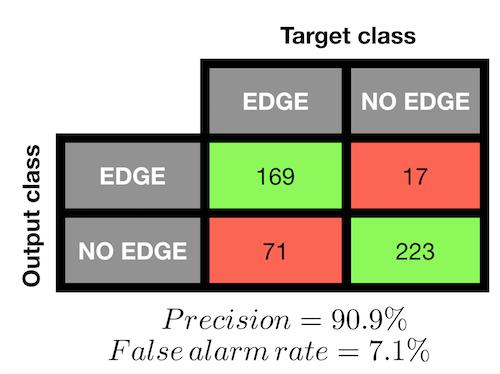}
         \caption{Subject 7}
    \end{subfigure}%
    \hfill
    \begin{subfigure}[t]{0.3\textwidth}
      \centering
         \includegraphics[width=\textwidth]{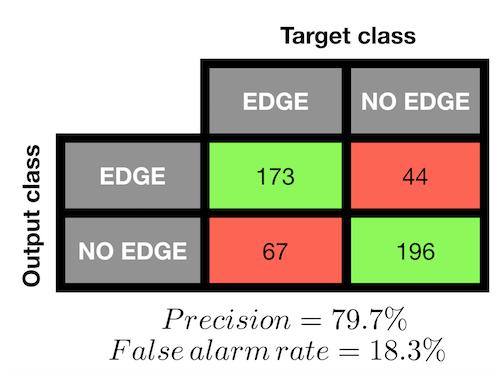}
         \caption{Subject 8}
    \end{subfigure}%
    \hfill
    \begin{subfigure}[t]{0.3\textwidth}
      \centering
         \includegraphics[width=\textwidth]{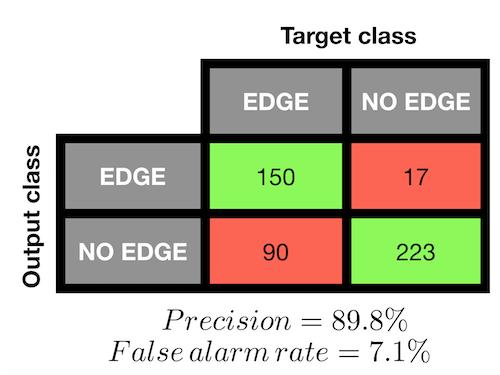}
         \caption{Subject 9}
    \end{subfigure}%
    \hfill
    \begin{subfigure}[t]{0.3\textwidth}
      \centering
         \includegraphics[width=\textwidth]{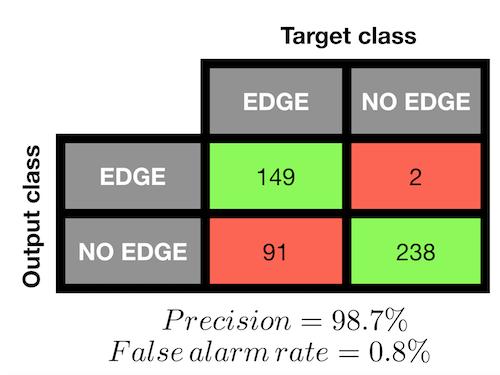}
         \caption{Subject 10}
    \end{subfigure}%
    \hfill
    \begin{subfigure}[t]{0.3\textwidth}
      \centering
         \includegraphics[width=\textwidth]{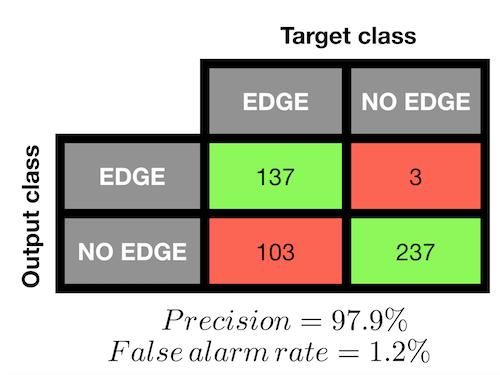}
         \caption{Subject 11}
    \end{subfigure}%
    \hfill
    \begin{subfigure}[t]{0.3\textwidth}
      \centering
         \includegraphics[width=\textwidth]{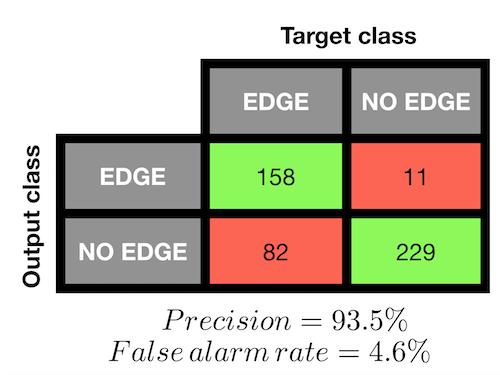}
         \caption{Subject 12}
    \end{subfigure}%
    \hfill
    \begin{subfigure}[t]{0.3\textwidth}
      \centering
         \includegraphics[width=\textwidth]{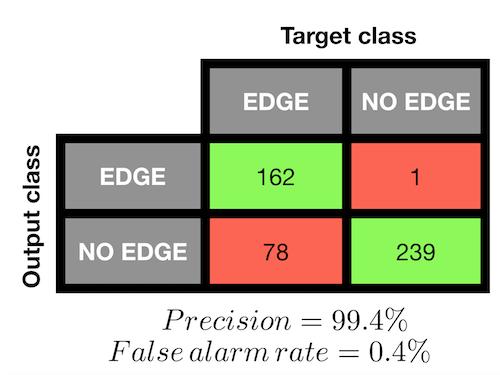}
         \caption{Subject 13}
    \end{subfigure}%
    \hfill
    \begin{subfigure}[t]{0.3\textwidth}
      \centering
         \includegraphics[width=\textwidth]{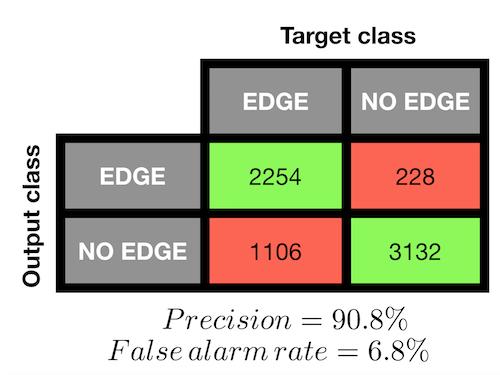}
         \caption{All subjects}
    \end{subfigure}%
    \hfill
    \caption{Global confusion matrices of each subject and of the average of all subjects. In particular, note how low the false alarm rate is for each subject, except subjects 2, 6 and 8. After this experiment, these three subjects admitted to have guessed a lot and not correctly understood that it was not desired and strongly discouraged to randomly guess. As such, their results are not trustworthy here. It is not worrying to have many true negatives as the dataset is generated such that many of the true videos have low $p_f$ and as such should make the subjects fail. Since the false alarm rate is particularly low, the precision of each subject is close to $100\%$ and as such we can trust a lot a decision to detect an edge. Also, the false alarm rates are similar to those of the a contrario algorithms. This leads us to only study the true positive rates. }
    \label{fig: exp3 static confusion humans}
 \end{figure}

\begin{figure}
    \centering
    \begin{subfigure}[t]{0.5\textwidth}
      \centering
         \includegraphics[width=\textwidth]{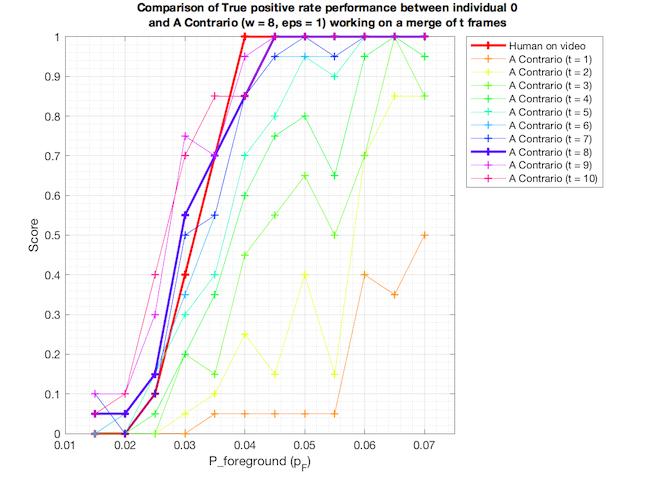}
         \caption{Subject 0}
    \end{subfigure}%
    \hfill
    \begin{subfigure}[t]{0.5\textwidth}
      \centering
         \includegraphics[width=\textwidth]{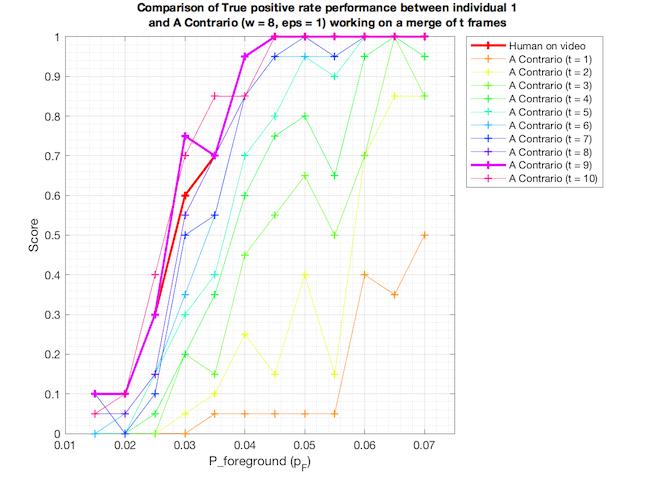}
         \caption{Subject 1}
    \end{subfigure}%
    \hfill
    \begin{subfigure}[t]{0.5\textwidth}
      \centering
         \includegraphics[width=\textwidth]{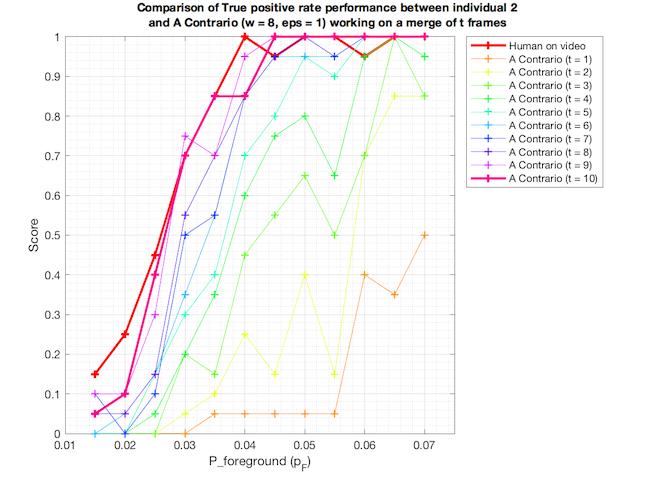}
         \caption{Subject 2}
    \end{subfigure}%
    \hfill
    \begin{subfigure}[t]{0.5\textwidth}
      \centering
         \includegraphics[width=\textwidth]{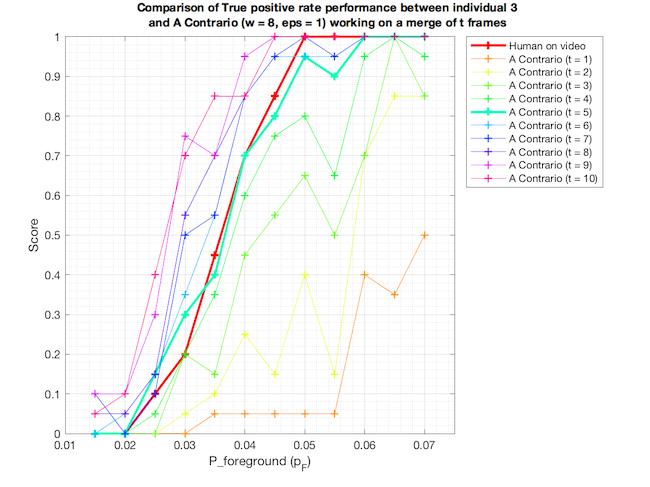}
         \caption{Subject 3}
    \end{subfigure}%
    \hfill
    \begin{subfigure}[t]{0.5\textwidth}
      \centering
         \includegraphics[width=\textwidth]{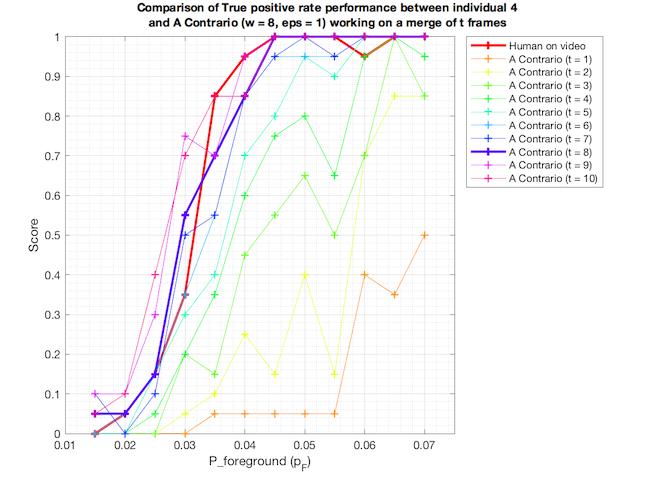}
         \caption{Subject 4}
    \end{subfigure}%
    \hfill
    \begin{subfigure}[t]{0.5\textwidth}
      \centering
         \includegraphics[width=\textwidth]{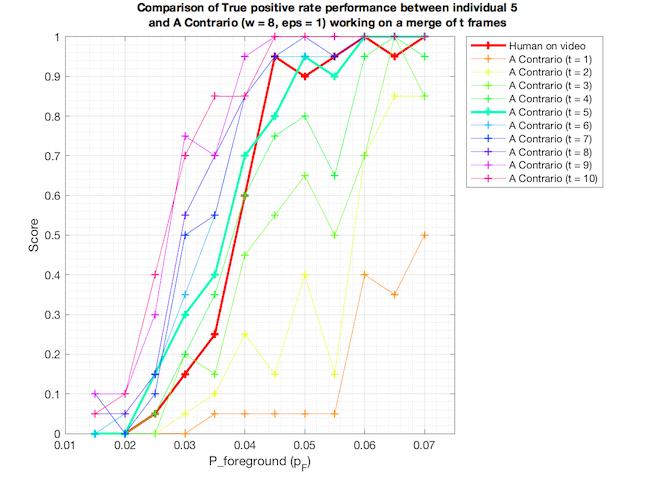}
         \caption{Subject 5}
    \end{subfigure}%
    \hfill
\end{figure}
\begin{figure} \ContinuedFloat
    \centering
    \begin{subfigure}[t]{0.5\textwidth}
      \centering
         \includegraphics[width=\textwidth]{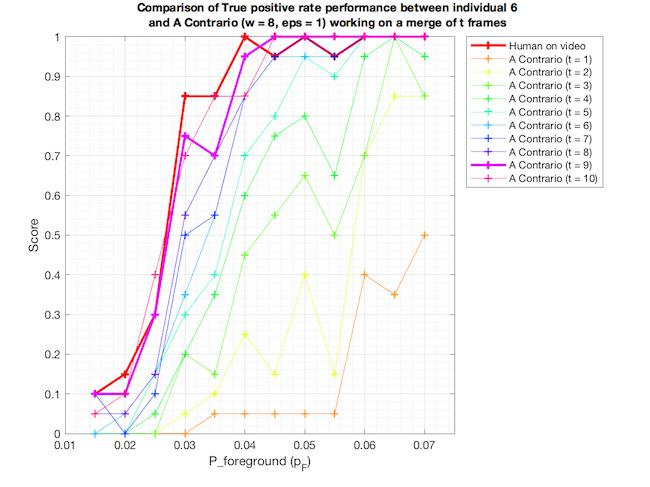}
         \caption{Subject 6}
    \end{subfigure}%
    \hfill
    \begin{subfigure}[t]{0.5\textwidth}
      \centering
         \includegraphics[width=\textwidth]{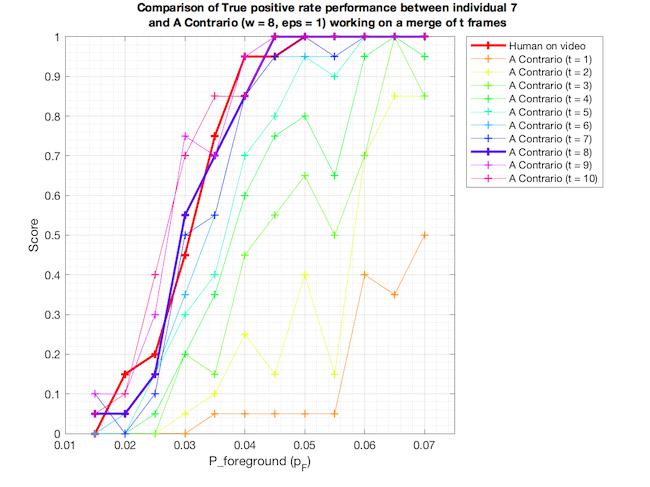}
         \caption{Subject 7}
    \end{subfigure}%
    \hfill
    \begin{subfigure}[t]{0.5\textwidth}
      \centering
        \includegraphics[width=\textwidth]{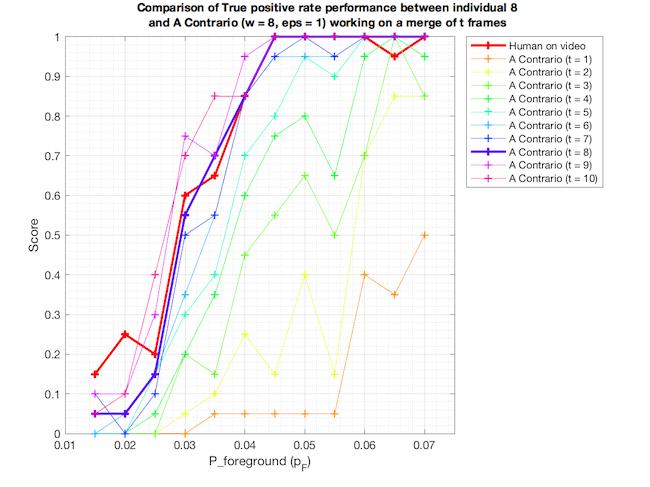}
         \caption{Subject 8}
    \end{subfigure}%
    \hfill
    \begin{subfigure}[t]{0.5\textwidth}
      \centering
         \includegraphics[width=\textwidth]{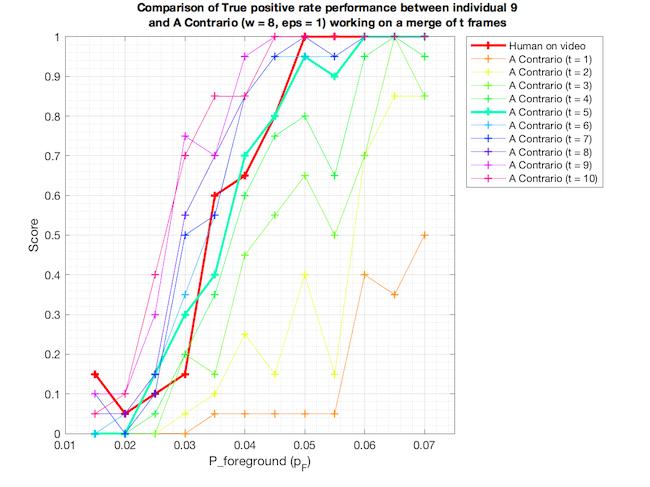}
         \caption{Subject 9}
    \end{subfigure}%
    \hfill
    \begin{subfigure}[t]{0.5\textwidth}
      \centering
         \includegraphics[width=\textwidth]{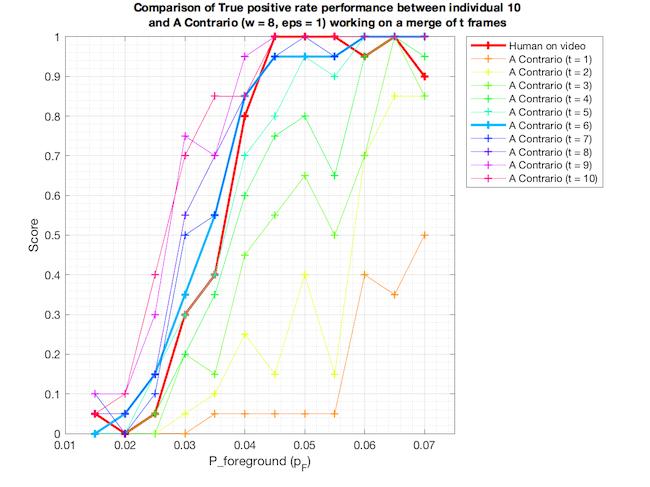}
         \caption{Subject 10}
    \end{subfigure}%
    \hfill
    \begin{subfigure}[t]{0.5\textwidth}
      \centering
        \includegraphics[width=\textwidth]{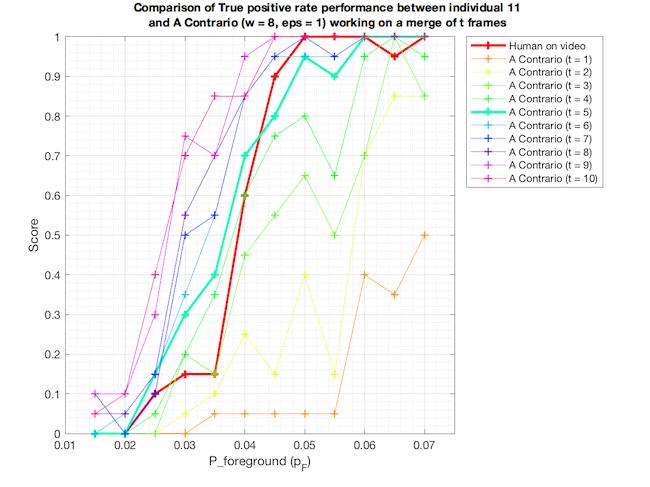}
         \caption{Subject 11}
    \end{subfigure}%
    \hfill
\end{figure}
\begin{figure} \ContinuedFloat
    \centering
    \begin{subfigure}[t]{0.5\textwidth}
      \centering
        \includegraphics[width=\textwidth]{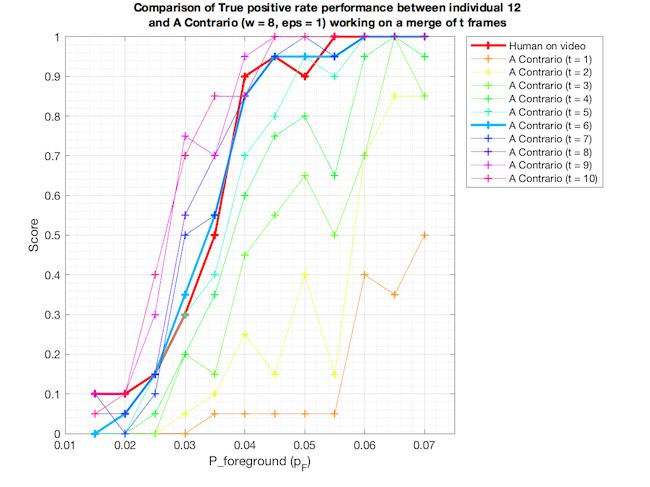}
         \caption{Subject 12}
    \end{subfigure}%
    \hfill
    \begin{subfigure}[t]{0.5\textwidth}
      \centering
         \includegraphics[width=\textwidth]{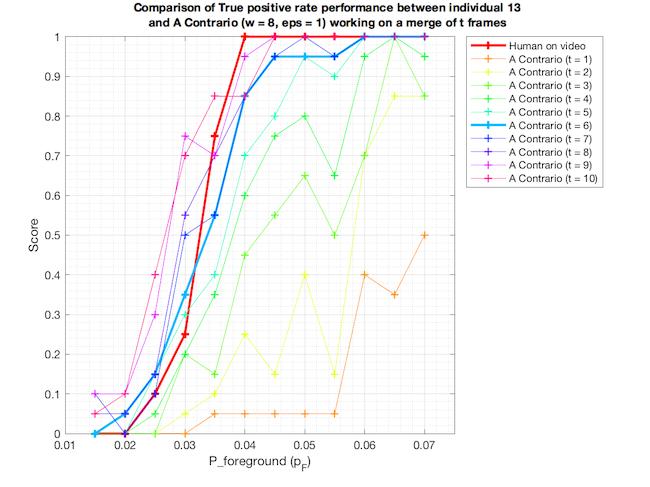}
         \caption{Subject 13}
    \end{subfigure}%
    \hfill
    \begin{subfigure}[t]{0.5\textwidth}
      \centering
         \includegraphics[width=\textwidth]{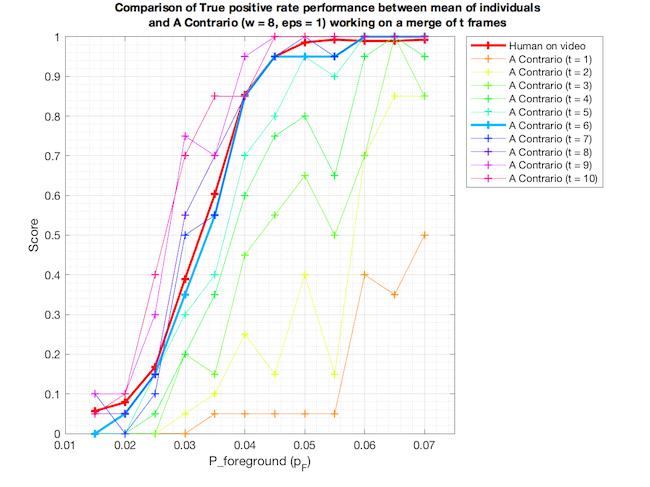}
         \caption{All subjects}
    \end{subfigure}%
    \hfill
    \caption{True positive rate of each subject with respect to $p_f$. The true positive rates increase from close to $0$ for low $p_f$ to close to $1$ for high $p_f$. We superimpose to the human performance curve the curves previously obtained by the single widths a contrarios working on $w=8$. The curves obtained for the subjects are somewhat similar to those obtained by the single width a contrario processes. In each plot, we have thickened the a contrario algorithm with single time integration $n_{f}\in\{1,\hdots,10\}$ frames such that its curve (or discrete vector) is closest in $L_2$ distance to the curve (or discrete vector) obtained by the subject. The best time integration seems to fluctuate between subjects and it is not clear from these plots how much better the best curve is compared to the others. See the next Figure for more clarity.}
    \label{fig: exp3 static TPR humans}
 \end{figure}

\begin{figure}
    \centering
    \begin{subfigure}[t]{0.5\textwidth}
      \centering
         \includegraphics[width=\textwidth]{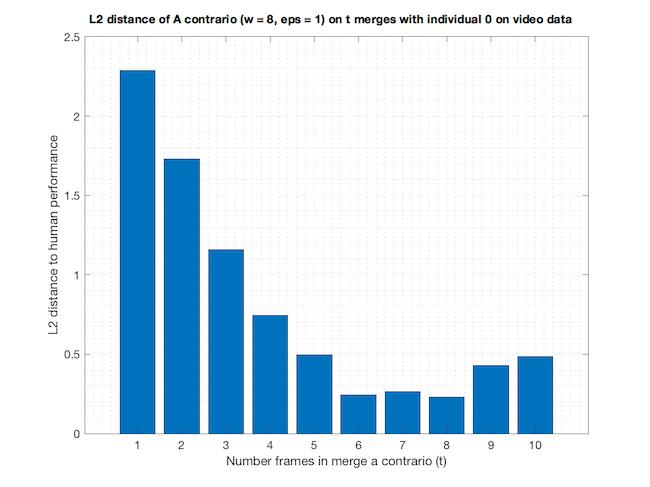}
         \caption{Subject 0}
    \end{subfigure}%
    \hfill
    \begin{subfigure}[t]{0.5\textwidth}
      \centering
         \includegraphics[width=\textwidth]{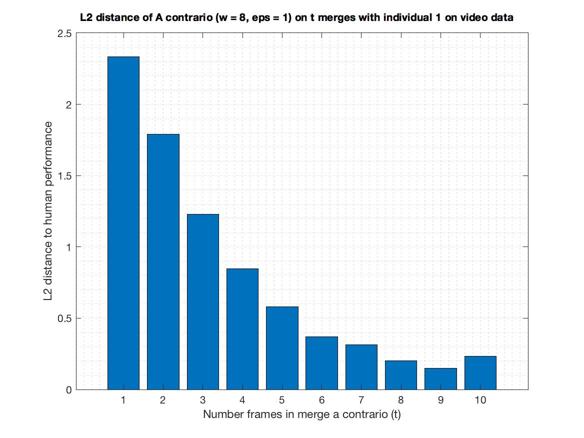}
         \caption{Subject 1}
    \end{subfigure}%
    \hfill
    \begin{subfigure}[t]{0.5\textwidth}
      \centering
         \includegraphics[width=\textwidth]{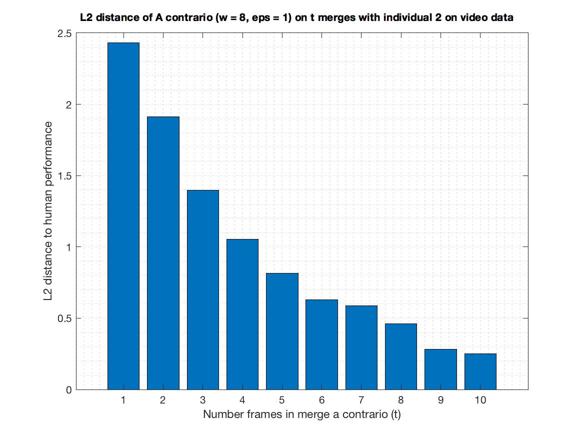}
         \caption{Subject 2}
    \end{subfigure}%
    \hfill
    \begin{subfigure}[t]{0.5\textwidth}
      \centering
         \includegraphics[width=\textwidth]{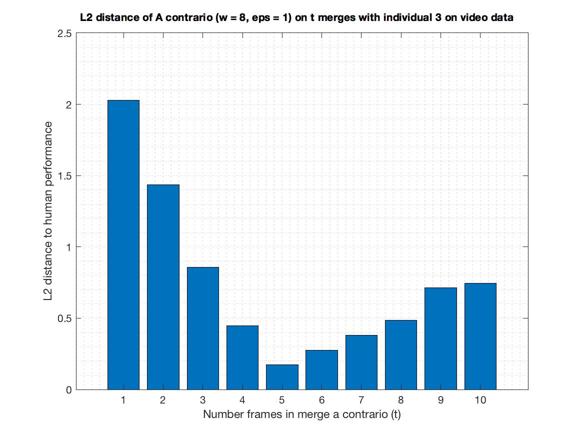}
         \caption{Subject 3}
    \end{subfigure}%
    \hfill
    \begin{subfigure}[t]{0.5\textwidth}
      \centering
         \includegraphics[width=\textwidth]{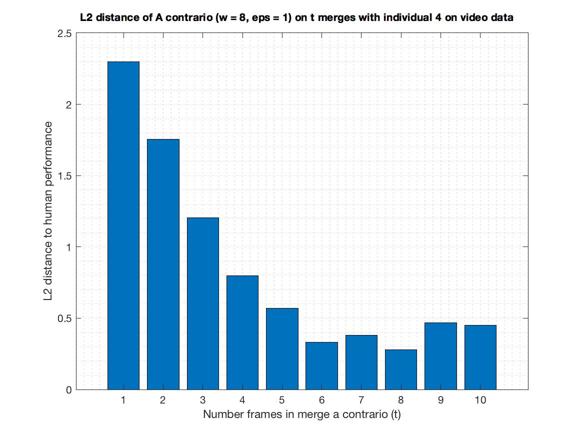}
         \caption{Subject 4}
    \end{subfigure}%
    \hfill
    \begin{subfigure}[t]{0.5\textwidth}
      \centering
         \includegraphics[width=\textwidth]{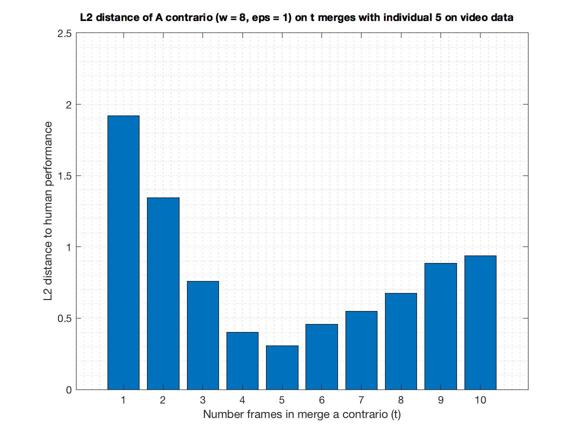}
         \caption{Subject 5}
    \end{subfigure}%
    \hfill
\end{figure}
\begin{figure} \ContinuedFloat
    \centering
    \begin{subfigure}[t]{0.5\textwidth}
      \centering
         \includegraphics[width=\textwidth]{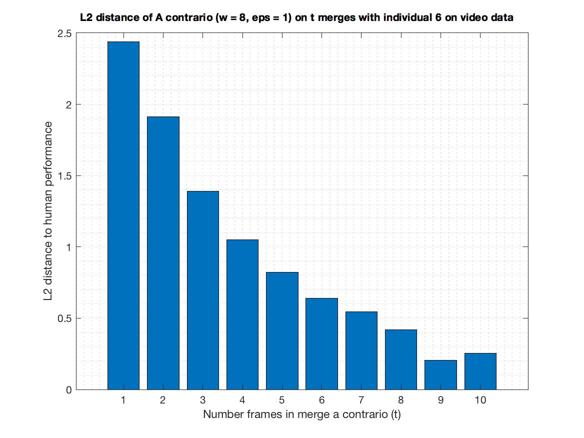}
         \caption{Subject 6}
    \end{subfigure}%
    \hfill
    \begin{subfigure}[t]{0.5\textwidth}
      \centering
         \includegraphics[width=\textwidth]{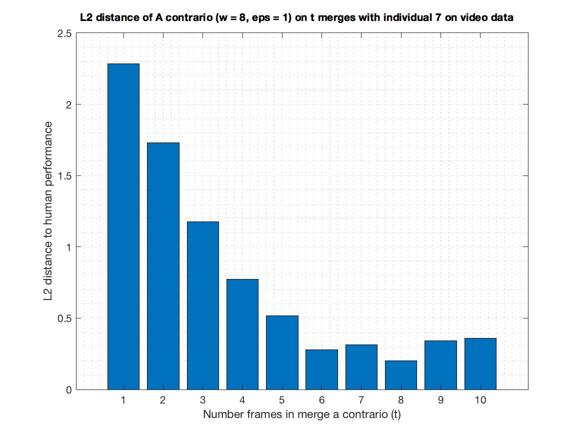}
         \caption{Subject 7}
    \end{subfigure}%
    \hfill
    \begin{subfigure}[t]{0.5\textwidth}
      \centering
         \includegraphics[width=\textwidth]{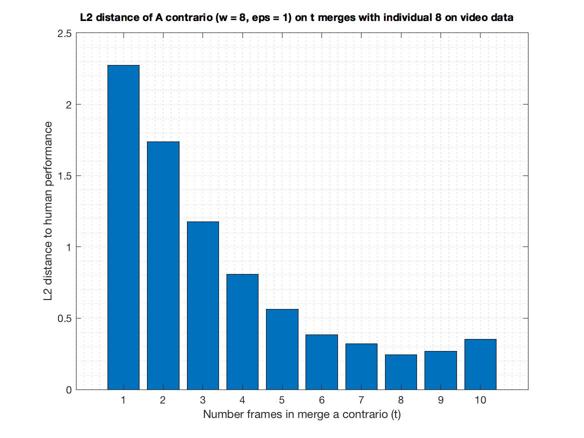}
         \caption{Subject 8}
    \end{subfigure}%
    \hfill
    \begin{subfigure}[t]{0.5\textwidth}
      \centering
         \includegraphics[width=\textwidth]{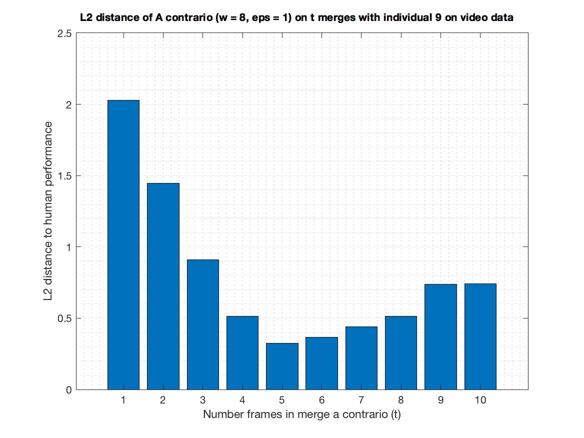}
         \caption{Subject 9}
    \end{subfigure}%
    \hfill
    \begin{subfigure}[t]{0.5\textwidth}
      \centering
         \includegraphics[width=\textwidth]{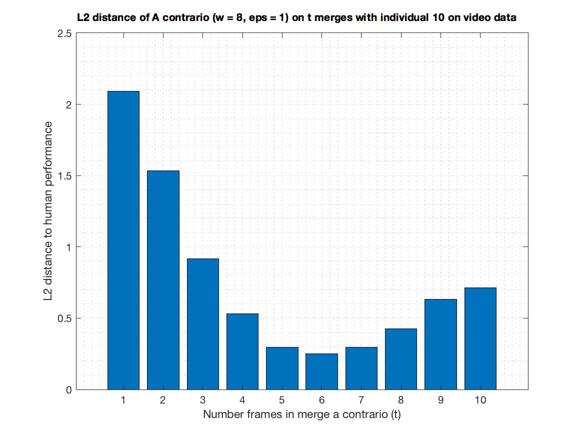}
         \caption{Subject 10}
    \end{subfigure}%
    \hfill
    \begin{subfigure}[t]{0.5\textwidth}
      \centering
         \includegraphics[width=\textwidth]{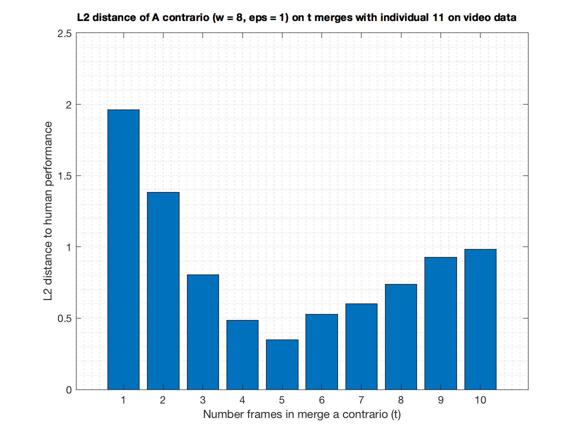}
         \caption{Subject 11}
    \end{subfigure}%
    \hfill
\end{figure}
\begin{figure} \ContinuedFloat
    \centering
    \begin{subfigure}[t]{0.5\textwidth}
      \centering
         \includegraphics[width=\textwidth]{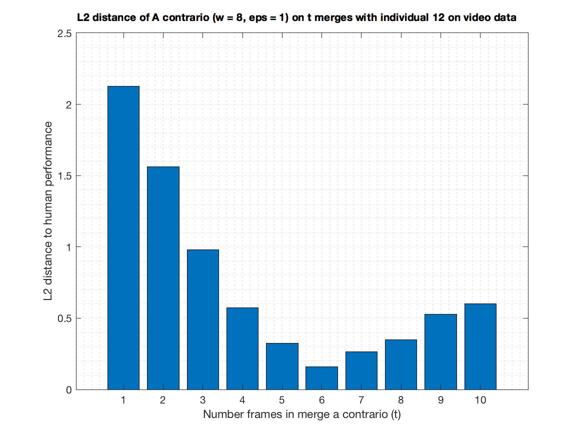}
         \caption{Subject 12}
    \end{subfigure}%
    \hfill
    \begin{subfigure}[t]{0.5\textwidth}
      \centering
         \includegraphics[width=\textwidth]{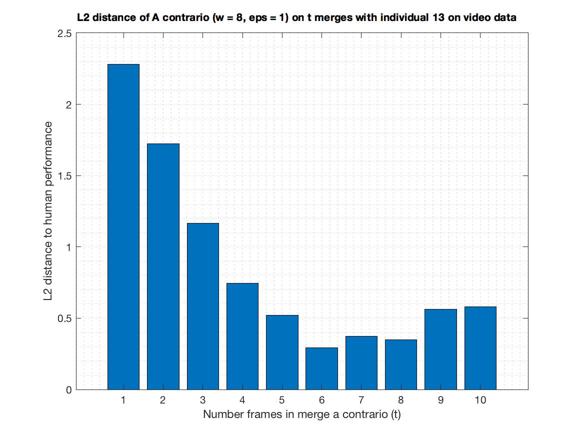}
         \caption{Subject 13}
    \end{subfigure}%
    \hfill
    \begin{subfigure}[t]{0.5\textwidth}
      \centering
         \includegraphics[width=\textwidth]{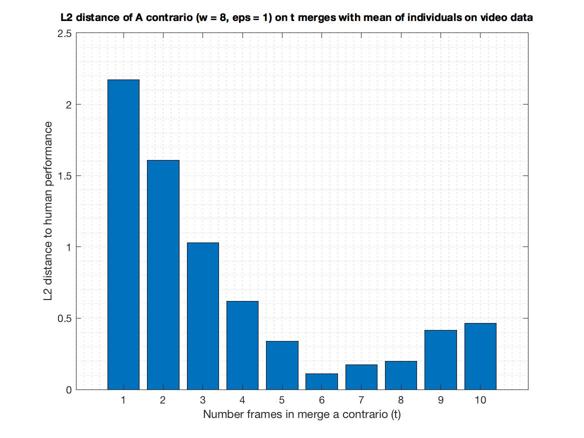}
         \caption{All subjects}
    \end{subfigure}%
    \hfill
    \caption{$L_2$ distances between the true positive rate curves of each subject and of those of each a contrario algorithm working on width $w=8$ and time integration $n_{f}\in\{1,\hdots,10\}$ frames. The time integrations that could suit our model are the time integrations giving the smallest $L_2$ error. There does not seem to be a single best value to take. Thus taking a range for the time integration seems wise. Qualitatively, for most individuals and on average, the range $n_{f}\in[6,8]$ frames seems like a good candidate range for optimal time integration. Recall that we cannot trust the results obtained by subjects 2, 6 and 8 since they guessed a lot thus giving a precision that is not close to $100\%$.}
    \label{fig: exp3 static L2 humans}
 \end{figure}

\clearpage
\newpage
\section{Dynamic Edge Experiment}
\label{sec: exp3.2}

In this Section we detail the experiment on the videos of jumping and moving dynamic edges. This will allow us to extract a further estimate of the time integration for our model of human perception. We will then compare this time integration with the previously obtained one on static edges and see whether they are compatible. From the comparison we extract a combined estimate for the time integration for our model. The experiments detailed here were approved by the Ethics committee of the Technion as conforming to the standards for performing psychophysics experiments on volunteering human subjects.

\paragraph{Stimuli} The stimuli is the random-dot video dynamic edge dataset we have generated and previously described. Recall that half of these videos are just noise without an underlying line. All subjects were presented with the entire dataset (480 videos). The order in which they are shown is chosen at random and therefore differs between subjects.

\paragraph{Apparatus} Each subject was tested with exactly the same display. The display is the screen of a MacBook Pro retina 13-inch display, from mid 2014. Each subject was seated directly in front of the screen. The images were displayed in a well-lit room. The average distance between the subjects' eyes and the screen was about $70\,\mathrm{cm}$. The $300\times 300$ pixels video is displayed having a size of $10.4 \,\mathrm{cm}\times10.4\,\mathrm{cm}$. This translates to a pixel size of approximately $0.35\,\mathrm{mm}\times0.35\,\mathrm{mm}$. On average the visual angle for observing the image is approximately $8.5°\times8.5°$. Below the video we display two buttons for user decision: a YES button and a NO button (see the procedure paragraph). The YES button is green with YES written in a large font on it. The NO button is red with NO written in a large font on it, the same font as on the YES button. Both buttons are horizontally aligned. The YES button is located on the left whereas the NO button is located on the right. Each button is of size $5.6\,\mathrm{cm}\times 2.8\,\mathrm{cm}$. The average visual angle for the buttons are approximately $4.6°\times2.3°$. The distance between the image and the horizontal line of boxes is $0.5\,\mathrm{cm}$ which translates to a visual angle of seperation of approximately $0.4°$. The buttons are separated by a distance of $5.8\,\mathrm{cm}$, which corresponds to a visual angle of approximately $4.7°$.  See Figure \ref{fig: screenshot exp 3.2} for a screenshot of the display.

\paragraph{Subjects} The experiments were performed on the first author and on the same other thirteen subjects as previously. All the other subjects were unfamiliar with psychophysical experiments. All subjects were international students of the university, coming from various countries around the globe including China, Vietnam, India, Germany, Greece, the United States of America and France. Gender parity was almost respected with eight males and six females. The age range of the subjects was between 20 and 31 years of age. All subjects had normal or corrected to normal vision. To the best of our knowledge, no subject had any mental condition. All other subjects were unaware of the aim of our study.

\paragraph{Procedure} The experiment is a yes-no task. Subjects were told that some of the videos they would be presented with will contain a line and some would not. The proportion of videos without a line was not given before-hand. However, subjects were told it was all right if they found that they had to click a lot on YES or a lot on NO and just focus on the current video at hand. Subjects were asked to click on the YES button if they believed that the video contained the line, and click NO if they thought it did not. They were encouraged to make a decision rather than waiting for the timer to run out. They were told that they were looking for a line of length two thirds of the width of the video frame and width one pixel. They were told that should the video be a video with the line in it, then the line would be there during the entire video. They were also told that if the line is present, it would jump regularly from a random position to another random position. They were told no video would contain any sequence were two underlying lines would be simultaneously present. In the entire experiment, a Matlab window covered the entire screen. In the centre we displayed the video. Below it, on the left would be a big green button with YES written on it and symmetrically on the right a big red button with NO written on it. Above the video was written the question \say{Is the line present?}, to which they had to answer by clicking on the appropriate button. For their response, subjects had the choice between using the built-in touchpad of the MacBook Pro or to use a Asus N6-Mini mouse. If the subject clicked on any area that is not within one of the two buttons then that click was dismissed. If the subject clicked on any of the buttons then the videos stops and moves on to the next one. Subjects had a 10 second limit to answer for each video, after which the next video would be automatically shown. If subjects ran out of time the next video was automatically shown. Subjects were not shown a progress bar. Subjects could not take a break once having started the experiment, but were allowed to abandon the experiment at any time, should they desire it. The time for each click to be made, along with the corresponding button clicked was saved.

\paragraph{Discussion} In order to estimate detection performance, we should normally look at both the true positive rate of the subjects with respect to $p_f$ but also the false alarm rate. Indeed, if we only look at one of these values and discard the other, we could end up rating some performance as excellent when in fact they are terrible. For instance imagine a subject that always finds the edge but when there isn't any always return a false alarm: he always clicks on YES. Then he detects all edge data, thus $100\%$ true positive rate, but also a $100\%$ false alarm rate, which is bad: the true decisions are not reliable. Thus if we omit the false alarm rate and only look at the true positive rate we get an incorrect evaluation of the performance. Nevertheless, this is what we are going to do. The justification for such a dangerous procedure is the following. In these experiments too, the false alarm rate of the subjects is particularly low, close to $0\%$. Also, it is similar to the false alarm rate of the a contrario algorithms. As such, we can disregard the false alarm rate and only concentrate on the true positive rate. The true positive rate with respect to $p_f$ provides a performance vector or curve for each subject. We can then compare humans and a contrario by comparing these vectors or curves. This is somewhat equivalent to fitting the true positive vector performance of humans to the columns of the previous performances on the static data (up to dilation since $p_f$ almost linearly increases with the number of frames merged). The comparison measure is chosen to be the $L_2$ distance between the performance vectors. For better display, we plot the $L_2$ distance of the subjects' performance and the performance of the corresponding a contrario algorithms based on various integration times $n_{f}$. The algorithms with lower distance thus perform the most similarly to the human and thus give us a candidate time integration. In practice, we see that for most subjects the algorithms with time integration in $n_f\in[4,10]$ frames perform similarly and the best and so this time integration range is a good candidate for the time integration of humans. This corresponds to a time integration of $\Delta t\in[0.13,0.33]$ seconds. We display the global confusion matrices and false alarm rates of the subjects in Figure \ref{fig: exp3 dynamic confusion humans}. We display the true positive rates with respect to $p_f$ in the Figure \ref{fig: exp3 dynamic TPR humans}. We display in Figure \ref{fig: exp3 dynamic L2 humans} the $L_2$ distance between the true positive rate curves of the subjects with respect to those of the a contrario algorithms working on $n_{f}$ frames for time integration.

\begin{figure}
    \centering
    \includegraphics[width=\textwidth]{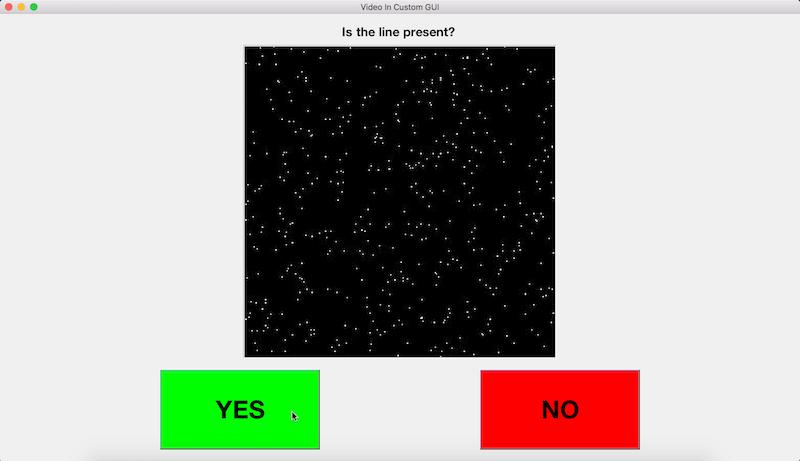}
    \hfill
    \caption[Short version]{Screenshot of the display.}
    \label{fig: screenshot exp 3.2}
\end{figure}

\begin{figure}
    \centering
    \begin{subfigure}[t]{0.3\textwidth}
      \centering
         \includegraphics[width=\textwidth]{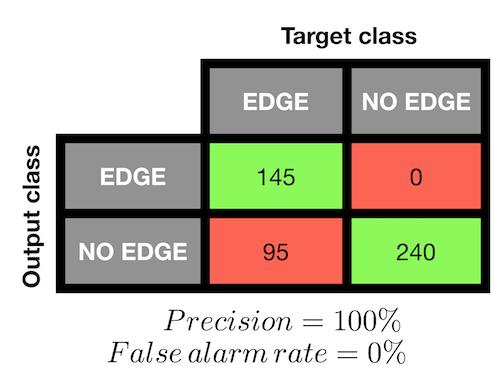}
         \caption{Subject 0}
    \end{subfigure}%
    \hfill
    \begin{subfigure}[t]{0.3\textwidth}
      \centering
         \includegraphics[width=\textwidth]{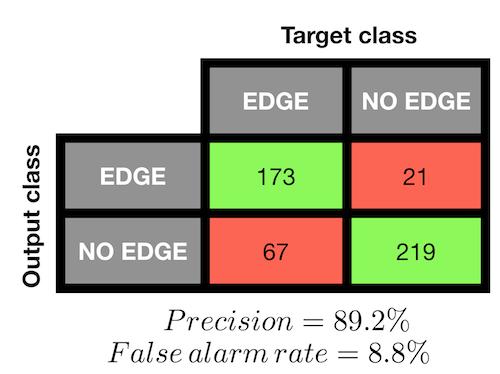}
         \caption{Subject 1}
    \end{subfigure}%
    \hfill
    \begin{subfigure}[t]{0.3\textwidth}
      \centering
         \includegraphics[width=\textwidth]{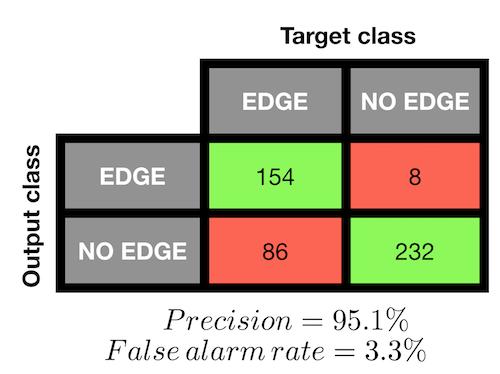}
         \caption{Subject 2}
    \end{subfigure}%
    \hfill
    \begin{subfigure}[t]{0.3\textwidth}
      \centering
         \includegraphics[width=\textwidth]{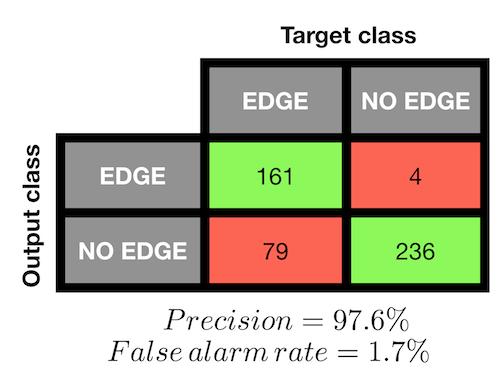}
         \caption{Subject 3}
    \end{subfigure}%
    \hfill
    \begin{subfigure}[t]{0.3\textwidth}
      \centering
         \includegraphics[width=\textwidth]{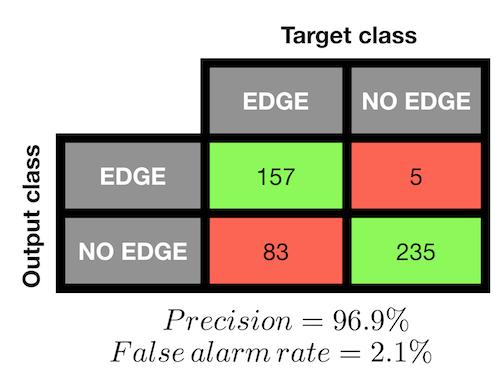}
         \caption{Subject 4}
    \end{subfigure}%
    \hfill
    \begin{subfigure}[t]{0.3\textwidth}
      \centering
         \includegraphics[width=\textwidth]{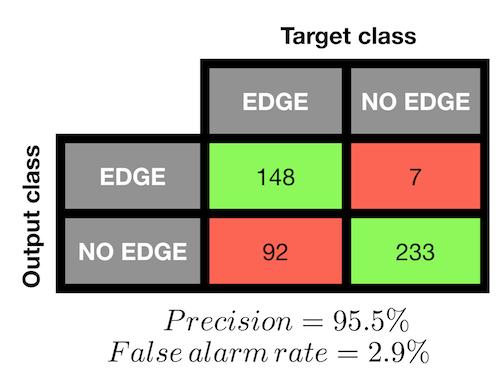}
         \caption{Subject 5}
    \end{subfigure}%
    \hfill
    \begin{subfigure}[t]{0.3\textwidth}
      \centering
         \includegraphics[width=\textwidth]{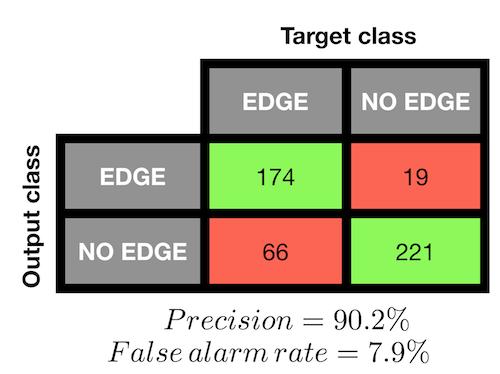}
         \caption{Subject 6}
    \end{subfigure}%
    \hfill
    \begin{subfigure}[t]{0.3\textwidth}
      \centering
         \includegraphics[width=\textwidth]{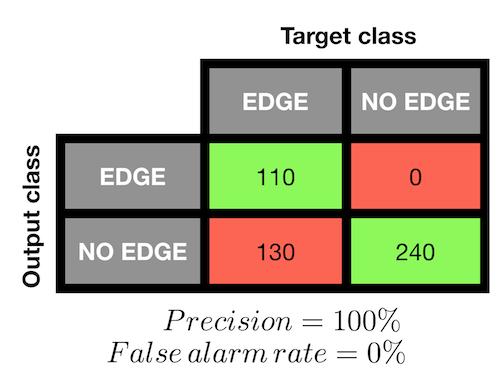}
         \caption{Subject 7}
    \end{subfigure}%
    \hfill
    \begin{subfigure}[t]{0.3\textwidth}
      \centering
         \includegraphics[width=\textwidth]{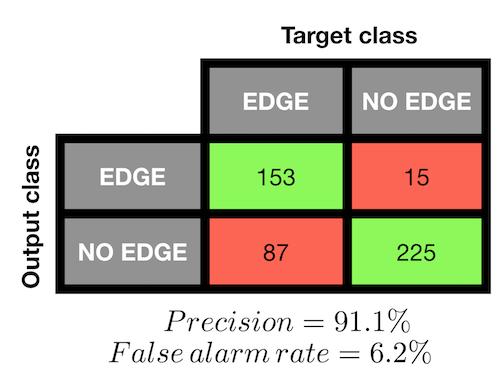}
         \caption{Subject 8}
    \end{subfigure}%
    \hfill
    \begin{subfigure}[t]{0.3\textwidth}
      \centering
         \includegraphics[width=\textwidth]{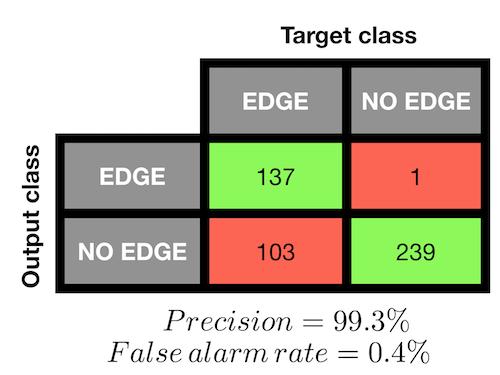}
         \caption{Subject 9}
    \end{subfigure}%
    \hfill
    \begin{subfigure}[t]{0.3\textwidth}
      \centering
         \includegraphics[width=\textwidth]{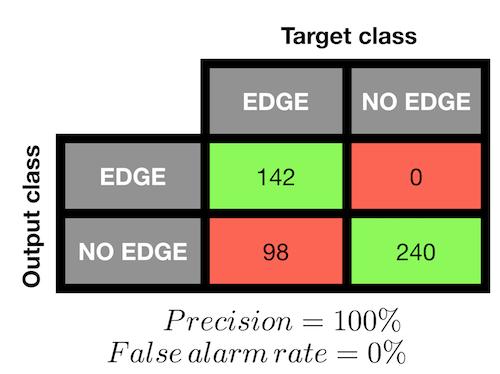}
         \caption{Subject 10}
    \end{subfigure}%
    \hfill
    \begin{subfigure}[t]{0.3\textwidth}
      \centering
         \includegraphics[width=\textwidth]{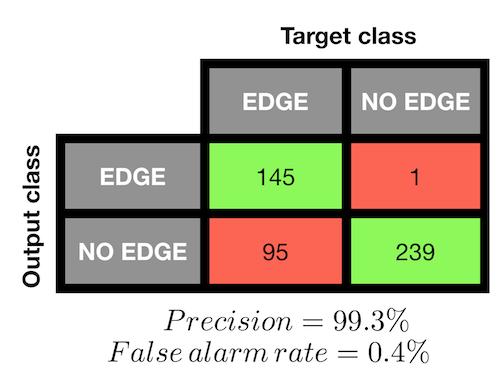}
         \caption{Subject 11}
    \end{subfigure}%
    \hfill
    \begin{subfigure}[t]{0.3\textwidth}
      \centering
         \includegraphics[width=\textwidth]{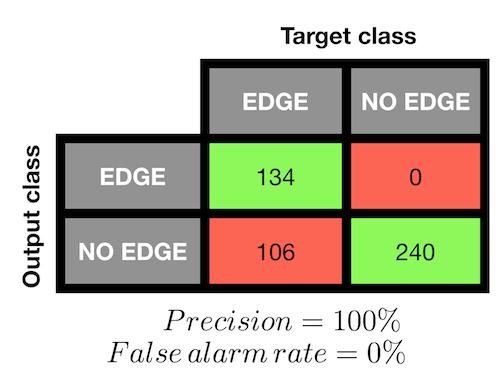}
         \caption{Subject 12}
    \end{subfigure}%
    \hfill
    \begin{subfigure}[t]{0.3\textwidth}
      \centering
         \includegraphics[width=\textwidth]{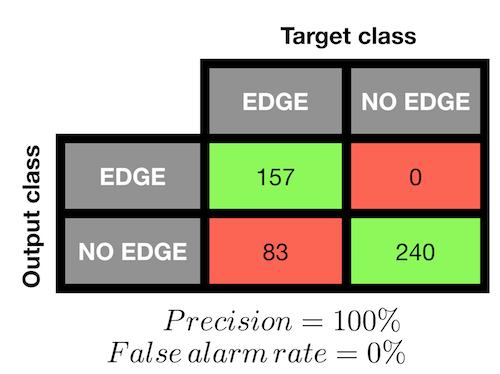}
         \caption{Subject 13}
    \end{subfigure}%
    \hfill
    \begin{subfigure}[t]{0.3\textwidth}
      \centering
         \includegraphics[width=\textwidth]{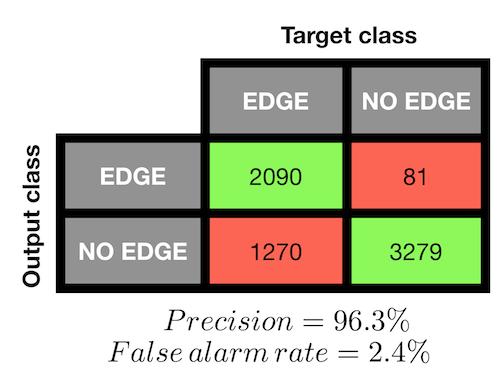}
         \caption{All subjects}
    \end{subfigure}%
    \hfill
    \caption{Global confusion matrices of each subject and of the average of all subjects. In particular, note how low the false alarm rate is for each subject. Note also how the subjects who had previously guessed a lot stopped doing so in this experiment since we stressed out that it was not all right to guess. It is not worrying to have many true negatives as the dataset is generated such that many of the true videos have low $p_f$ and as such should make the subjects fail. Since the false alarm rate is particularly low, the precision of each subject is close to $100\%$ and as such we can trust a lot a decision to detect an edge. Also, the false alarm rates are similar to those of the a contrario algorithms. This leads us to only study the true positive rates. }
    \label{fig: exp3 dynamic confusion humans}
 \end{figure}

\begin{figure}
    \centering
    \begin{subfigure}[t]{0.5\textwidth}
      \centering
         \includegraphics[width=\textwidth]{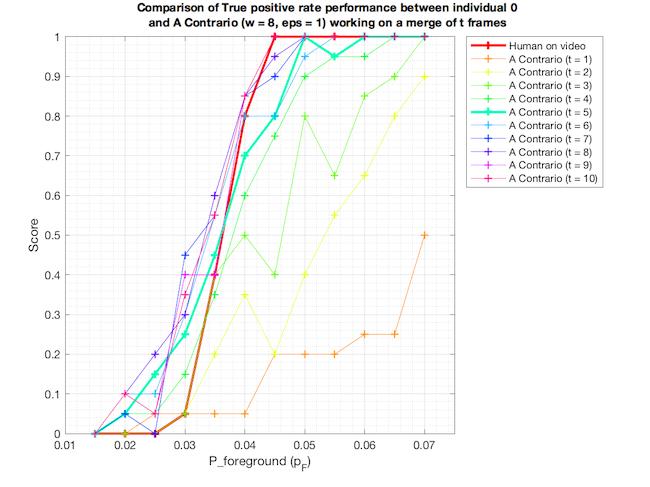}
         \caption{Subject 0}
    \end{subfigure}%
    \hfill
    \begin{subfigure}[t]{0.5\textwidth}
      \centering
         \includegraphics[width=\textwidth]{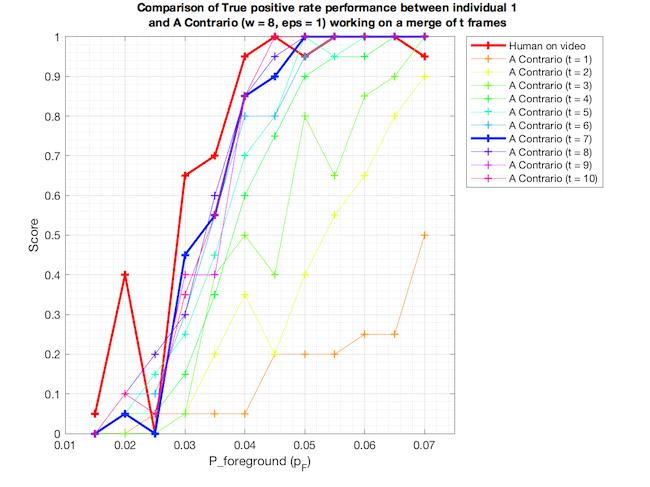}
         \caption{Subject 1}
    \end{subfigure}%
    \hfill
    \begin{subfigure}[t]{0.5\textwidth}
      \centering
         \includegraphics[width=\textwidth]{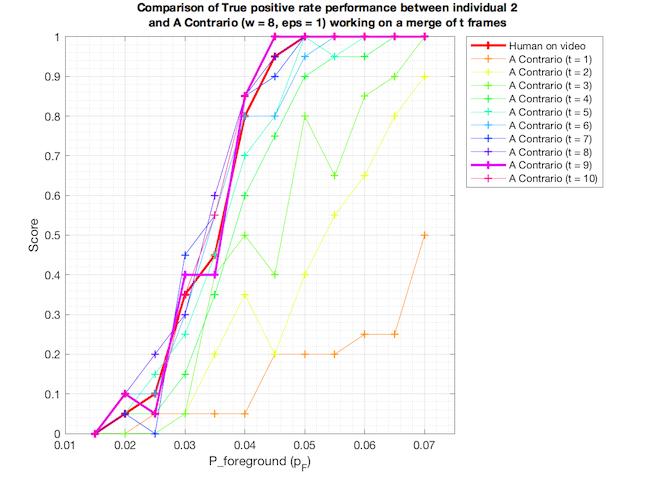}
         \caption{Subject 2}
    \end{subfigure}%
    \hfill
    \begin{subfigure}[t]{0.5\textwidth}
      \centering
         \includegraphics[width=\textwidth]{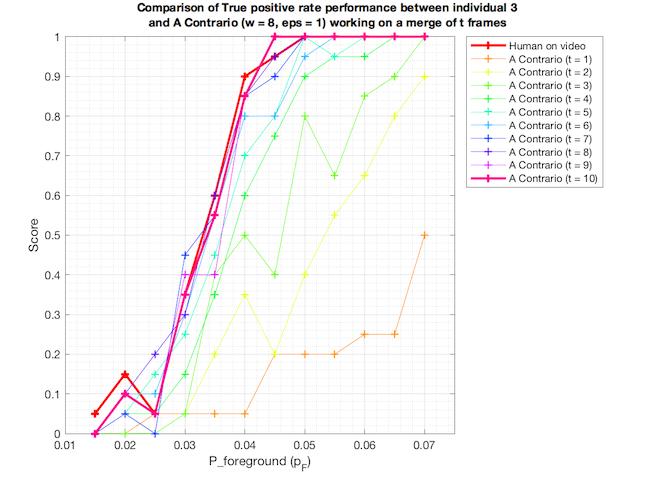}
         \caption{Subject 3}
    \end{subfigure}%
    \hfill
    \begin{subfigure}[t]{0.5\textwidth}
      \centering
         \includegraphics[width=\textwidth]{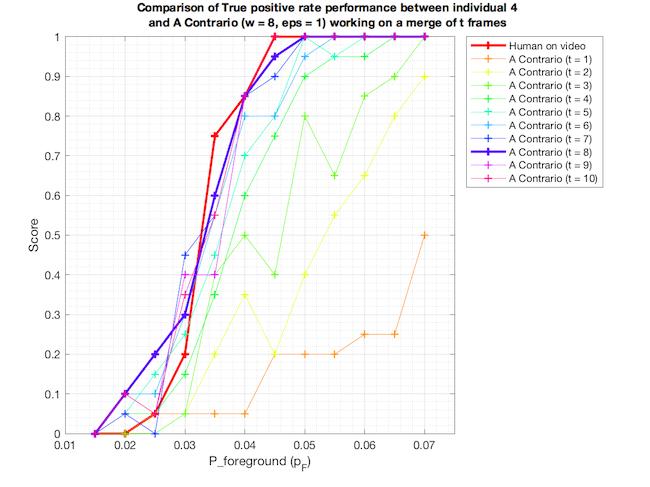}
         \caption{Subject 4}
    \end{subfigure}%
    \hfill
    \begin{subfigure}[t]{0.5\textwidth}
      \centering
         \includegraphics[width=\textwidth]{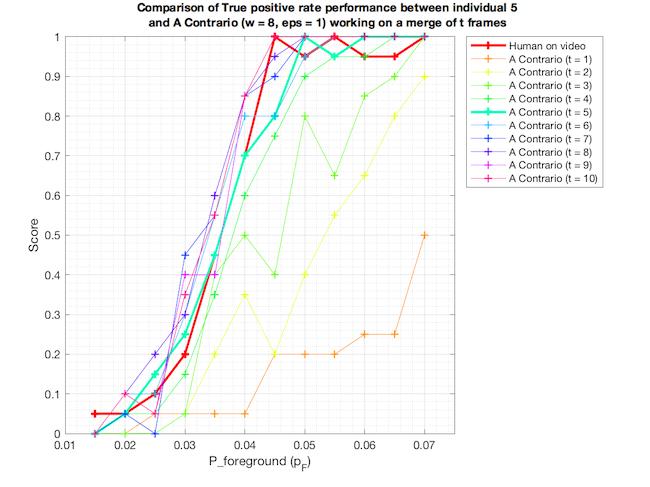}
         \caption{Subject 5}
    \end{subfigure}%
    \hfill
\end{figure}
\begin{figure} \ContinuedFloat
    \centering
    \begin{subfigure}[t]{0.5\textwidth}
      \centering
         \includegraphics[width=\textwidth]{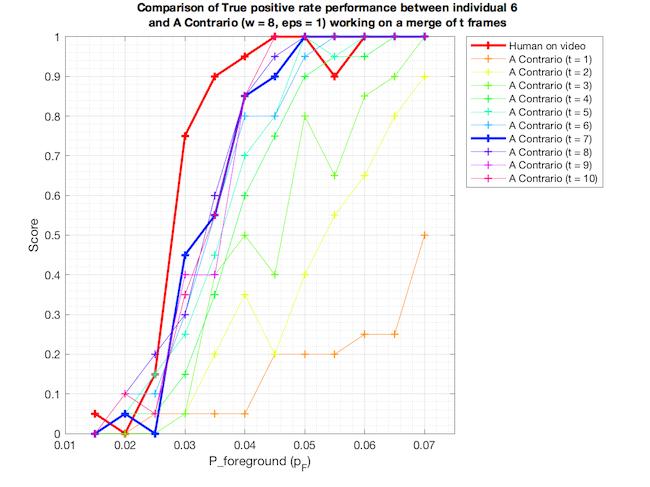}
         \caption{Subject 6}
    \end{subfigure}%
    \hfill
    \begin{subfigure}[t]{0.5\textwidth}
      \centering
         \includegraphics[width=\textwidth]{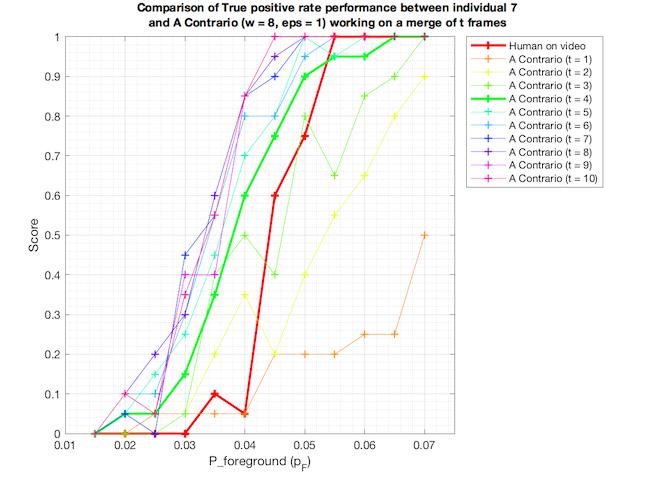}
         \caption{Subject 7}
    \end{subfigure}%
    \hfill
    \begin{subfigure}[t]{0.5\textwidth}
      \centering
        \includegraphics[width=\textwidth]{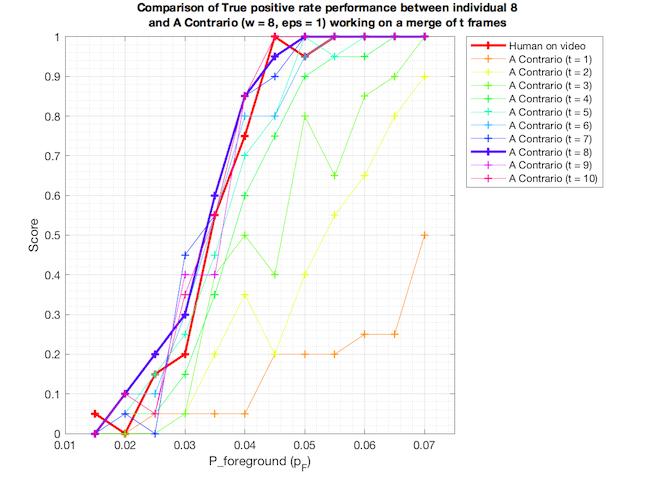}
         \caption{Subject 8}
    \end{subfigure}%
    \hfill
    \begin{subfigure}[t]{0.5\textwidth}
      \centering
         \includegraphics[width=\textwidth]{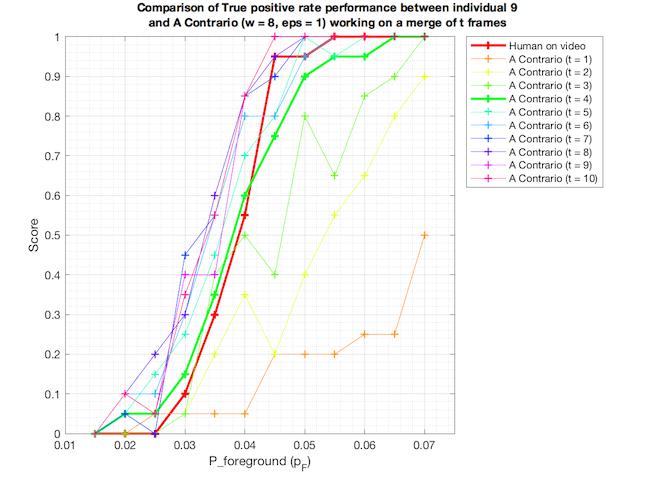}
         \caption{Subject 9}
    \end{subfigure}%
    \hfill
    \begin{subfigure}[t]{0.5\textwidth}
      \centering
         \includegraphics[width=\textwidth]{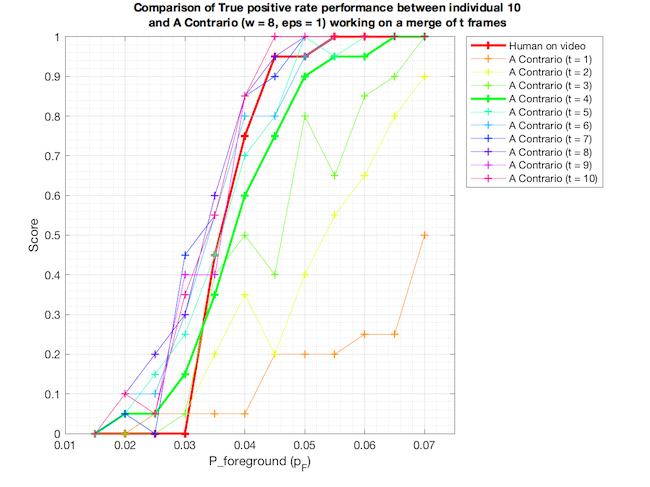}
         \caption{Subject 10}
    \end{subfigure}%
    \hfill
    \begin{subfigure}[t]{0.5\textwidth}
      \centering
        \includegraphics[width=\textwidth]{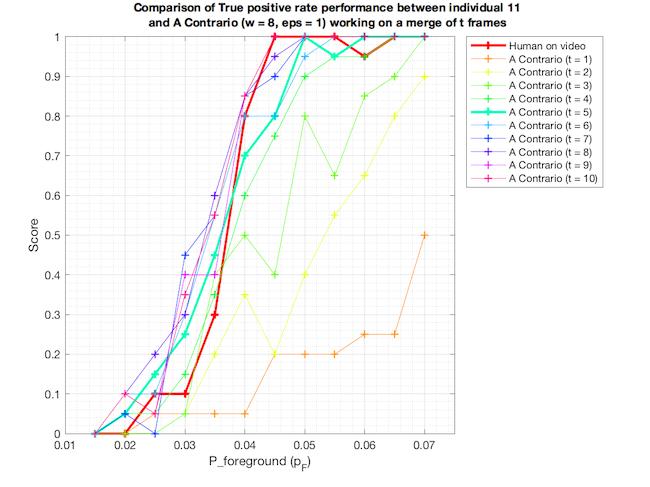}
         \caption{Subject 11}
    \end{subfigure}%
    \hfill
\end{figure}
\begin{figure} \ContinuedFloat
    \centering
    \begin{subfigure}[t]{0.5\textwidth}
      \centering
        \includegraphics[width=\textwidth]{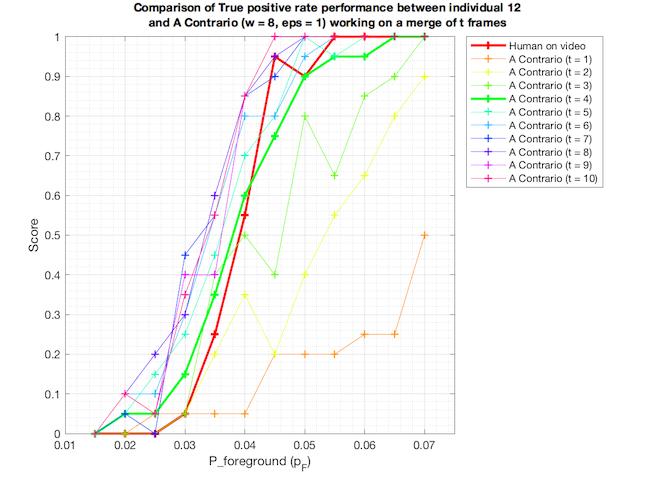}
         \caption{Subject 12}
    \end{subfigure}%
    \hfill
    \begin{subfigure}[t]{0.5\textwidth}
      \centering
         \includegraphics[width=\textwidth]{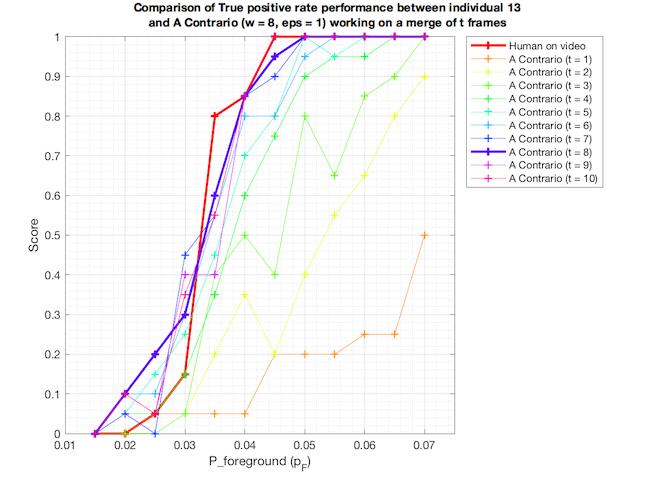}
         \caption{Subject 13}
    \end{subfigure}%
    \hfill
    \begin{subfigure}[t]{0.5\textwidth}
      \centering
         \includegraphics[width=\textwidth]{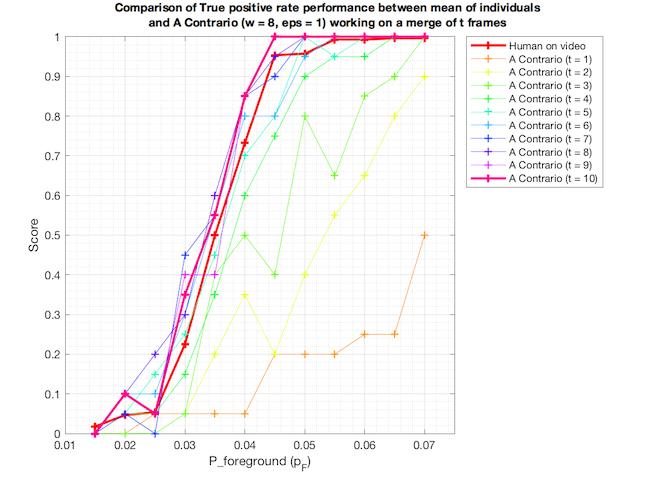}
         \caption{All subjects}
    \end{subfigure}%
    \hfill
    \caption{True positive rate of each subject with respect to $p_f$. The true positive rates increase from close to $0$ for low $p_f$ to close to $1$ for high $p_f$. We superimpose to the human performance curve the curves previously obtained by the single widths a contrarios working on $w=8$. The curves obtained for the subjects are somewhat similar to those obtained by the single width a contrario processes. In each plot, we have thickened the a contrario algorithm with single time integration $n_{f}\in\{1,\hdots,10\}$ frames such that its curve (or discrete vector) is closest in $L_2$ distance to the curve (or discrete vector) obtained by the subject. The best time integration seems to fluctuate between subjects and it is not clear from these plots how much better the best curve is compared to the others. See the next Figure for more clarity.}
    \label{fig: exp3 dynamic TPR humans}
 \end{figure}

\begin{figure}
    \centering
    \begin{subfigure}[t]{0.5\textwidth}
      \centering
         \includegraphics[width=\textwidth]{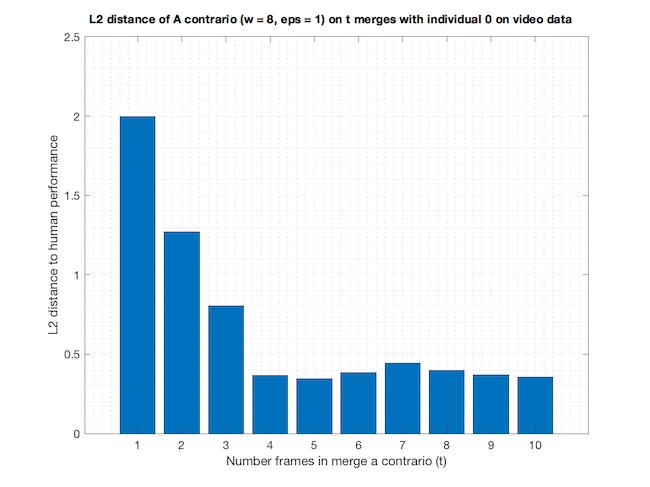}
         \caption{Subject 0}
    \end{subfigure}%
    \hfill
    \begin{subfigure}[t]{0.5\textwidth}
      \centering
         \includegraphics[width=\textwidth]{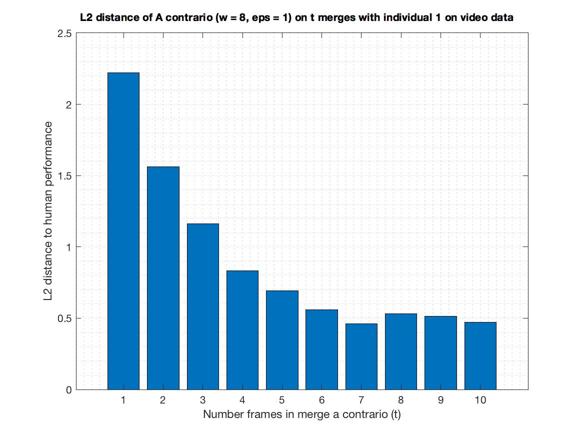}
         \caption{Subject 1}
    \end{subfigure}%
    \hfill
    \begin{subfigure}[t]{0.5\textwidth}
      \centering
         \includegraphics[width=\textwidth]{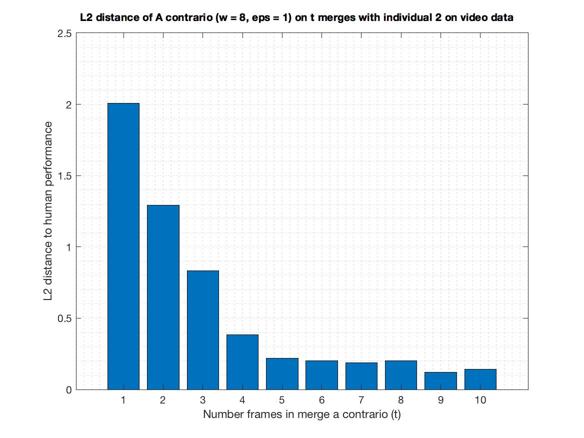}
         \caption{Subject 2}
    \end{subfigure}%
    \hfill
    \begin{subfigure}[t]{0.5\textwidth}
      \centering
         \includegraphics[width=\textwidth]{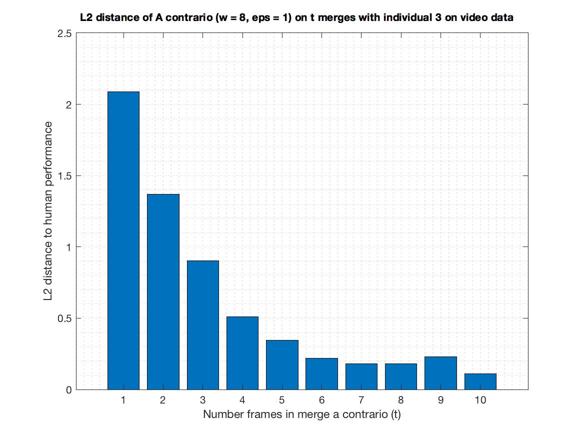}
         \caption{Subject 3}
    \end{subfigure}%
    \hfill
    \begin{subfigure}[t]{0.5\textwidth}
      \centering
         \includegraphics[width=\textwidth]{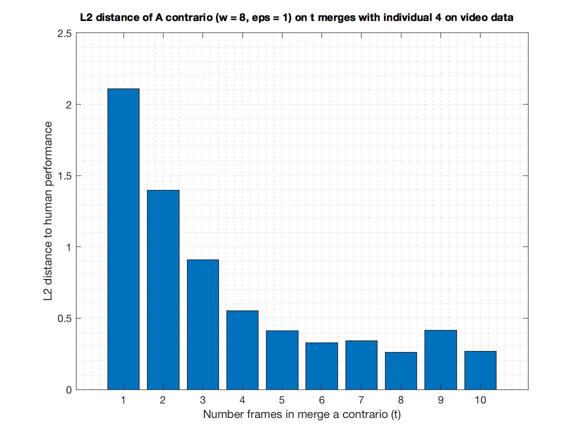}
         \caption{Subject 4}
    \end{subfigure}%
    \hfill
    \begin{subfigure}[t]{0.5\textwidth}
      \centering
         \includegraphics[width=\textwidth]{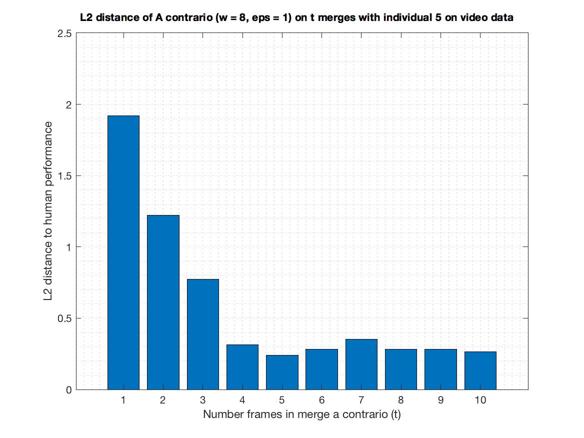}
         \caption{Subject 5}
    \end{subfigure}%
    \hfill
\end{figure}
\begin{figure} \ContinuedFloat
    \centering
    \begin{subfigure}[t]{0.5\textwidth}
      \centering
         \includegraphics[width=\textwidth]{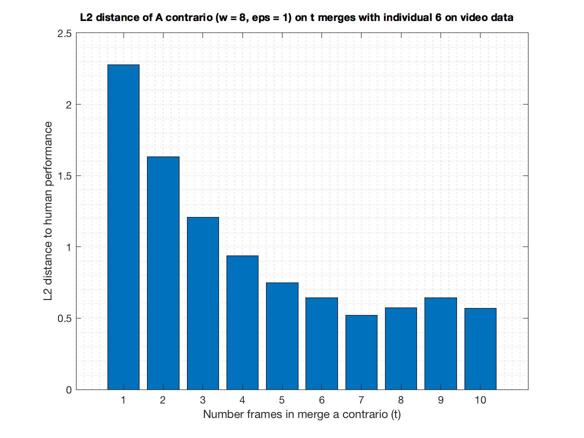}
         \caption{Subject 6}
    \end{subfigure}%
    \hfill
    \begin{subfigure}[t]{0.5\textwidth}
      \centering
         \includegraphics[width=\textwidth]{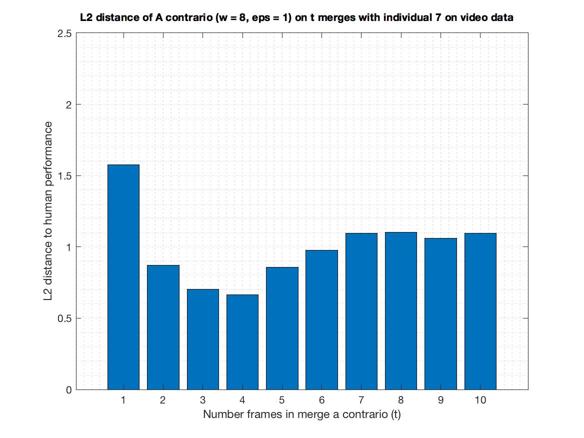}
         \caption{Subject 7}
    \end{subfigure}%
    \hfill
    \begin{subfigure}[t]{0.5\textwidth}
      \centering
         \includegraphics[width=\textwidth]{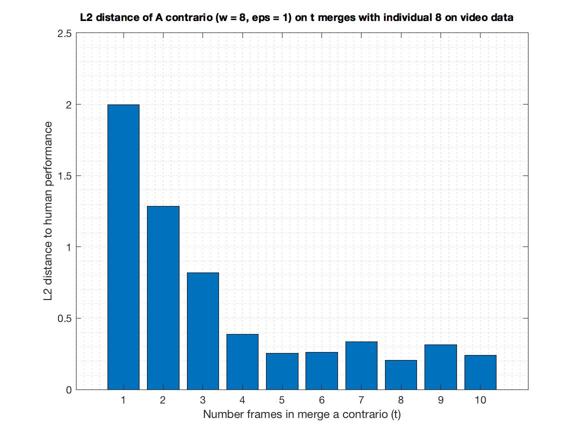}
         \caption{Subject 8}
    \end{subfigure}%
    \hfill
    \begin{subfigure}[t]{0.5\textwidth}
      \centering
         \includegraphics[width=\textwidth]{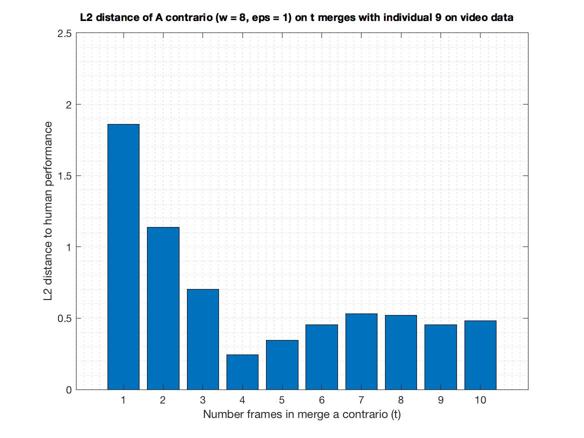}
         \caption{Subject 9}
    \end{subfigure}%
    \hfill
    \begin{subfigure}[t]{0.5\textwidth}
      \centering
         \includegraphics[width=\textwidth]{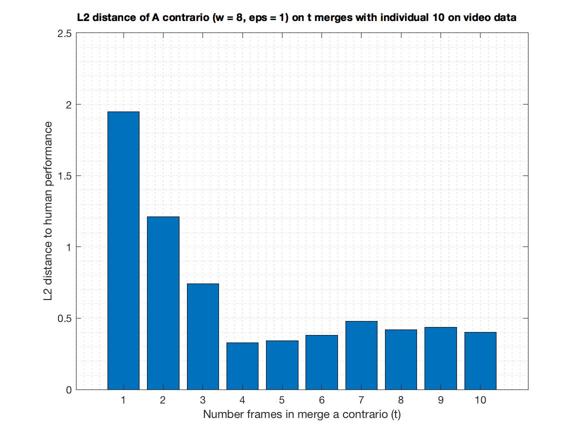}
         \caption{Subject 10}
    \end{subfigure}%
    \hfill
    \begin{subfigure}[t]{0.5\textwidth}
      \centering
         \includegraphics[width=\textwidth]{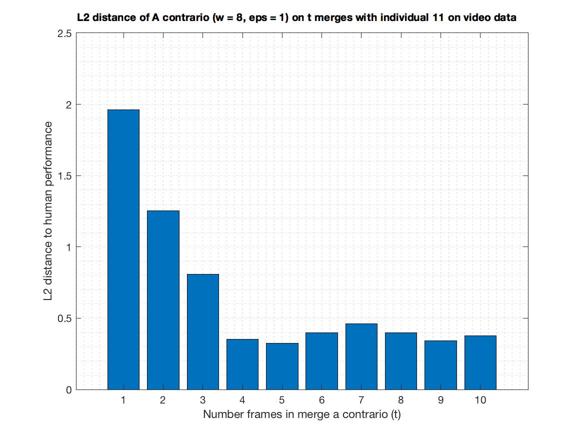}
         \caption{Subject 11}
    \end{subfigure}%
    \hfill
\end{figure}
\begin{figure} \ContinuedFloat
    \centering
    \begin{subfigure}[t]{0.5\textwidth}
      \centering
         \includegraphics[width=\textwidth]{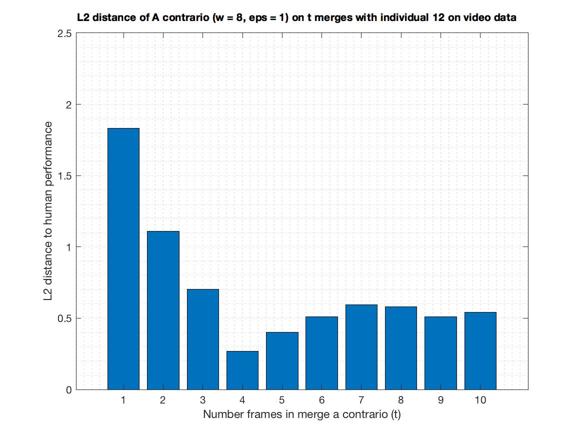}
         \caption{Subject 12}
    \end{subfigure}%
    \hfill
    \begin{subfigure}[t]{0.5\textwidth}
      \centering
         \includegraphics[width=\textwidth]{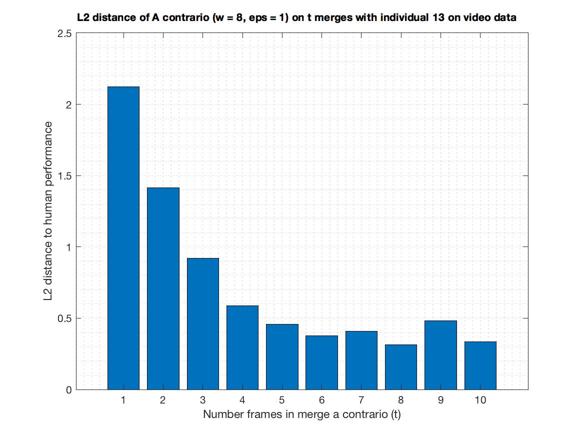}
         \caption{Subject 13}
    \end{subfigure}%
    \hfill
    \begin{subfigure}[t]{0.5\textwidth}
      \centering
         \includegraphics[width=\textwidth]{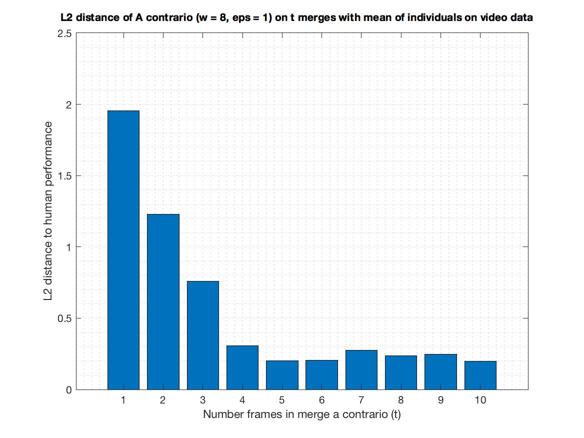}
         \caption{All subjects}
    \end{subfigure}%
    \hfill
    \caption{$L_2$ distances between the true positive rate curves of each subject and of those of each a contrario algorithm working on width $w=8$ and time integration $n_{f}\in\{1,\hdots,10\}$ frames. The time integrations that could suit our model are the time integration giving the smallest $L_2$ error. There does not seem to be a single best value to take. Thus taking a range for the time integration seems wise. Qualitatively, for most individuals and on average, the range $n_{f}\in[4,10]$ frames seems like a good candidate range for optimal time integration.}
    \label{fig: exp3 dynamic L2 humans}
 \end{figure}

\paragraph{Discussion of both video experiments} The dynamic experiments suggest that the sequential time integration by merging images and then performing a single width a contrario process with $w=8$ pixels (or $\theta = 0.23°$ in terms of visual angle) is reasonable if we select $n_f \in[6,8]$ frames for the static edge and $n_f\in[4,10]$ for the dynamic edge. However, in our model, the time integration should not depend on the dynamics of the edge. This choice could be controversial for a comparison with humans as in psychophysics, it is debatable whether the brain uses separate and very different channels for processing a static scene and moving visual data. Nevertheless, saccadic eye movements make every scene dynamic, including static displays, and it is possible that this allows us to see static objects that would otherwise perceptually fade \cite{martinez2004role}. As such, this suggests that humans process visual information as only motion information, and that intrinsic static information is converted to a dynamic input by the human eye. This gives us legitimacy for choosing a computational model for perception that does not depend on whether the input scene is static or not. We then choose a time integration that suits well the findings on the dynamic and static data, thus by taking the intersection of both regions of good candidates for the time integration we get that the time interval of $n_f \in[6,8]$ frames i.e. $\Delta t\in[0.2,0.27]$ seconds is a good candidate for the time integration. This value is particularly interesting. Johansson \cite{johansson1976spatio} found that within $0.2$ seconds all subjects were able to correctly group and understand his biological point light displays which suggests that this is the necessary time for grouping and understanding. Furthermore, it is known that the eye fixation duration, the time for which the eye focuses on an area when humans search through an image, corresponds to this range \cite{hooge1998adjustment,martinez2004role}. On the other hand, traditional integration times are considerably smaller. An explanation could be that in the classical experiments for measuring the time needed for visual integration, the visual data is entirely present at each instant during the duration of display. However, in our problem, the data is not seen in each frame of the video, thus it is logical that higher order mechanisms occur, thus increasing the time needed for integration. Further psychophysical experiments using data similar to ours should be done to better understand the relevance of the chosen time integration for our model with respect to the human brain. Moreover, the visual angle, $\theta = 0.23°$ or equivalently $\theta = 13.8$ arc minutes, is consistent with some psychophysical experiments \cite{unuma2010spatiotemporal}. Finally, we have shown that empirically humans perform similarly to a sequential time integrator followed by a single width a contrario algorithm. This gives us a simple model for human perception and further insight in the importance of Gestalt decisions of unlikeliness in human perception.

\clearpage
\newpage
\section{Future Work}
\label{sec: Future work}

In this report, we have achieved two objectives. The first was to design an algorithm to process very noisy random-dot video data. The second was to get a simple computational model for human perception on such data.

Future work would have to smooth the output of the algorithm. A \say{smoothness} assumption about the displacement of foreground objects was used in our algorithm, only implicitly in the localisation of the foreground in a small patch over merges of frames. A first possibility would be to filter the detected lines. This is a post-processing idea. A second more interesting approach would be to use such prior knowledge in a direct way in the a contrario algorithm. For this, we could stop splitting uniformly $\varepsilon$ into $N_T$ equal parts, of weight $\frac{\varepsilon}{N_T}$. Indeed, thanks to prior knowledge to where the edge was detected in previous frames, we could now use a weighted splitting of $\varepsilon$ that is not uniform, with lower weights in areas were we expect the line to be and high weights in areas where we do not expect it to be. This would imply that we require many more points for a detection in unexpected areas than in expected ones. Areas were we expect the line to be could be the area around the previous detection (continuous assumption) or an area displaced according to higher Taylor expansion estimations of the movement of the edge (smoothness assumption). The weights could be derived using boolean masks or more interestingly using continuous fields. For instance, we could use Gaussian weights, which somewhat closely resembles to Kalman filtering \cite{kalman1960new}. Splitting $\varepsilon$ not equally in a contrario has never been done to the best of our knowledge and seems like an ideal solution to smooth the results of the a contrario algorithm. A post-processing filtering could after also be done to further enforce smoothness as well.

Concerning the computational model of the human visual perception, a contrario is not the only possible candidate. Currently, great emphasis is made in the Statistics and Computer Science community on Deep Learning and its outstanding performance and inner workings are not fully understood. These systems manage to learn in almost any scenario how to do detection on image or video data, given sufficient data, to allow intelligent tweaking of network parameters. Therefore it would be interesting to see how Deep Learning would perform on our data, especially in conjunction with convolutional neural networks (CNN). Our intuition is that CNNs will fail to recover the foreground on a per frame basis as the number of points on the foreground is too low, thus convolving will not create much overlap between foreground points compared to background points and so estimating different densities will fail. However, by studying 3D patches of data and performing 3D convolutions, a well designed CNN might be able to recover the foreground. Nevertheless, we believe that it would require many training samples and a lot of tweaking of parameters, compared to the elegant, parameter-less, and train-less a contrario theory. Humans manage to process very well such data when having never perceived any data similar to it. This goes against, for comparing neural networks and the human brain, using a heavy training of neural networks using large datasets, and might be more in adequacy with learning theory using a small number of examples such as one-shot learning. Furthermore, Deep Learning is intrinsically a non Gestalt approach although some authors try to work holistically, but rather a bottom up approach.  Finally, we do not believe that Deep Learning will be able to perform comparably with humans. It will first perform much worse and then eventually brilliantly outperform humans. We would be very lucky if we found a dataset (described by its size and the type of data we are training on: edges, circles...) such that after training and convergence we reach similar performance to humans. For those reticent to the use of Deep Learning, other mathematical or computational approaches could also be studied. In practice, we found that the Hough transform, which is a weaker form than the a contrario framework except in the way the hypothesis are generated, was also a suitable candidate for recovering straight lines in our videos (see demo \cite{dagesHoughDemo}).

\begin{figure}
    \centering
    \begin{subfigure}[t]{0.7\textwidth}
      \centering
         \includegraphics[width=\textwidth]{Chapters/Pictures/model_data_acquisition.jpg}
         \caption{Data acquisition model}
    \end{subfigure}%
    \hfill
    \begin{subfigure}[t]{\textwidth}
      \centering
         \includegraphics[width=\textwidth]{Chapters/Pictures/model_a_contrario_merge.jpg}
         \caption{A contrario algorithm model on random-dot videos}
    \end{subfigure}%
    \hfill
    \begin{subfigure}[t]{\textwidth}
      \centering
         \includegraphics[width=\textwidth]{Chapters/Pictures/model_human.jpg}
         \caption{Black box model of the human perceptual system}
    \end{subfigure}%
    \hfill
    \begin{subfigure}[t]{\textwidth}
      \centering
         \includegraphics[width=\textwidth]{Chapters/Pictures/model_human_a_contrario.jpg}
         \caption{Algorithmic model for the human perceptual system}
    \end{subfigure}%
    \hfill
    \begin{subfigure}[t]{\textwidth}
      \centering
         \includegraphics[width=\textwidth]{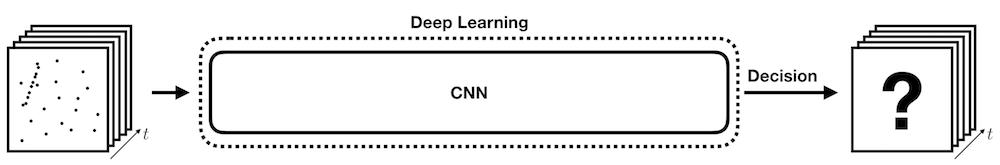}
         \caption{Deep Learning black box model}
    \end{subfigure}%
    \hfill
\caption{Summary of our models for perceiving our random-dot videos. Other interesting models exist. In particular, Deep Learning could provide a powerful model from an algorithmic perspective at the cost of tricky training and losing the full understanding of the underlying process and the meaningfulness of a detection. Also, it is not clear whether Deep Learning can provide an accurate model that performs similarly to humans, without overperforming or underperforming them.}
\label{fig: all models}
\end{figure}

Back to the a contrario process, we should also study its performance on more complex models, such as circles, parametric curves and non parametric curves. Preliminary tests on circles suggest that our model still holds, but proper experiments should be done. The a contrario framework must be better developed to handle non parametric curves as for now, up to the best of our knowledge, a contrario can handle such curves in very specific conditions, although we are optimistic that it will adapt to more complex possibilities. Also, when studying complex models, we should see whether we perform an a contrario testing all these complex models, and doing so, with high $N_T$, or whether we could model our perception by several independent a contrario algorithms working in parallel so that each works on a separate model, such as a line detector a contrario, a circle detector a contrario etc. This would allow $N_T$ to not explode. Indeed, as we have shown, our human perception is well adapted to the line-only a contrario model for edge structures. But we know that we effortlessly perceive circle, for example, even if we are asked to look for lines. However, if we model our perception by a single a contrario looking for all the complex shapes at once, $N_T$ might drastically increase, so that the algorithm might not detect lines anymore, even in cases when humans do. Further work should be done in this direction.

Another interesting phenomenon we have not studied in our previous research is the inversion domain. Find a demo at \cite{dagesInversionDemo}. It occurs when there are more white dots in the image than black dots, and when the foreground was supposed to be white and the background black. What now happens is that for us humans, the colour white is the background and the colour black the foreground: the black dots have the interesting information. If we want to keep the denomination of white as foreground and black as background, we have to go into the dual space and invert the image by taking its inverted colour image (thus the name inversion domain). In the dual space, we are now looking for an alignment of black pixels (instead of white) which means an alignment of background or equivalently a lack of foreground pixels. This translates mathematically, when looking at statistical unlikeliness in an a contrario test, to look at the left side of a tail of a binomial distribution: is it likely or not to have as many few foreground points than what we currently observe? To the best of our knowledge, the a contrario literature systematically looks at right sided tail distributions when performing the statistical test. However, going into the dual space and looking at the left side of the new distribution is not equivalent to looking at a right sided tail of the original distribution. To acknowledge the left sided tail of a distribution in a contrario is a novel approach in the literature. By going further, by realising that the primal space is the dual of the dual space, we then realise that a true a contrario test should consider both sides of the tail of the distribution. To the best of our knowledge, this is not done yet in the a contrario literature.

\clearpage
\newpage
\section{Conclusion}
\label{sec: Conclusion}

In this report we have studied the perception of information in very noisy video data. These videos are essentially white noise random dots with a few random dots on some foreground objects, that eventually moves in a non rigid fashion between frames. While we do not see anything on a per frame basis and algorithmic attempts on a per frame basis are doomed to fail due to the lack of information on the foreground and how it is drowned in the background noise, if we are shown this as a video, then we can easily see and understand what the foreground object is. The human visual apparatus is a magical system that manages to effortlessly process such nightmare data. We first designed an automated algorithm using the a contrario approach framework applied on merges of frame. Surprisingly, the algorithm performed somewhat similarly to humans for special parameters, namely the candidate width for the test of alignment and the time integration. This led us to design a model for human perception based on a sequential time integration followed by a statistical test of unlikeliness such as a contrario working on a specific width or visual angle. We challenged this model by performning psychophysical experiments and found that empirically the model held for a visual angle of $\theta=0.23°$ and a time integration of $\Delta t\in [0.2,0.27]$ seconds. Bot the visual angle and the time integration are consistent with some other psychophysics research. Thus, in this report, we have designed an algorithm that manages to process ultra noisy random-dot videos, with potential application to signals produced by very noisy sensors such as ultrasound imaging or Fluorescence Photoactivation Localization Microscopy (FPALM) images and a moving scene, and also yield useful statistical insight into the human brain black box. In a sense, we have in fact designed an artificial intelligence system and also studied some aspects of true intelligence in interpreting visual data.


\newpage
\label{Bibliography}
\bibliographystyle{alpha}  
\bibliography{Bibliography}  

\end{document}